\documentclass[sn-basic]{sn-jnl}

\renewcommand{\cite}{\citep}

\usepackage{graphicx}%
\usepackage{multirow}%
\usepackage{amsmath,amssymb,amsfonts}%
\usepackage{amsthm}%
\usepackage{mathrsfs}%
\usepackage[title]{appendix}%
\usepackage{xcolor}%
\usepackage{textcomp}%
\usepackage{manyfoot}%
\usepackage{booktabs}%
\usepackage{algorithm}%
\usepackage{algorithmicx}%
\usepackage{algpseudocode}%
\usepackage{listings}%
\usepackage{placeins}

\usepackage[T1]{fontenc}  
\usepackage[utf8]{inputenc} 

\usepackage{enumitem}
\usepackage{amsmath}
\usepackage{amsfonts}       
\usepackage{nicefrac} 
\usepackage{multirow}
\usepackage{microtype}  
\usepackage{array}
\usepackage{upgreek}      
\usepackage{graphicx}
\usepackage{makecell}
\usepackage{doi}
\usepackage{titlesec}
\usepackage{wrapfig}
\usepackage{subcaption}
\usepackage{array}
\usepackage{fix-cm}
\usepackage{amsthm}
\usepackage{thmtools,thm-restate}
\usepackage{soul}

\usepackage{cleveref}       

\usepackage{float}
\usepackage{amsthm}
\usepackage{mathrsfs}

\usepackage{amssymb}

\usepackage{hhline}

\usepackage{doi}
\usepackage{adjustbox}
\usepackage{xcolor}
\usepackage{anyfontsize}

\usepackage{tikz}

\usepackage{neuralnetwork}
\usepackage[normalem]{ulem}
\usepackage{boxedminipage}

\newcommand{\Rho}{\mathrm{P}}

\newtheoremstyle{spaced}%
  {10pt}   
  {10pt}   
  {\itshape} 
  {}        
  {\bfseries} 
  {.}       
  { }       
  {}        

\theoremstyle{spaced}

\newtheorem{example}{Example}
\newtheorem{remark}{Remark}

\newtheorem{proposition}{Proposition}
\newtheorem{definition}{Definition}

\newtheorem{corollary}{Corollary}
\newtheorem{lemma}{Lemma}
\newtheorem{notation}{Notation}
\newtheorem{method}{Method}
\newtheoremstyle{mystyle}%
  {3pt} 
  {3pt} 
  {\itshape} 
  {} 
  {\bfseries} 
  {.} 
  { } 
  {\thmname{#1} \thmnumber{#2}: \thmnote{\bfseries #3}} 

\theoremstyle{mystyle}

\usetikzlibrary{positioning, shapes.geometric, arrows.meta, calc, fit, chains, decorations.pathreplacing, calligraphy}

\usetikzlibrary{arrows,shapes,calc}

\usetikzlibrary{tikzmark}
\usetikzlibrary{positioning}
\usetikzlibrary{arrows,chains,positioning,scopes,quotes}
\usetikzlibrary{decorations.pathreplacing,calligraphy}
\usetikzlibrary{decorations.pathreplacing,patterns,positioning}
\usetikzlibrary{matrix}
\usetikzlibrary{chains}
\usetikzlibrary{shapes,arrows,matrix,fit,positioning,trees}
\usetikzlibrary{arrows,through,shapes,positioning,fit,calc}
\usetikzlibrary{patterns}
\usetikzlibrary{fadings}

\tikzstyle{line}=[draw]
\tikzstyle{arrow}=[draw, -latex] 

\tikzset{
    block3/.style = {rectangle, draw, fill=orange!10, text width=6.5em, text centered, rounded corners, minimum height=4em},
    block2/.style = {rectangle, draw, fill=lightgray!30, text width=6.5em, text centered, minimum height=4em},
       block4/.style = {rectangle, draw, fill=orange!20, text width=6.5em, text centered, rounded corners, minimum height=4em},
    thinBlock/.style={rectangle, draw, minimum width=5em, minimum height=3em, text centered, fill=olive!10},
    arrow/.style = {thick,->,>=stealth},
    output/.style = {rectangle, draw, fill=yellow!20, text width=6.9em, text centered, rounded corners, minimum height=4em},
    trapeziumBlock/.style={
    trapezium, 
    trapezium stretches=true, 
    inner sep=0.15em, 
    align=left, 
    shape border rotate=180, 
    draw, 
    fill=olive!10, 
    text width=4em, 
    minimum height=2em, 
    minimum width=3em
},
}

\newcommand{\myvaremphasis}[1]{\textit{\textsf{#1}}}

\begin{document}

\title[$\Pi$-NeSy: A Possibilistic  Neuro-Symbolic Approach]{$\Pi$-NeSy: A Possibilistic  Neuro-Symbolic Approach}


\author[1]{\fnm{Ismaïl} \sur{Baaj}}\email{ismail.baaj@assas-universite.fr}

\author[2,3]{\fnm{Pierre} \sur{Marquis}}\email{marquis@cril.fr}

\affil[1]{\orgdiv{LEMMA}, \orgname{Paris-Panthéon-Assas University}, \orgaddress{
\city{Paris}, \postcode{75006},
\country{France}}}

\affil[2]{\orgdiv{Univ. Artois, CNRS, CRIL}, \orgaddress{\city{Lens}, \postcode{F-62300},  \country{France}}}

\affil[3]{\orgdiv{Institut Universitaire de France}}

\hypersetup{
pdftitle={Pi-Nesy: A Possibilistic Neuro-Symbolic Approach},
pdfsubject={cs.AI},
pdfauthor={Ismaïl Baaj, Pierre Marquis},
pdfkeywords={Neuro-symbolic, Possibility theory, inference, learning},
}


\abstract{In this article, we introduce a neuro-symbolic approach that combines a low-level perception task performed by a neural network with a high-level reasoning task performed by a possibilistic rule-based system.  

The goal is to be able to derive for each input instance the degree of possibility that it belongs to a target (meta-)concept. This (meta-)concept is connected to intermediate concepts by a possibilistic rule-based system. The probability of each intermediate concept for the input instance is inferred using a neural network. The connection between the low-level perception task and the high-level reasoning task lies in the transformation of neural network outputs modeled by probability distributions (through  softmax activation) into possibility distributions.
The use of intermediate concepts is valuable for the explanation purpose: using the rule-based system, the classification of an input instance as an element of the (meta-)concept can be justified by the fact that intermediate concepts have been recognized.

From the technical side, our contribution consists of the design of efficient methods for defining the matrix relation and the equation system associated with a possibilistic rule-based system. The corresponding matrix and equation are key data structures used to perform inferences from a possibilistic rule-based system and to learn the values of the rule parameters in such a system according to a training data sample. Furthermore, leveraging recent results on the handling of inconsistent systems of fuzzy relational equations, an approach for learning rule parameters according to multiple training data samples is presented. Experiments carried out on the MNIST addition problems and the MNIST Sudoku puzzles problems  highlight the effectiveness of our approach compared with state-of-the-art neuro-symbolic ones.}




\maketitle

\section{Introduction}

One of the key challenges in Artificial Intelligence (AI) research is to combine neural and symbolic approaches.  Neural-based methods excel at low-level perception and learning on large data sets, while exhibiting robustness against noise and data ambiguity. However, neural methods may  encounter difficulties in generalization and they lack interpretability, hindering a good understanding of their decision processes. On the other hand, symbolic methods leverage domain knowledge to reason on data sets, providing transparency through explicit rule-based representations, thus enabling a better understanding of the decision-making process, which in turn fosters trust and intelligibility. However, symbolic methods are typically not good at capturing intricate low-level patterns and handling efficiently noisy or ambiguous data.

These challenges have led to the emergence of \emph{Neuro-Symbolic Computing (NeSy)}, which advocates for finding meeting points between neural and symbolic methods \cite{baaj2024synergies,Garcez2019NeuralSymbolicCA,Hitzler2021NeuroSymbolicAI,MARRA2024104062}. In addition to seeking synergies and complementarities, it is nowadays  desirable to develop  frameworks that enable neural and symbolic methods to interact.  Especially, a significant   objective is to develop two highly desirable capabilities for neuro-symbolic approaches: i) \textit{joint inference}: the ability to generate structured predictions based on pieces of knowledge and on examples involving intricate low-level features, ii) \textit{joint learning}: the ability to jointly learn and adapt parameters of neural and symbolic models. 

Another important component of NeSy approaches is the \emph{Uncertainty Theory} that is used. Because both data and pieces of knowledge are typically pervaded by uncertainty,  reasoning and learning from them requires to leverage an Uncertainty Theory that is suited to the nature of the uncertainty one has to deal with. In this perspective, Probability Theory is by far the dominant approach, so that one often equates NeSy with the sum "Neural+Logical+Probabilistic" \cite{DBLP:conf/nesy/RaedtMDDK19}. However, it is known for a while that Probability Theory (following the so-called Bayesian approach where a single probability distribution is considered) is not adequate to model every uncertain scenario, especially those situations where uncertainty is not due to randomness but comes from partial ignorance \cite{shafer1976,DBLP:books/sp/20/DenoeuxDP20}.
Accordingly, Dubois and Prade  have emphasized the use of \textit{Possibility Theory} as a valuable uncertainty theory on which neuro-symbolic approaches could be defined as well \cite{dubois2023reasoning}. Indeed, Possibility Theory is a qualitative framework for handling epistemic uncertainty (uncertainty due to a lack or limited amount of information) that appears as an interesting candidate to integrate both logic and neural networks when dealing with data and pieces of knowledge that are pervaded with this kind of uncertainty. 

In some sense, Possibility Theory lies between logical and probabilistic frameworks for uncertainty handling. 
Basically, in Possibility Theory, uncertainty is modeled by two dual measures called \textit{possibility} and \textit{necessity}, allowing to distinguish what is possible
without being certain at all and what is certain to some extent. These two measures  have simple set-based interpretations and they form the 
simplest non-trivial system of upper and lower probabilities \cite{dubois2023reasoning}.
In a nutshell, the possibility measure assigns to an event the highest plausibility of the individual elements composing it, while the dual necessity measure gives a lower bound on the certainty level of the event, indicating how impossible the complement of the event is. 
In \cite{dubois2023reasoning}, Dubois and Prade suggest that by taking advantage of recent tools introduced for performing inferences and learning with possibilistic rule-based systems \cite{baaj2022learning,baaj2021min}, one may be able to jointly use possibilistic reasoning capabilities and neural network-based learning, given that neural network results are interpreted as possibility distributions. 

\medskip
This is precisely the research line that we follow in this paper. We propose a neuro-symbolic approach called \textit{$\Pi$-NeSy} based on Possibility Theory, involving:
\begin{itemize}
    \item a low-level perception task performed by a neural network (Figure \ref{fig:low-levelperceptiontask}), which  is designed to determine (typically through a softmax transformation) the degree of probability of the occurrence of intermediate concepts from input data. 
    For example, if the neural network is trained to recognize whether an image represents a 0 or a 1, it will produce, from each input image, two intermediate concepts with their probabilities.  For a first input image (image 1), the concepts would be `image 1 is 0' and `image 1 is 1'. For a second input image (image 2), the concepts would be `image 2 is 0' and `image 2 is 1'. For each image, the probabilities of their two intermediate concepts form a probability distribution.

    \begin{figure}[H]
        \centering
        \scalebox{0.7}{
        \begin{tikzpicture}[node distance=2cm, auto]

\node (A) {Input data};

\node[block2, right=of A] (B) {Neural network};

\node[block3, right=of B, xshift=1cm] (C) {Probability dist.\\ of intermediate\\ concepts};

\node[trapeziumBlock, below=0.1cm of B] (L1) {Neural learning};

\draw[arrow] (A) -- (B);
\draw[arrow] (B) -- node[midway, above, align=center] {output\\(applying softmax\\activation)} (C);

\end{tikzpicture}
}
\smallskip
        \caption{$\Pi$-NeSy's low-level perception task.}
        \label{fig:low-levelperceptiontask}
    \end{figure}
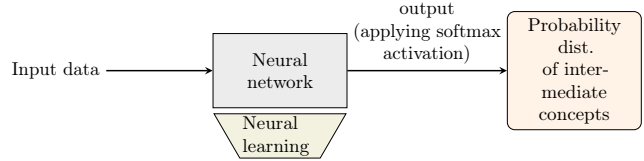

    \begin{figure}[H]
        \centering
\scalebox{0.7}{
\begin{tikzpicture}[node distance=0.3cm, auto]

\node[block3] (B2) {Probability dist.\\of intermediate concepts 1};
\node[block3, below=0.3cm of B2] (B2_2) {Probability dist.\\of intermediate concepts 2};
\node[block3, below=0.3cm of B2_2] (B2_3) {...};
\node[block3, below=0.3cm of B2_3] (B2_n) {Probability dist.\\of intermediate concepts $k$};

\node[block4, right=1.25cm of B2] (B4) {Possibility dist.\\of intermediate concepts 1};
\node[block4, below=0.3cm of B4] (B4_2) {Possibility dist.\\of intermediate concepts 2};
\node[block4, below=0.3cm of B4_2] (B4_3) {...};
\node[block4, below=0.3cm of B4_3] (B4_n) {Possibility dist.\\ of intermediate concepts $k$};

\node[block2, below right=-0.5cm and 1.5cm of B4_2] (C) {Possibilistic rule-based system};

\node[output, right=2cm of C] (D) {Possibility dist. of\\ (meta-)concepts};

\node[trapeziumBlock, below=0.1cm of C] (L2) {Possibilistic learning};

\draw[arrow] (B2) -- node[midway, above] {transf.} (B4);
\draw[arrow] (B2_2) -- (B4_2);
\draw[arrow] (B2_3) -- (B4_3);
\draw[arrow] (B2_n) -- (B4_n);

\draw[arrow] (B4.east) -- node[midway, above, sloped] {input 1} (C.west);
\draw[arrow] (B4_2.east) -- node[near start, above, sloped] {input 2} (C.west);
\draw[arrow] (B4_3.east) -- node[midway, above, sloped] {...} (C.west);
\draw[arrow] (B4_n.east) -- node[midway, above, sloped] {input $k$} (C.west);

\draw[arrow] (C) -- node[midway, above] {output} (D);

\end{tikzpicture}
}
\smallskip

        \caption{$\Pi$-NeSy's high-level reasoning task.}
        \label{fig:highlevelreasoningtask}
    \end{figure}
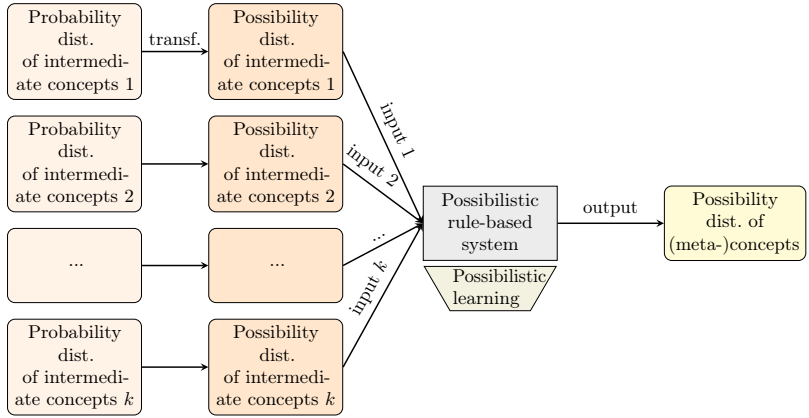
        \item a high-level reasoning task performed by a possibilistic rule-based system (Figure \ref{fig:highlevelreasoningtask}), which is designed to determine the degree of possibility of the occurrence of  a (meta-)concept by reasoning based on possibility distributions over intermediate concepts. 
\end{itemize}
The connection between these two tasks lies in the transformation of neural network outputs modeled by probability distributions into possibility distributions, using probability-possibility transformations \cite{delgado1987concept,dubois1982several,dubois1983unfair}.  Once transformed, the possibilistic rule-based system takes these possibility distributions as inputs. Continuing the previous example about digit recognition from two images (here, $k = 2$), the role of the possibilistic rule-based system could be to determine whether the two images represent the same digit. 

$\Pi$-NeSy requires explicit background knowledge to perform the high-level reasoning task. This background knowledge is modeled by a \textit{possibilistic rule-based system}, i.e., a  finite set of  possibilistic rules. Each possibilistic rule has the form ``if $p$ then $q$'', where $p$ is the rule premise and $q$ the conclusion of the rule. The rule premise $p$ is a conjunction of propositions of the form "$a(x)$ in $P$", where $a$ is an attribute applied to an item $x$, $P$ is a subset of the attribute domain $D_a$. The negation of the proposition  "$a(x)$ in $P$" is the proposition "$a(x)$ in $\overline{P}$" where $\overline{P}$ is the complement of $P$ in $D_a$, so the proposition "$a(x)$ in $P$"  can be viewed as a literal. The conclusion $q$ is also a proposition. Each rule is   associated with two parameters $r$ and $s$. These two rule parameters provide information on the uncertainty associated with the rule as follows: "if $p$ then $q$" is certain to a degree of $1-r$, while the converse rule "if $\neg p$ then $\neg q$" is certain to a degree of $1-s$ \cite{dubois2020possibilistic,farreny1986default}. $r$ and $s$ vary in $[0, 1]$. Especially, a rule "if $p$ then $q$" is totally certain when $r = 0$: in such a case, the necessity of its conclusion $q$ knowing its premise $p$ is equal to the maximal value $1$ (equivalently, the possibility of $\neg q$ knowing $p$ is equal to the minimal value $0$, stating that $\neg q$ is impossible when $p$ holds for sure).

\paragraph*{\texorpdfstring{$\Pi$-NeSy joint inference}{Pi-NeSy joint inference}}
The structure of our neuro-symbolic approach $\Pi$-NeSy, which enables joint inference, is based on explicit connections between the neural network and the possibilistic rule-based system. Specifically, when the neural network processes an input data item, it produces a probability distribution over a set of intermediate concepts. This set of intermediate concepts constitutes the domain of an input attribute in the possibilistic rule-based system, where each intermediate concept is a value within that domain. Thus, these sets of intermediate concepts allow us to link the outputs of the neural network to the inputs of the rule-based system.

To perform joint inference with $\Pi$-NeSy, a finite set of input data items is required, each of them being associated with a set of intermediate concepts. Joint inference is performed as follows:
\begin{enumerate}
    \item Each input data item is processed by the neural network, which produces a probability distribution over a set of intermediate concepts.
    \item These probability distributions are then transformed into possibility distributions.
    \item The resulting possibility distributions are used as inputs of the possibilistic rule-based system.
    \item Inference of possibilistic rule-based system is performed using the matrix relation from \cite{baaj2021min}, resulting in an output possibility distribution over the possible (meta-)concepts.
\end{enumerate}

\paragraph*{\texorpdfstring{$\Pi$-NeSy joint learning}{Pi-NeSy joint learning}} 
Joint learning in $\Pi$-NeSy relies on a training dataset where each example consists of input data items, each data item being associated with some intermediate concepts, and a targeted (meta-)concept. The sets of intermediate concepts for the data input items and the set of possible (meta-)concepts are determined from the training dataset.

Joint learning is performed in two steps: \textit{neural learning}, to learn neural network parameters based only on data items for identifying intermediate concepts; and then \textit{possibilistic learning}, to learn values of the possibilistic rule parameters based on examples from the training dataset 
using the equation system of \cite{baaj2022learning}.

The two types of learning (neural and possibilistic) are joined as follows. First, neural learning is performed based on the pairs of (input data item, targeted intermediate concepts) composing the examples of the dataset. Then, for each data item used for neural learning,  an inference is performed from the neural network to get an output probability distribution over the set of  intermediate concepts associated with the item. In the next step, these probability distributions are transformed into possibility distributions. Thus, a possibility distribution over the set of  intermediate concepts is generated for each data item of an example.

Possibilistic learning is then performed with training data based on these possibility distributions. Each training data sample used for possibilistic learning is derived from an example of the initial training dataset and consists of:
\begin{itemize}
    \item A set of input possibility distributions of the training data sample, which is given by the possibility distributions obtained for each data item in the example.
    \item A targeted output possibility distribution over the set of (meta-)concepts, where the targeted (meta-)concept of the example has a possibility degree equal to one, while the other (meta-)concepts have a possibility degree equal to zero.
\end{itemize}

\begin{example}\label{ex:firstrules-set}

As an illustration, the following neuro-symbolic problem is tackled using $\Pi$-NeSy: given two handwritten digits on two images, where the digits are either zero or one, e.g. \includegraphics[width=0.017\textwidth]{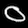} and \includegraphics[width=0.017\textwidth]{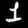}, how can we determine that the two handwritten digits are the same (or not)? This problem will be considered throughout the paper as a running example.

To address this problem with $\Pi$-NeSy, a convolutional neural network is used for the low-level perception task. The neural network is trained with image-label pairs, e.g., $\{$(\includegraphics[width=0.017\textwidth]{z_img__9158_ZERO.png}, 0) (\includegraphics[width=0.017\textwidth]{./z_img__ONE_41266.png}, 1)$\}$, and it  is designed to recognize the digit (an intermediate concept) on a given image (a data item) by providing a probability distribution over possible digits using softmax activation: for the image \includegraphics[width=0.017\textwidth]{z_img__9158_ZERO.png},  the following probability distribution is obtained: P(\includegraphics[width=0.017\textwidth]{z_img__9158_ZERO.png}) = [P(\includegraphics[width=0.017\textwidth]{z_img__9158_ZERO.png} = 0), P(\includegraphics[width=0.017\textwidth]{z_img__9158_ZERO.png} = 1)] = [0.9, 0.1]. This probability distribution can be   transformed into a possibility distribution,  
e.g., [$\Pi$(\includegraphics[width=0.017\textwidth]{z_img__9158_ZERO.png} = 0), $\Pi$(\includegraphics[width=0.017\textwidth]{z_img__9158_ZERO.png} = 1)] = [1, 0.1] (see Subsection \ref{sec:ppt}).

As sketched previously, $\Pi$-NeSy's high-level reasoning task relies on a possibilistic rule-based system. This system infers whether two images (alias two data items) represent the same digit (the (meta-)concept) by using two possibility distributions, each of them being associated with one of the two images. Each possibility distribution provides information on the uncertainty of the digit appearing in the image. As background knowledge, possibilistic rules that compose the possibilistic rule-based system are defined. Each rule involves two input attributes, $a_1$ and $a_2$, each of them corresponding to one of the two images. The domain for both attributes is $D_{a_1} = D_{a_2} = \{0, 1\}$, representing the potential labels for the images. The domain of the output attribute $b$ used in the possibilistic rule-based system is $D_b = D_{a_1} \times D_{a_2} = \{(0,0), (0,1), (1,0), (1,1)\}$, indicating the possible combinations of labels for the two images. The four rules composing the possibilistic rule-based system for this problem are:
\begin{itemize}
\item $R^1$: ``If $a_1(x) \in \{ 0 \}$ then $b(x) \in Q_1$'', where $Q_1 = \{ (0,0), (0,1)\}$, 
\item $R^2$: ``If $a_1(x) \in \{ 1 \}$ then $b(x) \in Q_2$'', where $Q_2 = \{ (1,0), (1,1)\}$, 
\item $R^3$: ``If $a_2(x) \in \{ 0 \}$  then $b(x) \in Q_3$'', where $Q_3= \{ (0,0), (1,0) \}$,
\item $R^4$: ``If $a_2(x) \in \{ 1 \}$ then $b(x) \in Q_4$'',  where  $Q_4 = \{(0,1), (1,1)\}$.
\end{itemize}

\smallskip
For this specific neuro-symbolic problem,  the domain knowledge is  \emph{certain}, i.e., the conclusion of any of the four rules above holds for sure when its premise hold. So, the rule parameters $s_i$ and $r_i$ associated with each rule $R^i$ are all equal to zero, see Subsection \ref{subsubsec:uncertain}. However, when needed, the values of the rule parameters can also be \emph{learned according to training data} (see Section \ref{sec:learning}). This facility is of the utmost value when the domain knowledge used is not certain.

To infer from the possibilistic rule-based system, possibility distributions $\pi_{a_1(x)} : D_{a_1} \rightarrow [0,1]$ and $\pi_{a_2(x)} : D_{a_2} \rightarrow [0,1]$, associated with the input attributes $a_1$ and $a_2$, must be generated. These possibility distributions are obtained by applying a probability-possibility transformation of the probability distributions produced by the neural network, which were obtained for each image associated with $a_1$ and $a_2$.  The inference process 
uses the matrix relation of \cite{baaj2021min} to produce an output possibility distribution $\pi_{b(x)}^\ast : D_b \rightarrow [0,1]$, which assigns a possibility degree to each pair of labels of $D_b$. $\pi_{b(x)}^\ast$ is obtained from the four rules, using $\pi_{a_1(x)}$ and $\pi_{a_2(x)}$. The pair  $u \in D_b$ with the highest possibility degree with respect to $\pi_{b(x)}^\ast$ gives the conclusion one looks for: if $u$ is either $(0,0)$ or $(1,1)$, then the two images are supposed to represent the same digit; if $u$ is either $(0,1)$ or $ (1,0)$, then the images are supposed to represent different digits. If multiple pairs have a possibility degree equal to the highest degree, the output possibility distribution $\pi_{b(x)}^\ast$ is ambiguous (no conclusion is drawn).
\end{example}

\medskip
The possibilistic framework on which $\Pi$-NeSy relies offers many advantages compared to alternative settings on which current neuro-symbolic approaches are based:

\begin{enumerate}
    \item \textbf{Computational complexity of inference}:
        Inference in probabilistic neuro-symbolic approaches has a high computational complexity \cite{maene2024hardness}. For instance, DeepProbLog \cite{manhaeve2018deepproblog} relies on Weighted Model Counting (WMC) \cite{chavira2008probabilistic}, making its inference process \#P-hard. Therefore, with these approaches, approximate inference methods are commonly used \cite{maene2024hardness}.\\
In contrast, performing an inference from a possibilistic rule-based system reduces to a min-max matrix product (where min is used as the product and max as the addition) of a vector by a matrix, see Subsection \ref{subsec:matrixrelpractical}. This operation is comparable  to a classical matrix product in terms of computational complexity \cite{duan2009fast}. Furthermore, constructing the matrix relation used to perform an inference 
\cite{baaj2021min} has a polynomial-time complexity with respect to the number of rules, see Proposition \ref{prop:complexity:nbOp}.

    \item \textbf{Handling uncertain domain knowledge}:
        The possibilistic framework is suited to reason with uncertain domain knowledge, through the use of possibilistic rules, see Subsection \ref{subsubsec:uncertain}.  Inference from a rule ``if $p$ then $q$'' consists in deriving the possibility degrees of the conclusion $q$ and that of $\neg q$ according to the possibility degrees of the rule premise $p$ and that of $\neg p$ and the rule parameters $s$ and $r$. In the possibilistic framework, the computation of  the possibility degrees of  $q$ and that of $\neg q$  relies on a max-min formula (see (\ref{eq:pi_q})) 
        which is \textit{analog} to the total probability theorem (see (\ref{eq:prob_q})) and where the rule parameters  $s$ and $r$ represent  conditional possibility distributions between the premise $p$ and the conclusion $q$ (or its negation), i.e., $s = \pi(q\mid \neg p)$ and $r = \pi(\neg q\mid p)$, akin to conditional probability distributions in probabilistic reasoning but using possibility distributions $\pi$ instead (see Subsection \ref{subsection:possibility}). Note that $s = \pi(q \mid \neg p)$ states that the necessity degree of $\neg q$ knowing $\neg p$ is equal to $1-s$, and that $r = \pi(\neg q \mid p)$ states that the necessity degree of $q$ knowing $p$ is equal to $1-r$.
        Especially, when the rule parameters are assigned to zero, i.e., $s = r = 0$, the rule "if $p$ then $q$" and the converse rule ``if $q$ then $p$'' (viewed as ``if $\neg p$ then $\neg q$'') \cite{dubois2020possibilistic,farreny1986default} are considered as equivalent.\\
         While the possibilistic framework used in $\Pi$-NeSy makes possible to exploit uncertain domain knowledge, most of the neuro-symbolic frameworks require to assume that domain knowledge is certain. For example, DeepProbLog \cite{manhaeve2018deepproblog} restricts uncertainty to probabilistic facts and the logical rules that are used are not subject to uncertainty. In NeuPSL \cite{pryor2023ijcai}, which is based on Probabilistic Soft Logic \cite{bach2017hinge},  a weight is attached to each rule to quantify its importance in the model. While this weight indicates a confidence level, it does not measure the certainty of the rule  "if $p$ then $q$" or of  its converse rule "if $\neg p$ then $\neg q$". In contrast, the possibilistic framework uses two parameters $s$ and $r$,  which represent the certainty of the rule and of the converse rule.
        
    \item \textbf{Learning}:
   The possibilistic learning method presented in Section \ref{sec:learning} focuses on learning 
 the values of the rule parameters of  
 possibilistic rules defined for addressing a multi-class classification problem. As the possibilistic framework relies on the combination of min and max functions, standard learning methods based on gradient descent    cannot be directly applied due to the non-differentiability of these two functions. The method introduced in Section \ref{sec:learning} is based on the work of \cite{baaj2022learning}, showing that the values of the rule parameters can be determined by solving a min-max equation system  constructed from the training data sample that is considered. A min-max equation system is similar to a system of linear equations, but using min and max operations instead of addition and multiplication  respectively. 
 In the min-max equation system  of  \cite{baaj2022learning}, the components of the unknown vector are the rule parameters.  \\
 \cite{baaj2022learning} shows that any solution of the min-max equation system yields values for the rule parameters that are compatible  with the considered training data sample. Thus, if the rule parameters are set using one of the solutions of the min-max equation system and then an inference from the rule-based system is performed using the input possibility distributions of the sample, the targeted output possibility distribution of the sample can be soundly derived.
 Whenever the min-max equation system is inconsistent, i.e., it has no solution  (which can happen when the data used to construct the equation system contain noise and/or outliers), one can take advantage of the tools introduced in \cite{BAAJ2024} to obtain \emph{approximate} solutions defined as follows: an approximate solution of an inconsistent equation system is a solution of one of the closest consistent equation systems. A consistent equation system is said to be close  to the given inconsistent equation system whenever the distance (based on the L-infinity norm) between the second member of the consistent equation system and the second member of the inconsistent equation system is minimal. \\
 To wrap up, our learning approach based on \cite{baaj2022learning,BAAJ2024} has the following advantages:
    \begin{itemize}
        \item  Using  \cite{BAAJ2024}, one can assess the extent to which a training data sample 
        is reliable  with respect to a  set of rules. The reliability measure is based on checking the consistency of a min-max equation system defined in \cite{baaj2022learning}, in order to determine how much the data must be perturbed (if so) to obtain a consistent equation system. \cite{BAAJ2024} indicates how to make minimal modifications to the data in order to obtain a consistent equation system and therefore how to get approximate solutions.

        \item In a neuro-symbolic approach where a neural model is connected to a possibilistic rule-based system, learning can be achieved sequentially, i.e., the parameters of the neural model are first learned and then the rule parameters of the possibilistic rule-based system. However, joint learning is also possible:
        the inputs of the training data  used for possibilistic learning can be directly derived from the results inferred from the neural model using the inputs from the data used to train the neural model.
When the possibilistic rules and their parameters have been defined beforehand, one can check whether these rule parameters align with those obtained from possibilistic learning based on training data. 
    \end{itemize}

         \item \textbf{Backpropagation}:
        A backpropagation process can be performed using a possibilistic rule-based system (see Subsection \ref{subsub:backpropagation}). This capability may aid in refining and improving  the neural model used. 
         \item \textbf{Explainability}:
        Possibilistic rule-based systems are endowed with explanatory capabilities \cite{baaj2021representation} (see also \cite{baajtel-03647652}). This is useful to develop neuro-symbolic approaches that are able to explain their inference results using intermediate concepts, see Subsection \ref{subsub:explainability}.
\end{enumerate}

\medskip
The contribution of this paper mainly consists in improving the computation of two key data structures associated with a possibilistic rule-based system and used in $\Pi$-NeSy: the matrix relation for performing inferences \cite{baaj2021min} and the equation system for learning \cite{baaj2022learning}. The generation of the matrix relation and the generation of the equation system are based on \textit{an explicit partition of the  domain of the output attribute $b(x)$} of the set of 
possibilistic rules that is considered as domain knowledge (see (\ref{eq:partition}) and \cite{baaj2021min}). This partition is constructed according to the conclusions of the rules. Rules concluding $b(x) \in \emptyset$ can be discarded without questioning the validity of the approach. 

In the following, a refined method for generating such a partition reduced to its non-empty subsets is first presented. 
This refined method is much better suited to practical cases than the one presented in \cite{baaj2021min}, the latter involving the generation of more than $2^n$ subsets (where $n$ is the number of rules) before reducing the partition to its non-empty subsets.  The partition generated by the refined method allows us to directly  construct, row by row, the matrices governing the matrix relation and those governing the equation system, where each row corresponds exactly to a subset of the partition.  As a result, the complexity of the  construction of the matrix relation and the equation system is reduced so that those two key ingredients of $\Pi$-NeSy can be used in practice.

We then address the possibilistic learning issue using the $\min-\max$ equation system of \cite{baaj2022learning}. When the equation system is consistent, i.e., it has solutions, explicit solutions of the system are computed. When the equation system is inconsistent which may happen when it is constructed from a noisy data sample (the possibility degrees used to describe the instance at hand are incorrect) or based on an outlier (the possibility degrees used to describe the degrees of membership of the instance to the targeted (meta-)concepts are incorrect), recent results about the handling of inconsistent $\min-\max$ equation systems presented in \cite{BAAJ2024} can be leveraged. The goal is to assess the quality of the training data sample used to construct the equation system and to obtain a new equation system which is consistent. The resulting system is close to the  inconsistent equation system considered before, in the sense that solutions of the resulting consistent equation system are approximate solutions of the inconsistent equation system. Consistency is recovered through a minimal update of the output possibility distribution associated with the sample, indicating its degree of membership to each (meta-)concept.   Using \cite{BAAJ2024}, we also show how to learn from \emph{multiple training data, i.e., sets of training data samples}.  This leads us to introduce a practical method for performing possibilistic learning called \textit{possibilistic cascade learning}, see Method \ref{meth:learningincascade}. 

Another contribution of the paper consists of experimental results. Experiments have been carried out to assess our neuro-symbolic approach $\Pi$-NeSy  on well-known neuro-symbolic datasets: MNIST-Additions problems \cite{manhaeve2018deepproblog} and MNIST Sudoku puzzles problems \cite{augustine2022visual}. For these two  problems, the empirical results found show that our neuro-symbolic approach $\Pi$-NeSy exhibits reasonable inference and learning times, and a classification accuracy that is quite good  compared to those achieved using recent neuro-symbolic approaches.

The rest of the paper is structured as follows. In Section \ref{sec:bg}, the necessary background for our neuro-symbolic approach $\Pi$-NeSy is reminded. In particular, we explain how to obtain a probability distribution as output of a neural network, we give a refresher about  Possibility Theory, and we recall the inference mechanism used for possibilistic rule-based system and two probability-possibility transformations.  In Section \ref{sec:practical-procedure-building-reduced-ES}, 
our inductive method for constructing the partition of the domain of the output attribute of the possibilistic rule-based system at hand according to the conclusions of the rules is described. 
In Section \ref{sec:learning}, a practical method for performing possibilistic learning (Method \ref{meth:learningincascade}) is presented. In Section \ref{sec:exp}, experiments carried out with $\Pi$-NeSy on well-known neuro-symbolic datasets  are reported. The resulting empirical results are compared to those obtained by state-of-the-art approaches using the same experimental protocol. Finally, the results are discussed and we conclude the paper with some perspectives.

\section{Background}
\label{sec:bg}
In this section, we briefly remind how neural networks work, and how they can output  a probability distribution through softmax activation. Then we give basic definitions about Possibility Theory,  describe the inference mechanism at work when dealing with possibilistic rule-based systems and present two probability-possibility transformations.

\subsection{Artificial neural networks}

The following paragraphs focus on concepts related to artificial neural networks (NNs) and  deep learning that are needed for understanding our work.  Extensive further details can be found, e.g., in \cite{goodfellow2016deep}. \\
NNs are computational models that learn to map complex input-output relationships through a series of interconnected layers consisting of  neurons (nodes). Though several types of neurons can be considered, some of them compute a weighted sum of their inputs, add a bias, and then apply a nonlinear activation function, such as the Rectified Linear Unit (ReLU), to the result. Complex patterns in the data can be identified using successive layers.

The training of NNs involves the minimization of a loss function, which quantifies the difference between the predicted outputs and the true outputs with respect to training data (supervised learning). This is achieved through an optimization process, typically using algorithms such as  gradient descent, where the parameters of the models (weights and biases) are iteratively adjusted based on the gradient of the loss function with respect to these parameters.

For classification tasks, the softmax function is commonly used in the final layer of a neural network. The softmax function converts the  raw output scores of the network  (real numbers) into probabilities by taking the exponential of each output score and then normalizing these exponentials so that they sum to 1. The probability for each class $i$ is given by the formula:

\begin{equation}\label{eq:softmax}
    \text{Softmax}(i) = \frac{e^{z_i}}{\sum_{j} e^{z_j}} 
\end{equation}

\noindent where $z_i$ is the output score for class $i$. This ensures that the output of the softmax function can be interpreted as a probability distribution over the classes, with each value representing the probability that the model assigns to the corresponding class.

\begin{example}\label{ex:nn:probs}
    As an example, the MNIST problem \cite{deng2012mnist} is considered. The goal is to recognize handwritten digits on images.  The training set contains 60,000 images, and the test set has 10,000 images. The neural network whose architecture is given in Table \ref{tab:NN} has been used. It has been trained using the Adadelta optimizer \cite{zeiler2012adadelta} with a learning rate of 1.0, adjusting it with a decay factor  of 0.7, over 20 epochs and a batch size of 64.

Given an input image, the neural network generates raw scores. These scores are then transformed into a probability distribution using a softmax activation, see  (\ref{eq:softmax}), which gives the probability that the image corresponds to each digit from 0 to 9. This probability distribution can be represented by a vector. For instance, from the two MNIST images \includegraphics[width=0.017\textwidth]{z_img__9158_ZERO.png} and \includegraphics[width=0.017\textwidth]{./z_img__ONE_41266.png}  one can obtain:
\begin{itemize}
    \item $P(\includegraphics[width=0.017\textwidth]{z_img__9158_ZERO.png}) = [P(\includegraphics[width=0.017\textwidth]{z_img__9158_ZERO.png} = 0),P(\includegraphics[width=0.017\textwidth]{z_img__9158_ZERO.png} = 1),\cdots,P(\includegraphics[width=0.017\textwidth]{z_img__9158_ZERO.png} = 9)] = [0.9742,$ $0,$ $0.0002,$ $0.0007,$ $0,$ $0.0237,$ $0,$ $0,$ $0,$ $0.0012]$,
    \item $P(\includegraphics[width=0.017\textwidth]{./z_img__ONE_41266.png}) = [P(\includegraphics[width=0.017\textwidth]{./z_img__ONE_41266.png}) = 0,P(\includegraphics[width=0.017\textwidth]{./z_img__ONE_41266.png}) = 1,\cdots,P(\includegraphics[width=0.017\textwidth]{./z_img__ONE_41266.png}) = 9] = [0,$ $0.9939,$ $0.0002,$ $0.0049,$ $0.0001,$ $0.0009,$ $0,$ $0,$ $0,$ $0]$.
\end{itemize}

\end{example}

\subsection{Possibility Theory}\label{subsection:possibility}
Possibility Theory is an uncertainty theory, which provides computable methods for the representation of incomplete and/or imprecise information. Initially introduced by Zadeh \cite{zadeh1978fuzzy} and considerably developed by Dubois and Prade \cite{dubois2023reasoning}, Possibility Theory models uncertainty by two dual measures, possibility and necessity, which are useful to distinguish what is possible without being certain at all and what is certain to  some extent.

In the following, we give some background on Possibility Theory, focusing on concepts needed to define  possibilistic rule-based systems and probability-possibility transformations (that will be presented in subsequent subsections).

Let $U$ be a set. Any subset  $A \subseteq U$ is called an \textit{event}. In particular, for each $u \in U$, the singleton $\{u\}$ is called an \textit{elementary event}. 

\begin{definition}
A \emph{possibility measure} on $U$ is a mapping 
$\Pi : 2^U \rightarrow [0,1]$, which assigns a degree $\Pi(A)$ to each event $A \subseteq U$ in order to assess to what extent the event $A$ is possible. It satisfies the following conditions:

\begin{itemize}
    \item $\Pi(\emptyset) = 0$ and $\Pi(U) = 1$,
    \item For any subset $\{ A_1,A_2,\dots,A_n \} \subseteq 2^U$, $\Pi(\bigcup_{i=1}^n A_i) = \max_{i=1,2,\dots,n} \Pi(A_i)$.
\end{itemize}
\end{definition}

For any event $A$, if $\Pi(A)$ is equal to 1, it means that $A$ is totally possible, while if $\Pi(A)$ is equal to 0, it means that $A$ is impossible. A possibility measure $\Pi$ has the following properties:
\begin{itemize}
    \item $\Pi(A \cup \overline{A}) = \max(\Pi(A),\Pi(\overline{A}))=1$.
    \item 
    For any $A_1 ,A_2\in 2^U$, if $A_1 \subseteq A_2$, then $\Pi(A_1) \leq \Pi(A_2)$. It follows that for any $A_1 ,A_2\in 2^U$, we have $\Pi(A_1 \cap A_2) \leq \min(\Pi(A_1),\Pi(A_2))$.
\end{itemize}

\smallskip
Likewise the notion of possibility measure, a \textit{necessity measure} is defined by:
\begin{definition}
\noindent A \emph{necessity measure} on $U$ is a mapping $N: 2^U \rightarrow [0,1]$, which assigns a degree $N(A)$ to each event $A\subseteq U$ in order to assess to what extent the event $A$ is certain. It satisfies:
\begin{itemize}
    \item $N(\emptyset) = 0$ and $N(U) = 1$,
    \item For any subset $\{ A_1,A_2,\dots,A_n \} \subseteq 2^U$, $N(\bigcap_{i=1}^n A_i) = \min_{i=1,2,\dots,n} N(A_i)$.
\end{itemize}
\end{definition}

 If $N(A) = 1$, it means that $A$ is certain. If $N(A) = 0$, the event $A$ is not certain at all, but this does not mean that $A$ is impossible. A necessity measure has the following properties:
\begin{itemize}
    \item $N(A \cap \overline{A}) = \min(N(A),N(\overline{A}))=0$.
    \item For any $A_1 ,A_2\in 2^U$, if $A_1 \subseteq A_2$, then $N(A_1) \leq N(A_2)$. It follows that for any $A_1 ,A_2\in 2^U$, we have $N(A_1 \cup A_2) \geq \max(N(A_1),N(A_2))$.
\end{itemize}

\smallskip
The two notions of possibility measure and of necessity measure are dual to each other in the following sense:
\begin{itemize}
    \item If $\Pi$ is a possibility measure, then the corresponding necessity measure $N$ is defined by the following formula:
    \[ N(A):= 1 - \Pi(\overline{A}).\]
    \item Reciprocally, if $N$ is a necessity measure, then the corresponding possibility measure $\Pi$ is defined by the following formula:
    \[ \Pi(A) := 1 - N(\overline A).\] 
\end{itemize}

A \textit{possibility distribution} on the set $U$ is defined by: 
\begin{definition}
A \emph{possibility distribution} $\pi$ on the set $U$ is a mapping $\pi: U \rightarrow [0,1]$, which assigns to each element $u \in U$ a possibility degree $\pi(u) \in [0,1]$. A possibility distribution is said to be \emph{normalized} if $\exists u \in U$ such that $\pi(u) = 1$.
\end{definition}

Any possibility measure $\Pi$ gives rise to a normalized possibility distribution $\pi$ defined by the formula:
\[\pi(u) = \Pi(\{u\}),  u\in U.\]

Therefore, for any subset $A\subseteq U$, we have:
\[ \Pi(A) = \sup_{x\in A} \pi(x) \quad \text{ and }  
N(A) =1 - \Pi(\overline A)= \min_{x \notin A} (1 - \pi(x)). \]

Reciprocally, a normalized possibility distribution $\pi$ gives rise to a possibility measure $\Pi$ 
and a necessity measure $N$ defined by:
\[ \text{for any } A \subseteq U,\quad \Pi(A) = \sup_{x\in A} \pi(x) \quad \text{ and } \quad  
N(A) =1 - \Pi(\overline A)= \inf_{x \notin A} (1 - \pi(x)).\]

Possibilistic conditioning is defined in both the qualitative and the quantitative frameworks of Possibility Theory. For a detailed overview, see \cite{dubois2023reasoning}. In the following,
the qualitative framework is used.

\subsection{Possibilistic handling of a rule-based system}
\label{subsec:possibilistic-rbs}

Let us now remind possibilistic rules and  possibilistic rule-based systems before focusing on the case of a cascade, i.e., when a possibilistic rule-based system uses two chained sets of possibilistic rules. 

\subsubsection{Possibilistic handling of uncertain rules}
\label{subsubsec:uncertain}

The possibilistic handling of a rule-based system was introduced in the 80's \cite{farreny1986default,farrency2013positive} and recently revisited in \cite{dubois2020possibilistic}. In this framework, the uncertainty of  rules ``if $p$ then $q$'' and  ``if $\neg p$ then $\neg q$'' is handled by a matrix calculus based on max-min composition:
$$\begin{bmatrix}
    \pi(q) \\ 
    \pi(\neg q) 
    \end{bmatrix} = \begin{bmatrix}
        \pi(q\mid p) &   \pi(q\mid \neg p) \\
          \pi(\neg q\mid p) &   \pi(\neg q\mid \neg p) 
    \end{bmatrix} \Box_{\min}^{\max} \begin{bmatrix}
        \pi(p) \\
        \pi(\neg p) 
    \end{bmatrix}$$
\noindent where the matrix product $\Box_{\min}^{\max}$ uses $\min$ as the product and $\max$ as the addition. The max-min product between a vector composed of the possibility degrees of  $p$ and $\neg p$ and the uncertainty propagation matrix, which contains the conditional possibility distributions between $p$ and $q$ (and its negation), yields the possibility degrees of $q$ and $\neg q$. 

\noindent One can remark that the obtained formula for $\pi(q)$:
\begin{equation}\label{eq:pi_q}
    \pi(q) = \max (\min( \pi(q \mid p), \pi(p)), \min( \pi(q \mid \neg p), \pi(\neg p)))
\end{equation}
\noindent is analog to the total probability theorem (shown in \cite{farreny1986default}):
\begin{equation}\label{eq:prob_q}
    \text{prob}(q) = \text{prob}(q \mid p)\cdot \text{prob}(p) + \text{prob}(q \mid \neg p)\cdot \text{prob}(\neg p).
\end{equation}

The uncertainty weights in the uncertainty propagation matrix obey a qualitative form of conditioning and encode the uncertainty of ``if $p$ then $q$'' and of ``if $\neg p$ then $\neg q$''. In  the formula (\ref{eq:pi_q}), one can see that the weights $ \pi(q \mid p)$ and $\pi(q \mid \neg p)$ act as thresholds for $\pi(p)$ and $\pi(\neg p)$ respectively.

\smallskip
The $\max-\min$ composition governing the matrix calculus and the normalization conditions $\max(\pi(p),\pi(\neg p))=1$, $\max(\pi(q\mid p),\pi(\neg q\mid p))=1$ and $\max(\pi(q\mid \neg p),\pi(\neg q\mid \neg p))=1$, ensure that the possibility degrees of the conclusion $q$ are normalized, i.e., $\max(\pi(q),\pi(\neg q))=1$. \\

In the possibilistic setting, an uncertainty propagation matrix of the form $\begin{bmatrix}
        \pi(q\mid p) &   \pi(q\mid \neg p) \\
          \pi(\neg q\mid p) &   \pi(\neg q\mid \neg p) 
    \end{bmatrix}=\begin{bmatrix}
        1 & s \\
        r & 1 
    \end{bmatrix}$ where $s,r \in [0,1]$ are the rule parameters, is associated with the rules  ``if $p$ then $q$''  and ``if $\neg p$ then $\neg q$'' to indicate that  ``if $p$ then $q$'' holds with certainty
$1 - r$ and that ``if $\neg p$ then $\neg q$'' holds with certainty $1 - s$ \cite{dubois2020possibilistic}. From such a matrix, assuming the normalization condition  $\max(\pi(p),\pi(\neg p))=1$, we have:

\begin{equation}\label{eq:pi_q2}
\pi(q) = \max(\pi(p),\min( s, \pi(\neg p))) = \max(\pi(p), s),
\end{equation}
\begin{equation}\label{eq:pi_neg_q}
\pi(\neg q) = \max(\pi(\neg p),\min( r, \pi(p)) = \max(\pi(\neg p),r).
\end{equation}

\smallskip
If the rule parameters are equal to zeros, i.e., $s = r = 0$,  a form of equivalence is expressed between ``if $p$ then $q$'' and ``if $q$ then $p$'' (viewed as if $\neg p$ then $\neg q$) \cite{dubois2020possibilistic,farreny1986default}.

The possibilistic handling of an uncertain rule is based on the $\max-\min$ matrix calculus, which  can be closely related to possibilistic logic \cite{dubois2020possibilistic}, the latter being highly compatible with classical logic \cite{dubois2004possibilistic}. 

Through simple examples, let us explain how uncertainty weights are involved in inference from a possibilistic rule:
\begin{example}
Let us start with a rule where $s = 0$ and $r = 0$, so both ``if $p$ then $q$'' and ``if $\neg p$ then $\neg q$'' are certain. In this case, the possibility degree of $q$ (resp. $\neg q$) will be equal to that of $p$ (resp. $\neg p$). In particular, we have:
\begin{itemize}
    \item if $(\pi(p),\pi(\neg p)) = (0,1)$, then $(\pi(q),\pi(\neg q)) = (0,1)$,
    \item if $(\pi(p),\pi(\neg p)) = (1,0)$, then $(\pi(q),\pi(\neg q)) = (1,0)$.
\end{itemize}

\noindent Let us now consider the rule parameters: $s = 0.3$ and $r = 0.5$. In this case, ``if $p$ then $q$'' is certain to a degree $1 - r = 0.5$ and ``if $\neg p$ then $\neg q$'' is certain to a degree $1 - s = 0.7$. We study the following cases:
\begin{itemize}
    \item if $(\pi(p),\pi(\neg p)) = (0,1)$, then $(\pi(q),\pi(\neg q)) = (s,1) = (0.3,1)$, so even if $\pi(p) = 0$, we have $\pi(q) = s > 0$ where $s$ is the rule parameter associated with ``if $\neg p$ then $\neg q$''.
    \item if $(\pi(p),\pi(\neg p)) = (1,0)$, then $(\pi(q),\pi(\neg q)) = (1,r) = (1,0.5)$, so even if $\pi(\neg p) = 0$, we have $\pi(\neg q) = r > 0$ where $r$ is the rule parameter associated with the rule ``if $p$ then $q$''.
\end{itemize}
\end{example}

We can therefore describe uncertain domain knowledge using this type of rule.

Finally, in order to specify that a rule ``if $p$ then $q$'' is totally certain, it is enough to set $r = 0$. Indeed, when $r = 0$, provided that the premise $p$ of the rule is totally certain (i.e., $\pi(p) = 1$ and $\pi(\neg p) = 0$), its conclusion $q$ is also totally certain (i.e., $\pi(q) = 1$ and $\pi(\neg q) = 0$).

\subsubsection{Possibilistic rule-based system}
\label{subsec:possibilistic-rule-based-system}

As sketched previously, a \textit{possibilistic rule-based system} is composed of $n$ if-then possibilistic rules $R^1,R^2,\dots,R^n$ \cite{dubois2020possibilistic}. Each rule $R^i$ is of the form ``if $p_i$ then $q_i$'' and is associated with an uncertainty propagation matrix which has two rule parameters $s_i,r_i$:
\[ \begin{bmatrix}  \pi(q_i | p_i) & \pi(q_i | \neg p_i) \\
 \pi(\neg q_i | p_i) & \pi(\neg q_i | \neg p_i) 
  \end{bmatrix} =\begin{bmatrix} 1 & s_i \\r_i & 1\end{bmatrix}.\] 
  
 \noindent The premise  $p_i = p^i_1 \wedge p^i_2 \wedge \dots \wedge p^i_k$ of $R^i$ is a conjunction of propositions $p_j^i$: ``$a_j^i(x) \in P_j^i$'', where $P_j^i$ is a subset of the domain $D_{a_j^i}$ of the attribute $a_j^i$ and $\overline{P_j^i}$ is its complement. The attribute $a_j^i$ is applied to an item $x$. The information about $a_j^i(x)$ is represented by a  possibility distribution $\pi_{a_j^i(x)}: D_{a_j^i} \rightarrow [0,1]$, which is supposed to be normalized, i.e., $\exists u \in D_{a_j^i}$ such that $\pi_{a_j^i(x)}(u) = 1$. The possibility degree of $p_j^i$ and that of its negation are computed using the possibility measure $\Pi$ by:
 \[ \pi(p_j^i) = \Pi(P_j^i) = \sup_{u \in P_j^i}\pi_{a_j^i(x)}(u) \text{ and } \pi(\neg p_j^i) = \Pi(\overline{P_j^i}) = \sup_{u \in  \overline{P_j^i}}\pi_{a_j^i(x)}(u).  \] 
  As $\pi_{a_j^i(x)}$ is normalized, we have $\max(\pi(p_j^i),\pi(\neg p_j^i))=1$. The necessity degree of  $p_j^i$ is defined with  the necessity measure $N$ by $n(p_j^i) = N(P_j^i) =  1 - \pi(\neg p_j^i) = \inf_{u \in \overline{P_j^i}}(1 - \pi_{a_j^i(x)}(u))$. \\

Let us stress that the notation $a_j^i$ is specifically used for distinguishing the attribute involved in the  proposition $p_j^i$ of the premise $p_i$. In cases when $i \neq i'$, two attributes $a_j^i$ and $a_{j'}^{i'}$ may represent the same concept and therefore share the same possibility distribution, i.e., ${\pi_{a_j^i(x)}} = {\pi_{a_{j'}^{i'}(x)}}$.\\

The  possibility degree of $p_i$ and that of its negation are defined by:
\begin{equation}\label{eq:pinegpiformulae}
    \pi(p_i) = \min_{j=1}^k \pi(p_j^i) \text{ and }  \pi(\neg p_i) = \max_{j=1}^k \pi(\neg p_j^i). 
\end{equation}

\noindent These values $\pi(p_i)$ and $\pi(\neg p_i)$ preserve the  normalization, i.e., $\max(\pi(p_i),\pi(\neg p_i))=1$ and are respectively noted  $\lambda_i$ and $\rho_i$. The necessity degree of $p_i$ is $n(p_i) = 1 - \pi(\neg p_i) = \min_{j=1}^k(1 - \pi(\neg p_j^i)) = \min_{j=1}^k n(p_j^i)$. The degrees $\lambda_i$ and $\rho_i$ are such that:
\begin{itemize}
    \item $\pi(p_i) = \lambda_i$ estimates to what extent $p_i$ is possible,
    \item  $n(p_i)=1 - \rho_i$ estimates to what extent $p_i$ is certain.
\end{itemize}

\smallskip
The conclusion $q_i$ of $R^i$ is of the form ``$b(x) \in Q_i$'', where $Q_i \subseteq D_b$ and $Q_i \neq \emptyset$.  \textit{Accordingly, in a possibilistic rule-based system, the conclusions of the rules are all about the same attribute $b$}.
The possibility degrees of $q_i$ and $\neg q_i$ are respectively noted $\alpha_i$ and $\beta_i$. They are defined by the following matrix calculus: \[ \begin{bmatrix} \pi(q_i) \\ \pi(\neg q_i) \end{bmatrix} = \begin{bmatrix}
    \alpha_i\\
    \beta_i
\end{bmatrix}= \begin{bmatrix} 1 & s_i \\r_i & 1\end{bmatrix} \Box_{\min}^{\max} \begin{bmatrix}\lambda_i \\ \rho_i \end{bmatrix},\] \noindent where the operator  $\Box_{\min}^{\max}$ uses $\min$ as the product and $\max$ as the addition. We still have $\max(\pi(p_i),\pi(\neg p_i))=1$, which implies, see (\ref{eq:pi_q}) and (\ref{eq:pi_neg_q}):
\begin{equation}
    \label{eq:alphabeta}
    \alpha_i = \max(s_i,\lambda_i) \text{ and }\beta_i = \max(r_i,\rho_i).
\end{equation}

\noindent The possibility distribution of the output attribute $b$ associated with $R^i$ is given by $\pi_{b(x)}^{\ast i}(u) = \alpha_i \mu_{Q_i}(u)   +  \beta_i \mu_{\overline{Q_i}}(u)$ for any $u \in D_b$, where $\mu_{Q_i}$ and $\mu_{\overline{Q_i}}$ are the characteristic functions of the sets $Q_i$ and $\overline{Q_i}$, respectively.

If the possibility distribution $\pi_{b(x)}^{\ast i}$ is not normalized, the rule $R^i$ is viewed as incoherent 
\cite{dubois2020possibilistic,dubois1994validation}.  Rules with a conclusion ``$b(x) \in \emptyset$'' are discarded to avoid this situation. More generally, incoherent rules $R^i$ can be "repaired" by considering an additional value in the domain of the attribute $b$ specifically for the purpose of ensuring that $\pi_{b(x)}^{\ast i}$ is normalized. 

\smallskip
When the possibilistic rule-based system consists of $n$ rules,  the output possibility distribution of the output attribute $b$ is defined by a min-based conjunctive~combination: 
\begin{equation}\label{eq:pistarbxu} 
   \pi^{\ast}_{b(x)}(u) = \min(\pi_{b(x)}^{\ast 1}(u), \pi_{b(x)}^{\ast 2}(u), \dots ,\pi_{b(x)}^{\ast n}(u)). 
\end{equation}

The concept of possibilistic rule-based system is illustrated at Example \ref{ex:firstrules-set}. For this specific example, an equivalence between the premise and the conclusion of each rule holds, so the parameters of the rules can be set to zero, i.e., $s_i = r_i = 0$.

In the context of our neuro-symbolic approach $\Pi$-NeSy, for joint inference, the neural model is connected to the possibilistic rule-based system by establishing that each domain $D_{a_j^i}$ of an input attribute $a_j^i$ is a set of  intermediate concepts and the domain of the output attribute $b$ is the set of (meta-)-concepts.

\subsubsection{Cascading rules}
\label{subsub:cascade}
 In the case of a cascade, a possibilistic rule-based system relies on a first set of $n$ if-then possibilistic rules $R^1,R^2,\dots,R^n$ and a second set of $m$ if-then possibilistic rules $R'^1, R'^2, \dots, R'^m$, where  both the conclusions of the rules $R^i$ and the premises of the rules $R'^j$ use the \textit{same attribute} $b$, making possible to chain the two sets of rules. Indeed, each rule $R'^j$ is of the form ``if $p'_j$ then $q'_j$'' where $p'_j$ is a proposition ``$b(x) \in Q'_j$'',  $Q'_j$ being a subset of $D_b$. The conclusion $q'_j$ is of the form ``$c(x) \in Q''_j$'' where $Q''_j$  is a subset of $D_c$, the domain of the attribute $c$. 
 
 The  possibility degrees associated with $R'^j$ are computed in the same way as those of the rules $R^i$: as $p'_j$ is a proposition, we compute $\lambda'_j = \pi(p'_j)$ and $\rho'_j = \pi(\neg p'_j)$ with the normalized possibility distribution of the attribute $b$. Similarly, $R'^j$ has an uncertainty propagation matrix  with its associated parameters $s'_j,r'_j$.

\begin{example}\label{ex:secondset-rules}(Example \ref{ex:firstrules-set}, cont'ed)\\
The set of possibilistic rules given in Example \ref{ex:firstrules-set} can be chained to another set $\{R^{'1}, R^{'2}\}$ consisting of two rules, where the domain of the output attribute $c$ of the second set is $D_c = \{ 0, 1 \}$:
\begin{itemize}
    \item $R^{'1}$: If $b(x) \in \{ (0,0), (1,1) \}$ then $c(x) \in \{ 1\}$, 
    \item $R^{'2}$: If $b(x) \in \{ (0,1), (1,0) \}$ then $c(x) \in \{ 0 \}$.
\end{itemize}
 Since the new rules are intended to represent equivalences, all the parameters of the rules are set to zero, i.e., $s'_j = r'_j = 0$.  
\noindent Then the statement $\pi_{c(x)}(1) = 1$ means that it is fully possible (i.e., with a possibility degree of 1) that the two MNIST images represent the same handwritten numbers. If $\pi_{c(x)}(0) = 1$, it is possible with a degree of 1 that the two MNIST images do not represent the same handwritten numbers. 
\end{example}

\subsection{Probability-possibility transformations}\label{sec:ppt}

In the following, we present the two main probability-possibility transformations from the literature.  The first method was introduced in \cite{dubois1983unfair} and named  ``antipignistic method'' in  \cite{dubois2020possibilistic}.  The second method   obeys the minimum specificity principle and was introduced in   \cite{delgado1987concept,dubois1982several}.

Transforming a probability distribution $p$ on $X$ (with its associated probability measure $P$) into a possibility distribution $\pi$ on $X$  (with its associated possibility measure $\Pi$ and necessity measure $N$) 
consists in \textit{finding a framing interval} $[N(A), \Pi(A)]$ of $P(A)$ for any subset $A \subseteq X$ \cite{dubois2006possibility,dubois1993possibility}:  the possibility measure $\Pi$ dominates the probability measure $P$. The transformation of the probability distribution $p$ into a possibility distribution $\pi$ should preserve the shape of the distribution: for $u,u' \in X$, $p(u) > p(u') \Longleftrightarrow \pi(u) > \pi(u')$. We also want to obtain a possibility distribution that is as specific as possible  (given two possibility distributions $\pi$ and $\pi'$,  $\pi$ is at least as specific as $\pi'$ if $\pi \leq \pi'$, i.e., $\forall u \in X, \pi(u) \leq \pi'(u)$ \cite{dubois2004probability}).
This last condition is motivated by the concern to keep as much information as possible, i.e., $\pi$ is more informative than $\pi'$.

The two transformations presented in the following are used in  our neuro-symbolic approach $\Pi$-NeSy:  the output probability distributions of the neural network are transformed into  normalized possibility distributions, in order  to connect our low-level perception task to our high-level reasoning task. 

\subsubsection{Antipignistic method}
\label{subsec:antipignistic}
The antipignistic method, recently reminded in \cite{dubois2020possibilistic}, was introduced  and motivated in \cite{dubois1983unfair}.

\noindent If $p$ is a probability distribution on $X$, let $P$ denote the probability measure on $X$ defined by $p$, i.e., $P(A) = \sum_{x\in A} p_x$ where $p_x = P(\{x\})$. 

\noindent The antipignistic method associates  a normalized possibility distribution $\pi$ on $X$ with $p$, which verifies that for all $A \subseteq X$:
 \[ N(A) \leq P(A) \leq \Pi(A),\]
\noindent where $N(A)$ and $\Pi(A)$ are the necessity and possibility measures defined by $\pi$.

\noindent {\it Let us suppose that the elements of $X$ are ordered} so that for $X=\{x_1, \dots, x_n\}$, we have  $p_1 \geq p_2 \dots \geq p_n\,$ \text{where} $\quad p_i =P(\{x_i\}).$ We call this assumption \textit{the decreasing assumption}.

\noindent The possibility degree $\pi_i = \pi(x_i)$ of $x_i$ where $1 \leq i \leq n$ is defined  by:

\begin{equation*}
\pi_i = i p_i + \sum_{j = i+1}^n p_j =  \sum_{j=1}^n \min(p_j, p_i).    
\end{equation*}

\noindent where the equality  $i p_i + \sum_{j = i+1}^n p_j =  \sum_{j=1}^n \min(p_j, p_i)$ holds because of the assumption  $p_1 \geq p_2 \dots \geq p_n$.

\noindent For all $A \subseteq X$, the necessity measure of $A$ can be computed as:
$$N(A) = \sum_{x\in A} \max(p_x  - \max_{y\notin A} p_y,0).$$
\noindent  Note that $X$ can be exhausted as follows:
\[ A_0 = \emptyset  \subset A_1 \subset A_2 \subset \dots \subset A_n = X,\,  \text{with }\,
A_i = \{x_1, x_2, \dots, x_i\}.\]
Then, we have:
\[ N(A) = \max_{0\leq k \leq n, A_k \subseteq A} N(A_k). \]
For $k = 0, 1, 2, \dots, n$, the computation of $N(A_k)$ by the preceding abstract formula (with the convention $p_{n+1} = 0 $) becomes:
\[ N(\emptyset) = 0,\, N(A_k) = \sum_{i=1}^k (p_i - p_{k+1}),\, N(X) = \sum_{i=1}^n p_i = 1.\]

\noindent We then have for all $A \subseteq X$:  $N(A) \leq P(A) \leq \Pi(A)$,  (see \cite{dubois1983unfair} for the proof and the   underlying semantics of this result). 

\noindent Note that the possibility distribution $\pi$ associated with such a probability distribution $p$ verifies:
\[ \pi_1 = 1, \quad 
\pi_i - \pi_{i+1 } = i(p_i  - p_{i+1}) \geq 0.\]
and then we have $\pi_1 = 1 \geq  \pi_2 \geq \dots \geq \pi_n$.

\noindent Observe that starting from a normalized possibility distribution $\pi$  that verifies  $\pi_1 = 1 \geq  \pi_2 \geq \dots \geq \pi_n$, the previous formula can also be used for associating with $\pi$ a probability distribution $p$ which verifies $p_1 \geq p_2 \geq \dots \geq p_n$. The probability distribution 
$p$ expressed in terms of $\pi$ is defined by: 
\[p_i = \sum_{j=i}^n \frac{1}{j}(\pi_j - \pi_{j+1}) \quad \text{with the convention } \quad  \pi_{n+1} = 0. \]

\noindent In conclusion, between the set of probability distributions $p$ on the set $\{1, 2, \dots, n\}$ which verifies $p_1 \geq p_2 \geq \dots \geq p_n$ and 
the set of   normalized possibility distributions
$\pi$ on the set $\{1, 2, \dots, n\}$ that verify 
$\pi_1=1 \geq  \pi_2 \geq \dots \geq \pi_n$
we have the following one-to-one correspondence:
\[ p \mapsto \pi : \pi_i =   i p_i + \sum_{j = i+1}^n p_j =  \sum_{j=1}^n \min(p_j, p_i),\] 
\[ \pi \mapsto p : p_i =   \sum_{j=i}^n \frac{1}{j}(\pi_j - \pi_{j+1}), \text{with the convention }  \pi_{n+1} = 0.\]
This one-to-one correspondence can be used on any set $X = \{x_1, x_2, \dots, x_n\}$ where the  domains of definition of each of the two mappings $p \mapsto \pi$ and $\pi \mapsto p$ satisfy the decreasing assumption. 

\noindent Finally, one can observe that the mapping $\pi \mapsto p$ preserves the shape of the distributions, i.e., $\pi_i > \pi_{i+1} \Longleftrightarrow p_{i} > p_{i+1}$.\\

\begin{example}\label{ex:probposs_antipignistic}
A neural network typically yields a probability distribution of the following form, where one value of the distribution is highly probable while the others are not:
\begin{align*}
    P(\includegraphics[width=0.017\textwidth]{z_img__9158_ZERO.png}) &= \left[P\left(\includegraphics[width=0.017\textwidth]{z_img__9158_ZERO.png} = 0\right), P\left(\includegraphics[width=0.017\textwidth]{z_img__9158_ZERO.png} = 1\right), \cdots, P\left(\includegraphics[width=0.017\textwidth]{z_img__9158_ZERO.png} = 9\right)\right]\\  
    &= \left[0.91, 0.01, 0.01, 0.01, 0.01, 0.01, 0.01, 0.01, 0.01, 0.01\right]
\end{align*}

\noindent We obtain the corresponding possibility distribution using the antipignistic method:

\begin{align*}
    \left[\Pi\left(\includegraphics[width=0.017\textwidth]{z_img__9158_ZERO.png} = 0\right), \Pi\left(\includegraphics[width=0.017\textwidth]{z_img__9158_ZERO.png} = 1\right), \cdots, \Pi\left(\includegraphics[width=0.017\textwidth]{z_img__9158_ZERO.png} = 9\right)\right]=
    \left[1.00, 0.10, 0.10, \cdots, 0.10\right].
\end{align*}

\noindent If we apply the antipignistic transformation to a more  ambiguous probability distribution such as:
\begin{align*}
    P(\includegraphics[width=0.017\textwidth]{z_img__9158_ZERO.png}) &= \left[P\left(\includegraphics[width=0.017\textwidth]{z_img__9158_ZERO.png} = 0\right), P\left(\includegraphics[width=0.017\textwidth]{z_img__9158_ZERO.png} = 1\right), \cdots, P\left(\includegraphics[width=0.017\textwidth]{z_img__9158_ZERO.png} = 9\right)\right]\\  
    &= \left[0.15, 0.14, 0.13, 0.12, 0.11, 0.09, 0.08, 0.07, 0.06, 0.05\right].
\end{align*}
\noindent  we obtain the following possibility distribution:
\begin{align*}
    \left[\Pi\left(\includegraphics[width=0.017\textwidth]{z_img__9158_ZERO.png} = 0\right), \Pi\left(\includegraphics[width=0.017\textwidth]{z_img__9158_ZERO.png} = 1\right), \cdots, \Pi\left(\includegraphics[width=0.017\textwidth]{z_img__9158_ZERO.png} = 9\right)\right]=
    \left[1.00, 0.99, 0.97, 0.94, 0.90, 0.80, 0.74, 0.67, 0.59, 0.50\right].
\end{align*}
 Using the probability-possibility transformation presented in Subsection \ref{subsec:probposs_mps},  we can obtain more specific possibility distributions than those obtained here, as we will see at Example \ref{ex:probposs_msp}.

\end{example}

Dubois and Prade state that the antipignistic method provides an intuitive ground to the perception of the idea of certainty \cite{dubois2020possibilistic}. However, the possibility distribution that is generated is less specific than the the one that can be obtained when using the probability-possibility transformation presented in the following subsection.

In the case of a neuro-symbolic approach where a neural model is connected to a possibilistic rule-based system based on the antipignistic transformation, the converse mapping from possibility distributions to probability distributions is valuable since it makes feasible the implementation of a \emph{backpropagation mechanism}, see Subsection \ref{subsub:backpropagation}: a targeted input possibility distribution of the possibilistic rule-based system can be transformed into a targeted output probability distribution for the neural network. Accordingly, the predictive performance of the neural model that is used can be improved thanks to the feedback received from the system of possibilistic rules.

\subsubsection{Method obeying minimum specificity principle}
\label{subsec:probposs_mps}

A second well-known transformation of probability distributions into possibility distributions is based on a totally different rationale and was introduced in \cite{delgado1987concept,dubois1982several}. The resulting possibility distribution is the most specific possibility distribution $\pi^\star$
whose associated possibility measure dominates the probability measure.

\noindent Starting from  a probability distribution $p$ on $\{1, 2, \dots ,n\}$, for all $i\in\{1, 2, \dots ,n\}$, we put:
\[ \pi_i^* = \sum_{j=i}^n p_j.\]
Then, we have:
\[ \pi_1^* = \sum_{j=1}^n p_j  = 1 \quad \text{ and }\quad  \pi_1^* =1 \geq \pi_2^* \geq \dots \geq \pi_n^*\geq 0, \]
so $\pi^*$
is a normalized possibility distribution on $\{1, 2, \dots, n\}$. 

\noindent Assuming  that we have $p_1 \geq p_2 \geq \dots \geq p_n$, we can associate with $p$, using the antipignistic method, 
the normalized possibility distribution $\pi$ defined by  
 $$\pi_i  = i p_i + \sum_{k = i+1}^n p_k.$$

\noindent Then, we have:
 \[ \forall i\in \{1, 2, \dots, n\} ,\,
 \pi_i^* \leq \pi_i.\]
 Indeed, we have:
 \[ \pi_i^* = \sum_{j=i}^n p_j = p_i +    \sum_{j=i +1}^n p_j \leq i p_i + \sum_{j = i+1}^n p_j = \pi_i.\]
 
\noindent  The inequality $\forall i\in \{1, 2, \dots, n\},\,
 \pi_i^* \leq \pi_i$ means that $\pi^*$ is at least as specific as $\pi$:  what we observe is that, with regard to extreme values, $\pi_i = 0 \Longrightarrow \pi_i^* = 0$ so $\pi^*$ restricts the possible values at least as much as $\pi$ and $\pi_i^* = 1 \Longrightarrow \pi_i = 1$ so $\pi^*$ is at least as informative as $\pi$.

This transformation $p \mapsto \pi^\ast$ is \textit{motivated by the concern to keep as much information as possible} \cite{dubois2020possibilistic}. 
The main drawback of the method obeying minimum specificity principle is that this transformation cannot be reversed, i.e., for this method,  there is no inverse mapping from possibility to probability. For a neuro-symbolic approach where a neural model is connected to a possibility rule-based system based on this probability-possibility transformation, this impossibility precludes any backpropagation of information from the symbolic rule-based system to the neural model.

\begin{example}\label{ex:probposs_msp}
To illustrate the probability-possibility transformation obeying  the minimum specificity principle, we use the same probability distributions as those used in Example \ref{ex:probposs_antipignistic} with the antipignistic method.
Based on the probability distribution that would typically be obtained from a neural network:
\begin{align*}
    P(\includegraphics[width=0.017\textwidth]{z_img__9158_ZERO.png}) &= \left[P\left(\includegraphics[width=0.017\textwidth]{z_img__9158_ZERO.png} = 0\right), P\left(\includegraphics[width=0.017\textwidth]{z_img__9158_ZERO.png} = 1\right), \cdots, P\left(\includegraphics[width=0.017\textwidth]{z_img__9158_ZERO.png} = 9\right)\right]\\  
    &= \left[0.91, 0.01, 0.01, 0.01, 0.01, 0.01, 0.01, 0.01, 0.01, 0.01\right],
\end{align*}
\noindent we obtain the following corresponding possibility distribution using the method obeying  the minimum specificity principle:
\begin{align*}
    \left[\Pi\left(\includegraphics[width=0.017\textwidth]{z_img__9158_ZERO.png} = 0\right), \Pi\left(\includegraphics[width=0.017\textwidth]{z_img__9158_ZERO.png} = 1\right), \cdots, \Pi\left(\includegraphics[width=0.017\textwidth]{z_img__9158_ZERO.png} = 9\right)\right]=
    \left[1.00, 0.09, 0.08, 0.07, 0.06, 0.05, 0.04, 0.03, 0.02, 0.01\right].
\end{align*}
From the probability distribution with higher levels of ambiguity:
\begin{align*}
    P(\includegraphics[width=0.017\textwidth]{z_img__9158_ZERO.png}) &= \left[P\left(\includegraphics[width=0.017\textwidth]{z_img__9158_ZERO.png} = 0\right), P\left(\includegraphics[width=0.017\textwidth]{z_img__9158_ZERO.png} = 1\right), \cdots, P\left(\includegraphics[width=0.017\textwidth]{z_img__9158_ZERO.png} = 9\right)\right]\\  
    &= \left[0.15, 0.14, 0.13, 0.12, 0.11, 0.09, 0.08, 0.07, 0.06, 0.05\right],
\end{align*}
\noindent we obtain the following possibility distribution using the method obeying the minimum specificity principle:
\begin{align*}
    \left[\Pi\left(\includegraphics[width=0.017\textwidth]{z_img__9158_ZERO.png} = 0\right), \Pi\left(\includegraphics[width=0.017\textwidth]{z_img__9158_ZERO.png} = 1\right), \cdots, \Pi\left(\includegraphics[width=0.017\textwidth]{z_img__9158_ZERO.png} = 9\right)\right]=
    \left[1.00, 0.85, 0.71, 0.58, 0.46, 0.35, 0.26, 0.18, 0.11, 0.05\right].
\end{align*}
 We can see that the possibility distributions generated in this example are more specific than those obtained in Example \ref{ex:probposs_antipignistic}.
\end{example}

In practice, the choice of a method for transforming a probability distribution into a possibility distribution mainly depends on the required specificity of the obtained possibility distribution. For some applications, the fact that the antipignistic method is a one-to-one correspondence between  probability measures and possibility measures can be useful.

\section{Practical procedures for building the matrix relation and the equation system associated with a possibilistic rule-based system}
\label{sec:practical-procedure-building-reduced-ES}

 For taking advantage of a possibilistic rule-based system (Subsection \ref{subsec:possibilistic-rule-based-system}) one needs to generate:
\begin{itemize}
    \item the matrix relation of \cite{baaj2021min}, see (\ref{eq:esFP}),  which is used to perform inferences from the possibilistic rule-based system at hand via a min-max matrix product. 
    
    \item the equation system of \cite{baaj2022learning}, see (\ref{eq:sigma}), which is used to learn the rule parameters of the possibilistic rule-based system according to training data samples.
\end{itemize}
 The essential tool governing  the matrix relation and the equation system is \textit{an explicit partition of the domain $D_b$ of the output attribute $b$} which is constructed according to the conclusions of the rules \cite{baaj2021min}. As explained 
 at the beginning of Subsection \ref{subsec:matrixrelpractical}  (resp. Subsection \ref{subsec:practical-building-learning}), in the matrix relation of \cite{baaj2021min} (resp. the equation system of \cite{baaj2022learning}), each component of its second member (and each row of its matrix) is directly  related to a subset of this partition. \\
 Following \cite{baaj2021min}, a partition of $D_b$ denoted  $(E_{k}^{(i)})_{1 \leq k \leq 2^{i}}$ and  built from the sets $Q_1,Q_2,\cdots,Q_n$ used in the conclusions of the rules $R^1, R^2, \cdots, R^n$ and their complements  $\overline{Q_1},\overline{Q_2},\cdots, \overline{Q_n}$ can be defined inductively as follows: 
\begin{subequations}
\label{eq:partition}
\begin{align}
&\text{\textbullet\,} E_{1}^{(1)} = Q_1\text{ and }E_{2}^{(1)} = \overline{Q_1}&\\ 
 &\text{and for  $i > 1$:}&\nonumber
 \\
&\text{\textbullet\,} E_{k}^{(i)} = \begin{cases} 
 E_{k}^{(i - 1)}  \cap Q_i & \text{ if }  1 \leq k \leq 2^{i-1}  \\
  E_{k -2^{i-1}}^{(i - 1)}   \cap \overline{Q_i} & \text{ if }  2^{i-1} < k \leq 2^{i}\end{cases}.&
\end{align}
\end{subequations}

\noindent This partition $(E_{k}^{(n)})_{1 \leq k \leq 2^{n}}$ can be reduced to its non-empty subsets (since the possibility degree of  the emptyset is equal to zero), and it contains at most $\min(\text{card}(D_b),2^n)$ non-empty subsets, see  \cite{baaj2021min}. Unfortunately,  the above inductive definition  (\ref{eq:partition}) does not allow us to construct a partition that is suitable for practical use since it involves generating $2^i$ subsets at each iteration  where $i$ ranges from $1$ to $n$ ($n$ being the number of rules).

In this section, in order to address this crucial computational issue, we present  an efficient inductive procedure for constructing a partition (\ref{eq:partition}) reduced to its non-empty subsets, see  Lemma \ref{lemma:LambdaN} and Proposition \ref{prop:isom}. This method has a polynomial-time complexity with respect to the number of rules, see Proposition \ref{prop:complexity:nbOp}.  Using it, the matrix relation and the equation system for learning, whose rows directly correspond to the non-empty subsets of the partition, can be generated efficiently, see Subsection \ref{subsec:matrixrelpractical} and Subsection \ref{subsec:practical-building-learning} respectively. Therefore, thanks to this method, the matrix relation and the equation system can be used in practice (the approach scales sufficiently well).

\subsection{A practical method for constructing a partition reduced to its non-empty subsets}
\label{subsec:partitionred}
 
\begin{notation}\label{nottrois}
    Let us consider:
    \begin{itemize}
    \item $D_b$ is the domain of the output attribute $b$, 
    \item $Q_1, Q_2, \cdots Q_n$ are the subsets of $D_b$ which are used in the conclusions of the rules $R^1, R^2\dots, R^n$, 
    \item $(E_k^{(n)})_{1 \leq k \leq 2^n}$   is the partition of $D_b$ formed from the subsets $Q_1, Q_2, \cdots Q_n$  and their complements by following the inductive method (\ref{eq:partition}), 
    \item $J^{(n)} = \left\{k \in\{1, 2, \cdots, 2^n\} \bigm\vert E_{k}^{(n)} \neq \emptyset \right\}$ is the set of the indexes of the non-empty subsets of the partition $(E_k^{(n)})_{1 \leq k \leq 2^n}$ and   $\omega^{(n)} = \text{card}(J^{(n)})$. The indexes in $J^{(n)}$ are arranged in a strictly increasing sequence: $1 \leq k^{(n)}_1 < k^{(n)}_2 < \dots < k^{(n)}_{\omega^{(n)}} \leq 2^n$.
\end{itemize}
\end{notation}

\noindent Our purpose is to give an inductive definition  of the ordered set $(J^{(n)},  \leq)$ whose elements are the indexes of the non-empty subsets $(E_k^{(n)})_{k \in J^{(n)}}$. To do this, we inductively define an ordered set $(\Lambda^{(n)}, \preceq^{(n)})$ isomorphic to the ordered set $(J^{(n)}, \leq)$, see proofs in Section \ref{sec:proofs}.

\noindent
\subsubsection{\texorpdfstring{Definition of the set $\Lambda^{(n)}$}{Definition of the set Lambda(n)}}

\noindent
We begin by giving a characterization of the non-empty sets of the partition $(E_{k}^{(n)})_{1 \leq k \leq 2^{n}}$:
\begin{restatable}{lemma}{lemmasubsettouniquetuplecmd}\label{lemma:subsettouniquetuple}
Each subset in the partition $(E_k^{(n)})_{1 \leq k \leq 2^{n}}$ is of the form:
 $$E_k^{(n)} =  T_1 \cap T_2 \cap \dots \cap T_n \quad \text{where}  \quad T_i\in \{Q_i, \overline{Q_i}\}.$$ 
Reciprocally, if $(T_1, T_2, \dots, T_n)$ is a tuple of sets such that   $T_i\in \{Q_i, \overline{Q_i}\}$, then $T_1 \cap T_2 \cap \dots \cap T_n$ is a subset of the partition  $(E_k^{(n)})_{1 \leq k \leq 2^n}$.

\noindent
Suppose that a set $E_k^{(n)}$ with $k \in \{1, 2, \dots, 2^n\}$ is non-empty, then there exists a \textbf{unique} tuple $(T_1,T_2,\dots,T_n)$ such that 
$E_k^{(n)} =  T_1 \cap T_2 \cap \dots \cap T_n$ and each $T_i$ is equal to $Q_i$ or $\overline{Q_i}$.
\end{restatable}
\noindent See Subsection \ref{subsec:proof:lemmasubsettouniquetuple} for the proof.

\noindent
Lemma  \ref{lemma:subsettouniquetuple} leads us to introduce for each $i \in \{1, 2, \dots, n\}$,  the following mapping:
\begin{equation}\label{eq:sigman}
    \sigma^{(i)} :   (t_1, t_2, \dots, t_i) \mapsto T_1 \cap T_2 \cap \dots \cap T_i \quad \text{ where }\quad     t_k = \pm k \quad \text{and} \quad 
 T_k = \begin{cases} Q_k & \text{ if } t_k = k\\ 
\overline{Q_k} & \text{ if }   t_k = -k \end{cases}.
\end{equation}
Then, we rely on $\sigma^{(1)}, \sigma^{(2)}, \cdots, \sigma^{(n)}$ to introduce:
\begin{definition}\label{def:LambdaN}
For each $i=1,2,\cdots,n$, let $\Lambda^{(i)}$ be the following set of tuples: 
\begin{equation}
    \Lambda^{(i)} = \bigg\{ \mu=(t_1,t_2,\cdots,t_i) \mid \text{for all } k \in \{1,2,\dots,i\}, \, t_k =\pm k \quad \text{ and } \quad \sigma^{(i)}(\mu) \neq \emptyset \bigg\},
\end{equation}
\noindent $\Lambda^{(i)}$ is  defined using the mapping $\sigma^{(i)}$ and thus relies on  the sets $Q_1,Q_2,\cdots,Q_i$ used in the conclusions of the rules $R^1,  \dots, R^i$ and their complements.
\end{definition}
\noindent
For each $i=1,2,\cdots,n$, we take advantage of the following result to compute inductively $\Lambda^{(i)}$:
\begin{restatable}{lemma}{LambdaNcmd}\label{lemma:LambdaN}
We have:
    \begin{subequations}
\label{eq:orderedset}
\begin{align}
&\text{\textbullet\,} \Lambda^{(1)} = \begin{cases} \{(1) \} & \text{ if } \overline{Q_1} = \emptyset \\[4pt]
\{(-1) \} & \text{ if } {Q_1} = \emptyset\\[4pt]
\{(1), (-1) \} & \text{ otherwise.}
\end{cases}&\\ 
 &\text{and for  $i > 1$:}&\nonumber
 \\
&\text{\textbullet\,} \Lambda^{(i)} = \Lambda^{(i)}_+ \bigcup \Lambda^{(i)}_- \text{where:}&\\
&\Lambda^{(i)}_+ = \bigg\{(t_1, t_2, \dots, t_{i-1}, i)\mid (t_1, t_2, \dots, t_{i-1})\in \Lambda^{(i-1)} \text{ and } \sigma^{(i-1)}(t_1, t_2, \dots, t_{i-1}) \cap Q_i \not=\emptyset\bigg\} &\nonumber\\
&\text{ and }\nonumber\\
&\Lambda^{(i)}_- = \bigg\{(t_1, t_2, \dots, t_{i-1}, -i)\mid (t_1, t_2, \dots, t_{i-1})\in \Lambda^{(i-1)} \text{ and } \sigma^{(i-1)}(t_1, t_2, \dots, t_{i-1}) \cap \overline{Q_i} \not=\emptyset\bigg\}.&\nonumber
\end{align}
\end{subequations}
\end{restatable}
\noindent See Subsection \ref{subsec:proof:LambdaN} for the proof.

\noindent
Based on Lemma \ref{lemma:subsettouniquetuple}, the following lemma shows that the set $\Lambda^{(n)} $ is directly related to the set $J^{(n)}$: 
\begin{restatable}{lemma}{lemmamutokmucmd}\label{lemma:mutokmu}
For each $\mu \in \Lambda^{(n)}$, there is a unique index $k_\mu \in J^{(n)}$  such that  $\sigma^{(n)}(\mu) = E_{k_\mu}^{(n)}$.  The mapping:
\begin{equation}\label{eq:defPsi}
\begin{tabular}{ll}
$\Psi  :$ & $\Lambda^{(n)} \rightarrow J^{(n)}$\\
& $\mu \mapsto k_\mu$ 
\end{tabular} 
\end{equation}
 is bijective and therefore  $\text{card}\, \Lambda^{(n)} = \text{card}\, J^{(n)} \leq \min(\text{card}\,D_b, 2^n)$ .
\end{restatable}
\noindent See Subsection \ref{subsec:proof:lemmamutokmu} for the proof.

\noindent
The following result is used to relate the images of the two  mappings\\ 
\begin{tabular}{ll}
$\Psi  :$ & $\Lambda^{(n)} \rightarrow J^{(n)}$\\
& $\mu \mapsto k_\mu$ 
\end{tabular} \\
\noindent and\\
\begin{tabular}{ll}
$\widetilde \Psi  :$ & $\Lambda^{(n-1)} \rightarrow J^{(n-1)}$\\
& $\widetilde \mu \mapsto k_{\widetilde \mu}$ 
\end{tabular}:
\begin{restatable}{lemma}{lemmamutocmd}\label{lemma:muto}
Let  $n > 1$ and  $\mu = (t_1, t_2, \dots, t_{n-1}, t_n)\in \Lambda^{(n)}$.  Set    
$\mu = (\widetilde \mu, t_n)$ with    $\widetilde \mu = (t_1, t_2, \dots, t_{n-1})$. Then, we have    $\widetilde \mu = (t_1, t_2, \dots, t_{n-1})  \in \Lambda^{(n-1)}$ and $t_n = \pm n$.  The associated indices $k_\mu $  and $k_{\widetilde\mu}$, see  Lemma \ref{lemma:mutokmu}, satisfy: 
\begin{equation}\label{eq:kmu10}
k_\mu = \begin{cases}
     k_{\widetilde\mu}    & \text{ if }\,  t_n = n \\
  k_{\widetilde\mu} + 2^{n-1}& \text{ if }\,  t_n = - n
\end{cases}.
 \end{equation}
\end{restatable}
\noindent See Subsection \ref{subsec:proof:lemmamuto} for the proof.
 
\noindent
\subsubsection{\texorpdfstring{Definition of the total order relation $\preceq^{(n)}$ on the set $\Lambda^{(n)}$}{Definition of the total order relation on the set Lambda(n)}}

\noindent
A total order  $\preceq^{(n)}$ can now be defined on the set $\Lambda^{(n)}$. 
We start from the following order relation $\preceq^{(1)}$   on the set $\Lambda^{(1)}$: 
\begin{itemize}
    \item if  $\Lambda^{(1)}  = \{(1)  \}$ or $\Lambda^{(1)}  = \{(-1)  \}$, since the set $\Lambda^{(1)}$ has only one element, the order relation $\preceq^{(1)}$ is defined as the equality relation, 
    \item if $\Lambda^{(1)}  = \{(1), (-1) \}$, the order relation $\preceq^{(1)}$ is the unique order   relation on the set $\{(1), (-1) \}$ satisfying the condition $(1) \preceq^{(1)} (-1)$.
\end{itemize}
 Assuming that $(\Lambda^{(i-1)} ,\preceq^{(i-1)})$, i.e., a total order $\preceq^{(i-1)}$ on the set $\Lambda^{(i-1)}$ has been defined, the total order $\preceq^{(i)}$ on the set $\Lambda^{(i)}$ can be defined as follows: 

\begin{definition}\label{def:orderrel}
Let $\mu =(t_1,t_2,\dots,t_i), \mu' = (t'_1, t'_2, \dots t'_{i}) \in \Lambda^{(i)}$: 
\begin{subequations}
\label{eq:relationorder}
\begin{align}
&\bullet\text{ if  }\mu = (t_1, t_2, \dots t_{i-1}, i) \text{ and }\mu' = (t'_1, t'_2, \dots t'_{i-1}, - i),\text{ then we take:}&\nonumber\\ 
&\quad \quad \quad\quad \quad \quad  \mu \prec^{(i)}\mu'&\\
& \text{ where } \mu \prec^{(i)}\mu' \text{ means that } \mu \preceq^{(i)}\mu' \text{ and } \mu' \npreceq^{(i)}\mu. &\nonumber\\ 
&\bullet \text{ if } t_i = t'_i,\, \text{then we take:}&\nonumber\\
&\quad \quad \quad \quad \quad \quad \mu \preceq^{(i)}\mu' \Longleftrightarrow
(t_1, t_2, \dots t_{i-1}) \preceq^{(i -1)} (t'_1, t'_2, \dots t'_{i-1}).&&
\end{align}
\end{subequations}
\end{definition}
\noindent Using Lemma \ref{lemma:LambdaN}, it is easy to show that if $(\Lambda^{(i-1)} \preceq^{(i-1)})$ is totally ordered, then $(\Lambda^{(i)} \preceq^{(i)})$ is totally ordered. Starting from $\preceq^{(1)}$, the orders  $\preceq^{(2)}, \preceq^{(3)},\cdots, \preceq^{(n)}$ can be defined using Definition \ref{def:orderrel}.
\subsubsection{\texorpdfstring{The order isomorphism $\Psi :  (\Lambda^{(n)}, \preceq^{(n)}) \rightarrow (J^{(n)}, \leq) $}{The order isomorphism psi : (Lambda(n), preceq(n)) to (J(n),leq)}}

Based on   Lemma \ref{lemma:mutokmu} and Lemma \ref{lemma:muto}, we prove that the mapping
\begin{center}
\begin{tabular}{ll}
$\Psi  :$ & $\Lambda^{(n)} \rightarrow J^{(n)}$\\
& $\mu \mapsto k_\mu$ 
\end{tabular} 
\end{center}
\noindent respects the order relations:
\begin{restatable}{proposition}{propisomcmd}\label{prop:isom}

 Let $\mu$ and $\mu'$ be in $\Lambda^{(n)}$. We have:
\begin{equation}\label{eq:kmu}
 \mu \preceq^{(n)} \mu' \Longleftrightarrow
 k_\mu \leq k_{\mu'}.
\end{equation}
\end{restatable}
\noindent See Subsection \ref{subsec:proof:propisom} for the proof. 

\noindent
From  Lemma \ref{lemma:mutokmu}, we know that $\text{card}\,\Lambda^{(n)} = \omega^{(n)}$, and, by definition, we have  $\omega^{(n)} := \text{card}\,J^{(n)}$. Moreover, we can establish the following result:
\begin{restatable}{proposition}{propnumberopcmd}\label{prop:complexity:nbOp}
    The number of operations  $\beta_n $ required to compute the set $\Lambda^{(n)}$   is less than or equal to $\dfrac{3\cdot \text{card}(D_b) \cdot n}{2}$.
\end{restatable}
\noindent See Subsection \ref{subsec:proof:propnumberop} for the proof. In contrast,  the method (\ref{eq:partition}) presented in \cite{baaj2021min} requires generating $2^i$ subsets at each iteration, where $i$ ranges from $1$ to $n$. As an immediate consequence of  Lemma \ref{lemma:mutokmu} and Proposition \ref{prop:isom}, the following property about the totally ordered sets $(\Lambda^{(n)} \preceq^{(n)})$ and $(J^{(n)} \leq)$ can be derived:
\begin{proposition}\label{prop:Des}
The two ordered sets $(\Lambda^{(n)} \preceq^{(n)})$ and $(J^{(n)} \leq)$ are such that
\begin{equation}\label{eq:LJ}
\Lambda^{(n)} = \{\mu_1, \mu_2, \dots, \mu_{\omega^{(n)}}\},   J^{(n)}  = \{k_{\mu_1}, k_{\mu_2}, \dots, k_{\mu_{\omega^{(n)}}}\}
 \text{ where }  \text{card}\,\Lambda^{(n)}   = \text{card}\,J^{(n)} =\omega^{(n)}   
\end{equation}
and for all $i < j$ we have $\mu_i \prec^{(n)} \mu_j$ and 
$k_{\mu_i} < k_{\mu_j}$.
\end{proposition}
\qed

\noindent To conclude, the construction of the totally ordered set $(\Lambda^{(n)},\preceq^{(n)})$ is similar to that of $(E_k^{(n)})_{1 \leq k \leq 2^{n}})$, see (\ref{eq:partition}), except that we only focus on non-empty intersections of sets and tuples of the form $(t_1,t_2,\dots,t_n)$ are processed as  indices.

\begin{example}(Example \ref{ex:firstrules-set}, 
cont'ed) \\
We remind that the possibilistic rule-based system of Example \ref{ex:firstrules-set} is given by:
\begin{itemize}
\item $R^1$: ``If $a_1(x) \in \{ 0 \}$ then $b(x) \in Q_1$'', where $Q_1 = \{ (0,0), (0,1)\}$, 
\item $R^2$: ``If $a_1(x) \in \{ 1 \}$ then $b(x) \in Q_2$'', where $Q_2 = \{ (1,0), (1,1)\}$, 
\item $R^3$: ``If $a_2(x) \in \{ 0 \}$  then $b(x) \in Q_3$'', where $Q_3= \{ (0,0), (1,0) \}$,
\item $R^4$: ``If $a_2(x) \in \{ 1 \}$ then $b(x) \in Q_4$'',  where  $Q_4 = \{(0,1), (1,1)\}$.
\end{itemize}

The totally ordered set of indexes associated with the set of possibilistic rules of Example \ref{ex:firstrules-set} is given by: 
\begin{equation}\label{eq:ex:orderedrel1}
    (-1, 2, -3, 4) \prec (1, -2, -3, 4) \prec (-1, 2, 3, -4) \prec (1, -2, 3, -4).
\end{equation}
The elements of this totally ordered set are associated with the following non-empty  subsets of the partition: $\overline{Q_1} \cap Q_2 \cap \overline{Q_3} \cap Q_4$,   $Q_1 \cap \overline{Q_2} \cap \overline{Q_3} \cap Q_4$, $\overline{Q_1} \cap Q_2 \cap Q_3 \cap  \overline{Q_4}$ and $Q_1 \cap \overline{Q_2} \cap Q_3 \cap \overline{Q_4}$. \\
The second set of possibilistic rules (Example \ref{ex:secondset-rules}), chained to the first one, consists of the following rules:
\begin{itemize}
    \item $R^{'1}$: If $b(x) \in \{ (0,0), (1,1) \}$ then $c(x) \in \{ 1\}$, 
    \item $R^{'2}$: If $b(x) \in \{ (0,1), (1,0) \}$ then $c(x) \in \{ 0 \}$.
\end{itemize}
\noindent The totally ordered set of indexes associated with the second set of possibilistic rules presented at Example \ref{ex:secondset-rules} is given by:
\begin{equation}\label{eq:ex:orderedrel2}
(-1,2) \prec (1,-2).
\end{equation}
The elements of this totally ordered set are associated with the following non-empty  subsets of the partition: $\overline{Q'_1} \cap Q'_2$ and $Q'_1 \cap \overline{Q'_2}$.
\end{example}

Finally, we introduce the following definitions which will be useful for giving explicit formulas for the coefficients of the matrices involved in the matrix relation and the equation system.
\begin{definition}\label{def:Lj}
For each $j\in \{1, 2, \dots, n\}$, we define:
\begin{equation}\label{eq:Lj}
\Lambda_j^{n, \top} = \{\mu = (t_1, t_2, \dots, t_n)\in \Lambda^{(n)}\,\mid \, t_j =j\} \text{ and }  \Lambda_j^{n,\bot} =
\{\mu = (t_1, t_2, \dots, t_n)\in \Lambda^{(n)}\,\mid \, t_j =-j\}.
 \end{equation}
\end{definition}
 \noindent
 Clearly, for any $j\in \{1, 2, \dots, n\}$, a partition of the set $\Lambda^{(n)} 
 $ is given by:
 \begin{equation}\label{eq:Lnpart}
 \Lambda^{(n)} = \Lambda_j^{n, \top} \cup\Lambda_j^{n, \bot}.
 \end{equation}

\subsection{A practical procedure for building the matrix relation}
\label{subsec:matrixrelpractical}

For performing inferences from a possibilistic rule-based system, we also take advantage of \cite{baaj2021min}, where a method is presented for constructing the matrices governing the following relation:
\begin{equation}\label{eq:esFP}
    O_n = M_n \Box_{\max}^{\min} I_n.
\end{equation} 
In this relation, the matrix product $\Box_{\max}^{\min}$ uses the function $\max$ as the product and the function $\min$ as the addition, and $n$ is the number of rules of the possibilistic rule-based system. The matrix $M_n$ is of size $(2^n, 2n)$ and the vectors $I_n$  and $O_n$  are of size  $(2n, 1)$ and $(2^n,1)$   respectively. The input vector $I_n$ is constructed using  the possibility degrees of the rule premises $\lambda_1,\rho_1, \lambda_2,\rho_2, \ldots, \lambda_n, \rho_n$, the matrix $M_n$ contains  the rule parameters $s_1, r_1, s_2, r_2, \ldots, s_n, r_n$, and the output vector $O_n$ describes the output possibility distribution  over a collection of \textit{mutually exclusive alternatives,}  i.e., the possible (meta-)concepts. These mutually exclusive alternatives are the non-empty subsets of the partition  $(E_{k}^{(n)})_{1 \leq k \leq 2^{n}}$  of the domain  $D_b$ of the output attribute $b$, see (\ref{eq:partition}). 
 
\noindent The relation (\ref{eq:esFP}) can be used to perform an inference (i.e., computing $O_n$) by jointly applying a set of  
possibilistic rules and fusing the results obtained with each rule (this is what the min-max product achieves). 

\noindent For each $i=1,2,\dots,n$, using \cite{baaj2021min}, the input vector $I_i$ of size $(2i,1)$ is defined thanks to the following inductive definition: 
\begin{equation}
    \label{eq:inputvector} 
    I_1 =  \begin{bmatrix}
\lambda_1 \\
\rho_1
\end{bmatrix}\text{ and  for $i > 1$, } I_i  = \begin{bmatrix}
{I}_{i-1} \\
\hline 
 \lambda_i \\
  \rho_i 
\end{bmatrix}. 
\end{equation}
Similarly, in \cite{baaj2021min},  for each $i=1,2,\dots,n$, the matrix $M_i$ of size $(2^i, 2i)$ is computed thanks to the following inductive definition which involves a block matrix construction:
\begin{subequations}
\label{eq:matrixmatrixrel}
\begin{align}
&\text{\textbullet\,} 
M_1 = \begin{bmatrix}
s_1 & 1 \\
1 & r_1 
\end{bmatrix}&\\ 
 &\text{and for  $i > 1$:}&\nonumber
 \\
&\text{\textbullet\,} M_i = 
\left[\begin{array}{@{}c|c@{}}
  {M}_{i-1}
  & S_i \\
\hline
  {M}_{i-1} &
  R_i
\end{array}\right]\text{, where }S_i = \begin{bmatrix}
s_i & 1\\
s_i & 1\\
\vdots & \vdots \\
s_i & 1 \\
\end{bmatrix}\text{ and }R_i = \begin{bmatrix}
1 & r_i \\
1 & r_i \\
\vdots & \vdots \\
1 & r_i
\end{bmatrix}\text{ are of size }(2^{i-1}, 2).
\end{align}
\end{subequations}

In \cite{baaj2021min}, it is shown that the rows of the matrix $M_n$ and the components of the output vector $O_n = [o_k^{(n)}]_{1\leq k \leq 2^n}$ are explicitly linked to the $2^n$ subsets of  the partition  $(E_{k}^{(n)})_{1 \leq k \leq 2^{n}}$: for $1 \leq k \leq 2^{n}$, the $k$-th row of $M_n$ (resp. the component $o_k^{(n)}$ of $O_n$) is associated with the subset $E_{k}^{(n)}$ of the partition $(E_{k}^{(n)})_{1 \leq k \leq 2^{n}}$. Furthermore, as the matrix relation is based on a min-max composition, each coefficient $o_k^{(n)}$ is equal to the $\min-\max$ matrix product of the $k$-th row of $M_n$ by $I_n$. 

\noindent Once the matrix relation and the output vector are constructed, they can be reduced so that the rows of the reduced matrix (resp. the components of the reduced output vector) correspond to the $\omega^{(n)}$ non-empty subsets of the partition $(E_{k}^{(n)})_{1 \leq k \leq 2^{n}}$. The min-max matrix product of  the input vector by the reduced matrix is a way to perform an \textit{inference} from the possibilistic rule-based system:  the possibility degree of each \textit{non-empty} subset $E_k^{(n)}$ of the partition $(E_{k}^{(n)})_{1 \leq k \leq 2^{n}}$ is equal to the coefficient $o_k^{(n)}$:

\begin{proposition}[\citealp{baaj2021min}]
\label{proposition:PiO}
We have:
\begin{equation}\label{eq:PiO}
\forall k\in J^{(n)},\, \Pi(E_k^{(n)}) = o_k^{(n)}    
\end{equation}   
where the possibility measure $\Pi$ is associated with the output possibility distribution $\pi_{b(x)}^{\ast}$  corresponding to the rules $R^1, R^2, \cdots, R^n$.
\end{proposition}

 As already explained, the main drawback of the method proposed in \cite{baaj2021min} is that  it involves the generation of a matrix $M_n$ with $2^n$ rows, which is then reduced so that its rows correspond to the $\omega^{(n)}$ non-empty subsets of the partition $(E_{k}^{(n)})_{1 \leq k \leq 2^{n}}$.

\noindent To avoid this computationally expensive generation, in the following, we present a method for  constructing  the output vector and the matrix 
directly reduced to the $\omega^{(n)}$ non-empty subsets of the partition 
based on the definition of the set $\Lambda^{(n)}$, see Definition \ref{def:LambdaN}. 
The reduced matrix equation is denoted by:
\begin{equation}\label{eq:redFP}
    \mathcal{O}_n = \mathcal{M}_n \Box_{\max}^{\min} I_n,
\end{equation}
\noindent where the output vector $\mathcal{O}_n$ is of size $(\omega^{(n)}, 1)$ and the matrix $\mathcal{M}_n$ is of size $(\omega^{(n)}, 2n)$. 
\noindent To introduce the method, a mapping $\tau$ which takes as input a $n$-tuple $\mu = (t_1,t_2,\cdots, t_n)$ of $\Lambda^{(n)}$ and returns a row matrix $H_\mu$ of size $(1,2n)$ is used:
\begin{equation}\label{eq:taumapping}
\begin{tabular}{ll}
$\tau :$ & $\Lambda^{(n)} \rightarrow [0,1]^{1 \times 2n}$\\
& $\mu = (t_1,t_2,\dots,t_n) \mapsto H_\mu = \begin{bmatrix}
        L_{t_1} & L_{t_2} & \cdots & L_{t_n} 
    \end{bmatrix}$ 
\end{tabular}
\end{equation}
\noindent where $H_\mu$ is constructed by block matrix construction such that each
\begin{equation}\label{eq:taumapping1}
    L_{t_i} = 
\begin{cases}\begin{bmatrix}s_i & 1 \end{bmatrix} & \text{ if } t_i = i \\[8pt] \begin{bmatrix}1 & r_i \end{bmatrix} & \text{ if } t_i = -i \end{cases}
\end{equation}
is a row matrix of size $(1,2)$ which uses a parameter of the rule $R^i$ (either $s_i$ or $r_i$). Using   (\ref{eq:kmu10}) and  (\ref{eq:matrixmatrixrel}), we prove:

 \begin{restatable}{lemma}{Hmucmd}\label{lemma:hmutokmu}
The row matrix $H_\mu$ is equal to the $k_{\mu}$-th row of the initial matrix $M_n$. 
 \end{restatable}
\noindent See Subsection \ref{subsec:proof:Hmu} for the proof.

\subsubsection{\texorpdfstring{Constructing the reduced matrix relation $\mathcal{O}_n = \mathcal{M}_n \Box_{\max}^{\min} I_n$}{Constructing the reduced matrix relation}}
Relying on  $\Lambda^{(n)} = \{\mu_1, \mu_2, \dots, \mu_{\omega^{(n)}}\}$,   such that for all $i < j$ we have $\mu_i \prec^{(n)} \mu_j$, see (\ref{eq:LJ}), Lemma \ref{lemma:hmutokmu} and  Proposition \ref{proposition:PiO}, we use the mapping $\sigma^{(n)}$, see (\ref{eq:sigman}), to construct the reduced output vector ${\cal O}_n$. We also use the mapping $\tau$, see (\ref{eq:taumapping}), to construct the reduced matrix ${\cal M}_n$:
\begin{equation}\label{eq:reducedonmnconstructed}
\mathcal{O}_n = \begin{bmatrix}
    \Pi(\sigma^{(n)}(\mu_1))\\
    \Pi(\sigma^{(n)}(\mu_2))\\
    \vdots\\
    \Pi(\sigma^{(n)}(\mu_{\omega^{(n)}}))\\
\end{bmatrix} \quad \text{ and } \quad 
    \mathcal{M}_n = \begin{bmatrix}
        \tau(\mu_1) = H_{\mu_1}\\ 
        \tau(\mu_2)= H_{\mu_2} \\
        \vdots \\
        \tau(\mu_{\omega^{(n)}}) = H_{\mu_{\omega^{(n)}}} 
    \end{bmatrix}.
\end{equation}
 
\noindent Using this construction,  for each $l\in \{1, 2, \dots, \omega^{(n)}\}$  the $l$-th row  of the matrix ${\cal M}_n$ is  $H_{\mu_l}$.

\noindent With the $n$ partitions $\Lambda^{(n)} =\Lambda_j^{n, \top} \cup\Lambda_j^{n, \bot}$ with $j \in \{1,2,\cdots,n\}$,  see Definition \ref{def:Lj} and (\ref{eq:Lnpart}),  each coefficient of a row of index $\mu\in \Lambda^{(n)}$ of the matrix ${\cal M}_n$ can be computed as follows: 
\begin{restatable}{lemma}{lemcoMcmd}\label{lemma:coM}
 Let ${\cal M}_n= [\dot{m}_{\mu, l}]_{\mu\in \Lambda^{(n)}, 1\leq l \leq 2n}$.   Then for any $\mu\in \Lambda^{(n)}$ and $j\in\{1, 2, \dots, n\}$, the following formulas can be used for computing the coefficients of ${\cal M}_n$:
 \begin{equation}\label{eq:coefredM}
 \dot{m}_{\mu, 2j-1}  =
   \begin{cases}
 s_j     & \text{ if }  \mu\in \Lambda_j^{n, \top} \\
  1 & \text{ if }  \mu\in \Lambda_j^{n, \bot}  
  \end{cases}
, \quad 
   \dot{m}_{\mu, 2j}  =
   \begin{cases}
 1     & \text{ if }  \mu\in \Lambda_j^{n, \top} \\
  r_j & \text{ if }  \mu\in \Lambda_j^{n, \bot}  
   \end{cases}.
 \end{equation}
\end{restatable}
\noindent See Subsection \ref{subsec:proof:coM} for the proof.

\begin{example}\label{ex:matrixrelred}(Example \ref{ex:firstrules-set}, 
cont'ed) \\
The totally ordered set associated with the set of  
possibilistic rules reported in Example \ref{ex:firstrules-set} is given by (\ref{eq:ex:orderedrel1}).

\noindent The matrix relation $\mathcal{O}_n = \mathcal{M}_n \Box_{\max}^{\min} I_n$ reduced the non-empty subsets of the partition is:
\begin{equation}\label{eq:ex:matrixrelation:firstsys}
    \begin{bmatrix}
    \Pi( \overline{Q_1} \cap Q_2 \cap \overline{Q_3} \cap Q_4 ) \\
     \Pi(  Q_1 \cap \overline{Q_2} \cap \overline{Q_3} \cap Q_4 ) \\
      \Pi(\overline{Q_1} \cap Q_2 \cap Q_3 \cap  \overline{Q_4}) \\
       \Pi( Q_1 \cap \overline{Q_2} \cap Q_3 \cap \overline{Q_4}) 
\end{bmatrix}= 
\begin{bmatrix}
    \Pi( \{ (1,1) \}) \\
     \Pi( \{ (0,1) \}) \\
      \Pi( \{ (1,0) \}) \\
       \Pi( \{ (0,0) \}) 
\end{bmatrix}
    =
    \begin{bmatrix}
        1 & r_1 & s_2 & 1 & 1 & r_3 & s_4 & 1\\
        s_1 & 1 & 1 & r_2 & 1 & r_3 & s_4 & 1\\
        1 & r_1 & s_2 & 1 & s_3 & 1 & 1 & r_4\\
        s_1 & 1 & 1 & r_2 & s_3 & 1 & 1 & r_4
    \end{bmatrix} \Box_{\max}^{\min} \begin{bmatrix}
    \lambda_1\\
    \rho_1 \\
    \lambda_2 \\
    \rho_2 \\
     \lambda_3 \\
    \rho_3\\
    \lambda_4\\
    \rho_4
\end{bmatrix}.
\end{equation}

\noindent The totally ordered set associated with the second set of 
possibilistic rules reported in Example \ref{ex:secondset-rules}  is given by (\ref{eq:ex:orderedrel2}). The corresponding reduced matrix relation is:
\begin{equation}\label{eq:ex:matrixrelation:secondsys}
    \begin{bmatrix}
    \Pi( \overline{Q'_1} \cap Q'_2 )\\
     \Pi(  Q'_1 \cap \overline{Q'_2} ) 
\end{bmatrix}= 
\begin{bmatrix}
    \Pi( \{ 0 \}) \\
     \Pi( \{ 1 \}) 
\end{bmatrix}
    =
    \begin{bmatrix}
        1 & r'_1 & s'_2 & 1 \\
        s'_1 & 1 & 1 & r'_2 
    \end{bmatrix} \Box_{\max}^{\min} \begin{bmatrix}
    \lambda'_1\\
    \rho'_1 \\
    \lambda'_2 \\
    \rho'_2 
\end{bmatrix}.
\end{equation}

\end{example}

In the case of a cascade (Subsection \ref{subsub:cascade}), an input-output relation between the two matrix relations associated with each set
of 
possibilistic rules can be established, see  \cite{baaj2021min}.

\subsection{\texorpdfstring{A practical method for building the reduced equation system $(\Upsigma_n)$}{Practical method for building the reduced equation system (Sigma\_n)}}
\label{subsec:practical-building-learning}

One of the challenges arising when using a possibilistic rule-based system is the determination of the values of the parameters $s_i,r_i$ of the possibilistic rules for a specific use case. To address this problem, \cite{baaj2022learning} proposes a learning method which is based on the solving of a min-max equation system  constructed from a training data sample. The method for building this equation system is \textit{closely related} to the one used for building the matrix relation $O_n = M_n \Box_{\max}^{\min} I_n$, see (\ref{eq:esFP}). The motivations behind such a learning approach are highlighted in the next section, where a possibilistic learning method that extends the work of \cite{baaj2022learning} is introduced. This new method can be used to process \emph{data sets, i.e., sets of training data samples}.   In this subsection, we start by recalling the construction of the equation system presented in \cite{baaj2022learning}, which is based on the partition $(E_k^{(n)})_{1 \leq k \leq 2^n}$. The $\min-\max$ equation system of \cite{baaj2022learning} is: 

\begin{equation}\label{eq:sigma} 
	(\Sigma_n): Y_n = \Gamma_n  \Box_{\max}^{\min} X,
\end{equation} 
\noindent where:
\begin{equation} \label{eq:y-and-x} 
Y_n = \begin{bmatrix} y_1\\y_2\\
	\vdots\\
	
	y_{2^n}\end{bmatrix} 
 \text{ and } X = \begin{bmatrix}s_1\\
	r_1\\
	s_2 \\
	r_2\\
	\vdots\\
	s_n\\
	r_n
\end{bmatrix}, \end{equation} 
The second member $Y_n$ is of size $(2^n,1)$. Its components which correspond to the non-empty subsets of the partition $(E_k^{(n)})_{1 \leq k \leq 2^n}$, see (\ref{eq:partition}), describes an output possibility distribution over a collection of mutually exclusive alternatives 
  i.e., the possible (meta-)concepts. The  components of the unknown vector $X$ of size $(2n,1)$ are the rule parameters $s_1,r_1,s_2,r_2,\cdots,s_n,r_n$. 
The matrix $\Gamma_n$ of the equation system $(\Sigma_n)$ is of size $(2^n,2n)$ and is constructed by induction  using the possibility degrees  $\lambda_1,\rho_1,\lambda_2,\rho_2,\cdots, \lambda_n, \rho_n$ of the rule premises.

\cite{baaj2022learning} shows that for each $i=1,2,\dots,n$, the matrix $\Gamma_i$ of size $(2^i, 2i)$ can be generated using the following inductive definition which involves a block matrix construction:
\begin{subequations}
\label{eq:matrixeqlearn}
\begin{align}
&\text{\textbullet\,} 
\Gamma_1 = \begin{bmatrix}
\lambda_1 & 1 \\
1 & \rho_1 
\end{bmatrix}&\\ 
 &\text{and for  $i > 1$: }&\nonumber
 \\
&\text{\textbullet\,} 
\Gamma_i = 
\left[\begin{array}{@{}c|c@{}}
  {\Gamma}_{i-1}
  & L_i \\
\hline
  {\Gamma}_{i-1} &
  \Rho_i
\end{array}\right],\text{where } L_i = \begin{bmatrix}
\lambda_i & 1\\
\lambda_i & 1\\
\vdots & \vdots \\
\lambda_i & 1 \\
\end{bmatrix}\text{ and }\Rho_i = \begin{bmatrix}
1 & \rho_i \\
1 & \rho_i \\
\vdots & \vdots \\
1 & \rho_i
\end{bmatrix}\text{ are of size }(2^{i-1}, 2). 
\end{align}
\end{subequations}
We can observe that:
\begin{lemma}\label{lemma:gammnsubstitue}
    The matrix $\Gamma_n = M_n(\lambda_i, \rho_i)$  is obtained by substituting $s_i$ by $\lambda_i$ and $r_i$ by $\rho_i$  for $i=1,2,\dots,n$ in the matrix $M_n(s_i, r_i)$ of the matrix relation (\ref{eq:esFP}).
\end{lemma}
\qed

Thus,   the rows of the matrix $\Gamma_n$   are  linked to the $2^n$ subsets of  the partition  $(E_{k}^{(n)})_{1 \leq k \leq 2^{n}}$: for $1 \leq k \leq 2^{n}$, the $k$-th row of $\Gamma_n$ is associated with the subset $E_{k}^{(n)}$ of the partition $(E_{k}^{(n)})_{1 \leq k \leq 2^{n}}$. The drawback of the method of \cite{baaj2022learning} for building the equation system $(\Sigma_n)$ in practice is the same one as for the method used to build the matrix relation: 
it involves generating a matrix $\Gamma_n$ with $2^n$ rows and then to reduce it so that the rows of the matrix correspond to the $\omega^{(n)}$ non-empty subsets of the partition $(E_{k}^{(n)})_{1 \leq k \leq 2^{n}}$.

\noindent

In the following, we present a new  method for  constructing  
the output vector and the matrix of the equation system directly reduced to the $\omega^{(n)}$ non-empty subsets of the partition $(E_{k}^{(n)})_{1 \leq k \leq 2^{n}}$. We proceed as for the matrix relation, i.e., we consider an inductive definition of the ordered set     $(\Lambda^{(n)}, \preceq^{(n)})$, see Subsection \ref{subsec:partitionred}, and we take advantage of Proposition \ref{prop:Des}.

\noindent  The reduced equation system is denoted by:
\begin{equation}\label{eq:upsigma}
    (\Upsigma_n): \dot{Y}_n = \dot{\Gamma_n} \Box_{\max}^{\min} X, 
\end{equation}
 where the vector $\dot{Y}_n$ is of size $(\omega^{(n)}, 1)$ and the matrix $\dot{\Gamma_n}$ is of size $(\omega^{(n)}, 2n)$.  We first need the following mapping $\kappa$ which takes as input a $n$-tuple $\mu = (t_1,t_2,\dots,t_n)$ of $\Lambda^{(n)}$ and returns a row matrix $K_\mu$ of size $(1,2n)$, defined as follows:
\begin{equation}\label{eq:mappingkappa}
\begin{tabular}{ll}
$\kappa :$ & $\Lambda^{(n)} \rightarrow [0,1]^{1 \times 2n}$\\
& $\mu = (t_1,t_2,\dots,t_n) \mapsto K_\mu = \begin{bmatrix}
        N_{t_1} & N_{t_2} & \cdots & N_{t_n} 
    \end{bmatrix}$ 
\end{tabular}
\end{equation}
\noindent where $K_\mu$ is constructed by block matrix construction such that each $$N_{t_i} = \begin{cases}\begin{bmatrix}\lambda_i & 1 \end{bmatrix} & \text{ if } t_i = i \\[8pt] \begin{bmatrix}1 & \rho_i \end{bmatrix} & \text{ if } t_i = -i \end{cases}$$ is a row matrix of size $(1,2)$ which uses the possibility degree of the premise of the rule $R^i$ denoted $\lambda_i$ or  the possibility degree of the negation of the premise, i.e., $\rho_i$. Using Lemma \ref{lemma:gammnsubstitue}, and by applying Lemma \ref{lemma:hmutokmu}, we immediately get:
\begin{lemma}\label{HmuG}
The row matrix $K_\mu$ is equal to the $k_{\mu}$-th row of the initial matrix $\Gamma_n$.
 \end{lemma}
\qed

\subsubsection{\texorpdfstring{Constructing the reduced equation system $(\Upsigma_n): \dot{Y}_n = \dot{\Gamma_n} \Box_{\max}^{\min} X$}{Constructing the reduced equation system}}
\noindent
As for the construction of the matrix relation, we rely on  $\Lambda^{(n)} = \{\mu_1, \mu_2, \dots, \mu_{\omega^{(n)}}\}$,   such that for all $i < j$ we have $\mu_i \prec^{(n)} \mu_j$, see (\ref{eq:LJ}). We use the mapping $\sigma^{(n)}$, see (\ref{eq:sigman}), to construct the reduced vector $\dot{Y}_n$ and the mapping $\kappa$, see (\ref{eq:mappingkappa}),  to construct the reduced matrix $\dot{\Gamma_n}$:
\begin{equation}\label{eq:buildupsigman}
\dot{Y}_n = \begin{bmatrix}
    \Pi(\sigma^{(n)}(\mu_1))\\
    \Pi(\sigma^{(n)}(\mu_2))\\
    \vdots\\
    \Pi(\sigma^{(n)}(\mu_{\omega^{(n)}}))\\
\end{bmatrix} \quad \text{ and } \quad 
    \dot{\Gamma_n} = \begin{bmatrix}
        \kappa(\mu_1) = K_{\mu_1}\\ 
        \kappa(\mu_2)= K_{\mu_2} \\
        \vdots \\
        \kappa(\mu_{\omega^{(n)}}) = K_{\mu_{\omega^{(n)}}} 
    \end{bmatrix}.
\end{equation}
 
\noindent \noindent Here, $\Pi$ is the possibility measure associated with a possibility distribution, described in $\dot{Y}_n$ by the possibility degrees of the non-empty subsets of the partition $(E_{k}^{(n)})_{1 \leq k \leq 2^{n}}$.\\
With this construction,  for each $l\in \{1, 2, \dots, \omega^{(n)}\}$  the $l$-th row  of the matrix ${\dot \Gamma }_n$ is  $K_{\mu_l}$.  
The vector  $\dot{Y}_n$  is analogous to that of the reduced output vector in the reduced matrix relation, as described in (\ref{eq:reducedonmnconstructed}).

\noindent Using the $n$ partitions of  $\Lambda^{(n)} =\Lambda_j^{n, \top} \cup\Lambda_j^{n, \bot}$ with $j \in \{1,2,\cdots,n\}$, see Definition \ref{def:Lj} and (\ref{eq:Lnpart}),  each coefficient of a row of index $\mu\in \Lambda^{(n)}$ of the matrix $\dot{\Gamma}_n$ can be computed as follows: 
\begin{restatable}{lemma}{lemcoGcmd}\label{lem:coG}
 Let $\dot{\Gamma}_n= [\dot{\gamma}_{\mu, l}]_{\mu\in \Lambda^{(n)}, 1\leq l \leq 2n}$.   Then for any $\mu\in \Lambda^{(n)}$ and $j\in\{1, 2, \dots, n\}$, the following formulas can be used for computing the coefficients of $\dot{\Gamma}_n$:
 \begin{equation}\label{eq:coefredG}
 \dot{\gamma}_{\mu, 2j-1}  =
   \begin{cases}
 \lambda_j     & \text{ if }  \mu\in \Lambda_j^{n, \top} \\
  1 & \text{ if }  \mu\in \Lambda_j^{n, \bot}  
  \end{cases}
, \quad 
   \dot{\gamma}_{\mu, 2j}  =
   \begin{cases}
 1     & \text{ if }  \mu\in \Lambda_j^{n, \top} \\
  \rho_j & \text{ if }  \mu\in \Lambda_j^{n, \bot}  
   \end{cases}.
 \end{equation}
\end{restatable}
\noindent The proof is easy from Lemma \ref{lemma:gammnsubstitue}  and Lemma \ref{lemma:coM}.
\qed

\begin{example}\label{ex:eqsyslearn} (Example \ref{ex:firstrules-set}, 
cont'ed) \\
    From the first set of possibilistic rules considered (see Example \ref{ex:firstrules-set}), the following equation system is produced using (\ref{eq:ex:orderedrel1}):
    \begin{equation}\label{eq:ex:learning-system1}
\begin{bmatrix}
    \Pi( \overline{Q_1} \cap Q_2 \cap \overline{Q_3} \cap Q_4 ) \\
     \Pi(  Q_1 \cap \overline{Q_2} \cap \overline{Q_3} \cap Q_4 ) \\
      \Pi(\overline{Q_1} \cap Q_2 \cap Q_3 \cap  \overline{Q_4}) \\
       \Pi( Q_1 \cap \overline{Q_2} \cap Q_3 \cap \overline{Q_4}) 
\end{bmatrix}= 
\begin{bmatrix}
    \Pi( \{ (1,1) \}) \\
     \Pi( \{ (0,1) \}) \\
      \Pi( \{ (1,0) \}) \\
       \Pi( \{ (0,0) \}) 
\end{bmatrix}
    =
    \begin{bmatrix}
        1 & \rho_1 & \lambda_2 & 1 & 1 & \rho_3 & \lambda_4 & 1\\
        \lambda_1 & 1 & 1 & \rho_2 & 1 & \rho_3 & \lambda_4 & 1\\
        1 & \rho_1 & \lambda_2 & 1 & \lambda_3 & 1 & 1 & \rho_4\\
        \lambda_1 & 1 & 1 & \rho_2 & \lambda_3 & 1 & 1 & \rho_4
    \end{bmatrix} \Box_{\max}^{\min} X
\end{equation}
\noindent where $X = \begin{bmatrix}
    s_1 \\
    r_1 \\
    s_2 \\ 
    r_2 \\
    s_3 \\
    r_3 \\
    s_4 \\
    r_4
\end{bmatrix}$.\\
From the second set of 
possibilistic rules (see Example \ref{ex:secondset-rules}), the following reduced equation system is produced  using (\ref{eq:ex:orderedrel2}):
\begin{equation}\label{eq:ex:learning-system2}
\begin{bmatrix}
    \Pi( \overline{Q'_1} \cap Q'_2 )\\
     \Pi(  Q'_1 \cap \overline{Q'_2} ) 
\end{bmatrix}= 
\begin{bmatrix}
    \Pi( \{ 0 \}) \\
     \Pi( \{ 1 \}) 
\end{bmatrix}
    =
    \begin{bmatrix}
        1 & \rho'_1 & \lambda'_2 & 1 \\
        \lambda'_1 & 1 & 1 & \rho'_2 
    \end{bmatrix} \Box_{\max}^{\min} X' \text{ where }X' = \begin{bmatrix}
    s'_1 \\
    r'_1 \\
    s'_2 \\ 
    r'_2 
\end{bmatrix}.
\end{equation}

\end{example}

From a practical perspective, constructing the reduced equation system $(\Upsigma_n)$  from a training data sample enables us to leverage the learning method of \cite{baaj2022learning}. Previously, 
\cite{baaj2022learning} showed that solving the non-reduced equation system $(\Sigma_n)$ constructed with a training data sample using Sanchez's result \cite{sanchez1976resolution} yields solutions for the rule parameters of a possibilistic rule-based system that are compatible with the training data sample. 
In the next section,  this approach is presented in more details but based on the reduced equation system $(\Upsigma_n)$, see Subsection \ref{subsec:learnruleparam}, and the reduced equation system is also related to the reduced matrix relation, see (\ref{eq:releqsyscompatible}).

\subsubsection{\texorpdfstring{Constructing the equation system $(\Upsigma_n): \dot{Y}_n = \dot{\Gamma_n} \Box_{\max}^{\min} X$ from a training data sample}{Constructing the reduced equation system from a training data sample}}
\label{subsubsec:constructEStrainingdata}
\noindent Let us consider:
\begin{itemize}
    \item a possibilistic rule-based system that uses the input attributes $a_1, a_2, \ldots, a_k$ and an output attribute $b$. Using Lemma \ref{lemma:LambdaN}  and Definition \ref{def:orderrel}, we know how to construct the totally ordered set $(\Lambda^{(n)}, \preceq^{(n)})$ associated with the set of rules.
    \item a training data sample, which consists of a set of input possibility distributions $\widetilde{\pi}_{a_1(x)}, \widetilde{\pi}_{a_2(x)}, \ldots, \widetilde{\pi}_{a_k(x)}$, associated with the input attributes $a_1, a_2, \ldots, a_k$, and a targeted  output possibility distribution $\widetilde{\pi}_{b(x)}$ associated with the output attribute $b$.
\end{itemize}

\noindent The equation system $(\Upsigma_n): \dot{Y}_n = \dot{\Gamma_n} \Box_{\max}^{\min} X$ can be constructed from the training data sample by following the two steps: 
\begin{enumerate}
    \item Using (\ref{eq:buildupsigman}), first build the matrix $\dot{\Gamma_n}$ with the possibility degrees $\lambda_i = \pi(p_i), \rho_i = \pi(\neg p_i)$ of the rule premises computed using the formulas in (\ref{eq:pinegpiformulae}) and the input possibility distributions $\widetilde{\pi}_{a_1(x)}, \widetilde{\pi}_{a_2(x)}, \ldots, \widetilde{\pi}_{a_k(x)}$.
    \item Then construct the second member $\dot{Y}_n$ using  $(\Lambda^{(n)}, \preceq^{(n)})$, the mapping $\sigma^{(n)}$ (see (\ref{eq:sigman})), and the possibility measure $\widetilde{\Pi}$ associated with $\widetilde{\pi}_{b(x)}$ as in (\ref{eq:buildupsigman}): \begin{equation*}
\dot{Y}_n = \begin{bmatrix}
    \widetilde{\Pi}(\sigma^{(n)}(\mu_1))\\
    \widetilde{\Pi}(\sigma^{(n)}(\mu_2))\\
    \vdots\\
    \widetilde{\Pi}(\sigma^{(n)}(\mu_{\omega^{(n)}}))\\
\end{bmatrix}.
\end{equation*}

\end{enumerate}

\noindent Let us illustrate such a construction:
\begin{example}\label{ex:conssystg}
{(Example \ref{ex:firstrules-set}, 
cont'ed)} \\
We consider a training data sample associated with the first set of rules (see Example \ref{ex:firstrules-set}), which consists of input possibility distributions for the attributes $a_1$ and $a_2$ and a targeted output possibility distribution for the attribute $b$, all provided by an expert. These distributions reflect the expert’s uncertainty regarding the classification of the sample; in the absence of uncertainty, the possibility distribution associated with each attribute assigns a possibility degree of one to a single element of the attribute domain and of zero to all others.
By inferring the rule premises of the first set of rules using the input possibility distributions, we obtain: $\lambda_1 = 0.12, \rho_1 =1,  \lambda_2= 0.18,  \rho_2= 1, \lambda_3= 1,  \rho_3= 0.43, \lambda_4=  1, \rho_4 = 0.66$. The targeted output possibility distribution is as follows:  $ \Pi( \{ (1,1) \}) =0.2,
     \Pi( \{ (0,1) \}) = 0.87,
      \Pi( \{ (1,0) \}) = 0.2$, and $
       \Pi( \{ (0,0) \}) = 1$. \\

\noindent Then, from those values, the reduced equation system can be constructed using (\ref{eq:ex:learning-system1}):
\begin{equation}\label{eq:eqtd1}
    \begin{bmatrix}
    \Pi( \{ (1,1) \}) \\
     \Pi( \{ (0,1) \}) \\
      \Pi( \{ (1,0) \}) \\
       \Pi( \{ (0,0) \}) 
\end{bmatrix}=\begin{bmatrix}
    0.2\\
    0.87\\
    0.2\\
    1
\end{bmatrix}= \begin{bmatrix}
    1 &  1&   0.18& 1&   1&   0.43& 1&   1\\
    0.12& 1&   1&   1&   1&   0.43& 1&   1\\
    1&  1&   0.18& 1&   1&   1&   1&   0.66\\
    0.12& 1&   1&   1&   1&   1&   1&   0.66
\end{bmatrix}
\Box_{\max}^{\min} X.
\end{equation}

\noindent Similarly, the second reduced equation system can be constructed, see (\ref{eq:ex:learning-system2}), with  the possibility degrees of the rule premises $\lambda'_1 = 0.25, \rho'_1= 1,  \lambda'_2= 0.35, \rho'_2= 1$ (which can be derived using input possibility distributions), and the following targeted output possibility distribution: $\Pi( \{ 0 \})= 0.4,
     \Pi( \{ 1 \}) = 1$.  

\noindent The second reduced equation system is:
\begin{equation}\label{eq:eqtd2}
    \begin{bmatrix}
    \Pi( \{ 0 \}) \\
     \Pi( \{ 1 \}) 
\end{bmatrix}=\begin{bmatrix}
    0.4\\
    1
\end{bmatrix} = \begin{bmatrix}
    1 & 1 & 0.35 & 1 \\
    0.25  & 1 &  1 & 1   
   
\end{bmatrix} \Box_{\max}^{\min} X'.
\end{equation}
\end{example}

\section{Learning rule parameters of a possibilistic rule-based system according to training data}
\label{sec:learning}

In this section, the focus is on  \textit{possibilistic learning}, which is involved in the high-level reasoning task of $\Pi$-NeSy. The goal is to determine how to learn the values of the rule parameters of a possibilistic rule-based system according to \emph{a set of training data samples}, where each training data sample is given by the possibility degrees of the rule premises inferred from input possibility distributions and a targeted  output possibility distribution. In the context of $\Pi$-NeSy, for joint learning, the input possibility distributions of the training data samples result from probability-possibility transformations of the output probability distributions generated by the neural network. 

In AI, model parameters are often learned from training data by minimizing 
a loss function 
using gradient descent methods. However, in our case,  such an approach is hampered by the non-differentiability of the error function to be used, given that the possibilistic framework is based on combinations of $\min$ and $\max$ functions.  \cite{baaj2022learning} showed that this learning problem can be tackled by solving the min-max equation system $(\Sigma_n)$, see (\ref{eq:sigma}), which is constructed from a  training data sample. By solving the min-max equation system $(\Sigma_n)$  using Sanchez's result  \cite{sanchez1976resolution},  \cite{baaj2022learning}  shows that  a solution of this system can be used to assign values to the rule parameters that are compatible with the training data sample in the sense of (\ref{eq:releqsyscompatible}). As to learning, the results presented in \cite{BAAJ2024} are also useful since they concern the handling of \emph{inconsistent} min-max equation systems. Indeed, inconsistency may happen when dealing with the equation system $(\Sigma_n)$, for instance when the data used for constructing $(\Sigma_n)$ contain noise or outliers. In particular, the following research questions have been considered in \cite{BAAJ2024}: how to find a consistent equation system(s) that is as close as possible to the considered inconsistent equation system by perturbing the second member of the inconsistent system as little as possible, and how to obtain approximate solutions of the considered inconsistent equation system (which are solutions of one of the consistent equation systems that is as close as possible to the considered inconsistent equation system)?

In this section,  we start by studying how to solve the reduced equation system $(\Upsigma_n)$, see (\ref{eq:upsigma}), since the results of \cite{baaj2022learning} were given for the equation system $(\Sigma_n)$ not reduced to the non-empty subsets of the partition (Subsection \ref{subsec:learnruleparam}). When $(\Upsigma_n)$ is consistent, two solutions of this  equation system are computed, and one of them is the lowest solution of it (see Lemma \ref{lem:coED} and Proposition \ref{prop:coEDU}). The rest of the section focuses on the introduction of a possibilistic learning method based on the work of \cite{baaj2022learning} and \cite{BAAJ2024}, that is efficient enough to be used in practice. In the following, the steps involved in this method are outlined. 

Consider that we want to learn the values of the rule parameters of a set of possibilistic rules according to a set of training data samples. The method first processes each training data sample one by one. 
For a given training data sample, the associated min-max equation system $(\Upsigma_n)$ is considered, see (\ref{eq:upsigma}).  \cite{BAAJ2024} is then used to assess the quality of the training data sample with respect to the set of rules (Subsection \ref{subsec:poorqualtrainingdata}). Indeed, using the min-max equation system $(\Upsigma_n)$ constructed from the training data sample, we can assess whether the training data sample is of good quality or contains noise or outliers. The evaluation consists in measuring the \emph{Chebyschev distance} (based on $L_\infty$ norm) between the min-max equation system $(\Upsigma_n)$ and the closest consistent equation system(s) (this distance is equal to zero if the equation system $(\Upsigma_n)$ is consistent), see Subsection \ref{subsec:handlinginconsisteq}. Then, for an inconsistent equation system $(\Upsigma_n)$, the results of \cite{BAAJ2024} are leveraged to make minimal modifications to the training data sample used for constructing it in order to obtain a consistent equation system (Definition \ref{def:reduced-es} of Subsection \ref{subsec:apptotrainingdatasample}). Once this process has been achieved, each training data sample is associated with a consistent equation system. Since these consistent equation systems share the same unknown vector whose components are the rule parameters, we can construct an equation system denoted by $(\mathbf{\Sigma})$, that stacks all these equation systems into one, see (\ref{eq:stackedES}) in Subsection \ref{subsubsec:learningmultipledata}. We then use \cite{BAAJ2024} to compute an approximate solution of the resulting stacked equation system (a solution of one of its closest consistent equation systems). The obtained approximate solution represents the entire training data set and it can be used to assign values to the rule parameters of the set of rules with respect to the data. \\ Finally, we introduce a method named ``Possibilistic cascade learning'', see  Method \ref{meth:learningincascade}, which can be exploited to extend the learning method presented in Subsection \ref{subsubsec:learningmultipledata} to a cascade (Subsection \ref{subsub:cascade}), i.e., when a possibilistic rule-based system uses two chained sets of possibilistic rules. 
We also explain how to determine values for the hyperparameters (tolerance thresholds, see Definition \ref{def:reliabletau}) which can be used to apply Method \ref{meth:learningincascade} to very large training datasets.

Method \ref{meth:learningincascade} is a key component of  our neuro-symbolic approach $\Pi$-NeSy. It is used both for joint learning (combining neural learning and possibilistic learning) and for learning the values of the parameters of the rules of a cascade. 

\smallskip
The method introduces two novel features that set it apart from current neuro-symbolic approaches: \emph{assessing the quality of a  training data sample with respect to a set of rules}, and \emph{making minimal modifications to poor-quality training data to precisely approximate the values of the rule parameters}.

\subsection{Learning rule parameters according to a training data sample}
\label{subsec:learnruleparam}
The learning method of \cite{baaj2022learning} for determining  values for the rule parameters of a set of 
possibilistic rules of a possibilistic rule-based system according to a training data sample is based on the non-reduced equation system $(\Sigma_n)$, see (\ref{eq:sigma}), using methods for solving $\min-\max$ equation systems introduced by \cite{sanchez1976resolution}.  In the following,  the solving of the reduced equation system $(\Upsigma_n)$, see (\ref{eq:upsigma}) is investigated. We show how to compute two specific solutions of this system when it is consistent.

We remind that the coefficients of the matrix $\dot{\Gamma}_n$  of the reduced equation system  $(\Upsigma_n): \dot{Y}_n = \dot{\Gamma_n} \Box_{\max}^{\min} X$, see (\ref{eq:upsigma}),  are given by the   formulas  (\ref{eq:coefredG})  derived from the  $n$ partitions 
$\Lambda^{(n)} = \Lambda_j^{n, \top} \cup \Lambda_j^{n, \bot}$ with $j \in \{1,2,\cdots,n\}$, see Definition \ref{def:Lj} and (\ref{eq:Lnpart}).

\noindent Two column-vectors are introduced:
\begin{equation}\label{eq:twosoldef}
    E^{n,\downarrow} := \dot{\Gamma}^t_n  \Box_\epsilon^{\max } \dot{Y}_n =   \begin{bmatrix} e^{n,\downarrow}_1 \\  e^{n,\downarrow}_2 \\
	\vdots\\

	e^{n,\downarrow}_{2n}  \end{bmatrix} \quad \text{ and } \quad 
        E^{n,\uparrow} =     \begin{bmatrix} e_{1}^{n,\uparrow} \\  e_{2}^{n,\uparrow}\\
	\vdots\\

	e_{2n}^{n,\uparrow}  \end{bmatrix}
\end{equation}
\noindent where the matrix $\dot{\Gamma}^t_n$ is the transpose of the matrix $\dot{\Gamma}_n$ and the matrix product $\Box_\epsilon^{\max}$ uses $\max$ as addition and $\epsilon-$product as product.  $\epsilon-$product is defined  by: 
\begin{equation}\label{eq:epsilonproduct}
    x \epsilon  y =  \begin{cases}
y & \text{if }   x < y\\
0 & \text{if }   x \geq  y \end{cases} \text{ in } [0, 1].
\end{equation} The coefficients of the two vectors $E^{n,\downarrow}$ and $E^{n,\uparrow}$  are computed using the possibility degrees of the rule premises and  the coefficients of  the   vector   $\dot{Y}_n =[y_\mu]_{\mu\in\Lambda^{(n)}}$ using the sets $\Lambda_j^{n, \top}$ and $\Lambda_j^{n, \bot}$ with $j \in \{1,2,\dots,n\}$:

\begin{definition}\label{def:coEU}
The coefficients of  the vector   $E^{n,\uparrow}$  are  defined by:
\begin{equation}\label{eq:formulasejtop}
  \text{for } j=1,2,\dots,n, 
  \quad e^{n,\uparrow}_{2j - 1}: = \max_{\mu\in \Lambda_j^{n, \top}} y_\mu \quad 
\text{ and }\quad 
    e^{n,\uparrow}_{2j} := \max_{\mu\in \Lambda_j^{n, \bot}} y_\mu.
\end{equation}
 If $\Lambda_j^{n, \top} = \emptyset$ (resp. $\Lambda_j^{n, \bot} = \emptyset$), we have $\max_{\mu\in \Lambda_j^{n, \top}} y_\mu = 0$ (resp. $\max_{\mu\in \Lambda_j^{n, \bot}} y_\mu= 0$).
    
\end{definition}
\begin{restatable}{lemma}{lemcoEDcmd}\label{lem:coED}
The coefficients of  the vector $E^{n,\downarrow}:= \dot{\Gamma}^t_n  \Box_\epsilon^{\max } \dot{Y}_n$ are  given by: \\
For $j=1,2,\dots,n$: 
\begin{equation}\label{eq:eninitialformula}
    e^{n,\downarrow}_{2j-1} = \lambda_j\, \epsilon \max_{\mu\in \Lambda_j^{n, \top}} y_\mu  = \lambda_j \,\epsilon \,e^{n,\uparrow}_{2j - 1} \text{ and } e^{n,\downarrow}_{2j} = \rho_j \, \epsilon \max_{\mu\in \Lambda_j^{n, \bot}} y_\mu = \rho_j\, \epsilon \,e^{n,\uparrow}_{2j}.
\end{equation}
For $l= 1, 2, \dots, 2n$, we have $e^{n,\downarrow}_{l} \leq e^{n,\uparrow}_{l}$, i.e., the vector inequality 
$E^{n,\downarrow} \leq E^{n,\uparrow}$ holds.
\end{restatable}
\noindent See Subsection \ref{subsec:proof:coED} for the proof.

\smallskip
Observe that it is guaranteed that the two solutions $E^{n,\downarrow}$ and $E^{n,\uparrow}$ are distinct. 
Indeed, for $j=1,2,\cdots,n$,  since we have $\max(\lambda_j,\rho_j)=1$ and assuming that $e^{n,\uparrow}_{2j - 1} > 0$ and $e^{n,\uparrow}_{2j} > 0$,  
we  must have $e^{n,\uparrow}_{2j - 1} > e^{n,\downarrow}_{2j-1}$ or $e^{n,\uparrow}_{2j} > e^{n,\downarrow}_{2j}$. In other words, provided that there exists $j \in \{1,2, \cdots,n\}$ such that $e^{n,\uparrow}_{2j - 1} > 0$ and $e^{n,\uparrow}_{2j} > 0$, the two column-vectors $E^{n,\downarrow}$ and $E^{n,\uparrow}$ yields two different solutions of the system $(\Upsigma)$.

\smallskip
We also have the following result:

\begin{restatable}{proposition}{propcoEDUcmd}\label{prop:coEDU}
Let $\dot{\Gamma_n} \Box_{\max}^{\min} E^{n,\downarrow} = [u_\mu]_{\mu\in \Lambda^{(n)}}$. Then, for each $\mu\in \Lambda^{(n)}$, we have:
\begin{equation}\label{eq:autreen}
u_\mu = \min_{1 \leq j \leq n} z_{\mu, j}\quad \text{; } \quad 
z_{\mu, j} := \begin{cases}
\max(\lambda_j, e_{2j-1}^{n, \uparrow})     & \text{ if }  \mu\in \Lambda_j^{n, \top}   \\
  \max(\rho_j, e_{2j}^{n, \uparrow}) & \text{ if }  \mu\in \Lambda_j^{n, \bot} 
\end{cases}.
    \end{equation}
Moreover, the following equality holds: 
\begin{equation}\label{eq:ee}
\dot{\Gamma_n} \Box_{\max}^{\min} E^{n,\downarrow} = \dot{\Gamma_n} \Box_{\max}^{\min} E^{n,\uparrow} 
.
\end{equation}

\end{restatable}
\noindent See Subsection \ref{subsec:proof:coEDU} for the proof. 

\smallskip
By Sanchez's result \cite{sanchez1976resolution}, the reduced equation system $(\Upsigma_n)$ is consistent iff $\dot{Y}_n = \dot{\Gamma_n} \Box_{\max}^{\min} E^{n,\downarrow}$. If this equality is satisfied, the lowest solution of the system $(\Upsigma_n)$ is the vector $E^{n,\downarrow}$ and by     Proposition \ref{prop:coEDU},  the vector $E^{n,\uparrow}$ is another solution of the system $(\Upsigma_n)$. In fact,    any vector $X$ such that $E^{n,\downarrow}\leq X \leq  E^{n,\uparrow}$  (where $\leq$ denotes the usual component-wise order between vectors) is a solution of the reduced equation system $(\Upsigma_n)$:
\begin{corollary}\label{cor:consist}
Suppose that the reduced equation system  $(\Upsigma_n): \dot{Y}_n = \dot{\Gamma_n} \Box_{\max}^{\min} X$ is consistent. Then  any vector $X$ such that $E^{n,\downarrow}\leq X \leq  E^{n,\uparrow}$ is a solution of the reduced equation system $(\Upsigma_n)$.   
\end{corollary}
\begin{proof}
As by Lemma \ref{lem:coED}, we have  $E^{n,\downarrow} \leq E^{n,\uparrow}$, we can deduce that for any vector $X$ such that $E^{n,\downarrow}\leq X \leq  E^{n,\uparrow}$:
\[ \dot{Y}_n=  \dot{\Gamma_n} \Box_{\max}^{\min} E^{n,\downarrow} \leq \dot{\Gamma_n} \Box_{\max}^{\min} X \leq 
  \dot{\Gamma_n} \Box_{\max}^{\min} E^{n,\uparrow} = \dot{Y}_n. \]
\end{proof}

Accordingly, whenever the reduced equation system $(\Upsigma_n)$ is consistent, it has an uncountable number of solutions.

\begin{example}
{(Example \ref{ex:firstrules-set}, 
cont'ed)} \\
The equation system (\ref{eq:eqtd1}) of Example \ref{ex:conssystg} is consistent, and it has the following solutions: 
$$ E^{n,\downarrow} = \begin{bmatrix}
   1.0\\ 0\\ 0.2\\ 0\\ 0\\ 0.87\\ 0\\ 1.0
\end{bmatrix} \text{ and } E^{n,\uparrow}= \begin{bmatrix}
    1.0\\0.2\\ 0.2\\ 1.0\\ 1.0\\ 0.87\\ 0.87\\ 1.0
\end{bmatrix}.$$

\noindent The equation system (\ref{eq:eqtd2}) of Example \ref{ex:conssystg} is also consistent, and it has the following solutions: 
$$ E^{n,\downarrow} = \begin{bmatrix}
    1\\ 0\\ 0.4\\ 0 
\end{bmatrix} \text{ and } E^{n,\uparrow}= \begin{bmatrix}
    1\\ 0.4\\ 0.4\\ 1
\end{bmatrix}.$$
\end{example}

\paragraph*{Relating the equation system to the matrix relation} 
The connection between the equation system and the matrix relation is presented in \cite{baaj2022learning} and detailed in the following. Let us consider a possibilistic rule-based system that uses input attributes $a_1, a_2, \ldots, a_k$ and an output attribute $b$, and a training data sample. This sample consists of a set of input possibility distributions $\widetilde{\pi}_{a_1(x)}, \widetilde{\pi}_{a_2(x)}, \ldots, \widetilde{\pi}_{a_k(x)}$, associated with the input attributes $a_1, a_2, \ldots, a_k$, and a targeted  output possibility distribution $\widetilde{\pi}_{b(x)}$ associated with the output attribute $b$.

\noindent
If the reduced equation system $(\Upsigma_n)$ constructed with the training data sample (see Subsection \ref{subsubsec:constructEStrainingdata}) is consistent, we can set the values of the parameters $s_1,r_1,s_2,r_2,\cdots,s_n,r_n$ of the rules $R^1,R^2,\cdots,R^n$ to the coefficients of a solution  $X = [x_l]_{1 \leq l \leq 2n}$  of the equation system $(\Upsigma_n)$ so that: 
$$s_i:= x_{2i -1} \text{ and } r_i := x_{2i} \quad \text{ for all } i\in\{1, 2, \dots, n\}.$$
Then, the matrix $\mathcal{M}_n$ can be built using  (\ref{eq:reducedonmnconstructed}).
 Inference from the possibilistic rule-based system using the possibility degrees of the rule premises computed with the input possibility distributions $\widetilde{\pi}_{a_1(x)}, \widetilde{\pi}_{a_2(x)}, \ldots, \widetilde{\pi}_{a_k(x)}$ of the training data sample (which define the input vector $I_n$, see (\ref{eq:inputvector})), and the learned rule parameter values $(s_i, r_i)$, yields an output possibility distribution $\pi^*_{b(x)}$, see (\ref{eq:pistarbxu}), whose possibility measure $\Pi$  satisfies  (\ref{eq:PiO}). The compatibility of the targeted output possibility distribution $\widetilde{\pi}_{b(x)}$  of the training data sample with the inferred output possibility distribution $\pi^*_{b(x)}$  is ensured by the following results:
 
\begin{restatable}{theorem}{theoremYOcmd}\label{theorem:YO}
The following equality holds:
\begin{equation}\label{eq:releqsyscompatible}
    \dot{Y}_n = \mathcal{O}_n  \quad \text{where} \quad     \mathcal{O}_n := \mathcal{M}_n \Box_{\max}^{\min} I_n.
\end{equation}
 \end{restatable}
\noindent See Subsection \ref{subsec:proof:YO} for the proof. 
Using  the previous notations, two corollaries can be deduced from this result:
\begin{corollary}
For each non-empty subset $E_k^{(n)}$  with $k \in J^{(n)}$ of the partition of $D_b$, we have:
\[ \Pi(E_k^{(n)}) =  \widetilde\Pi(E_k^{(n)})\]
\noindent where $\Pi, \widetilde\Pi$ are the possibility measures associated with $\pi^*_{b(x)},\widetilde{\pi}_{b(x)}$ respectively.
\end{corollary}
\begin{proof}
By (\ref{eq:buildupsigman}) and Lemma \ref{lemma:mutokmu},  we have $\dot{Y}_n= [\widetilde \Pi(E_{k_\mu}^{(n)})]_{\mu\in \Lambda^{(n)}}$. Likewise, by (\ref{eq:PiO}), we have ${\cal O}_n = [\Pi(E_{k_\mu}^{(n)})]_{\mu\in \Lambda^{(n)}}$.  
The mapping 
\begin{center}
\begin{tabular}{ll}
$\Psi :$ & $\Lambda^{(n)} \rightarrow J^{(n)}$\\
& $\mu \mapsto k_\mu
$ 
\end{tabular}
\end{center}
being bijective, see Lemma \ref{lemma:mutokmu}, we deduce from the equality $\dot{Y}_n= {\cal O}_n$   that $\Pi(E_k^{(n)}) =  \widetilde\Pi(E_k^{(n)})$ for all $k \in J^{(n)}$.
\end{proof}

\begin{corollary}
  Suppose that $\text{card}(E_k^{(n)}) = 1$ for all $k\in J^{(n)}$ i.e., the non-empty subsets of the partition  $D_b$ are singletons. Then,  we have:  $$\widetilde{\pi}_{b(x)}(u) = \pi^*_{b(x)}(u) \quad \text{ for each } u \in D_b.$$
\end{corollary}
\qed

\subsection{Assessing the quality of the data used to construct the equation system \texorpdfstring{$(\Upsigma_n)$}{(Upsigma n)}}
\label{subsec:poorqualtrainingdata}
The reduced equation system $(\Upsigma_n): \dot{Y}_n = \dot{\Gamma_n} \Box_{\max}^{\min} X$, see (\ref{eq:upsigma}), may be inconsistent for several reasons:
\begin{itemize}
    \item The equation system may be constructed with a training data sample that is altered by noise or is an outlier,
    \item The possibilistic rule base is not well-specified and this may prevent a desired result from being obtained.\\ 
    This situation can occur if the propositions in the premises and conclusions of the rules are poorly designed. For example, an error may arise in the conclusion "$b(x) \in Q$" if the set $Q$ is incorrectly defined; specifically, if it includes elements that do not align with the premise, leading to inconsistent outputs from the possibilistic rule-based system.
     Additionally, compatible premises with conflicting conclusions  can create similar inconsistencies.
\end{itemize}

 Assuming that the rule base is well-specified, the extent to which the training data are of poor quality can be measured using recent results on the handling of inconsistent $\min-\max$ equations systems  \cite{BAAJ2024}. Since, as far as we know, this feature is new for neuro-symbolic approaches, we now explain how the quality of the data can be assessed thanks to it.

\paragraph*{\texorpdfstring{Chebyshev distance associated with the second member $\dot{Y}_n$}{Chebyshev distance associated with the second member dot{Y}n}} 

When the reduced equation system $(\Upsigma_n)$ is constructed from a training data sample, it may be inconsistent if the data is of poor quality. The extent to which the training data sample is of poor quality can be measured  by studying the inconsistency of the equation system $(\Upsigma_n)$ using the work in \cite{BAAJ2024}.
 \cite{BAAJ2024} studies the problem of finding the closest consistent min-max equation systems to a given  min-max inconsistent equation system, assuming all consistent equation systems use the same matrix: that of the inconsistent equation system. Here, a consistent equation system is said to be as close as possible to the initial inconsistent system when the distance (in terms of $L_\infty$ norm) between the second member of the consistent system and that of the inconsistent one is minimal. As explained in the following, the work of \cite{BAAJ2024} can be exploited to compute this minimum distance efficiently, thus providing a means to measure the extent to which the training data sample used to construct the  equation system $(\Upsigma_n)$ is of poor quality.

Let us explain how to apply the approach presented in \cite{BAAJ2024} to the equation system $(\Upsigma_n)$.  If $(\Upsigma_n)$ is inconsistent, the goal is to find a closest consistent equation system of the form $d = \dot{\Gamma_n} \Box_{\max}^{\min} X$, which uses the same matrix $\dot{\Gamma_n} = \begin{bmatrix} \dot{\gamma}_{\mu, l} \end{bmatrix}_{\mu\in \Lambda^{(n)}, 1 \leq l \leq 2n}$ as $(\Upsigma_n)$, where $X$ is an unknown vector, and $d$ is a second member that is as close as possible to the second member $\dot{Y}_n$  of $(\Upsigma_n)$. The minimal distance is called the Chebyshev distance associated with the second member $\dot{Y}_n = \begin{bmatrix} y_\mu \end{bmatrix}_{\mu\in \Lambda^{(n)}}$ of $(\Upsigma_n)$ and it is defined by, see \cite{BAAJ2024}:
$$\nabla = \min_{d \in \mathcal{D}} \Vert \dot{Y}_n - d \Vert_{\infty},$$ 
where the set $\mathcal{D}$ given by 
\begin{align}\label{eq:setD}
    \mathcal{D} &= \{d \mid d = \begin{bmatrix} d_\mu \end{bmatrix}_{\mu\in \Lambda^{(n)}} \text{ with } d_\mu \in [0,1] \text{ and the equation system } d = \dot{\Gamma_n} \Box_{\max}^{\min} X \text{ is consistent} \} 
\end{align}
is formed by the second members of consistent equation systems defined with the same matrix: the matrix  $\dot{\Gamma_n}$ of $(\Upsigma_n)$. An obvious observation is that, except for trivial inconsistent min-max equation systems, the set $\mathcal{D}$ is not countable. However, leveraging previous work of  \cite{cuninghame1995residuation}, which showed that the Chebyshev distance $\nabla$ is the minimal solution of a particular vector inequality (see Corollary 5 in  \cite{BAAJ2024}), \cite{BAAJ2024} provided an explicit analytical formula for computing $\nabla$ in the context of a min-max equation system (see Section 6 of \cite{BAAJ2024}).

\noindent For the case of the equation system $(\Upsigma_n)$, the analytical formula of the Chebyshev distance is (we use the notation $x^+ = \max(x,0)$):
\begin{equation}\label{eq:Nabla}
    \nabla = \max_{\mu\in \Lambda^{(n)}} \nabla_\mu
\end{equation}
where for each $\mu\in \Lambda^{(n)}$: 
\begin{equation} \label{eq:Nablamini} 
\nabla_\mu =  \min_{1 \leq l \leq 2n}\,\max[ (\dot{\gamma}_{\mu, l}-y_\mu)^+,  \max_{\mu'\in \Lambda^{(n)}}\,  \,\min(\frac{(y_{\mu'} - y_\mu)^+}{2}, (y_{\mu'} - \dot{\gamma}_{\mu',l})^+)].
\end{equation}
\cite{BAAJ2024} shows that the following equivalence holds:
\begin{equation}\label{eq:nablaconsist}
    \nabla = 0 \Longleftrightarrow \text{ the equation system } (\Upsigma_n) \text{ is consistent.}
\end{equation}

\noindent For the equation system $(\Upsigma_n)$ constructed using a training data sample, the Chebyshev distance $\nabla$ is a measure of the reliability of this data sample. Given that the possibilistic rule base system is well-specified, the more reliable the training data sample used to build the equation system $(\Upsigma_n)$, the lower the Chebyshev distance  $\nabla$ of the equation system $(\Upsigma_n)$.

Accordingly, we can evaluate  the quality of a training data sample  based on a tolerance threshold in the following way: if the equation system $(\Upsigma_n)$ constructed from the training data sample is consistent or, more generally, if the Chebyshev distance remains  strictly  lower than the preset threshold, then the training data sample is considered as reliable. Formally:
\begin{definition}\label{def:reliabletau}
Let $\tau > 0$ be a fixed tolerance threshold 
used to assess whether the training data are reliable. The equation system $(\Upsigma_n)$ is constructed from a \textit{reliable} training data sample if its Chebyshev distance  $\nabla$,  see (\ref{eq:Nabla}) is strictly lower than $\tau$, i.e., $\nabla < \tau$.
\end{definition}
\noindent As a matter of illustration, let us consider again our running example:
 
\begin{example}{(Example \ref{ex:firstrules-set}, 
cont'ed)} \\
\label{ex:chebdist-12}
We assume that a threshold $\tau$ is set to $0.01$. We consider a training data sample associated with the first set of rules (Example \ref{ex:firstrules-set}), where the possibility degrees of the rule premises inferred with input possibility distributions are $\lambda_1 = 1, \rho_1 =0,  \lambda_2= 0,  \rho_2= 1, \lambda_3= 1,  \rho_3= 0, \lambda_4=  0, \rho_4 = 1$  and the following  output possibility distribution is targeted: $\Pi( \{ (1,1) \}) =0.005,
     \Pi( \{ (0,1) \}) = 0.0,
      \Pi( \{ (1,0) \}) = 0.0,
       \Pi( \{ (0,0) \}) = 1$.  \\
The reduced equation system, see (\ref{eq:ex:learning-system1}), is:
\begin{equation}\label{eq:incsys-ex12}
    \begin{bmatrix}
    \Pi( \{ (1,1) \}) \\
     \Pi( \{ (0,1) \}) \\
      \Pi( \{ (1,0) \}) \\
       \Pi( \{ (0,0) \}) 
\end{bmatrix}=
\begin{bmatrix}
  0.005\\ 0\\ 0\\ 1
\end{bmatrix}= \begin{bmatrix}
   1& 0& 0& 1& 1& 0& 0& 1\\
    1& 1& 1& 1& 1& 0& 0& 1\\
    1& 0& 0& 1& 1& 1& 1& 1\\
    1& 1& 1& 1& 1& 1& 1& 1
\end{bmatrix}
\Box_{\max}^{\min} X.
\end{equation}

The Chebyshev distance (\ref{eq:Nabla}) associated with the equation system can now be computed. We have $\nabla = \max_{\mu \in {\Lambda^{(n)}}} \nabla_i$ where $\nabla_{\mu_1} = 0$, $\nabla_{\mu_2} = 0.0025$, $\nabla_{\mu_3} = 0.0025$ and $\nabla_{\mu_4} = 0$. So, the Chebyshev distance associated with the equation system $(\Upsigma_n)$ is equal to $\nabla = 0.0025$ and therefore the equation system $(\Upsigma_n)$  is inconsistent. Since $\tau = 0.01$, the data sample is considered as reliable nevertheless.

\end{example}

In a similar way, we can  compute the Chebyshev distance associated with the equation system $(\Omega)$, see (\ref{eq:sysOmega}), which is involved in the backpropagation mechanism presented in Subsection \ref{subsub:backpropagation}, and assess the quality of the data used to construct $(\Omega)$.

\subsection{Handling the inconsistency of the equation system \texorpdfstring{$(\Upsigma_n)$}{(Upsigma n)}}
\label{subsec:handlinginconsisteq}
In this subsection, we show how to handle the equation system $(\Upsigma_n)$ when it is inconsistent, i.e., when its Chebyshev distance $\nabla$, see (\ref{eq:Nabla}) is strictly positive. Based on \cite{BAAJ2024}, we show how to minimally update the data used to construct the inconsistent equation system $(\Upsigma_n)$ in order to get consistent equation systems that are as close as possible to $(\Upsigma_n)$, and how to obtain approximate solutions of the inconsistent equation system $(\Upsigma_n)$ (which are solutions of the consistent equation systems that are as close as possible to $(\Upsigma_n)$).

\paragraph*{Chebyshev approximations of the second member $\dot{Y}_n$}
\noindent From an inconsistent equation system  $(\Upsigma_n)$, one can derive a \textit{consistent} equation system as close as possible  to $(\Upsigma_n)$ by computing a Chebyshev approximation of the second member $\dot{Y}_n$ of the system $(\Upsigma_n)$, which is a vector $d = [d_\mu]_{\mu\in \Lambda^{(n)}}$ such that  $\Vert \dot{Y}_n - d\Vert_{\infty} = \nabla$  and  the equation system $d = \dot{\Gamma_n} \Box_{\max}^{\min} X$, where $X$ is an unknown vector, is  consistent, see \cite{BAAJ2024}. 

\noindent The lowest Chebyshev approximation of $\dot{Y}_n$ is defined as the following vector \cite{BAAJ2024}: \begin{equation}\label{eq:lowest}
   \widetilde{\dot{Y}_n}= \dot{\Gamma_n} \Box_{\max}^{\min} ({\dot{\Gamma_n}}^t \Box_{\epsilon}^{\max} \underline{\dot{Y}_n}({\nabla})) \quad \text{ where } \quad \underline{\dot{Y}_n}(\mathbf{\nabla}) = \begin{bmatrix}
        (y_\mu - {\nabla})^+
    \end{bmatrix}_{\mu \in \Lambda^{(n)}}.
\end{equation}
For computing $ \widetilde{\dot{Y}_n}$, a matrix product $\Box_\epsilon^{\max}$ is performed, which uses $\epsilon-$product, see (\ref{eq:epsilonproduct}) as the product and $\max$ as the addition, followed by a matrix product $\Box_{\max}^{\min}$, which uses the function $\max$ as the product and $\min$ as the addition. We have $\Vert \dot{Y}_n- \widetilde{\dot{Y}_n}\Vert_{\infty} = \nabla$ and the equation system $\widetilde{\dot{Y}_n} = \dot{\Gamma_n} \Box_{\max}^{\min}X$ is consistent. 
Deriving other Chebyshev approximations of $\dot{Y}_n$ would be more demanding computationally speaking, see \cite{BAAJ2024}.

 \paragraph*{\texorpdfstring{Approximate solutions of the equation system $(\Upsigma_n)$}{Approximate solutions of the equation system (Upsigma n)}} The methods introduced in \cite{BAAJ2024} can be exploited to obtain approximate solutions of an inconsistent equation system which are defined as follows: \textit{an approximate solution of an inconsistent equation system is a solution of one of the closest consistent equation systems} (an equation system whose matrix is that of the inconsistent equation system and whose second member is a Chebyshev approximation of the second member of the inconsistent equation system). 

\noindent The lowest approximate solution of the inconsistent equation system $(\Upsigma_n)$ is the following vector \cite{BAAJ2024}:
\begin{equation}\label{eq:lowestapprox}
    {\dot{\Gamma_n}}^t \Box_{\epsilon}^{\max} \underline{\dot{Y}_n}({\nabla}),
\end{equation}
\noindent and thus it is also the lowest solution of the consistent equation system $\widetilde{\dot{Y}_n} = \dot{\Gamma_n} \Box_{\max}^{\min}X$ (see \cite{BAAJ2024}), where $\widetilde{\dot{Y}_n}$ is the lowest Chebyshev approximation of $\dot{Y}_n$, see (\ref{eq:lowest}). Any approximate solution $x^\ast$ of the inconsistent equation system $(\Upsigma_n)$ such as the vector ${\dot{\Gamma_n}}^t \Box_{\epsilon}^{\max} \underline{\dot{Y}_n}({\nabla})$ 
satisfies $\Vert {\dot{\Gamma_n}} \Box_{\max}^{\min} x^\ast -  \dot{Y}_n   \Vert_{\infty} = \nabla$, which means that the vector $\dot{\Gamma_n} \Box_{\max}^{\min} x^\ast$ is a Chebyshev approximation of $\dot{Y}_n$.  
Again, obtaining other approximate  solutions 
would be more expensive in terms of computational costs, see \cite{BAAJ2024}.

 \begin{remark}
 \noindent If the equation system $(\Upsigma_n)$ is consistent, i.e., we have $\nabla = 0$, then the lowest Chebyshev approximation $\widetilde{\dot{Y}_n}$ of $\dot{Y}_n$ is equal to the second member $\dot{Y}_n$  of the equation system $(\Upsigma_n)$, i.e., we have $\widetilde{\dot{Y}_n}=\dot{\Gamma_n} \Box_{\max}^{\min} ({\dot{\Gamma_n}}^t \Box_{\epsilon}^{\max} \underline{\dot{Y}_n}({\nabla})) = \dot{Y}_n$,  and the lowest approximate solution of $(\Upsigma_n)$ is equal to the lowest solution of the equation system $(\Upsigma_n)$, i.e., we have $
    {\dot{\Gamma_n}}^t \Box_{\epsilon}^{\max} \underline{\dot{Y}_n}({\nabla}) = {\dot{\Gamma_n}}^t \Box_{\epsilon}^{\max} \dot{Y}_n.$
\end{remark}

\begin{example}
{(Example \ref{ex:firstrules-set}, 
cont'ed)} \\
In Example \ref{ex:chebdist-12}, we explained how to compute the Chebyshev distance $\nabla = 0.0025$ associated with the second member $\dot{Y}_n=\begin{bmatrix}
  0.005\\ 0\\ 0\\ 1
\end{bmatrix}$ of the equation system (\ref{eq:incsys-ex12}), which describes a targeted output possibility distribution. In the following, we indicate how to compute the lowest Chebyshev approximation of the second member $\dot{Y}_n$ of the equation system (\ref{eq:incsys-ex12}) and the lowest approximate solution of the equation system (\ref{eq:incsys-ex12}). 

From $\nabla$ and $\dot{Y}_n$, the following vector can first be computed:
$$\underline{\dot{Y}_n}(\mathbf{\nabla}) = \begin{bmatrix}
       ( y_\mu - {\nabla})^+
    \end{bmatrix}_{\mu \in \Lambda^{(n)}} = \begin{bmatrix}
        0.0025\\
        0\\
        0\\
        0.9975
    \end{bmatrix}.$$
\noindent Then, the lowest Chebyshev approximation of $\dot{Y}_n$ is computed as the following vector: $$\widetilde{\dot{Y}_n} =  \dot{\Gamma_n} \Box_{\max}^{\min} ({\dot{\Gamma_n}}^t \Box_{\epsilon}^{\max} \underline{\dot{Y}_n}({\nabla}))= \begin{bmatrix}
  0.0025\\0.0025\\ 0.0025\\ 1
\end{bmatrix}.$$ The equation system $\widetilde{\dot{Y}_n} = \dot{\Gamma_n} \Box_{\max}^{\min} X$ is consistent.  We have $\Vert \dot{Y}_n- \widetilde{\dot{Y}_n}\Vert_{\infty} = \nabla$. We observe that a minimal  update of the targeted output possibility distribution described in the vector $\dot{Y}_n$  has been performed to obtain the second member $\widetilde{\dot{Y}_n}$ of the equation system.

\noindent The lowest approximate solution of the equation system (\ref{eq:incsys-ex12}) is the vector:
\begin{equation}
   {\dot{\Gamma_n}}^t \Box_{\epsilon}^{\max} \underline{\dot{Y}_n}({\nabla})= \begin{bmatrix}
        0.0\\ 
        0.0025\\ 
        0.0025\\ 
        0.0\\ 
        0.0\\ 
        0.0025\\ 
        0.0025\\ 
        0.0
    \end{bmatrix}.
\end{equation}
It is the lowest solution of the consistent equation system $\widetilde{\dot{Y}_n} = \dot{\Gamma_n} \Box_{\max}^{\min} X$.
\end{example}

\subsection{\texorpdfstring{Application to training data samples of multi-class classification problems}{Application to training data samples of multi-class classification problems}}
\label{subsec:apptotrainingdatasample}

\noindent The results from \cite{BAAJ2024} can be applied in a multi-class classification problem to handle the inconsistency of the equation system $(\Upsigma_n)$, which is constructed from a training data sample. This training data sample consists of input possibility distributions and a targeted output possibility distribution, where only the target (meta-)concept has a possibility degree of one, and all other concepts have a possibility degree of zero. 

In the following, we show how 
the method presented in \cite{BAAJ2024} can be leveraged to slightly modify the training data sample in order to recover consistency. Additionally, we show that considering a threshold $\tau$ strictly below $0.5$ for the Chebyshev distance $\nabla$ of $(\Upsigma_n)$, see (\ref{eq:Nabla}), is suitable for addressing multi-class classification problems, as it allows us to determine whether or not the training data sample is reliable.

\paragraph*{Constructing a consistent equation system close to the inconsistent equation system $(\Upsigma_n)$}
Assuming that the equation system $(\Upsigma_n)$ has been constructed from a training data sample (Subsection \ref{subsubsec:constructEStrainingdata}), and that the Chebyshev distance $\nabla$ of $(\Upsigma_n)$, see (\ref{eq:Nabla}), has been computed, a consistent equation system close to $(\Upsigma_n)$ can be generated, where:
\begin{itemize}
    \item the matrix is $\dot{\Gamma_n}$, 
    \item the second member is the lowest Chebyshev approximation $\widetilde{\dot{Y}_n}$ of  $\dot{Y}_n$.
\end{itemize}
The adjustments required to recover consistency only apply to the targeted output possibility distribution described  in the initial second member $\dot{Y}_n$ and thus to the possibility degrees used to describe the degrees of
membership of the instance to the  (meta-)concepts. In practice, an equation system built from a high-quality training data sample is often nearly consistent, meaning its Chebyshev distance $\nabla$ is close to zero.

The following definition characterizes a consistent equation system close to $(\Upsigma_n)$ with respect to the tolerance threshold $\tau$, see Definition \ref{def:reliabletau}. 

\begin{definition}\label{def:reduced-es}
A \textit{consistent equation system built using a reliable training data sample with respect to the tolerance threshold $\tau$}  (Definition \ref{def:reliabletau}) is defined by:
   \begin{equation}\label{eq:cs-sigma-good}
    (\widetilde \Upsigma_n): \widetilde{\dot{Y}_n} = \dot{\Gamma_n} \Box_{\max}^{\min} X 
\end{equation}
\noindent where:
\begin{itemize}
    \item the Chebyshev distance $\nabla$, see (\ref{eq:Nabla}), of the equation system $(\Upsigma_n)$ is such that $\nabla < \tau$,
    \item the second member $\widetilde{\dot{Y}_n}$ is the lowest Chebyshev approximation, see (\ref{eq:lowest}) of the second member $\dot{Y}_n$ of the equation system $(\Upsigma_n)$.
\end{itemize} 
\end{definition}

\paragraph*{Determining whether a training data sample is reliable} 
For multi-class classification problems,   the second member $\dot{Y}_n$ of the equation system $(\Upsigma_n)$ describes a targeted output possibility distribution where only one (meta-)concept (the target) is fully possible (i.e., with a possibility degree of one), while all other concepts are impossible (each having a possibility degree of zero). In other words,
the second member $\dot{Y}_n$  is a vector in which one component is equal to one and all the others are equal to zero.
Given that the rule base related to the equation system $(\Upsigma_n)$ is well-specified, since the lowest Chebyshev approximation  $\widetilde{\dot{Y}}$ is such that $\Vert \dot{Y} - \widetilde{\dot{Y}} \Vert_{\infty} = \nabla$, where $\nabla$ is the Chebyshev distance  of  $(\Upsigma_n)$, see (\ref{eq:Nabla}),  the following result holds: if $\nabla < 0.5$, then the component of the lowest Chebyshev approximation $\widetilde{\dot{Y}_n}$  with the highest value is \textit{unique} and its index is exactly the index of the component that is set to one in $\dot{Y}_n$. This justifies the use of the $0.5$ threshold: if $\nabla < \tau < 0.5$, the training data sample can be considered reliable to some extent, since the highest  possibility degree in $\widetilde{\dot{Y}_n}$ is the expected one. When $\nabla \geq  0.5$, the training data sample may contain outliers, as the highest possibility degree in $\widetilde{\dot{Y}_n}$ might not align with the expected target (meta-)concept (so the training data sample is not reliable in this sense), but a consistent equation system can still be obtained using the lowest Chebyshev approximation $\widetilde{\dot{Y}_n}$ of $\dot{Y}_n$, see (\ref{eq:lowest}). Further reasoning to determine whether the training data sample is reliable may also be done by checking whether the component with the highest degree  in the lowest Chebyshev approximation $\widetilde{\dot{Y}_n}$ is unique and is the expected one.

\subsection{Learning rule parameters with multiple training data}
\label{subsubsec:learningmultipledata}
\cite{BAAJ2024} shows how to tackle the problem of determining values of rule parameters from a training data set, i.e., a set of training data samples. The approach is as follows: 
\begin{enumerate}
    \item For each sample in the data set, its associated equation system $(\Upsigma)$ is constructed, see (\ref{eq:upsigma}),
    \item  All the equation systems are then stacked into one equation system denoted by $(\mathbf{\Sigma})$,
    \item  Methods for solving $\min-\max$ equation system  and for handling inconsistent $\min-\max$ equation systems are leveraged to obtain (approximate) solutions of $(\mathbf{\Sigma})$ for the rule parameters that are compatible with all the training data samples.
\end{enumerate}
Based on the same ideas, we now present an approach for learning the values of the rule parameters from a training data set.\\
 \paragraph*{Selecting reliable training data samples}
To address the problem of learning the values of rule parameters using very large datasets, our approach focuses on a subset of the training data. The goal is to select the training data samples that are reliable with respect to a tolerance threshold $\tau$ (Definition \ref{def:reliabletau}). The value of $\tau$ can be determined using a validation dataset, see Subsection \ref{subsec:thresholdsdetermination}. This selection process is uniquely motivated by the fact that datasets can be very large, and that possibilistic rule-based systems can have a large number of rules, which can lead to memory and computational problems when using laptops to learning the values of rule parameters.
Of course, this selection process is not mandatory when dealing with commonly-sized datasets and/or with a sufficiently powerful computer (we have verified this experimentally). A simple way of testing learning without the selection process  is to set the tolerance threshold $\tau$ to a value strictly greater than $1$ (a strict inequality is used in Definition \ref{def:reliabletau}), which leads to selecting all the training data samples.

\noindent The selection process is as follows. We suppose that the value of the tolerance threshold $\tau$  is set, and that a subset of  $N$ training data samples, which are considered reliable with respect to $\tau$, has been extracted from the set of training data. Thus, 
for each selected training data sample, the reduced equation system $(\Upsigma)$ built with it has a Chebyshev distance $\nabla$, see (\ref{eq:Nabla}), which is strictly lower than~$\tau$.

\paragraph*{Learning according to the selected training data samples} Our next goal is then to learn the values of the rule parameters according to these $N$ selected reliable training data samples.

\noindent Using Definition \ref{def:reduced-es}, $N$  consistent  reduced equation systems are constructed, based on the selected reliable training data samples. For $i = 1,2, \dots, N$, each equation system is denoted by:
\[
(\widetilde{\Upsigma}^{(i)}): \widetilde{\dot{Y}^{(i)}} = \dot{\Gamma^{(i)}} \Box_{\max}^{\min} X.
\]
So, according to Definition \ref{def:reduced-es}, for each equation system $(\widetilde{\Upsigma}^{(i)})$, its second member $\widetilde{\dot{Y}^{(i)}}$ is the lowest Chebyshev approximation of its initial second member ${\dot{Y}^{(i)}}$ and all these equation systems have the same unknown: a vector $X$ whose components are the rule parameters. From the matrices $\dot{\Gamma^{(1)}}, \dot{\Gamma^{(2)}},\dots, \dot{\Gamma^{(N)}}$ and the second members $\widetilde{\dot{Y}^{(1)}}, \widetilde{\dot{Y}^{(2)}}, \dots, \widetilde{\dot{Y}^{(N)}}$   of the equation systems, a new matrix and a new column vector are formed: 
\begin{equation}
     \mathbf{\Gamma} = \begin{bmatrix}
        \dot{\Gamma^{(1)}}\\
         \dot{\Gamma^{(2)}}\\
        \vdots\\
         \dot{\Gamma^{(N)}}
    \end{bmatrix} \quad \text{ and } \quad \mathbf{Y} =  \begin{bmatrix}
        \widetilde{\dot{Y}^{(1)}}\\
        \widetilde{\dot{Y}^{(2)}}\\
        \vdots \\
        \widetilde{\dot{Y}^{(N)}}
    \end{bmatrix} 
\end{equation}
\noindent The following equation system  stacks the  consistent  equation systems $(\widetilde \Upsigma^{(1)}), (\widetilde \Upsigma^{(2)}),\dots,(\widetilde \Upsigma^{(N)})$ into one:
\begin{equation}\label{eq:stackedES}
    (\mathbf{\Sigma}): \mathbf{Y} = \mathbf{\Gamma} \Box^{\min}_{\max} X, 
\end{equation}
\noindent where    $\mathbf{\Gamma} = \begin{bmatrix}
    \mathbf{\gamma}_{\mathbf{ij}}
\end{bmatrix}_{1 \leq \mathbf{i} \leq \mathbf{n}, 1 \leq \mathbf{j} \leq 2\mathbf{n}}$ and $\mathbf{Y} = \begin{bmatrix}
    \mathbf{y}_{\mathbf{i}}
\end{bmatrix}_{1 \leq \mathbf{i} \leq \mathbf{n}}$. We have $\mathbf{n} = N\cdot \text{card}( \Lambda^{(n)})$ where $\Lambda^{(n)}$ is common to all equation systems $(\widetilde \Upsigma^{(1)}), (\widetilde \Upsigma^{(2)}),\dots,(\widetilde \Upsigma^{(N)})$.

\noindent By solving the equation system $(\mathbf{\Sigma})$, solutions compatible with all the reliable training data samples are obtained, which allow us to assign values to the rule parameters. In our approach, the lowest approximate solution of the equation system $(\mathbf{\Sigma})$ is used, see \cite{BAAJ2024}. It is given by the following vector:  
\begin{equation}\label{eq:lowestapproxsol}
    {\mathbf{\Gamma}}^t \Box_{\epsilon}^{\max} \underline{{\mathbf{Y}}}({\boldsymbol{{\nabla}}})
\end{equation}

\noindent where $\underline{{\mathbf{Y}}}({\boldsymbol{{\nabla}}}) = \begin{bmatrix}
        (\mathbf{{y_i}} - {\boldsymbol{{\nabla}}})^+
    \end{bmatrix}_{1 \leq i \leq \textbf{n}}$  is constructed using the Chebyshev distance $\boldsymbol{{\nabla}}$    associated with the equation system  $(\mathbf{\Sigma})$. The formula of the Chebyshev distance associated with a  $\min-\max$  equation system introduced in    \cite{BAAJ2024}, which is adapted for  
    the equation system $(\Upsigma_n)$ in (\ref{eq:Nabla}), can be straightforwardly adapted for the equation system $(\mathbf{\Sigma})$.\\
    
\begin{remark}If the equation system $(\mathbf{\Sigma})$ is consistent, its Chebyshev distance is equal to zero and so the lowest approximate solution is equal to its lowest solution, i.e., we have  $\mathbf{\Gamma}^t \Box_{\epsilon}^{\max} \mathbf{Y} = {\mathbf{\Gamma}}^t \Box_{\epsilon}^{\max} \underline{{\mathbf{Y}}}({\boldsymbol{{\nabla}}})$.  
\end{remark}

\noindent There are many reasons for the equation system $(\mathbf{\Sigma})$ to be inconsistent. For instance, some training data samples may conflict with others, such as when nearly identical instances are associated with significantly different output possibility distributions. However, in our experiments (Section \ref{sec:exp}),  we observed that, by constructing for each reliable training data sample an associated \textit{consistent} equation system using Definition \ref{def:reduced-es} and then stacking these consistent equation systems into the equation system $(\mathbf{\Sigma})$, the resulting equation system $(\mathbf{\Sigma})$ was  consistent or almost consistent with regard to the value of its Chebyshev distance ${\boldsymbol{{\nabla}}}$. 

\begin{example}
Consider the equation system
$(\widetilde \Upsigma^{(1)}): \widetilde{\dot{Y}^{(1)}} = \dot{\Gamma^{(1)}} \Box_{\max}^{\min} X$, which is associated with a reliable training data sample:
$$\begin{bmatrix}
   0\\
   0\\
   0\\
   1
\end{bmatrix}= \begin{bmatrix}
  1& 0& 0& 1& 1& 0& 0& 1\\
    1& 1& 1& 1& 1& 0& 0& 1\\
    1& 0& 0& 1& 1& 1& 1& 1\\
    1& 1& 1& 1& 1& 1& 1& 1
\end{bmatrix}
\Box_{\max}^{\min} X 
$$
and the equation system $(\widetilde \Upsigma^{(2)}): \widetilde{\dot{Y}^{(2)}} = \dot{\Gamma^{(2)}}  \Box_{\max}^{\min} X$ which is associated with another reliable training data sample:
$$\begin{bmatrix}
  0.0025\\ 0.0025\\ 0.0025\\ 1
\end{bmatrix}= \begin{bmatrix}
   1& 0& 0& 1& 1& 0& 0& 1\\
    1& 1& 1& 1& 1& 0& 0& 1\\
    1& 0& 0& 1& 1& 1& 1& 1\\
    1& 1& 1& 1& 1& 1& 1& 1
\end{bmatrix}
\Box_{\max}^{\min} X. 
$$ 
\end{example}
\noindent We stack $(\widetilde \Upsigma^{(1)})$ and  $(\widetilde \Upsigma^{(2)})$ to form the equation system $(\mathbf{\Sigma})$:
\begin{equation*}
    \begin{bmatrix}
  0\\0\\0\\1\\ 0.0025\\ 0.0025\\ 0.0025\\ 1
\end{bmatrix}= \begin{bmatrix}
1& 0& 0& 1& 1& 0& 0& 1\\
    1& 1& 1& 1& 1& 0& 0& 1\\
    1& 0& 0& 1& 1& 1& 1& 1\\
    1& 1& 1& 1& 1& 1& 1& 1\\
   1& 0& 0& 1& 1& 0& 0& 1\\
    1& 1& 1& 1& 1& 0& 0& 1\\
    1& 0& 0& 1& 1& 1& 1& 1\\
    1& 1& 1& 1& 1& 1& 1& 1
\end{bmatrix}
\Box_{\max}^{\min} X. 
\end{equation*}
\noindent The Chebyshev distance of the equation system $(\mathbf{\Sigma})$ is ${\boldsymbol{\nabla}} = 0.00125$. 
\noindent From $\underline{{\mathbf{Y}}}(\boldsymbol{\nabla}) = \begin{bmatrix}
        (\mathbf{y_i} - \boldsymbol{\nabla})^+
    \end{bmatrix}_{1 \leq i \leq 8} = \begin{bmatrix}
        0\\0\\ 0\\ 0.99875\\ 0.00125\\ 0.00125\\ 0.00125\\ 0.99875
    \end{bmatrix}$, we compute the lowest approximate solution of  the equation system $(\mathbf{{\Sigma}})$ which is the vector ${\mathbf{\Gamma}}^t \Box_{\epsilon}^{\max}\underline{{\mathbf{Y}}}(\boldsymbol{\nabla})=\begin{bmatrix}
         0\\ 0.00125\\ 0.00125\\ 0\\ 0\\ 0.00125\\ 0.00125\\ 0
    \end{bmatrix}$. So we obtain approximate values for the rule parameters: $s_1 = r_2 = s_3 = r_4 = 0$ and $r_1 = s_2  = r_3 = s_4 = 0.00125$. Those values are approximately compatible with the two reliable training data samples at the same time. In other words, for each training data sample, inference based on the possibility degrees of the rule premises computed with the input possibility distributions and the learned rule parameter values, allows us to get an output possibility distribution very close to the targeted one. In the inferred distribution, the concept with the highest possibility degree consistently corresponds to the intended targeted concept for each sample, while all other concepts have possibility degrees very close to zero. This indicates that the approximate rule parameter values enable the system to reliably produce the expected results, closely aligning with the training data samples.

\subsection{Learning rule parameters of a cascade according to training data}

In what follows, a method for learning the values of the rule parameters according to training data for the case of   a cascade (Subsection \ref{subsub:cascade}), i.e., a possibilistic rule-based system that uses two chained sets of 
possibilistic rules, is presented. This learning method follows the forward chaining of sets of rules. Thus, the values of the rule parameters of the first set of 
rules of the cascade according to training data are learned using the approach given in  Subsection \ref{subsubsec:learningmultipledata}. Then inferences are performed from the first set of rules, in order to obtain an output possibility distribution 
for each training data sample used for learning.
Finally,  using the  approach given in  Subsection \ref{subsubsec:learningmultipledata}, the values of the  parameters of the rules of the second set of rules, which is chained to the first set, are learned. Each training data sample uses as input data one of the output possibility distributions inferred using the first set of rules.

The proposed method may take advantage of two hyperparameters, $\tau_1$ and $\tau_2$, which are tolerance thresholds (Definition \ref{def:reliabletau}), each of which being associated with a set of 
possibilistic rules, in order to limit the size of the training data used for learning (the selection process is  described in Subsection \ref{subsubsec:learningmultipledata}).  In the next subsection, we explain how the values  of $\tau_1$ and $\tau_2$ can be determined based on  accuracy measures obtained using a validation dataset. 
We assume that each training data sample is given by input possibility distributions associated with the input attributes of the first set of rules, a targeted output possibility distribution associated with the output attribute of the first set of rules, and another targeted output possibility distribution associated with the output attribute of the second set of rules.

\bigskip
\begin{method}[Possibilistic cascade learning]\label{meth:learningincascade}
\mbox{}
\begin{enumerate}[start=1]

    \item First, the rule parameters of the first set of  possibilistic rules are learned.
    \begin{itemize}
    \item For each training data sample, the following possibility distributions are considered:
    \begin{itemize}
        \item the input possibility distributions, which are associated with the input attributes of the first set of rules, 
        \item the targeted output possibility distribution, which is associated with the output attribute of the first set of rules.
    \end{itemize}
    \item For each training data sample, the possibility degrees of the rule premises of the first set of rules are computed using the input possibility distributions.
    \item For each training data sample, the equation system $(\Upsigma)$, see (\ref{eq:upsigma}), is constructed based on the first set of rules using the possibility degrees of
the rule premises and its targeted output possibility distribution.  We check whether the training data sample is reliable  with respect to $\tau_1$, see Definition \ref{def:reliabletau}. If it is the case, its associated consistent equation system is constructed using Definition \ref{def:reduced-es}. 
    \item All the consistent equation systems constructed using Definition \ref{def:reduced-es} are stacked  into a single equation system $(\mathbf{\Sigma})$, see (\ref{eq:stackedES}).
    \item The rule parameters of the first set of rules are learned by solving the equation system  $(\mathbf{\Sigma})$ and its lowest approximate solution is considered, see  (\ref{eq:lowestapproxsol}). From it, the values for the rule parameters of the first set of 
    possibilistic rules are obtained.
\end{itemize}
\item The values of the parameters of the first set of rules are set. 
\item  Using these rule parameters, the matrix relation of the first set of rules is constructed,  see Subsection \ref{subsec:matrixrelpractical}. 
\item For each training data sample, an inference step is performed using the constructed matrix relation  with the input possibility distributions of the first set of rules associated with the training data sample. From it, for each training data sample, its inferred output possibility distribution related to the first set of rules is obtained.

\item The next step consists in learning the rule parameters of the second set of 
possibilistic rules.
\begin{itemize}
    \item For each training data sample, the targeted output possibility distribution associated with the output attribute of the second set of rules is considered.
    \item For each training data sample, the possibility degrees of the rule premises of the second set of rules are computed. We rely on  the output possibility distribution obtained by the inference of the training data sample using the matrix relation associated with the first set of rules.
    \item For each training data sample, the equation system $(\Upsigma)$ is constructed, see (\ref{eq:upsigma}), associated with the second set of rules, using the possibility degrees of the premises of the rules of the second set of rules and the targeted output possibility distribution associated with the output attribute of the second set of rules.  We check whether the training data sample is reliable with respect to $\tau_2$, see Definition \ref{def:reliabletau}. If it is the case, its associated consistent equation system is constructed using Definition \ref{def:reduced-es}. 
    \item All the consistent equation systems constructed using Definition \ref{def:reduced-es} are stacked into a single equation system $(\mathbf{\Sigma})$, see (\ref{eq:stackedES}).
    \item The rule parameters of the second set of rules are computed by solving the equation system  $(\mathbf{\Sigma})$. Its lowest approximate solution is considered, see (\ref{eq:lowestapproxsol}). That way, values for the rule parameters of the second set of 
    possibilistic rules are obtained.
\end{itemize}
\item The parameters of the second set of rules are set.
\item Using these rule parameters, the matrix relation associated with the second set of rules is constructed, see Subsection \ref{subsec:matrixrelpractical}.
\end{enumerate}

\end{method}

After applying Method \ref{meth:learningincascade}, values for the rule parameters of the two sets of 
rules in the possibilistic rule-based system are determined, based on the training data. The two matrices governing the matrix relations associated with the two sets of rules in the possibilistic rule-based system have also been constructed. Given input possibility distributions associated with the input attributes of the first set of rules, inference from the possibilistic rule-based system can then be achieved using these matrices. 

The method is given for a cascade of two chained sets of 
possibilistic rules. Obviously enough, it can easily be extended to cascades with more sets of 
possibilistic rules.

\begin{example}{(Example \ref{ex:firstrules-set}, 
cont'ed)} \\
We consider again the cascade described in  Example \ref{ex:firstrules-set} and Example \ref{ex:secondset-rules} for addressing  the following neuro-symbolic problem: given two MNIST images, are the two handwritten digits on the images the same? 

\noindent The thresholds $\tau_1 = \tau_2$ are set to $0.05$.  
In the example, the same probability-to-possibility transformation is used: the antipignistic method, see Subsection \ref{sec:ppt}. Four training data samples are considered: 
\begin{itemize}
    \item The first sample is composed of:
    \begin{itemize}
        \item Two MNIST images \includegraphics[width=0.017\textwidth]{z_img__9158_ZERO.png} and \includegraphics[width=0.017\textwidth]{./z_img__ONE_41266.png} labeled $0$ and $1$ respectively in the MNIST dataset.
        \item The label ``different'' (instead of ``same''), which means that the two images have different labels.
        \item The probability distributions produced by the neural network for each of the two MNIST images \includegraphics[width=0.017\textwidth]{z_img__9158_ZERO.png} and \includegraphics[width=0.017\textwidth]{./z_img__ONE_41266.png} are: 
        \begin{itemize}
       \item[-]  $[P^{(1)}_{i_1}(\includegraphics[width=0.017\textwidth]{z_img__9158_ZERO.png} = 0), 
P^{(1)}_{i_1}(\includegraphics[width=0.017\textwidth]{z_img__9158_ZERO.png} = 1)] = [0.995, 0.005]$,
\item[-] $[P^{(1)}_{i_2}(\includegraphics[width=0.017\textwidth]{z_img__ONE_41266.png} = 0), 
P^{(1)}_{i_2}(\includegraphics[width=0.017\textwidth]{z_img__ONE_41266.png} = 1)] = [0.02, 0.98]$.\end{itemize}  They are transformed into possibility distributions (one per image, using the antipignistic transformation):
\begin{itemize}
       \item[-]$[\Pi^{(1)}_{i_1}(\includegraphics[width=0.017\textwidth]{z_img__9158_ZERO.png} = 0), 
\Pi^{(1)}_{i_1}(\includegraphics[width=0.017\textwidth]{z_img__9158_ZERO.png} = 1)] = [1, 0.01]$,
\item[-] $[\Pi^{(1)}_{i_2}(\includegraphics[width=0.017\textwidth]{z_img__ONE_41266.png} = 0), 
\Pi^{(1)}_{i_2}(\includegraphics[width=0.017\textwidth]{z_img__ONE_41266.png} = 1)] = [0.04, 1]$.
\end{itemize}
\item A targeted output possibility distribution associated with the output attribute $b$  of the first set of  possibilistic rules (Example \ref{ex:firstrules-set}). We remind that the domain of $b$  is $D_b = \{ (0,0), (0,1), (1,0), (1,1) \}$.  The targeted  output possibility distribution is as follows: it assigns a possibility degree of $1$ to the value $(0,1)$ i.e., $\Pi^{(1)}_t\{(0,1)\} = 1$ (since the labels of two images \includegraphics[width=0.017\textwidth]{z_img__9158_ZERO.png} and \includegraphics[width=0.017\textwidth]{./z_img__ONE_41266.png} in the MNIST dataset are $0$ and $1$ respectively), and a possibility degree of zero for all other output values.
\item A targeted output possibility distribution associated with the output attribute $c$ of the second set of possibilistic rules (Example \ref{ex:secondset-rules}). We remind that the domain of $c$  is $\{ 0, 1 \}$,  where the output value $0$ means that the two MNIST images represent two different handwritten digits, while the output value $1$ means that the two MNIST images represent the same handwritten digit. The targeted output possibility distribution of the sample is as follows: $ \Pi^{'(1)}_t(\{0\}) = 1$ and $\Pi^{'(1)}_t(\{1\}) = 0$,  since in the  sample, the two images represent two different digits.
    \end{itemize}
    \item The second sample is composed of:
    \begin{itemize}
      \item Two MNIST images \includegraphics[width=0.017\textwidth]{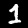} and \includegraphics[width=0.017\textwidth]{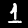} both labeled $1$  in the MNIST dataset.
        \item The label ``same'', which means that the two images have the same label.
        \item The probability distribution produced by the neural network for each of the two MNIST images \includegraphics[width=0.017\textwidth]{z_img_tr_1_a.png} and \includegraphics[width=0.017\textwidth]{./z_img_tr_1_b.png} are $[P^{(2)}_{i_1}(\includegraphics[width=0.017\textwidth]{z_img_tr_1_a.png} = 0), 
P^{(2)}_{i_1}(\includegraphics[width=0.017\textwidth]{z_img_tr_1_a.png} = 1)] = [0.015, 0.985]$
and $[P^{(2)}_{i_2}(\includegraphics[width=0.017\textwidth]{z_img_tr_1_b.png} = 0), 
P^{(2)}_{i_2}(\includegraphics[width=0.017\textwidth]{z_img_tr_1_b.png} = 1)] = [0.01, 0.99]$. We transform them into possibility distributions and obtain: $[\Pi^{(2)}_{i_1}(\includegraphics[width=0.017\textwidth]{z_img_tr_1_a.png} = 0), 
\Pi^{(2)}_{i_1}(\includegraphics[width=0.017\textwidth]{z_img_tr_1_a.png} = 1)] = [0.03, 1]$
and $[\Pi^{(2)}_{i_2}(\includegraphics[width=0.017\textwidth]{z_img_tr_1_b.png} = 0), 
\Pi^{(2)}_{i_2}(\includegraphics[width=0.017\textwidth]{z_img_tr_1_b.png} = 1)] = [0.02, 1]$. 
\item A targeted output possibility distribution associated with the output attribute $b$  of the first set of  possibilistic rules. The targeted  output possibility distribution is as follows: it assigns a possibility degree of $1$ to the value $(1,1)$ i.e., $\Pi^{(2)}_t\{(1,1)\} = 1$, and a possibility degree of zero for all other output values.
\item A targeted output possibility distribution associated with the output attribute $c$ of the second set of possibilistic rules is generated as follows: $ \Pi^{'(2)}_t(\{0\}) = 0$ and $\Pi^{'(2)}_t(\{1\}) = 1$,  since in the  sample, the two images represent the same digit.
    \end{itemize}
\item The third sample is composed of:
\begin{itemize}
   \item Two MNIST images \includegraphics[width=0.017\textwidth]{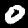} and \includegraphics[width=0.017\textwidth]{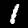}, which are  labeled $0$ and $1$ respectively  in the MNIST dataset,
        \item The label of the training data sample is ``different'', since the two images have a different label.
        \item The probability distribution produced by the neural network for each of the two MNIST images \includegraphics[width=0.017\textwidth]{z_img_tr_3_a.png} and \includegraphics[width=0.017\textwidth]{./z_img_tr_3_b.png} are: $[P^{(3)}_{i_1}(\includegraphics[width=0.017\textwidth]{z_img_tr_3_a.png} = 0), 
P^{(3)}_{i_1}(\includegraphics[width=0.017\textwidth]{z_img_tr_3_a.png} = 1)] = [0.5, 0.5]$
and $[P^{(3)}_{i_2}(\includegraphics[width=0.017\textwidth]{z_img_tr_3_b.png} = 0), 
P^{(3)}_{i_2}(\includegraphics[width=0.017\textwidth]{z_img_tr_3_b.png} = 1)] = [0.05, 0.995]$ respectively. From these probability distributions, we observe that the result is ambiguous for the first image. We transform these probability distributions into possibility distributions, and we obtain:
$[\Pi^{(3)}_{i_1}(\includegraphics[width=0.017\textwidth]{z_img_tr_3_a.png} = 0), 
\Pi^{(3)}_{i_1}(\includegraphics[width=0.017\textwidth]{z_img_tr_3_a.png} = 1)] = [1, 1]$
and $[\Pi^{(3)}_{i_2}(\includegraphics[width=0.017\textwidth]{z_img_tr_3_b.png} = 0), 
\Pi^{(3)}_{i_2}(\includegraphics[width=0.017\textwidth]{z_img_tr_3_b.png} = 1)] = [0.1, 1]$. 
\item A targeted output possibility distribution associated with the output attribute $b$  of the first set of possibilistic rules,  which is the same as for the first sample.  The targeted  output possibility distribution is as follows: it assigns a possibility degree of $1$ to the value $(0,1)$ i.e., $\Pi^{(3)}_t\{(0,1)\} = 1$, and a possibility degree of zero for all other output values.
\item A targeted output possibility distribution associated with the output attribute $c$ of the second set of possibilistic rules, which is generated as follows: $ \Pi^{'(3)}_t(\{0\}) = 1$ and $\Pi^{'(3)}_t(\{1\}) = 0$.
\end{itemize}
\item The fourth sample is composed of: 
\begin{itemize}
    \item Two MNIST images \includegraphics[width=0.017\textwidth]{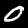} and \includegraphics[width=0.017\textwidth]{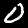} both labeled $0$  in the MNIST dataset.
        \item The  label ``same'', indicating that the two images have the same label.
        \item The probability distribution produced by the neural network for each of the two MNIST images \includegraphics[width=0.017\textwidth]{z_img_tr_4_a.png} and \includegraphics[width=0.017\textwidth]{./z_img_tr_4_b.png} are: $[P^{(4)}_{i_1}(\includegraphics[width=0.017\textwidth]{z_img_tr_4_a.png} = 0), 
P^{(4)}_{i_1}(\includegraphics[width=0.017\textwidth]{z_img_tr_4_a.png} = 1)] = [ 0.995,  0.005]$
and $[P^{(4)}_{i_2}(\includegraphics[width=0.017\textwidth]{z_img_tr_4_b.png} = 0), 
P^{(4)}_{i_2}(\includegraphics[width=0.017\textwidth]{z_img_tr_4_b.png} = 1)] = [0.025, 0.975]$, respectively. From these probability distributions, we observe that the result is incorrect for the second image (misclassification). We transform these probability distributions into possibility distributions: $[\Pi^{(4)}_{i_1}(\includegraphics[width=0.017\textwidth]{z_img_tr_4_a.png} = 0), 
\Pi^{(4)}_{i_1}(\includegraphics[width=0.017\textwidth]{z_img_tr_4_a.png} = 1)] = [1, 0.01]$
and $[\Pi^{(4)}_{i_2}(\includegraphics[width=0.017\textwidth]{z_img_tr_4_b.png} = 0), 
\Pi^{(4)}_{i_2}(\includegraphics[width=0.017\textwidth]{z_img_tr_4_b.png} = 1)] = [0.05, 1]$. 
\item A targeted output possibility distribution associated with the output attribute $b$  of the first set of possibilistic rules. The targeted  output possibility distribution is generated as follows: it assigns a possibility degree of $1$ to the value $(0,0)$ i.e., $\Pi^{(4)}_t\{(0,0)\} = 1$, and a possibility degree of zero for all other output values.
\item A targeted output possibility distribution associated with the output attribute $c$ of the second set of possibilistic rules is constructed as follows: $ \Pi^{'(4)}_t(\{0\}) = 0$ and $\Pi^{'(4)}_t(\{1\}) = 1$,  since in the  sample, the two images represent the same digit.
\end{itemize}

\end{itemize}

\noindent The learning of the rule parameters of the first set of possibilistic rules, see Example \ref{ex:firstrules-set}, begins using these four training data samples. We process the first sample.
The possibility degrees of the premises of the first set of 
possibilistic rules (Example \ref{ex:firstrules-set}) are computed from the possibility distributions of the training data sample (note: the image \includegraphics[width=0.017\textwidth]{z_img__9158_ZERO.png} is linked to the input attribute $a_1$ and the image \includegraphics[width=0.017\textwidth]{z_img__ONE_41266.png} is linked to the input attribute $a_2$). We obtain: $\lambda_1^{(1)} = 1, \rho_1^{(1)} = 0.01$,   $\lambda_2^{(1)} = 0.01,  \rho_2^{(1)} = 1$,  $\lambda_3^{(1)} = 0.04, \rho_3^{(1)} = 1$ and  $\lambda_4^{(1)} = 1, \rho_4^{(1)} = 0.04$. We rely on the possibility degrees of the premises and  the targeted output possibility distribution of the sample which is associated with the output attribute $b$  of the first set of possibilistic rules  to build the  equation system associated with the first set of 
possibilistic rules, see (\ref{eq:ex:learning-system1}) of Example \ref{ex:eqsyslearn}:

\begin{equation}\label{ex15:inc-eq-sys-1}\begin{bmatrix}
    \Pi^{(1)}_t\{(1,1)\}\\
     \Pi^{(1)}_t\{(0,1)\}\\
      \Pi^{(1)}_t\{(1,0)\}\\
       \Pi^{(1)}_t\{(0,0)\}
\end{bmatrix}=\begin{bmatrix}
   0\\
    1 \\
     0 \\
     0 
\end{bmatrix}
    =
    \begin{bmatrix}
        1 & 0.01 & 0.01 & 1 & 1 & 1 & 1 & 1\\
        1 & 1 & 1 & 1 & 1 & 1 & 1 & 1\\
        1 & 0.01 & 0.01 & 1 & 0.04 & 1 & 1 & 0.04\\
        1 & 1 & 1 & 1 & 0.04 & 1 & 1 & 0.04
    \end{bmatrix}  \Box_{\max}^{\min} \begin{bmatrix}
        s_1\\
        r_1\\
        s_2\\
        r_2\\
        s_3\\
        r_3\\
        s_4\\
        r_4
    \end{bmatrix}.\end{equation}

\noindent This equation system is inconsistent, since its Chebyshev distance, see (\ref{eq:Nabla}), is equal to 0.04. As we have $0.04 < \tau_1$, the training data sample is assessed as reliable with respect to $\tau_1$ (Definition \ref{def:reliabletau}). We construct a consistent equation system using Definition \ref{def:reduced-es}:
\begin{equation}\label{ex15:cons-eq-sys-1}
    \begin{bmatrix}
   0.01\\
    1 \\
     0.01 \\
     0.04 
\end{bmatrix}
    =
    \begin{bmatrix}
        1 & 0.01 & 0.01 & 1 & 1 & 1 & 1 & 1\\
        1 & 1 & 1 & 1 & 1 & 1 & 1 & 1\\
        1 & 0.01 & 0.01 & 1 & 0.04 & 1 & 1 & 0.04\\
        1 & 1 & 1 & 1 & 0.04 & 1 & 1 & 0.04
    \end{bmatrix}  \Box_{\max}^{\min} \begin{bmatrix}
        s_1\\
        r_1\\
        s_2\\
        r_2\\
        s_3\\
        r_3\\
        s_4\\
        r_4
    \end{bmatrix}.\end{equation}

\noindent The vector $\begin{bmatrix}
   0.01\\
    1 \\
     0.01 \\
     0.04 
\end{bmatrix}$ in (\ref{ex15:cons-eq-sys-1}) is the lowest Chebyschev approximation of the second member 
$\begin{bmatrix}
   0\\
    1 \\
     0 \\
     0 
\end{bmatrix}$ of the equation system (\ref{ex15:inc-eq-sys-1}) (thus, a minimal update of the targeted output possibility distribution has been performed), which allows us to obtain the consistent equation system (\ref{ex15:cons-eq-sys-1}).


\noindent We process the second sample. The possibility degrees of the premises of the first set of 
possibilistic rules (Example \ref{ex:firstrules-set}) are computed from the possibility distributions of the training data sample (note: the image \includegraphics[width=0.017\textwidth]{z_img_tr_1_a.png} is linked to the input attribute $a_1$ and the image \includegraphics[width=0.017\textwidth]{z_img_tr_1_b.png} is linked to the input attribute $a_2$). We obtain: $\lambda_1^{(2)} = 0.03, \rho_1^{(2)} = 1$,   $\lambda_2^{(2)} = 1,  \rho_2^{(2)} = 0.03$,  $\lambda_3^{(2)} = 0.02, \rho_3^{(2)} = 1$ and  $\lambda_4^{(2)} = 1, \rho_4^{(2)} = 0.02$. We use the possibility degrees of the premises and the targeted output possibility distribution of the sample which is associated with the output attribute $b$  of the first set of possibilistic rules  to build the equation system:
\begin{equation}
    \label{ex15:inc-eq-sys-2}
\begin{bmatrix}
    \Pi^{(2)}_t\{(1,1)\}\\
     \Pi^{(2)}_t\{(0,1)\}\\
      \Pi^{(2)}_t\{(1,0)\}\\
       \Pi^{(2)}_t\{(0,0)\}
\end{bmatrix}=\begin{bmatrix}
   1\\
    0 \\
     0 \\
     0 
\end{bmatrix}
    =\begin{bmatrix}
        1 & 1 & 1& 1 & 1 & 1 & 1 & 1\\
        0.03 & 1 & 1 & 0.03 & 1 & 1 & 1 & 1\\
        1 & 1 & 1 & 1 & 0.02 & 1 & 1 & 0.02\\
        0.03 & 1 & 1 & 0.03 & 0.02 & 1 & 1 & 0.02
    \end{bmatrix}\Box_{\max}^{\min} \begin{bmatrix}
        s_1\\
        r_1\\
        s_2\\
        r_2\\
        s_3\\
        r_3\\
        s_4\\
        r_4
    \end{bmatrix}.\end{equation}
\noindent This equation system is inconsistent, since its Chebyshev distance, see (\ref{eq:Nabla}), is equal to 0.03. As we have $0.03 < \tau_1$, the training data sample is assessed as reliable with respect to $\tau_1$ (Definition \ref{def:reliabletau}).
We construct a consistent equation system using Definition \ref{def:reduced-es}:
\begin{equation}\label{ex15:cons-eq-sys-2}
\begin{bmatrix}
   1\\
    0.03 \\
     0.02 \\
     0.02 
\end{bmatrix}
    =\begin{bmatrix}
        1 & 1 & 1& 1 & 1 & 1 & 1 & 1\\
        0.03 & 1 & 1 & 0.03 & 1 & 1 & 1 & 1\\
        1 & 1 & 1 & 1 & 0.02 & 1 & 1 & 0.02\\
        0.03 & 1 & 1 & 0.03 & 0.02 & 1 & 1 & 0.02
    \end{bmatrix}\Box_{\max}^{\min} \begin{bmatrix}
        s_1\\
        r_1\\
        s_2\\
        r_2\\
        s_3\\
        r_3\\
        s_4\\
        r_4
    \end{bmatrix}.\end{equation}
\noindent The vector $\begin{bmatrix}
    1\\
    0.03 \\
     0.02 \\
     0.02 
\end{bmatrix}$ in (\ref{ex15:cons-eq-sys-2}) is the lowest Chebyschev approximation of the second member 
$\begin{bmatrix}
   1\\
    0 \\
     0 \\
     0 
\end{bmatrix}$ of the equation system (\ref{ex15:inc-eq-sys-2}) (again, a minimal update of the targeted output possibility distribution has been achieved), which allows us to obtain the consistent equation system (\ref{ex15:cons-eq-sys-2}).

\noindent We process the third sample. We compute  the possibility degrees of the premises of the first set of possibilistic rules  from the possibility distributions of the training data sample (the image \includegraphics[width=0.017\textwidth]{z_img_tr_3_a.png} is linked to the input attribute $a_1$ and the image \includegraphics[width=0.017\textwidth]{z_img_tr_3_b.png} is linked to the input attribute $a_2$):   $\lambda_1^{(3)} = 1, \rho_1^{(3)} = 1$,   $\lambda_2^{(3)} = 1,  \rho_2^{(3)} = 1$,  $\lambda_3^{(3)} = 0.1, \rho_3^{(3)} = 1$ and  $\lambda_4^{(3)} = 1, \rho_4^{(3)} = 0.1$. We then construct the equation system associated with the first of  possibilistic rules using the possibility degrees of
the premises and the targeted output possibility distribution  of the sample which is associated with the output attribute $b$:
$$\begin{bmatrix}
    \Pi^{(3)}_t\{(1,1)\}\\
     \Pi^{(3)}_t\{(0,1)\}\\
      \Pi^{(3)}_t\{(1,0)\}\\
       \Pi^{(3)}_t\{(0,0)\}
\end{bmatrix}=\begin{bmatrix}
   0\\
    1 \\
     0 \\
     0 
\end{bmatrix}
    =\begin{bmatrix}
        1 & 1 & 1& 1 & 1 & 1 & 1 & 1\\
        1 & 1 & 1 & 1 & 1 & 1 & 1 & 1\\
        1 & 1 & 1 & 1 & 0.1 & 1 & 1 & 0.1\\
        1 & 1 & 1 & 1 & 0.1 & 1 & 1 & 0.1
    \end{bmatrix}\Box_{\max}^{\min} \begin{bmatrix}
        s_1\\
        r_1\\
        s_2\\
        r_2\\
        s_3\\
        r_3\\
        s_4\\
        r_4
    \end{bmatrix}.$$
\noindent This equation system is inconsistent, since its Chebyshev distance is equal to 1. As we have $1 > \tau_1$, the training data sample is assessed as \textit{not} reliable with respect to $\tau_1$ (Definition \ref{def:reliabletau}). We can compute  the lowest Chebyshev approximation of the second member of the equation system, see  (\ref{eq:lowest}), which is the vector $\begin{bmatrix}
    1\\
    1\\ 
    0.1\\ 
    0.1
\end{bmatrix}$.


\noindent We process  the fourth sample.
The possibility degrees of the premises of the first set of 
possibilistic rules are computed from the possibility distributions of the training data sample (the image \includegraphics[width=0.017\textwidth]{z_img_tr_4_a.png} is linked to the input attribute $a_1$ and the image \includegraphics[width=0.017\textwidth]{z_img_tr_4_b.png} is linked to the input attribute $a_2$). We obtain: $\lambda_1^{(4)} = 1, \rho_1^{(4)} = 0.01$,   $\lambda_2^{(4)} = 0.01,  \rho_2^{(4)} = 1$,  $\lambda_3^{(4)} = 0.05, \rho_3^{(3)} = 1$ and  $\lambda_4^{(4)} = 1, \rho_4^{(4)} = 0.05$. The equation system associated with the first set of 
possibilistic rules, see (\ref{eq:ex:learning-system1}), is built  using the
possibility degrees of the premises and the targeted output possibility distribution of the sample which is associated with the output attribute $b$:
$$\begin{bmatrix}
    \Pi^{(4)}_t\{(1,1)\}\\
     \Pi^{(4)}_t\{(0,1)\}\\
      \Pi^{(4)}_t\{(1,0)\}\\
       \Pi^{(4)}_t\{(0,0)\}
\end{bmatrix}=\begin{bmatrix}
   0\\
    0 \\
     0 \\
     1
\end{bmatrix}
    =\begin{bmatrix}
        1 & 0.01 & 0.01 & 1 & 1    & 1 & 1 & 1\\
        1 & 1    & 1    & 1 & 1    & 1 & 1 & 1\\
        1 & 0.01 & 0.01 & 1 & 0.05 & 1 & 1 & 0.05\\
        1 & 1    & 1    & 1 & 0.05 & 1 & 1 & 0.05
    \end{bmatrix}\Box_{\max}^{\min} \begin{bmatrix}
        s_1\\
        r_1\\
        s_2\\
        r_2\\
        s_3\\
        r_3\\
        s_4\\
        r_4
    \end{bmatrix}.$$
\noindent This equation system is inconsistent, since its Chebyshev distance is equal to 1. As we have $1 > \tau_1$, the training data sample is assessed as \textit{not} reliable with respect to $\tau_1$ (Definition \ref{def:reliabletau}). We can compute using (\ref{eq:lowest}) the lowest Chebyshev approximation of the second member of the equation system, which is the vector $\begin{bmatrix}
   0.01\\
   1\\ 
   0.01\\
   0.05
\end{bmatrix}$.

\noindent We have finished the sample-by-sample preprocessing. Based on the threshold $\tau_1$, the first two training data samples are considered reliable, whereas the last two are not.  The two consistent equation systems obtained using the first two  training data samples, see (\ref{ex15:cons-eq-sys-1}) and (\ref{ex15:cons-eq-sys-2}), are then stacked to give the equation system $(\mathbf{\Sigma})$, see (\ref{eq:stackedES}):
$$
\begin{bmatrix}
0.01\\
    1 \\
     0.01 \\
     0.04 \\
   1\\
    0.03 \\
     0.02 \\
     0.02 
\end{bmatrix}
    =\begin{bmatrix}
    1 & 0.01 & 0.01 & 1 & 1 & 1 & 1 & 1\\
        1 & 1 & 1 & 1 & 1 & 1 & 1 & 1\\
        1 & 0.01 & 0.01 & 1 & 0.04 & 1 & 1 & 0.04\\
        1 & 1 & 1 & 1 & 0.04 & 1 & 1 & 0.04\\
        1 & 1 & 1& 1 & 1 & 1 & 1 & 1\\
        0.03 & 1 & 1 & 0.03 & 1 & 1 & 1 & 1\\
        1 & 1 & 1 & 1 & 0.02 & 1 & 1 & 0.02\\
        0.03 & 1 & 1 & 0.03 & 0.02 & 1 & 1 & 0.02
    \end{bmatrix}\Box_{\max}^{\min} \begin{bmatrix}
        s_1\\
        r_1\\
        s_2\\
        r_2\\
        s_3\\
        r_3\\
        s_4\\
        r_4
    \end{bmatrix}.
$$
The equation system $(\mathbf{\Sigma})$ is consistent and its lowest solution 
is the vector $\begin{bmatrix}
    0 \\
    0 \\
    \vdots \\
    0 
\end{bmatrix}$ 
(by the way, we could check that the vector $\begin{bmatrix}0.03
\\
0.01\\
0.01\\
0.03\\
0.02\\
1\\
1 \\
0.02
\end{bmatrix}$ is another solution of the system $(\mathbf{\Sigma})$).
\noindent Following our learning method (Method \ref{meth:learningincascade}), the rule parameters of the first set of rules are set using the lowest solution of the equation system $(\mathbf{\Sigma})$: $s_1 = r_1 = s_2 = r_2 = \cdots = s_4 = r_4 = 0$. This shows that the rules are certain. Using these rule parameters, the matrix relation associated with the first set of rules,  see (\ref{eq:ex:matrixrelation:firstsys}), is constructed:

$$\begin{bmatrix}
    \Pi\{(1,1)\}\\
     \Pi\{(0,1)\}\\
      \Pi\{(1,0)\}\\
       \Pi\{(0,0)\}
\end{bmatrix}=
    \begin{bmatrix}
        1 & 0 & 0 & 1 & 1 & 0 & 0 & 1\\
        0 & 1 & 1 & 0 & 1 & 0 & 0 & 1\\
        1 & 0 & 0 & 1 & 0 & 1 & 1 & 0\\
        0 & 1 & 1 & 0 & 0 & 1 & 1 & 0
    \end{bmatrix} \Box_{\max}^{\min} \begin{bmatrix}
        \lambda_1\\
        \rho_1\\
        \lambda_2\\
        \rho_2\\
        \lambda_3\\
        \rho_3\\
        \lambda_4\\
        \rho_4
    \end{bmatrix}.$$

An inference is then performed from the first set 
of possibilistic rules using the constructed matrix relation with the possibility degrees of the rule premises of each training data sample. For the first sample, we compute $\lambda_1^{(1)} = 1, \rho_1^{(1)} = 0.01$,   $\lambda_2^{(1)} = 0.01,  \rho_2^{(1)} = 1$,  $\lambda_3^{(1)} = 0.04, \rho_3^{(1)} = 1$ and  $\lambda_4^{(1)} = 1, \rho_4^{(1)} = 0.04$. Using the $\min-\max$ matrix product, the inferred output possibility distribution is obtained: 
\begin{equation}\label{eq:ex15:pi1_o}
\begin{bmatrix}
    \Pi^{(1)}_o\{(1,1)\}\\
     \Pi^{(1)}_o\{(0,1)\}\\
      \Pi^{(1)}_o\{(1,0)\}\\
       \Pi^{(1)}_o\{(0,0)\}
\end{bmatrix}=\begin{bmatrix}
    0.01\\
    1\\
    0.01\\ 
    0.04
\end{bmatrix}
    =
    \begin{bmatrix}
        1 & 0 & 0 & 1 & 1 & 0 & 0 & 1\\
        0 & 1 & 1 & 0 & 1 & 0 & 0 & 1\\
        1 & 0 & 0 & 1 & 0 & 1 & 1 & 0\\
        0 & 1 & 1 & 0 & 0 & 1 & 1 & 0
    \end{bmatrix} \Box_{\max}^{\min} \begin{bmatrix}
    1\\ 
    0.01\\
    0.01\\ 
    1\\ 
    0.04\\ 
    1\\ 
    1\\ 
    0.04
\end{bmatrix}.\end{equation}

\noindent We do the same for the second training data sample. We remind that the possibility degrees of the rule premises are $\lambda_1^{(2)} = 0.03, \rho_1^{(2)} = 1$,   $\lambda_2^{(2)} = 1,  \rho_2^{(2)} = 0.03$,  $\lambda_3^{(2)} = 0.02, \rho_3^{(2)} = 1$ and  $\lambda_4^{(2)} = 1, \rho_4^{(2)} = 0.02$. We obtain:
\begin{equation}\label{eq:ex15:pi2_o}
\begin{bmatrix}
    \Pi^{(2)}_o\{(1,1)\}\\
     \Pi^{(2)}_o\{(0,1)\}\\
      \Pi^{(2)}_o\{(1,0)\}\\
       \Pi^{(2)}_o\{(0,0)\}
\end{bmatrix}=\begin{bmatrix}
    1\\
    0.03\\ 
    0.02\\
    0.02
\end{bmatrix}
    =
    \begin{bmatrix}
        1 & 0 & 0 & 1 & 1 & 0 & 0 & 1\\
        0 & 1 & 1 & 0 & 1 & 0 & 0 & 1\\
        1 & 0 & 0 & 1 & 0 & 1 & 1 & 0\\
        0 & 1 & 1 & 0 & 0 & 1 & 1 & 0
    \end{bmatrix} \Box_{\max}^{\min} \begin{bmatrix}
    0.03\\ 
    1\\ 
    1\\ 
    0.03\\ 
    0.02\\ 
    1\\
    1\\ 
    0.02
\end{bmatrix}.\end{equation}
Similarly, for the third training data sample, for which the possibility degrees of the rule premises are $\lambda_1^{(3)} = 1, \rho_1^{(3)} = 1$,   $\lambda_2^{(3)} = 1,  \rho_2^{(3)} = 1$,  $\lambda_3^{(3)} = 0.1, \rho_3^{(3)} = 1$ and  $\lambda_4^{(3)} = 1, \rho_4^{(3)} = 0.1$. We get:
\begin{equation}\label{eq:ex15:pi3_o}
\begin{bmatrix}
    \Pi^{(3)}_o\{(1,1)\}\\
     \Pi^{(3)}_o\{(0,1)\}\\
      \Pi^{(3)}_o\{(1,0)\}\\
       \Pi^{(3)}_o\{(0,0)\}
\end{bmatrix}=\begin{bmatrix}
    1\\
    1\\
    0.1\\
    0.1
\end{bmatrix}
    =
    \begin{bmatrix}
        1 & 0 & 0 & 1 & 1 & 0 & 0 & 1\\
        0 & 1 & 1 & 0 & 1 & 0 & 0 & 1\\
        1 & 0 & 0 & 1 & 0 & 1 & 1 & 0\\
        0 & 1 & 1 & 0 & 0 & 1 & 1 & 0
    \end{bmatrix} \Box_{\max}^{\min} \begin{bmatrix}
    1\\
    1\\
    1\\
    1\\
    0.1\\
    1\\
    1\\
    0.1
\end{bmatrix}.\end{equation}
 For the fourth training data sample, where the possibility degrees of the rule premises are $\lambda_1^{(4)} = 1, \rho_1^{(4)} = 0.01$,   $\lambda_2^{(4)} = 0.01,  \rho_2^{(4)} = 1$,  $\lambda_3^{(4)} = 0.05, \rho_3^{(3)} = 1$ and  $\lambda_4^{(4)} = 1, \rho_4^{(4)} = 0.05$, we obtain:
\begin{equation}\label{eq:ex15:pi4_o} 
\begin{bmatrix}
    \Pi^{(4)}_o\{(1,1)\}\\
     \Pi^{(4)}_o\{(0,1)\}\\
      \Pi^{(4)}_o\{(1,0)\}\\
       \Pi^{(4)}_o\{(0,0)\}
\end{bmatrix}=\begin{bmatrix}
    0.01\\
    1\\
    0.01\\
    0.05
\end{bmatrix}
    =
    \begin{bmatrix}
        1 & 0 & 0 & 1 & 1 & 0 & 0 & 1\\
        0 & 1 & 1 & 0 & 1 & 0 & 0 & 1\\
        1 & 0 & 0 & 1 & 0 & 1 & 1 & 0\\
        0 & 1 & 1 & 0 & 0 & 1 & 1 & 0
    \end{bmatrix} \Box_{\max}^{\min} \begin{bmatrix}
   1\\
   0.01\\
   0.01\\
   1\\
   0.05\\
   1\\
   1\\
   0.05
\end{bmatrix}.\end{equation}

\noindent Now, one can start learning  the rule parameters of the second set of 
possibilistic rules, see Example \ref{ex:secondset-rules}. 
To do so, one considers its associated reduced equation system, see (\ref{eq:ex:learning-system2}), of Example \ref{ex:eqsyslearn}. We process sample-by-sample. 
For the first sample, we compute the possibility degrees of the rule premises of the second set of rules using the inferred output possibility distribution of the first set of rules derived from the sample, see (\ref{eq:ex15:pi1_o}). We get $\lambda_1^{'(1)} = 0.04$, $\rho_1^{'(1)} = 1$, $\lambda_2^{'(1)} = 1$ and $\rho_2^{'(1)} = 0.04$. Using the possibility degrees of the rule premises and the targeted output possibility distribution  of the sample associated with the output attribute $c$ of the second set of rules, we construct the equation system:
\begin{equation}\label{ex15:inc-eq-sys-3}\begin{bmatrix}
   \Pi^{'(1)}_t( \{ 0 \}) \\
    \Pi^{'(1)}_t(\{ 1 \})
\end{bmatrix}=\begin{bmatrix}
    1\\
    0
\end{bmatrix}
    =
    \begin{bmatrix}
        1    & 1 & 1 & 1 \\
        0.04 & 1 & 1 & 0.04 
    \end{bmatrix} \Box_{\max}^{\min}  \begin{bmatrix}
    s'_1 \\
    r'_1 \\
    s'_2 \\ 
    r'_2 
\end{bmatrix}. \end{equation}

\noindent This equation system is inconsistent, since its Chebyshev distance is equal to 0.04. As we have $0.04 < \tau_2$, the training data sample is assessed as reliable with respect to $\tau_2$ (Definition \ref{def:reliabletau}). We construct a consistent equation system using Definition \ref{def:reduced-es}:
\begin{equation}\label{ex15:cons-eq-sys-3}
    \begin{bmatrix}
   1\\
   0.04
\end{bmatrix}
    =
    \begin{bmatrix}
        1    & 1 & 1 & 1 \\
        0.04 & 1 & 1 & 0.04 
    \end{bmatrix}  \Box_{\max}^{\min} \begin{bmatrix}
    s'_1 \\
    r'_1 \\
    s'_2 \\ 
    r'_2 
    \end{bmatrix}.\end{equation}

    \noindent The vector $\begin{bmatrix}
   1\\
   0.04
\end{bmatrix}$ in (\ref{ex15:cons-eq-sys-3}) is the lowest Chebyschev approximation of the second member 
$\begin{bmatrix}
   1\\
   0
\end{bmatrix}$ of the equation system (\ref{ex15:inc-eq-sys-3}) (thus, a minimal update of the targeted output possibility distribution has been performed), which allows us to obtain the consistent equation system  (\ref{ex15:cons-eq-sys-3}).

Similarly, we process the second sample. We compute the possibility degrees of the rule premises of the second set of rules using the inferred output possibility distribution of the first set of rules derived from the sample, see (\ref{eq:ex15:pi2_o}), and obtain: $\lambda_1^{'(2)} = 1$, $\rho_1^{'(2)} =0.03$, $\lambda_2^{'(2)} = 0.03$ and $\rho_2^{'(2)} = 1$. We rely on the possibility degrees of the rule premises and the targeted output possibility distribution of the sample associated with the output attribute $c$ of the second set of rules to build the equation system:
\begin{equation}\label{ex15:inc-eq-sys-4}\begin{bmatrix}
   \Pi^{'(2)}_t( \{ 0 \}) \\
    \Pi^{'(2)}_t(\{ 1 \})
\end{bmatrix}=\begin{bmatrix}
    0\\
    1
\end{bmatrix}
    =
    \begin{bmatrix}
        1 & 0.03 & 0.03 & 1 \\
        1 & 1    & 1    & 1
    \end{bmatrix} \Box_{\max}^{\min}  \begin{bmatrix}
    s'_1 \\
    r'_1 \\
    s'_2 \\ 
    r'_2 
\end{bmatrix}.\end{equation}

\noindent This equation system is inconsistent, since its Chebyshev distance is equal to 0.03. As we have $0.03 < \tau_2$, the training data sample is assessed as reliable with respect to $\tau_2$ (Definition \ref{def:reliabletau}). We construct a consistent equation system using Definition \ref{def:reduced-es}:
\begin{equation}\label{ex15:cons-eq-sys-4}
    \begin{bmatrix}
   0.03\\
   1
\end{bmatrix}
    =
    \begin{bmatrix}
         1 & 0.03 & 0.03 & 1 \\
        1 & 1    & 1    & 1
    \end{bmatrix}  \Box_{\max}^{\min} \begin{bmatrix}
    s'_1 \\
    r'_1 \\
    s'_2 \\ 
    r'_2 
    \end{bmatrix}.\end{equation}

    \noindent This time, the vector $\begin{bmatrix}
   0.03\\
   1
\end{bmatrix}$ in (\ref{ex15:cons-eq-sys-4}) is the lowest Chebyschev approximation of the second member 
$\begin{bmatrix}
   0\\
   1
\end{bmatrix}$ of the equation system (\ref{ex15:inc-eq-sys-4}) (again, a minimal update of the targeted output possibility distribution has been achieved), which allows us to obtain the consistent equation system  (\ref{ex15:cons-eq-sys-4}).

\noindent We process the third sample. 
 We compute the possibility degrees of the rule premises of the second set of rules using the inferred output possibility distribution of the first set of rules derived from the sample, see (\ref{eq:ex15:pi3_o}), and get: $\lambda_1^{'(3)} = 1$, $\rho_1^{'(3)} =1$, $\lambda_2^{'(3)} = 1$ and $\rho_2^{'(3)} = 1$. The equation system is constructed using the possibility degrees of the rule premises and the targeted output possibility distribution  of the sample associated with the output attribute $c$ of the second set of rules:
$$\begin{bmatrix}
   \Pi^{'(3)}_t( \{ 0 \}) \\
    \Pi^{'(3)}_t(\{ 1 \})
\end{bmatrix}=\begin{bmatrix}
    1\\
    0
\end{bmatrix}
    =
    \begin{bmatrix}
        1 & 1 & 1 & 1 \\
        1 & 1    & 1    & 1
    \end{bmatrix} \Box_{\max}^{\min}  \begin{bmatrix}
    s'_1 \\
    r'_1 \\
    s'_2 \\ 
    r'_2 
\end{bmatrix}.$$

\noindent This equation system is inconsistent, since its Chebyshev distance is equal to 1. As we have $1 > \tau_2$, the training data sample is assessed as \textit{not} reliable with respect to $\tau_2$ (Definition \ref{def:reliabletau}).   We can compute using (\ref{eq:lowest}) the lowest Chebyshev approximation of the second member of the equation system, which is the vector $\begin{bmatrix}
  1\\
  1
\end{bmatrix}$.\\
We process the fourth sample. We compute the possibility degrees of the rule premises of the second set of rules using the inferred output possibility distribution of the first set of rules derived from the sample, see (\ref{eq:ex15:pi4_o}). We get $\lambda_1^{'(4)} = 0.05$, $\rho_1^{'(4)} =1$, $\lambda_2^{'(4)} = 1$ and $\rho_2^{'(4)} = 0.05$. Using these possibility degrees of the rule premises and the targeted output possibility distribution  of the sample associated with the output attribute $c$ of the second set of rules, we construct the equation system:
$$\begin{bmatrix}
   \Pi^{'(4)}_t( \{ 0 \}) \\
    \Pi^{'(4)}_t(\{ 1 \})
\end{bmatrix}=\begin{bmatrix}
    0\\
    1
\end{bmatrix}
    =
    \begin{bmatrix}
        1    &   1    & 1    & 1 \\
        0.05 &   1    & 1    & 0.05
    \end{bmatrix} \Box_{\max}^{\min}  \begin{bmatrix}
    s'_1 \\
    r'_1 \\
    s'_2 \\ 
    r'_2 
\end{bmatrix}.$$

\noindent This equation system is inconsistent, since its Chebyshev distance is equal to 1. As we have $1 > \tau_2$, the training data sample is assessed as \textit{not} reliable with respect to $\tau_2$ (Definition \ref{def:reliabletau}).   We can compute using (\ref{eq:lowest}) the lowest Chebyshev approximation of the second member of the equation system, which is the vector $\begin{bmatrix}
  1\\
  0.05
\end{bmatrix}$.

\noindent Since two reliable training data sample are available (the first and second samples), we construct the equation system $(\mathbf{\Sigma}')$, see (\ref{eq:stackedES}), by stacking  their associated consistent   equation systems, see (\ref{ex15:cons-eq-sys-3}) and (\ref{ex15:cons-eq-sys-4}):
\begin{equation}
    \begin{bmatrix}
   1\\
   0.04\\
   0.03\\
   1
\end{bmatrix}
    =
    \begin{bmatrix}
        1    & 1 & 1 & 1 \\
        0.04 & 1 & 1 & 0.04\\
         1 & 0.03 & 0.03 & 1 \\
        1 & 1    & 1    & 1
    \end{bmatrix}  \Box_{\max}^{\min} \begin{bmatrix}
    s'_1 \\
    r'_1 \\
    s'_2 \\ 
    r'_2 
    \end{bmatrix}.
\end{equation}

\noindent The  equation system $(\mathbf{\Sigma'})$ is consistent. The lowest solution of the equation system $(\mathbf{\Sigma'})$ is the vector $\begin{bmatrix}
    0\\
    0\\
    0\\
    0
\end{bmatrix}$, and, consequently, the parameters of the second set of rules are set as follows: $s'_1 = r'_1 = s'_2 = r'_2 = 0$. Again, this shows that the rules are certain. We could check that he vector $\begin{bmatrix}0.04
\\
0.03\\
0.03\\
0.04
\end{bmatrix}$ is another solution of the system $(\mathbf{\Sigma'})$.

\noindent Using the parameters, the matrix relation of the second set of rules is built using (\ref{eq:ex:matrixrelation:secondsys}) of Example \ref{ex:matrixrelred}:
\begin{equation*}
\begin{bmatrix}
    \Pi( \{ 0 \}) \\
     \Pi( \{ 1 \}) 
\end{bmatrix}=
    \begin{bmatrix}
        1 & 0 & 0 & 1 \\
       0 & 1 & 1 & 0 
    \end{bmatrix} \Box_{\max}^{\min} \begin{bmatrix}
        \lambda'_1\\
        \rho'_1\\
        \lambda'_2\\
        \rho'_2
    \end{bmatrix}.
\end{equation*}
\end{example}

To be able to perform learning in cascade (Method \ref{meth:learningincascade}) in practice with very large datasets, it remains to determine how to set the values of the tolerance thresholds $\tau_1$ and $\tau_2$. This problem is addressed in the next subsection.

Notably, the learning method we have presented is intended to learn rule parameter values based on training data composed of input and output possibility distributions. If, instead, the rule parameter values are set and the goal is to determine targeted input possibility distributions from training data samples composed of output possibility distributions, the backpropagation mechanism presented in Subsection \ref{subsub:backpropagation} can be leveraged.

\subsection{Determining the values of the tolerance thresholds using a validation dataset}
\label{subsec:thresholdsdetermination}

To apply our learning paradigm to very large datasets, thresholds (Definition \ref{def:reliabletau}) can be exploited to select reliable training data samples. Indeed, Method \ref{meth:learningincascade}  uses two thresholds $\tau_1$ and $\tau_2$ for learning the values of the rule parameters of a cascade. 

\noindent We now show how the values of these two thresholds can be determined by an iterative process based on a validation dataset. Let $\epsilon$ be a real number $> 0$. We define the set $\mathcal{T}$ of candidate positive values for the thresholds as follows:
\begin{equation}\label{eq:setTvalidationstep}
     \mathcal{T} = \left\{ t_i = \left(\frac{i}{l}\right)^h \cdot (1 + \epsilon) \, \middle| \, i = 1, \ldots, l \right\}, 
\end{equation}
Here $l$ is a parameter denoting the total number of candidate values and $h \geq 1$ is another parameter used for controlling the number of values of $\mathcal{T}$ that are close to zero: the highest the value of $h$, the more there are values of $\mathcal{T}$ close to zero. Obviously, we have $1 + \epsilon \in \mathcal{T}$.

The definition of $\mathcal{T}$ aims to point out many different small values for the thresholds in order to find a large subset of training data that are considered highly reliable  (Definition \ref{def:reliabletau}). We remind that the lower the Chebyshev distance associated with an equation system $(\Upsigma)$ constructed with a training data sample, the higher the reliability of the training data sample used to build this equation system (Subsection \ref{subsec:poorqualtrainingdata}). If the values set for thresholds are too small, learning in cascade (Method \ref{meth:learningincascade}) may not be possible (no training data would be considered reliable enough). But values for the thresholds which make learning in cascade possible can always be found in the set $\mathcal{T}$. In fact, for a threshold value $t_i$, two possible situations may arise:

\begin{itemize} 
\item All equation systems constructed from training data samples have a Chebyshev distance $\nabla$ greater than or equal to $t_i$ (which implies that $i < l$). In this case, no training data sample is considered reliable with respect to $t_i$, and learning cannot be applied with $t_i$. 
\item Some equation systems constructed from training data samples have a Chebyshev distance $\nabla$ strictly below $t_i$. These training data samples are then considered as reliable with respect to $t_i$, allowing learning to be applied using them. 
\end{itemize}

The worst-case scenario occurs when each equation system constructed from a training data sample has a Chebyshev distance $\nabla$ equal to one. In this case, the highest threshold value $t_l = 1 + \epsilon$ ensures that all training data samples are selected, enabling learning to be performed.

\smallskip
For a cascade, the values of the thresholds $\tau_1$ and $\tau_2$ are determined by iteratively testing values from the set $\mathcal{T}$:
\begin{itemize}
    \item At start, $\tau_1$ and $\tau_2$ are set to the minimum possible value $t_1 \in \mathcal{T}$. 
At each iteration of the process, the value of $\tau_1$ and/or the value of $\tau_2$ is updated by selecting the next value in the ordered set $\mathcal{T}$, ensuring a gradual increase of the tested threshold values.
    \item For each pair of threshold values tested, we first check if it is possible to apply Method \ref{meth:learningincascade} with these thresholds. If this is the case, the accuracy $A(\tau_1,\tau_2)$ of the model built using Method \ref{meth:learningincascade} is evaluated with these fixed thresholds on the validation dataset.
    \item This iterative process stops when the accuracy has not been improved by at least a minimum improvement value 
    after a specified number of consecutive evaluations (called the stagnation parameter). When this condition is met, the thresholds are set to the lowest tested values that allow us to obtain the (stabilized) accuracy on the validation dataset.\\
    In our experiments, the metric used has been classification accuracy, but other metrics for evaluating the predictive performance of a model could be used instead, e.g., precision, recall, F-score, AUC.
\end{itemize}

This iterative approach is designed to avoid the inclusion of unreliable training data samples and to take into account computational constraints: the goal is to test as few values as possible for the thresholds (while expecting good accuracy) and, when dealing with large datasets, we want to avoid having stacked equation systems, see (\ref{eq:stackedES}), that are too large.

Of course, this method can easily be extended to find thresholds for cascades with more than two sets of 
possibilistic rules. Furthermore, if multiple sets of 
rules perform the same task, a common threshold can be used for them.

\section{Experiments}
\label{sec:exp}

We evaluated $\Pi$-NeSy on the following neuro-symbolic challenges:
\begin{itemize}
    \item MNIST Addition-$k$ problems \cite{manhaeve2018deepproblog} where $k \in \{1,2,4,15,100\}$: the goal is to determine the sum of two numbers of $k$ digits, where each digit is represented on one  MNIST image. For instance, if \( k = 2 \), given two pairs of images \(\includegraphics[width=0.017\textwidth]{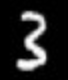}\), \(\includegraphics[width=0.017\textwidth]{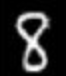}\) and \(\includegraphics[width=0.017\textwidth]{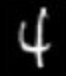}\), \(\includegraphics[width=0.017\textwidth]{./z_img__ONE_41266.png}\), the task is to form the numbers 38 and 41, then add them to predict the sum, which is 79.\\
    As $k$ increases, the complexity of the problem grows, given that the model must accurately recognize more digits from  MNIST images. Consequently, the probability of having at least one digit recognition error in one of the two numbers becomes higher for larger values of $k$. Therefore, handling higher values of $k$ significantly increases the overall complexity of the task.

\item 
MNIST Sudoku puzzle problems \cite{augustine2022visual}, where  Sudoku puzzles of size 4x4 or 9x9 are considered. Each 4x4 Sudoku puzzle is constructed using 16 MNIST images, while each 9x9 Sudoku puzzle uses 81 MNIST images. In the 4x4 Sudoku puzzles, the images represent digits between 0 and 3, and in the 9x9 Sudoku puzzles, the digits range from 0 to 8. The goal is to verify the correctness of the puzzle, ensuring there are no repeated digits in any row, column, or sub-grid.\\
For instance, given the visual Sudoku puzzles 4x4 in Figure \ref{fig:visualSudoku}, a model should predict that the Sudoku puzzle (a) is valid and that the Sudoku puzzle (b) is not.\\
The complexity of MNIST Sudoku puzzle problems also increases with the puzzle dimension: verifying the correctness of a 9x9 Sudoku requires accurately recognizing more digits and satisfying more constraints than for a 4x4 Sudoku.
\begin{figure}[H]
\centering
\begin{subfigure}{.5\textwidth}
  \centering
  \includegraphics[width=.2\linewidth]{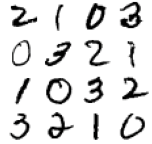}
  \caption{A valid Sudoku puzzle.}
  \label{fig:sub1}
\end{subfigure}%
\begin{subfigure}{.5\textwidth}
  \centering
  \includegraphics[width=.2\linewidth]{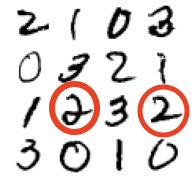}
  \caption{An invalid Sudoku puzzle.}
  \label{fig:sub2}
\end{subfigure}
\caption{ 4x4 Sudoku puzzles.}
\label{fig:visualSudoku}
\end{figure}

\end{itemize}

\paragraph*{Our approach $\Pi$-NeSy}

In $\Pi$-NeSy, for each problem, the low-level perception task consists of recognizing handwritten digits on images. It is performed by a neural network, which produces a probability distribution of the digits represented on the image.  The architecture of  $\Pi$-NeSy's neural network is given in Table \ref{tab:NN}. It includes four convolution layers and three fully connected layers. To transform the outputs of the neural network into a probability distribution, a softmax activation is applied to the output scores.

For the experiments, the following setup of  $\Pi$-NeSy's neural network has been considered. The Adadelta optimizer \cite{zeiler2012adadelta} has been used with the following hyperparameters:  learning rate of 1.0, adjusting it with a decay factor  of 0.7, over 20 epochs. A batch size of 64 was used for both the MNIST Addition-$k$ problems where $k \in \{1,2,4,15,100\}$ and the MNIST Sudoku 9x9 problem, while a batch size of 32 was applied for the MNIST Sudoku 4x4 problem.   Training the model 10 times (each time from randomly initialized weights) using the MNIST training dataset (60,000 items) and evaluating each model on the MNIST test dataset (10,000 items), an average test accuracy of $99.52 \pm 0.06\%$ has been obtained.

In $\Pi$-NeSy, the high-level reasoning task is  performed by a possibilistic rule-based system, which performs inferences based on possibility distributions obtained from probability-possibility transformations applied to the output probability distributions of the neural network. Two versions of $\Pi$-NeSy have been considered, depending on the probability-possibility transformation used (see Section \ref{sec:bg}): 
\begin{itemize}
    \item $\Pi$-NeSy-1 is based on the antipignistic method,
     \item $\Pi$-NeSy-2 uses the transformation method obeying the  minimum specificity principle.
\end{itemize}
In order to find threshold values for performing  learning in cascade (Method \ref{meth:learningincascade}) of the rule parameters of  $\Pi$-NeSy's possibilistic rule-based system, the set $\mathcal{T}$ of possible values for thresholds, see (\ref{eq:setTvalidationstep}), has been obtained based on the  setting $l = 30$, $h = 5$, and $\epsilon=0.001$. The threshold value search process in which this set is involved (see Subsection \ref{subsec:thresholdsdetermination}) is configured with the following setting: minimum improvement is set to $0.01$ and stagnation parameter is set to $1$.

\begin{table}[ht]
\centering
\footnotesize
\begin{tabular}{@{}cll@{}}
\toprule
\textbf{Order} & \textbf{Layer} & \textbf{Configuration} \\
\midrule
1 & \makecell[l]{Convolutional,\\ReLU, Batch Normalization} & \makecell[l]{Kernel Size: 5x5, Stride: 1, Padding: Same,\\Output Channels: 32, Bias: Yes} \\
\midrule
2 & \makecell[l]{Convolutional, \\Batch Normalization, Max Pooling, Dropout} & \makecell[l]{Kernel Size: 5x5, Stride: 1, Padding: Same,\\Output Channels: 32, Bias: No,\\Pooling Size: 2x2, Dropout: 0.25} \\
\midrule
3 & \makecell[l]{Convolutional,\\ReLU, Batch Normalization} & \makecell[l]{Kernel Size: 3x3, Stride: 1, Padding: Same,\\Output Channels: 64, Bias: Yes} \\
\midrule
4 & \makecell[l]{Convolutional,\\Batch Normalization, Max Pooling, Dropout} & \makecell[l]{Kernel Size: 3x3, Stride: 1, Padding: Same,\\Output Channels: 64, Bias: No,\\Pooling Size: 2x2, Dropout: 0.25} \\
\midrule
5 & \makecell[l]{Fully Connected,\\ReLU, Batch Normalization} & \makecell[l]{Input Shape: 3136, Output Shape: 256,\\Bias: No} \\
\midrule
6 & \makecell[l]{Fully Connected,\\ReLU, Batch Normalization} & \makecell[l]{Input Shape: 256, Output Shape: 128,\\Bias: No} \\
\midrule
7 & \makecell[l]{Fully Connected,\\ReLU, Batch Normalization, Dropout} & \makecell[l]{Input Shape: 128, Output Shape: 84,\\Bias: No, Dropout: 0.25} \\
\midrule
8 & Output & Output Shape: 10 \\
\bottomrule
\end{tabular}
\caption{Neural network architecture used for $\Pi$-NeSy experiments.}
\label{tab:NN}
\end{table}
The code of $\Pi$-NeSy  is shared publicly under Apache 2.0 license.\footnote{\url{https://github.com/ibaaj/pi-nesy}} This software relies on a Python library created with Pybind11 \cite{jakob2017pybind11}, which interfaces with a C++ program to perform possibilistic learning and inference computations via Apple Metal, OpenCL or CUDA for GPU acceleration, or multithreaded CPUs. For the experiments with $\Pi$-NeSy reported in this article, Apple Metal has been chosen. The experiments have been conducted on a MacBook Air M2 (16 GB of RAM) under macOS Sequoia 15.2, using its integrated GPU (8 cores). No timeout mechanism has been considered in the experiments.

\normalsize

\subsection{\texorpdfstring{MNIST Addition-$k$ problems}{MNIST Addition k problems}}
\label{subsec:exp:mnistadd}

We followed the methodology described in \cite{ manhaeve2021neural,pryor2023ijcai} to generate the MNIST Addition-$k$ datasets, and we used the same experimental setup as the one presented in this paper. Several neuro-symbolic approaches for solving the MNIST Addition-$k$ problem explicitly represent all possible sums, e.g., Logic Tensor Networks \cite{badreddine2022logic} or NeuPSL \cite{pryor2023ijcai}, resulting in exponential inference complexity with respect to the number $k$ of digits. In contrast, recent approaches like DeepSoftLog \cite{maene2024soft} and A-NeSI \cite{van2024nesi}, along with $\Pi$-NeSy, offer an alternative formulation that scales up linearly with $k$. In these three neuro-symbolic approaches, the digits of the sum and any potential carry are computed in an iterative way.

\subsection{Training, validation and test data}
\label{subsec:trainvaltest}

In our experiments,  three pairwise disjoint datasets have been considered: the training dataset $\mathcal{A}_{\text{train}}$, the validation dataset $\mathcal{A}_{\text{valid}}$, and the test dataset $\mathcal{A}_{\text{test}}$. Their sizes are denoted by $N_{\text{train}}$, $N_{\text{valid}}$, and $N_{\text{test}}$ respectively.

These datasets are generated from the MNIST dataset, which is composed of a set of 60,000 training images and a set of 10,000 test images. The generation process used is the same one as in \cite{pryor2023ijcai}:
\begin{itemize}
    \item \textbf{Training Dataset $\mathcal{A}_{\text{train}}$}: From the 60,000 MNIST training images, 50,000 images are randomly selected without replacement. We form $N_{\text{train}} = \left\lfloor \frac{50000}{2\cdot k} \right\rfloor$ pairs of tuples from these images. Each tuple, containing $k$ images, is formed by selecting images randomly without replacement from the pool of 50,000 images.
    \item \textbf{Validation Dataset $\mathcal{A}_{\text{valid}}$}: From the remaining 10,000 MNIST training images,  $N_{\text{valid}} = \left\lfloor \frac{10000}{2\cdot k} \right\rfloor$ pairs of tuples are created. Each tuple consists of $k$ images selected randomly without replacement.
    \item \textbf{Test Dataset $\mathcal{A}_{\text{test}}$}: $N_{\text{test}} = \left\lfloor \frac{10000}{2\cdot k} \right\rfloor$ pairs of tuples are formed from the MNIST test dataset. Just like for the other datasets, each tuple includes $k$ images randomly selected without replacement.
\end{itemize}

\medskip
Each example in these datasets consists of two tuples of images $\myvaremphasis{A} = (\myvaremphasis{m}_1, \myvaremphasis{m}_2, \ldots, \myvaremphasis{m}_k)$ and $\myvaremphasis{B} = (\myvaremphasis{m}_{k+1}, \myvaremphasis{m}_{k+2}, \ldots, \myvaremphasis{m}_{2k})$, where each tuple contains $k$ labeled images of handwritten digits forming a number. Each element $\myvaremphasis{m}_i = (\myvaremphasis{p}_i, \myvaremphasis{a}_i)$ includes an image $\myvaremphasis{p}_i$ and its associated label $\myvaremphasis{a}_i \in \{0, 1, \ldots, 9\}$. The number associated with $\myvaremphasis{A}$ is calculated as $\sum_{i=0}^{k-1} \myvaremphasis{a}_{k-i} \cdot 10^{i}$, and for $\myvaremphasis{B}$, it is $\sum_{i=0}^{k-1} \myvaremphasis{a}_{2k-i} \cdot 10^{i}$.

For each example, the goal is to calculate the sum of $\myvaremphasis{A}$ and $\myvaremphasis{B}$, resulting in a new number $\myvaremphasis{Y}$, which consists of $k+1$ digits represented as $(\myvaremphasis{y}_0, \myvaremphasis{y}_1, \ldots, \myvaremphasis{y}_k)$, where $\myvaremphasis{y}_0$ is the most significant digit if $\myvaremphasis{y}_0 \neq 0$.

\smallskip
For processing the training data with $\Pi$-NeSy, the following data are generated for each sample:
\begin{itemize}
    \item The tuple $\myvaremphasis{c}_k = (\myvaremphasis{a}_k, \myvaremphasis{a}_{2k})$ is formed and the carry $\myvaremphasis{w}_k$ is calculated: it is equal to 1 if $\myvaremphasis{a}_k + \myvaremphasis{a}_{2k} \geq 10$, otherwise it is equal to 0.
    \item For $i = k-1$ down to 1, we consider $\myvaremphasis{c}_i = (\myvaremphasis{a}_i, \myvaremphasis{a}_{k+i}, \myvaremphasis{w}_{i+1})$ using $\myvaremphasis{a}_i$ from $\myvaremphasis{A}$, $\myvaremphasis{a}_{k+i}$ from $\myvaremphasis{B}$, and the carry $\myvaremphasis{w}_{i+1}$.
    \item The carry $\myvaremphasis{w}_i$ is set to 1 if $\myvaremphasis{a}_i + \myvaremphasis{a}_{k+i} + \myvaremphasis{w}_{i+1} \geq 10$, otherwise it is set to 0.
    \item The $i$-th digit of the sum $\myvaremphasis{y}_i$ for $i = 1, 2, \ldots, k$ is calculated as the sum of the elements in $\myvaremphasis{c}_i$ modulo 10.
    \item We set $\myvaremphasis{y}_0 = \myvaremphasis{w}_1$.
\end{itemize}

\begin{example}
Suppose $k = 3$, and let:
\begin{itemize}
    \item $\myvaremphasis{A} = (\myvaremphasis{m}_1, \myvaremphasis{m}_2, \myvaremphasis{m}_3)$ where $\myvaremphasis{m}_1 = (\includegraphics[width=0.02\textwidth]{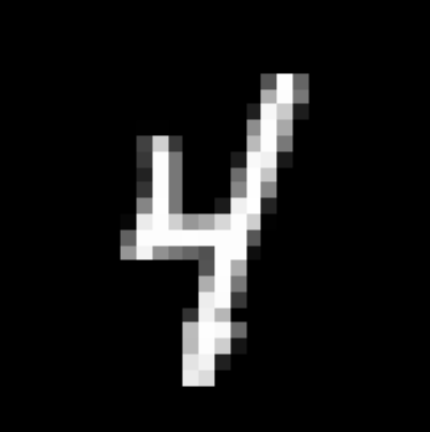}, 4)$, $\myvaremphasis{m}_2 = (\includegraphics[width=0.02\textwidth]{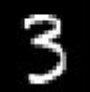}, 3)$, $\myvaremphasis{m}_3 = (\includegraphics[width=0.02\textwidth]{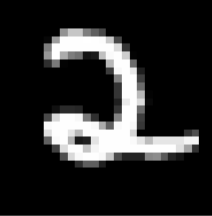}, 2)$, and
    \item $\myvaremphasis{B} = (\myvaremphasis{m}_4, \myvaremphasis{m}_5, \myvaremphasis{m}_6)$ where $\myvaremphasis{m}_4 = (\includegraphics[width=0.02\textwidth]{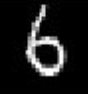}, 6)$, $\myvaremphasis{m}_5 = (\includegraphics[width=0.02\textwidth]{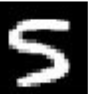}, 5)$, $\myvaremphasis{m}_6 = (\includegraphics[width=0.02\textwidth]{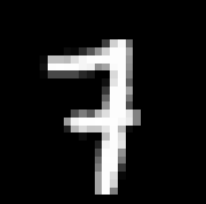}, 7)$.
\end{itemize}
\begin{itemize}
    \item We compute:
    \begin{itemize}
        \item $\myvaremphasis{c}_3 = (2, 7)$, $\myvaremphasis{w}_3= 0$ (since $2 + 7 = 9 < 10$), hence $\myvaremphasis{y}_3 = 9$.
        \item $\myvaremphasis{c}_2 = (3, 5, 0)$, $\myvaremphasis{w}_2 = 0$ (since $3 + 5 = 8 < 10$), hence $\myvaremphasis{y}_2 = 8$.
        \item $\myvaremphasis{c}_1 = (4, 6, 0)$, $\myvaremphasis{w}_1 = 1$ (since $4 + 6 + 0 = 10 \geq 10$), hence $\myvaremphasis{y}_1 = 0$.
    \end{itemize}
    \item $\myvaremphasis{y}_0 = \myvaremphasis{w}_1 = 1$.
\end{itemize}
The result $\myvaremphasis{Y}$ of the addition is $(1, 0, 8, 9)$.
\end{example}

\subsubsection{Possibilistic rule-based system for  \texorpdfstring{MNIST Addition-$k$ problems}{MNIST Addition k problems}}
\label{subsec:rulesformnistadd}
We now introduce a possibilistic rule-based system designed to address the MNIST Addition-$k$ problems. The system uses a set of attributes that mirror the structured data of each  training data sample in our datasets. 

Attributes $a_1, a_2, \ldots, a_k$ and $a_{k+1}, a_{k+2}, \ldots, a_{2k}$ represent the digits of two numbers to be added. Result digit attributes $y_0, y_1, \ldots, y_k$ correspond to the digits of the sum. Tuple attributes $c_1, c_2, \ldots, c_k$ and carry attributes $w_1, w_2, \ldots, w_k$ are used for the computation of each digit of the result.

\paragraph*{Attributes}
The following attributes are used to build the possibilistic rule-based system:

\begin{itemize}
    \item The attributes $a_1, a_2, \ldots, a_k$ are related to the first number, and the attributes $a_{k+1}, a_{k+2}, \ldots, a_{2k}$ are related to the second number. For each $i = 1,2,\ldots,2k$, the domain of the attribute $a_i$ is $D_{a_i} = \{0, 1, \ldots, 9\}$.
    \item The attributes $y_0, y_1, \ldots, y_k$ are related to the resulting number. We have $D_{y_0} = \{0,1\}$ and for $i=1,\cdots,k$, we have  $D_{y_i} = \{0, 1, \ldots, 9\}$.
    \item The attributes $c_1, c_2, \ldots, c_k$ are used for computing the digits of the resulting number. Their related domains are:
    \begin{itemize}
        \item $D_{c_k} = \{(u,v) \mid u, v \in \{0,1,\ldots,9\}\}$,
        \item $D_{c_1} = D_{c_2} = \ldots = D_{c_{k-1}} = \{(u,v,w) \mid u, v \in \{0,1,\ldots,9\}, w \in \{0,1\}\}$.
    \end{itemize}
    \item The attributes $w_1, w_2, \ldots, w_k$ are used for computing the carries involved in the sum. For $i=1,2,\cdots,k$, we have $D_{w_i} = \{0, 1\}$.
\end{itemize}

\paragraph*{Sets of rules}

To address the MNIST Addition-$k$ problems, we generate the following sets of rules, based on the attributes presented in the previous paragraph:

\begin{itemize} \item One set containing 20 rules, all using the output attribute $c_k$. \item $k - 1$ sets, each containing 21 rules. In each of these sets, all rules use the same output attribute $c_i$, where  $i \in \{1, 2, \dots, k - 1\}$. \item $k$ sets, each containing one rule. In each of these sets, the rule uses the output attribute $w_i$, where $i \in \{1, 2, \dots, k\}$. \item $k$ sets, each containing 10 rules. In each of these sets, all rules use the same output attribute $y_i$, where $i \in \{1, 2, \dots, k\}$. \item One additional set containing a single rule, which uses the output attribute $y_0$. \end{itemize}

\smallskip
In total, this results in $1 + (k - 1) + k + k + 1= 3 \cdot k + 1$ sets, and $20 + (k-1) \cdot 21 + k + 10 \cdot k + 1 = 32 \cdot k$ rules. 

\smallskip
In the following,  we describe each set of rules.

\smallskip
Each attribute $c_i$ is the output attribute of a set of possibilistic rules:

\begin{itemize}
    \item The set of rules related to $c_k$ is composed of 20 rules, which are generated using $a_k$ and $a_{2k}$ as follows:
    \begin{itemize}
        \item for each $j=0, 1, \cdots, 9$, the following rule, which involves the attribute $a_k$, is generated: 
        \begin{center}
            ``if $a_{k}(x) \in \{j\}$ then  $c_k(x) \in \{ (j,v) \mid v \in \{ 0, 1, \cdots, 9 \} \}$.''
        \end{center}
        \item for each $j=0, 1, \cdots, 9$, the following rule, which uses $a_{2k}$, is generated:
        \begin{center}
            ``if $a_{2k}(x) \in \{j\}$ then  $c_{k}(x) \in \{ (u,j) \mid u \in \{ 0, 1, \cdots, 9 \} \}$.''
        \end{center}
    \end{itemize}
    \item The set of rules related to $c_i$ for $i \in \{ 1, 2, \dots, k-1 \}$ is composed of 21 rules, which are generated using $a_i, a_{k+i}$ and $w_{i+1}$ as follows:
    \begin{itemize}
        \item for each $j=0, 1, \cdots, 9$, the following rule is generated: 
        \begin{center}
            ``if $a_{i}(x) \in \{j\}$ then  $c_i(x) \in \{ (j,v,w) \mid v \in \{ 0, 1, \cdots, 9 \}, w \in \{0,1\} \}$.''
        \end{center}
        \item for each $j=0, 1, \cdots, 9$, the following rule is generated: 
        \begin{center}
            ``if $a_{k+i}(x) \in \{j\}$ then  $c_{i}(x) \in \{  (u,j,w) \mid u \in \{ 0, 1, \cdots, 9 \}, w \in \{0,1\} \}$.''
        \end{center}
        \item we also add the following rule:
        \begin{center}
            ``if $w_{i+1}(x) \in \{0\}$ then  $c_{i}(x) \in \{  (u,v,0) \mid u,v \in \{ 0, 1, \cdots, 9 \} \}$.''
        \end{center}
    \end{itemize}
\end{itemize}

\smallskip
Each attribute  $w_i$ is the output attribute of a set composed of one rule:
\begin{itemize}
\item The set of rules which uses $w_k$ as output attribute contains the following rule, which uses attribute $c_k$:
    \begin{center}
        ``if $c_k(x) \in \{ (u,v) \mid u,v \in \{ 0,1,\dots,9\} \text{ s.t. } u+v \geq 10 \}$ then $w_k(x) \in \{ 1 \}$''
    \end{center}
    \item The set of rules whose output attribute is $w_i$ with $i \in \{ 1, 2, \cdots, k-1 \}$  is composed of the following  rule that relies on $c_i$:
    \begin{center}
        ``if $c_i(x) \in \{ (u,v,w) \mid u,v \in \{ 0,1,\dots,9\}, w \in \{0,1\} \text{ s.t. } u+v+w \geq 10 \}$ then $w_i(x) \in \{ 1 \}$''
    \end{center}
\end{itemize}

\smallskip
Each attribute $y_i$ is the output attribute of a set of possibilistic rules:
\begin{itemize}
 \item The set of rules which uses $y_k$ as output attribute is composed of 10 rules, which involve attribute $c_k$. For each $j = 0, 1, \cdots 9$, the following rule is generated:
    \begin{center}
        ``if $c_k(x) \in \{ (u,v) \mid u,v \in \{ 0,1,\cdots,9\} \text{ s.t. } (u+v)\mod 10 = j \}$ then $y_k(x) \in \{ j \}$''
    \end{center}
    \item The set of rules which uses $y_i$ with  $i \in \{ 1, 2, \dots, k-1 \}$  as output attribute is composed of  10 rules, which rely on attribute $c_i$. For each $j = 0, 1, \cdots 9$, the following rule is generated:
    \begin{center}
        ``if $c_i(x) \in \{ (u,v,w) \mid u,v \in \{ 0,1,\cdots,9\}, w\in \{0,1\} \text{ s.t. } (u+v+w) \mod 10 = j \}$ then $y_i(x) \in \{ j\}$''
    \end{center}
   \item The set of rules which uses $y_0$ as output attribute is composed of  the following rule which uses attribute $w_1$:
    \begin{center}
        ``if $w_1(x) \in \{ 0 \}$ then $y_0(x) \in \{ 0 \}$.''
    \end{center}
\end{itemize}

\smallskip
For each set of rules, we observe that the partition associated with the output attribute (see Section \ref{sec:practical-procedure-building-reduced-ES})  is composed of all possible singletons of the domain of this output attribute.

\paragraph*{Inference}

Given the possibility distributions $\pi_{a_1(x)}, \pi_{a_2(x)}, \cdots, \pi_{a_{2\cdot k}(x)}$, inference from the possibilistic rule-based system is achieved as follows:
\begin{enumerate}
    \item Inference from the set of rules whose output attribute is $c_k$ using $\pi_{a_k(x)}$ and $\pi_{a_{2k}(x)}$. The output possibility distribution $\pi^\ast_{c_k}(x)$ is obtained.
    \item Inference from the set of rules whose output attribute is $w_k$, which takes as input $\pi^\ast_{c_k}(x)$. The result is $\pi^\ast_{w_{k}(x)}$.
    \item For each $i = k-1$ down to $1$, the following steps are achieved to obtain successively $\pi^\ast_{c_i}(x)$ and then  $\pi^\ast_{w_i}(x)$:
    \begin{enumerate}
        \item Inference from the set of rules whose output attribute is $c_i$ using $\pi_{a_i(x)}$ and $\pi_{a_{k+i}(x)}$ and $\pi_{w_{i+1}^\ast(x)}$. The output possibility distribution $\pi^\ast_{c_i}(x)$ is obtained.
        \item Inference from the set of rules whose output attribute is $w_i$ using $\pi^\ast_{c_i(x)}$.The output possibility distribution $\pi^\ast_{w_i}(x)$ is obtained.
    \end{enumerate}
    \item For each $i=1,2,\cdots,k$, the next step is inference from the set of rules whose output attribute is $y_i$ using $\pi^\ast_{c_i(x)}$. The output attribute $\pi^\ast_{y_i}(x)$ is obtained.
    \item The final step is inference from the set of rules whose output attribute is $y_0$ using $\pi^\ast_{w_1(x)}$. The output attribute $\pi^\ast_{y_0}(x)$ is obtained.
\end{enumerate}

For each $i = 0, 1, \ldots, k$, from the output possibility distributions $\pi^\ast_{y_i}(x)$, the alternative (digit) with the highest possibility degree is retained to determine the $i$-th digit of the sum. Here, the index $i = 0$ corresponds to the most significant digit of the sum, while the index $i = k$ represents the least significant digit.

If, for a given $i \in \{0, 1, \ldots, k\}$, more than one alternative has the highest possibility degree in $\pi^\ast_{y_i}(x)$ (for instance, if two alternatives both have a possibility degree of 1), a misclassification (ambiguity) occurs.

\subsubsection{Learning}

\paragraph*{Neural learning} The neural network  has been trained with  the 50,000  MNIST image examples used to form the training dataset in $\mathcal{A}_{\text{train}}$. The neural network has then been tested with the 10,000 MNIST image examples used to form  $\mathcal{A}_{\text{test}}$.

\paragraph*{Possibilistic learning}

Possibilistic learning consists in finding values for the rule parameters of each set of rules according to training data. For this purpose, the examples in the MNIST Addition-$k$ training dataset $\mathcal{A}_{\text{train}}$ have been used. With each example are associated the following possibility distributions, which are generated using the data corresponding to the example, i.e.,  
$(\myvaremphasis{p}_1, \myvaremphasis{p}_2, \cdots, \myvaremphasis{p}_{2k}, \myvaremphasis{c}_1, \myvaremphasis{c}_2, \cdots,  \myvaremphasis{c}_k, \myvaremphasis{w}_1, \myvaremphasis{w}_2, \cdots,  \myvaremphasis{w}_k, \myvaremphasis{y}_0, \myvaremphasis{y}_1, \cdots,  \myvaremphasis{y}_k)$, see Section \ref{subsec:trainvaltest}:
\begin{itemize}
    \item For each MNIST image $\myvaremphasis{p}_i$ occurring in the two tuples of size $k$ composing the  training data sample, a possibility distribution is obtained by applying a probability-possibility transformation to the probability distribution produced by the neural network for predicting the image label. Therefore, for each training data sample in $\mathcal{A}_{\text{train}}$, the results of the neural network gives rise to $2k$ possibility distributions, each of which being related to one image.
    \item For each $i=1,2,\cdots,k$, for the set of rules whose output attribute is $c_i$ (resp. $w_i$), a targeted output possibility distribution is formed, which assigns a possibility degree equal to one to the tuple which is $\myvaremphasis{c}_i$ (resp. $\myvaremphasis{w}_i$)  while all other tuples have a possibility degree of zero. 
    \item For each $i=0,1,2,\cdots,k$, for the set of rules whose output attribute is $y_i$, a targeted output possibility distribution is formed, which assigns a possibility degree equal to one to the digit which is $\myvaremphasis{y}_i$  while all other digits have a possibility degree of zero. 
\end{itemize}
\noindent In total, for each  training data sample, we have $2k$ input  possibility distributions (associated with the attributes $a_1,a_2,\cdots,a_{2k}$) and $3k+1$ targeted output possibility distributions (associated with the output attributes $c_1,c_2,\cdots,c_k$, $w_1,w_2,\cdots,w_k$, $y_0,y_1,\cdots,y_k$ of the set of rules).

Possibilistic learning is performed in cascade using the  training data samples from $\mathcal{A}_{\text{train}}$: 
\begin{enumerate}
    \item For each $i$ from $k$ to $1$, the parameters for the set of rules whose output attribute is $c_i$ are first learned, and then those for the set of rules whose output attribute is  $w_i$.
    \item Finally, for each $i$ from $0$ to $k$, the parameters for the set of rules whose output attribute is $y_i$ are learned.
\end{enumerate}

The possibilistic learning phase relies on three thresholds $\tau_1,\tau_2$ and $\tau_3$. The first threshold $\tau_1$ is associated with the sets of rules whose output attribute is $c_i$ with $i \in \{1,2,\cdots,k\}$, since all these sets of rules are designed to perform the same task.  Similarly, the same threshold $\tau_2$ is used for all the sets of rules  whose output attribute is $w_i$ with $i \in \{1,2,\cdots,k\}$, and the same threshold $\tau_3$ is used for all the set of rules  whose output attribute is $y_i$ with $i \in \{0,1,\cdots,k\}$.
The approach given in Subsection \ref{subsec:thresholdsdetermination} is used to find the  values of these three thresholds based on the validation dataset $\mathcal{A}_{\text{valid}}$.

Learning in cascade (Method \ref{meth:learningincascade}) the rule parameters of the possibilistic rule-based system is performed, and then the matrix relations associated with all the sets of rules are constructed using the learned rule parameters.

\noindent \paragraph{Evaluation}

As expected, the evaluation is based on the MNIST Addition-$k$ test dataset $\mathcal{A}_{\text{test}}$. For each MNIST image in a given   training data sample of $\mathcal{A}_{\text{test}}$, the neural network produces a probability distribution for its label. These distributions are then transformed into possibility distributions, the latter being  used as input by the possibilistic rule-based system to perform inference using its matrix relations. For each  training data sample, we obtain $k+1$ output possibility distributions $\pi^\ast_{y_0}(x), \pi^\ast_{y_1}(x), \cdots, \pi^\ast_{y_k}(x)$.

In our experiments with  $\Pi$-NeSy, the correctness of a prediction has been assessed by verifying, for each $i = 0, 1, \ldots, k$, that the digit with the highest possibility degree in the output attribute distribution $\pi^\ast_{y_i}(x)$ matches the expected digit $\myvaremphasis{y}_i$ of the given  training data sample.

If, for any given $i$, there is more than one alternative (digit) with the highest possibility degree in $\pi^\ast_{y_i}(x)$ (for example, if two digits both have a possibility degree of 1), $\Pi$-NeSy's prediction is  viewed as incorrect (we consider this case of ambiguity to be a case of misclassification). Note that while we have taken such a case of misclassification into account in our empirical protocol, it did not occur in our experiments.

\subsubsection{Results}

\begin{table}[ht]
\centering
\setlength{\tabcolsep}{1.5pt}
\renewcommand{\arraystretch}{1.5}

\begin{tabular}{lcccccc}
\hline
Digits per number $k$ & 1 & 2 & 4 & 15 & 100\\  
\hline  
\makecell[l]{LTN\\ \cite{badreddine2022logic}} & 80.5 $\pm$ 23.3 & 77.5 $\pm$ 35.6 & timeout & - & - \\[2mm] 
\makecell[l]{NeuPSL\\ \cite{pryor2023ijcai}} & 97.3 $\pm$ 0.3 & 93.9 $\pm$ 0.4 & timeout & - & - \\[2mm] 
\makecell[l]{DeepProbLog\\ \cite{manhaeve2021neural}} & 97.2 $\pm$ 0.5 & 95.2 $\pm$ 1.7 & timeout & - & - \\[2mm] 
\makecell[l]{NeurASP\\ \cite{yang2023neurasp}} & 97.3 $\pm$ 0.3 & 93.9 $\pm$ 0.7 & timeout & - & - \\[2mm] 
\makecell[l]{DeepStochLog\\ \cite{winters2022deepstochlog}} & 97.9 $\pm$ 0.1 & 96.4 $\pm$ 0.1 & 92.7 $\pm$ 0.6 & timeout & - \\[2mm] 
\makecell[l]{Embed2Sym\\ \cite{aspis2022embed2sym}} & 97.6 $\pm$ 0.3 & 93.8 $\pm$ 1.4 & 91.7 $\pm$ 0.6 & 60.5 $\pm$ 20.4 & timeout \\[2mm] 
\makecell[l]{A-NeSI\\ \cite{van2024nesi}} & 97.7 $\pm$ 0.2 & 96.0 $\pm$ 0.4 & 92.6 $\pm$ 0.8 & 75.9 $\pm$ 2.2 & overflow \\[2mm] 
\makecell[l]{DeepSoftLog\\ \cite{maene2024soft}}  & 98.4 $\pm$ 0.1 & 96.6 $\pm$ 0.3 & 93.5 $\pm$ 0.6 & 77.1 $\pm$ 1.6 & 25.6 $\pm$ 3.4 \\[2mm] 
$\Pi$-NeSy-1 & \textbf{99.0 $\pm$ 0.1} & 97.9 $\pm$ 0.3 & \textbf{96.2 $\pm$ 0.5} & 86.6 $\pm$ 1.3 & 34.2 $\pm$ 4.6\\[2mm]
$\Pi$-NeSy-2 & 98.9 $\pm$ 0.2 & \textbf{98.0 $\pm$ 0.2} & 95.9 $\pm$ 0.4 & \textbf{87.1 $\pm$ 1.1} & \textbf{35.8 $\pm$ 5.5}\\[2mm]

\hline
\end{tabular}

\caption{Average accuracy and standard deviation on the test dataset for the MNIST Addition-$k$ problem over 10 runs (best results in bold).
}
\label{tab:mnist_add_results}
\end{table}

We present in Table \ref{tab:mnist_add_results} the accuracy obtained by $\Pi$-NeSy on MNIST Addition-$k$ problems (where $k \in \{1,2,4,15,100\}$), once averaged over 10 independent runs. For a fixed $k$, each of the ten runs consists of the following steps:
\begin{enumerate}
    \item Generate the three datasets: $\mathcal{A}_{\text{train}}$, $\mathcal{A}_{\text{valid}}$, and $\mathcal{A}_{\text{test}}$ from the MNIST dataset.
    \item  Train the neural networks using the MNIST images in the training data samples of $\mathcal{A}_{\text{train}}$.
    \item Apply the possibilistic learning method (Method \ref{meth:learningincascade}) using the training data samples in $\mathcal{A}_{\text{train}}$ and the  output probability distributions generated by the neural network from the MNIST images in each sample.
    \item Perform inference of the test samples in $\mathcal{A}_{\text{test}}$. For each test sample, the neural networks generate output probability distributions based on its MNIST images, and the possibilistic rule-based system performs high-level reasoning to produce the final sum, which is represented by possibility distributions.
    \item We measure accuracy based on the inference results of the test samples.
\end{enumerate}

This procedure is executed independently for each of the 10 runs, and the final performance metric (accuracy) is obtained by averaging the results.

For comparison purposes, we have included in  Table \ref{tab:mnist_add_results} the results obtained by the following approaches: Logic Tensor Networks (LTN) \cite{badreddine2022logic}, NeuPSL \cite{pryor2023ijcai}, DeepProbLog \cite{manhaeve2021neural}, NeurASP \cite{yang2023neurasp}, DeepStochLog \cite{winters2022deepstochlog}, Embed2Sym \cite{aspis2022embed2sym}, A-NeSI \cite{van2024nesi} and DeepSoftLog \cite{maene2024soft} using the same datasets generation method and evaluation protocol (these results were recently reported in \cite{maene2024soft,van2024nesi}).

From the results obtained, it appears that our $\Pi$-NeSy approach outperforms other neuro-symbolic methods on the MNIST Addition-$k$ problems (in particular, for $k=15$ and $k = 100$). The choice of probability-possibility transformation has little impact on performance, except for $k=15$ and $k=100$, where the method based on  minimum specificity principle ($\Pi$-NeSy-2) is slightly more effective than the antipignistic one ($\Pi$-NeSy-1). 
\begin{table}[ht]
\centering

\begin{tabular}{@{}cccccc@{}}
\toprule
k & Training set size & Test set size  & Approach & Learning time  & Inference time per test sample \\ \midrule
1 & 25000 & 5000 & $\Pi$-NeSy-1 & 1845.62 $\pm$ 162.41 & 0.002 $\pm$ 0.00 \\
1 & 25000 & 5000 & $\Pi$-NeSy-2 & 1868.79 $\pm$ 133.10 & 0.002 $\pm$ 0.00 \\
2 & 12500  & 2500 &  $\Pi$-NeSy-1 & 1510.01 $\pm$ 79.64 & 0.004 $\pm$ 0.00 \\
2 & 12500 & 2500 & $\Pi$-NeSy-2 & 1617.90 $\pm$ 97.70 & 0.004 $\pm$ 0.00 \\
4 & 6250 & 1250 & $\Pi$-NeSy-1 & 1396.77 $\pm$ 63.98 & 0.007 $\pm$ 0.00 \\
4 & 6250 & 1250 & $\Pi$-NeSy-2 & 1430.09 $\pm$ 48.23 & 0.007 $\pm$ 0.00 \\
15 & 1666 & 333 & $\Pi$-NeSy-1 & 1270.12 $\pm$ 27.50 & 0.027 $\pm$ 0.00 \\
15 & 1666 & 333 & $\Pi$-NeSy-2 & 1318.40 $\pm$ 28.99 & 0.028 $\pm$ 0.00 \\
100 & 250 & 50 & $\Pi$-NeSy-1 & 1211.84 $\pm$ 18.94 & 0.181 $\pm$ 0.00 \\
100 & 250 & 50 & $\Pi$-NeSy-2 & 1266.96 $\pm$ 17.70 & 0.180 $\pm$ 0.00 \\ \bottomrule
\end{tabular}

\caption{Each row in the table represents a MNIST Addition-$k$ problem addressed using either $\Pi$-NeSy-1 or $\Pi$-NeSy-2.
Each row includes the average learning time in seconds (averaged over 10 runs) and the average inference time per test sample in seconds (averaged over $10 \cdot N_{\text{test}}$ test samples)}.
The processing times for both approaches are similar and small enough. As $k$ increases, the learning time decreases because the possibilistic learning method selects fewer reliable training data (see Section \ref{subsec:appendix:posslearn:mnistadd}).
\label{tab:times_mnistadd}
\end{table}

Empirically, using the specified configuration (as detailed in the beginning of this section), the inference and learning times of $\Pi$-NeSy turned out to be reasonable  for MNIST Addition-$k$ problems, see Table \ref{tab:times_mnistadd}. 

In Section \ref{subsec:appendix:posslearn:mnistadd}, we give a detailed analysis of the empirical results presented in Table \ref{tab:mnist_add_results}. In this analysis, we show that when performing possibilistic learning  using $\Pi$-NeSy-1 and $\Pi$-NeSy-2 on the MNIST Addition-$k$ problems where $k \in \{1,2,4,15,100\}$, the thresholds used by both approaches are always close to zero for these problems and have the same value. We also show that possibilistic learning always yields values close to zero for the rule parameters, and that a small amount of training data is considered reliable.

\subsubsection{Other results}
\label{subsec:mnistaddotherresults}

\phantomsection
\paragraph{Ablation study}

We carried out an ablation study using $\Pi$-NeSy. We performed the MNIST Addition-$k$ problems where $k \in \{1,2,4,15,100\}$ using the same settings but without possibilistic learning:   we do not use thresholds to select training data samples (in practice, setting each threshold to a value strictly greater than 1 is equivalent to not using a threshold), and we manually set all rule parameters of possibilistic rule-based systems to zero (since, in MNIST Addition-$k$ problems, the rules are assumed to be certain). As we operated with the same fixed random seeds as those used in the experiments reported in Table \ref{tab:mnist_add_results} and with deterministic pytorch operations, we obtained the same results as those reported in Table \ref{tab:mnist_add_results}. Note by the way that possibilistic learning does not require the use of random numbers. 

\paragraph{Low data settings}

\begin{table}[ht]
\centering
\setlength{\tabcolsep}{1.5pt}
\renewcommand{\arraystretch}{1.5}
\begin{tabular}{lccccc}
\hline
Digits per number $k$ & 1 & 2 & 4 & 15 & 100\\
\hline
$\Pi$-NeSy-1 & 93.5 $\pm$ 2.9 & 86.3 $\pm$ 5.2 & 80.2 $\pm$ 7.0 & 39.3 $\pm$ 16.2 & 1.0 $\pm$ 1.6\\[2mm]
$\Pi$-NeSy-2 & 93.9 $\pm$ 2.6 & 89.6 $\pm$ 3.1 & 80.6 $\pm$ 5.0 & 40.6 $\pm$ 12.4 & 0.4 $\pm$ 1.2\\[2mm]
\hline
\end{tabular}
\caption{Average accuracy and standard deviation on the test dataset for the MNIST Addition-$k$ problem over 10 runs in low data settings.}
\label{tab:mnist_add_results_low_data_settings}
\end{table}

We also addressed MNIST Addition-$k$ problems in the context of low data settings. In each experiment, we randomly selected 2,500 images without replacement from the MNIST training dataset (instead of 50,000), to form $N_{\text{train}} = \left\lfloor \frac{2500}{2\cdot k} \right\rfloor$ pairs of tuples from these images in  the training dataset $\mathcal{A}_{\text{train}}$, while the validation dataset and the test dataset are generated using the standard settings. 
In Table \ref{tab:mnist_add_results_low_data_settings} we present the average accuracy results we obtained based on 10 independent runs. Interestingly, our approach $\Pi$-NeSy shows quite robust in low data settings. That said, as $k$ increases, task complexity becomes more apparent: for $k = 15$, we observe very high variability, and for $k=100$, $\Pi$-NeSy is unable to address the problem.
As expected, results are worse than under normal conditions.

In low data settings, we encountered a few difficulties in comparing floats, linked to the thresholds used in Method \ref{meth:learningincascade}.
To overcome this problem, we modified the parameters for generating the set $\mathcal{T}$ of candidate thresholds, see (\ref{eq:setTvalidationstep}), by setting $h = 2$, which generates candidate thresholds with higher values (for other experiments, $h = 5$ has been used).

In Subsection \ref{subsec:appendix:lds:posslearn:mnistadd}, an in-depth analysis of the characteristics and performance of our possibilistic learning method in these experiments  is reported.

\paragraph{Using DeepSoftLog convolutional neural network}

We also addressed the MNIST Addition-$k$ problems using DeepSoftLog convolutional neural network (CNN) \cite{maene2024soft} for digit recognition on MNIST images and  our possibilistic rule-based system (Subsection \ref{subsec:rulesformnistadd}) for the high-level reasoning task. The network uses the LeNet architecture (Table \ref{tab:NN_DeepSoftLog}), is optimized with the AdamW optimizer, and applies a cosine annealing learning rate schedule along with cross-entropy loss. 

We trained 10 independent instances of the CNN on the entire MNIST training dataset (60,000 samples), and evaluate their performance on the standard MNIST test dataset (10,000 samples) \footnote{DeepSoftLog uses the complete MNIST training dataset (60,000 samples) to train its neural networks for MNIST Addition-$k$ problems; see \url{https://github.com/jjcmoon/DeepSoftLog/blob/87765eda25dc12b5faf6dbda71aa9fea3a0b2ac0/deepsoftlog/experiments/mnist_addition/dataset.py}
} and on a validation dataset consisting of 10,000 samples randomly taken from the MNIST training dataset and augmented with small perturbations (small rotations, slight translations, and minor scaling) to simulate natural variations.

The CNN hyperparameters were manually fixed according to the results obtained on the validation dataset  (Table \ref{tab:mnist_add_hyperparams_DeepSoftLog}). For these 10 instances, the average validation accuracy was $99.82 \pm 0.05\%$, and the average test accuracy was $99.20 \pm 0.11\%$.  For tackling the MNIST Addition-$k$ problems, the CNN instance with the highest validation accuracy was chosen. This instance achieves $99.89\%$ on the validation dataset and $99.35\%$ on the test dataset \footnote{We were unable to achieve such high performance levels using the same data splits (Subsection \ref{subsec:trainvaltest}) used for training our neural network (In our experiments reported in Table \ref{tab:mnist_add_results}, our CNN (Table \ref{tab:NN}) is trained on a training dataset of 50,000 MNIST training samples taken without replacement from the MNIST training dataset and a validation dataset composed of the remaining 10,000 samples in the MNIST training dataset is used)}.

\begin{table}[ht]
\centering
\footnotesize
\begin{tabular}{@{}cll@{}}
\toprule
\textbf{Order} & \textbf{Layer} & \textbf{Configuration} \\
\midrule
1 & \makecell[l]{Convolutional, Max Pooling, ReLU} 
  & \makecell[l]{
    Kernel Size: 5x5, Stride: 1, Padding: 0,\\
    Input Channels: 1, Output Channels: 6, Bias: Yes,\\
    Pooling: 2x2 (Stride: 2) \\
    \textit{(Input: 1$\times$28$\times$28 $\rightarrow$ Output: 6$\times$12$\times$12)}
  } \\
\midrule
2 & \makecell[l]{Convolutional, Max Pooling, ReLU} 
  & \makecell[l]{
    Kernel Size: 5x5, Stride: 1, Padding: 0,\\
    Input Channels: 6, Output Channels: 16, Bias: Yes,\\
    Pooling: 2x2 (Stride: 2) \\
    \textit{(Input: 6$\times$12$\times$12 $\rightarrow$ Output: 16$\times$4$\times$4)}
  } \\
\midrule
3 & \makecell[l]{Fully Connected, ReLU} 
  & \makecell[l]{
    Input Shape: 16$\times$4$\times$4 = 256,\\
    Output Shape: 120, Bias: Yes
  } \\
\midrule
4 & \makecell[l]{Fully Connected, ReLU} 
  & \makecell[l]{
    Input Shape: 120,\\
    Output Shape: 84, Bias: Yes
  } \\
\midrule
5 & \makecell[l]{Output (Fully Connected)} 
  & \makecell[l]{
    Input Shape: 84,\\
    Output Shape: 10, Bias: Yes
  } \\
\bottomrule
\end{tabular}
\caption{Architecture of the convolutional neural network  of DeepSoftLog \cite{maene2024soft} for the low-level perception task.}
\label{tab:NN_DeepSoftLog}
\end{table}

\begin{table}[ht]
\centering
\footnotesize
\begin{tabular}{@{}ll@{}}
\toprule
\textbf{Hyperparameter}          & \textbf{Value} \\ \midrule
Batch Size                       & 128            \\
Training Epochs                  & 130            \\
Learning Rate                    & 0.001          \\
Weight Decay                     & 0.0001         \\
Cosine Schedule $T_{\max}$       & 100            \\
Cosine Schedule $\eta_{\min}$    & 0.00001        \\ \bottomrule
\end{tabular}
\caption{Hyperparameters used for training DeepSoftLog CNN.}
\label{tab:mnist_add_hyperparams_DeepSoftLog}
\end{table}

\noindent We also considered the MNIST Addition-$k$ problems similarly as before, except that we slightly modified the dataset generation and reuse the previously trained convolutional neural network rather than re-training it at each run. At each run, the training dataset $\mathcal{A}_{\text{train}}$, the validation dataset $\mathcal{A}_{\text{valid}}$, and the test dataset $\mathcal{A}_{\text{test}}$ are generated following the procedure outlined in Subsection~\ref{subsec:trainvaltest}. Specifically, $\mathcal{A}_{\text{train}}$ is derived from the complete MNIST training dataset, yielding $N_{\text{train}} = \left\lfloor \frac{60000}{2\cdot k} \right\rfloor$ examples (instead of  $\left\lfloor \frac{50000}{2\cdot k} \right\rfloor$ previously); the validation dataset $\mathcal{A}_{\text{valid}}$ is generated using the set of 10,000 MNIST training images slightly altered  previously (used for selecting the best-performing CNN instance); and the test dataset $\mathcal{A}_{\text{test}}$ is constructed using images from the MNIST test dataset.

For each MNIST Addition-$k$ problem, the final test accuracy is calculated as the average of the test accuracies obtained for each of the 10 runs. The results are presented in Table \ref{tab:deepsoftlog_cnn_mnist_add_results}.

\begin{table}[ht]
\centering
\setlength{\tabcolsep}{1.5pt}
\renewcommand{\arraystretch}{1.5}
\begin{tabular}{lccccc}
\hline
Digits per number $k$ & 1 & 2 & 4 & 15 & 100\\
\hline
\makecell[l]{DeepSoftLog\\ \cite{maene2024soft}}  & 98.4 $\pm$ 0.1 & 96.6 $\pm$ 0.3 & 93.5 $\pm$ 0.6 & 77.1 $\pm$ 1.6 & 25.6 $\pm$ 3.4 \\[2mm] 
$\Pi$-NeSy-1 & \textbf{98.7 $\pm$ 0.01} & 97.4 $\pm$ 0.04 & 94.9 $\pm$ 0.09 & \textbf{82.0 $\pm$ 0.40} & \textbf{29.4 $\pm$ 2.84}\\[2mm]
$\Pi$-NeSy-2 & 98.7 $\pm$ 0.02 & \textbf{97.4 $\pm$ 0.03} & \textbf{95.0 $\pm$ 0.08} & 81.9 $\pm$ 0.47 & 26.8 $\pm$ 4.49\\[2mm]
\hline
\end{tabular}
\caption{Average accuracy and standard deviation on the test dataset for the MNIST Addition-$k$ problem over 10 runs with $\Pi$-NeSy using DeepSoftLog's CNN. Best results in bold.}
\label{tab:deepsoftlog_cnn_mnist_add_results}
\end{table}

For these settings, $\Pi$-NeSy-1 and $\Pi$-NeSy-2 exhibit better performance than DeepSoftLog for all MNIST Addition-$k$ problems. In the most challenging case, i.e., when $k=100$, $\Pi$-NeSy-1 is slightly better than $\Pi$-NeSy-2: for the latter, we obtained two test accuracies of 16.0 and 22.0, which lowers its final average test accuracy and substantially increases its variability.
\noindent These results reveal an interesting insight: for MNIST Addition-$k$ problems, a combination of a convolutional neural network (here, the one used by DeepSoftLog) and a  rule-based system outperforms the entire DeepSoftLog system.

In Section \ref{sec:appendix:posslearn:withDSL}, we provide an in-depth analysis of the results obtained by our possibilistic learning method in these experiments. Especially, it turns out that for these experiments, almost all training data samples were considered reliable.

\subsection{MNIST Sudoku puzzle problems}

We also tackled the MNIST Sudoku puzzles problem \cite{augustine2022visual}, in which the Sudoku puzzles use MNIST images to represent digits: a 4x4 Sudoku puzzle includes 16 MNIST images where each image represents the digit 0,1,2, or 3, while a 9x9 Sudoku puzzle uses 81 MNIST images that can represent the digits 0,1, ... or 8. 

Similarly to the neuro-symbolic approaches A-NeSI \cite{van2024nesi} and DeepSoftLog \cite{maene2024soft}, the Sudoku generator described in \cite{augustine2022visual} has been used for generating datasets. The same experimental protocol as in \cite{maene2024soft,van2024nesi} has been considered. 

\subsubsection{Possibilistic rule-based systems for MNIST Sudoku puzzle problems}
\label{subsec:exp:sudoku}

\noindent For the sake of brevity, we only present the possibilistic rule-based system designed for addressing 4x4 MNIST Sudoku puzzles. Our approach is based on comparing pairs of MNIST image digits that must be labeled differently (these comparisons are called ``small constraints'' in \cite{demoen2014redundant}). The rule-based system for addressing 4x4 MNIST Sudoku puzzles can be straightforwardly adapted to deal with 9x9 MNIST Sudoku puzzles.

The possibilistic rule-based system for addressing MNIST Sudoku 4x4 puzzles problem consists of a specific cascade: fifty-six sets of  eight rules are chained to one set with a single rule. In total, we have 449 rules which are formed using 73 attributes: 
\begin{itemize}
    \item 16 input attributes $a_{ij}$ where $i,j \in \{1,2,3,4\}$, such that each attribute $a_{ij}$ corresponds to the handwritten digit at coordinate $(i,j)$ in the Sudoku puzzle and its domain is $D_{a_{ij}} = \{ 0, 1, 2, 3\}$,
    \item 56 attributes denoted by ${b_{(i,j,i',j')}}$ to compare each pair of attributes $a_{ij}, a_{i'j'}$ which must be labeled differently (in rows, columns and sub-grids). The domain of each attribute ${b_{(i,j,i',j')}}$ is  $D_{b_{(i,j,i',j')}} = \{ (u,v) \mid u,v \in \{ 0,1,2,3\} \}$.
    \item 1 attribute $c$, whose domain is $D_c = \{0,1\}$, which assesses whether the Sudoku is valid (1) or not (0).
\end{itemize}

A set $F$ of constraints is generated where each constraint is a 4-tuple of the form $(i,j,i',j')$ such that if $(i,j,i',j') \in F$ then the two attributes  $a_{i,j},a_{i',j'}$  must be labeled differently in order to ensure that the Sudoku is valid. The set $F$ is constructed by combining tuples from the rows, columns and sub-grids of the Sudoku, as follows:
\begin{align*}
    F = &\, \bigcup_{i = 1}^{4} \left\{ (i,j,i,j') \,\middle|\, j,j' \in \{1,2,3,4\},\, j' > j \right\} \\
        &\, \cup \bigcup_{j = 1}^{4} \left\{ (i,j,i',j) \,\middle|\, i,i' \in \{1,2,3,4\},\, i' > i \right\} \nonumber\\ 
        &\, \cup \bigcup_{S = 1}^{4} \left\{ (R(S,i),C(S,i),R(S,i'),C(S,i')) \,\middle|\, i,i' \in \{1,2,3,4\},\, i' > i \right\} \nonumber
\end{align*}
\noindent where for a sub-grid of index $S \in \{1,2,3,4\}$ (by convention $1,2,3,4$ are the top left, top right, bottom left, bottom right sub-grids respectively), the following two functions give respectively the row and the column in the coordinate of the $i$-th element ($i \in \{1,2,3,4\}$): 
\begin{align*}
    R(S, i) &= \left( \left\lfloor \frac{S - 1}{2} \right\rfloor \cdot 2 \right) + \left\lfloor \frac{i - 1}{2} \right\rfloor + 1,
\\
    C(S, i) &= \left( (S - 1) \mod 2 \right) \cdot 2 + ((i - 1) \mod 2) + 1. 
\end{align*}
The set $F$ is composed of 56 tuples. For each 4-tuple $(i,j,i',j') \in F$, a set of eight possibilistic rules is built as follows. Let $(i,j,i',j') \in F$ be fixed, for each $k=0,1,2,3$, two rules are generated:
\begin{itemize}
    \item ``If ${a_{ij}}(x) \in \{ k \}$ then ${b_{(i,j,i',j')}}{(x)} \in \{ (k,l) \mid l \in \{0,1,2,3\}\}$'',
    \item ``If ${a_{i'j'}}(x) \in \{ k \}$ then ${b_{(i,j,i',j')}}{(x)} \in \{ (l,k) \mid l \in \{0,1,2,3\}\}$''.
\end{itemize}

\smallskip
For each $(i,j,i',j') \in F$, the output possibility distribution $\pi^\ast_{b_{(i,j,i',j')}(x)} $  is exploited to determine whether the two handwritten digits corresponding to $a_{i,j},a_{i',j'}$ are labeled differently or not. The 56 sets of  possibilistic rules are chained to one set of one rule,  which 
is used to determine 
whether the Sudoku is valid. The premise of the rule is the conjunction of all 56 propositions  of the form ``$b_{(i,j,i',j')}(x)  \in \{ (k,l) \mid k,l \in \{0,1,2,3\}, \text{ s.t. } k \neq l \}$''  for  all $(i,j,i',j') \in F$  and the conclusion is ``$c(x) \in \{1\}$''. 

The possibilistic rule-based system for addressing  MNIST Sudoku puzzles of size 9x9 can be constructed similarly. It uses 81 input attributes $a_{i,j}$ (where $i, j \in \{1, 2, \ldots, 9\}$) to represent Sudoku cells, each having a domain $D_{a_{i,j}} = \{0, 1, \ldots, 8\}$, and 810 attributes $b_{(i,j,i',j')}$ to compare pairs of Sudoku cells that must differ to satisfy Sudoku rules (there are 810 small constraints in a 9x9 Sudoku puzzle, see \cite{demoen2014redundant}), with a domain $D_{b_{(i,j,i',j')}} = \{ (u,v) \mid u, v \in \{0, 1, \ldots, 8\}\}$. A set of rules where the output attribute \(c\) indicates whether the Sudoku is valid or not is generated. This set of rules is composed of a single rule, where its premise is the conjunction of 810 propositions of the form ``$b_{(i,j,i',j')}(x)  \in \{ (k,l) \mid k,l \in \{0,1,\cdots,8\}, \text{ s.t. } k \neq l \}$'' and its conclusion is ``$c(x) \in \{1\}$''.

\subsubsection{Protocol}

\paragraph*{Training, validation and test data}

As with A-NeSI \cite{van2024nesi} and DeepSoftLog \cite{maene2024soft}, the Sudoku generator presented in \cite{augustine2022visual} is used to construct ten data splits with the following parameters: a corruption chance of 0.50 and no overlap (each MNIST image is used only once). Each data split includes a training dataset of 200 Sudoku puzzles (100 correct, 100 incorrect), a validation dataset and a test dataset of 100 Sudoku puzzles each (50 correct, 50 incorrect).

Such data splits are generated for Sudoku puzzles of size 4x4 and 9x9. We seamlessly integrate the three datasets (train dataset $\mathcal{A}_{\text{train}}$, validation dataset $\mathcal{A}_{\text{validation}}$ and test dataset $\mathcal{A}_{\text{test}}$) that are produced by the Sudoku puzzle generator of \cite{augustine2022visual}  with $\Pi$-NeSy. 

Each  training data sample in these three datasets consists of a 4x4 (resp. 9x9) Sudoku puzzle composed of 16 (resp. 81) MNIST images of a handwritten digit, the true label of each MNIST image, and the label of the Sudoku puzzle which states if the Sudoku puzzle is valid or not.

\paragraph*{Neural learning}
As for MNIST Addition-$k$ problems, neural learning is done first using all the images contained in the Sudoku puzzles of the training dataset $\mathcal{A}_{\text{train}}$. For testing, all the images contained in the Sudoku puzzles of $\mathcal{A}_{\text{test}}$ are used.

\paragraph*{Possibilistic learning}
For performing possibilistic learning, each Sudoku puzzle in the MNIST Sudoku training dataset $\mathcal{A}_{\text{train}}$ is associated with the following items:
\begin{itemize}
    \item For each MNIST image in the Sudoku puzzle, a possibility distribution is obtained by applying a probability-possibility transformation of the probability distribution produced by the neural network with each MNIST image.
    \item With each set of rules associated with a tuple $(i,j,i',j') \in F$,  a  targeted output possibility distribution of the output attribute $b_{(i,j,i',j')}$ is associated; it assigns a possibility degree equal to one to the tuple $(k,l)$, where $k$ and $l$ are the true labels of the two MNIST images of coordinate $(i,j)$ and $(i',j')$ respectively, while all other tuples have a possibility degree of zero. 
    \item A targeted output possibility distribution of the output attribute $c$ of the set of the rules at the end of the chain in the cascade is generated, it assigns:
    \begin{itemize}
        \item a possibility degree equal to one to the output value $1$ if the Sudoku is valid, otherwise a possibility degree equal to zero, 
        \item a possibility degree equal to zero to the output value $0$ if the Sudoku is valid, otherwise a possibility degree equal to one. 
    \end{itemize} 
\end{itemize}
In the cascade,  all the sets of rules associated with the sets in $F$ are designed for  performing the same task. Therefore, for all these sets of rules, the same threshold $\tau_1$ is used. For the set of rules at the end of the chain, a second threshold $\tau_2$ is used. The search for threshold values can be carried out in a similar way to a classical cascade using the approach given in Subsection \ref{subsec:thresholdsdetermination}.

Finally,  learning in cascade is achieved (Method \ref{meth:learningincascade}) and the matrix relations associated with each set of rules built with the learned rule parameters are finally obtained.

\paragraph*{Evaluation}
In each experiment, the accuracy of $\Pi$-NeSy is evaluated using the  MNIST Sudoku  test dataset $\mathcal{A}_{\text{test}}$ in the same way as done for the MNIST Addition-$k$ problem. For each Sudoku puzzle of $\mathcal{A}_{\text{test}}$, $\Pi$-NeSy produces an output possibility distribution $\pi^\ast_{c}(x)$. The alternative with the highest degree of possibility in $\pi^\ast_{c}(x)$ is selected. Given a Sudoku puzzle, $\Pi$-NeSy's prediction is valid if the chosen alternative is the number 1 and the Sudoku puzzle is correct, or if the chosen alternative is the number 0 and the Sudoku puzzle is incorrect. Similarly to what has been done for the MNIST Addition-$k$ problem, we consider as a case of misclassification the situation when the two alternatives 0 and 1 in $\pi^\ast_{c}(x)$ have an equal possibility degree (ambiguity). However, such a scenario has not been encountered in our experiments.

\subsubsection{Results}
\label{subsub:mnistsudokuoresults}
The  results of $\Pi$-NeSy on MNIST Sudoku puzzles are reported in Table \ref{tab:visual_Sudoku_classification}. We proceed similarly as with the MNIST Addition-$k$ problem: for the MNIST Sudoku 4x4 (resp. 9x9) problem, we conduct 10 independent runs based on the 10 generated data splits. In each run, the neural network has been trained; possibilistic learning has been carried out using the training set $\mathcal{A}_{\text{train}}$, and subsequently the accuracy of the model has been evaluated based on the test puzzles from $\mathcal{A}_{\text{test}}$. Finally, the reported results are the average of the test accuracies obtained over the ten independent runs. For comparison purposes, we have included the results reported in \cite{maene2024soft,van2024nesi} and obtained by a CNN baseline, NeuPSL, A-NeSI and DeepSoftLog following the same experimental protocol. 
\begin{table}[ht]
\centering

\renewcommand{\arraystretch}{1.5}
\begin{tabular}{lcc}
\hline
{Approach} & {4 x 4  Sudoku}   & {9 x 9 Sudoku}   \\
\hline
CNN & 51.5 $\pm$ 3.34 &51.2 $\pm$ 2.20\\
NeuPSL \cite{pryor2023ijcai}       & 89.7 $\pm$ 2.20  & 51.5 $\pm$ 1.37  \\
A-NeSI \cite{van2024nesi}       & 89.8 $\pm$ 2.08  & 62.3 $\pm$ 2.20  \\
DeepSoftLog \cite{maene2024soft}           & 94.2 $\pm$ 1.84  & 65.0 $\pm$ 1.94  \\
$\Pi$-NeSy-1 & \textbf{96.8 $\pm$ 1.78} & 73.3 $\pm$ 6.03\\
$\Pi$-NeSy-2 & 96.1 $\pm$ 1.92 & \textbf{74.1 $\pm$ 4.50}\\
\hline
\end{tabular}
\caption{Accuracy and standard deviation on the test dataset for the MNIST Sudoku problem over 10 runs (best results in bold). Results for a CNN baseline, NeuPSL, A-NeSI and DeepSoftLog were reported in  \cite{maene2024soft,van2024nesi}.}
\label{tab:visual_Sudoku_classification}
\end{table}

In the 4x4 configuration, $\Pi$-NeSy achieves slightly better performances than DeepSoftLog and exhibits superior performance than NeuPSL, A-NeSI, and the CNN baseline. In the 9x9 configuration, $\Pi$-NeSy significantly outperforms all other methods. The choice of probability-possibility transformation has little impact on performance, but we note that $\Pi$-NeSy-1 is better than $\Pi$-NeSy-2 for the 4x4 configuration, while $\Pi$-NeSy-2 is better than $\Pi$-NeSy-1 for the 9x9 configuration.

The good results obtained by $\Pi$-NeSy are partly due to the performance of the neural network of $\Pi$-NeSy. For each experiment, the neural network is trained using MNIST images of the Sudoku puzzles in the training dataset $\mathcal{A}_{\text{train}}$, which includes 3200 MNIST images for the 4x4 configuration and 16200 MNIST images for the 9x9 configuration. For testing, MNIST images of $\mathcal{A}_{\text{test}}$ (1600 in the 4x4 configuration, 8100 in the 9x9 configuration) are used and the neural model of $\Pi$-NeSy-1 (resp. $\Pi$-NeSy-2)  obtains average accuracies of $99.58\pm 0.24$\% (resp. $99.46 \pm 0.26$) in the 4x4 configuration and $99.09\pm0.25\%$ (resp. $99.08 \pm 0.19\%$) in the 9x9 configuration.

Based on empirical results with the specified configuration (detailed in the beginning of this section),  the inference and learning times of $\Pi$-NeSy are reasonable  for MNIST Sudoku problems, see Table \ref{tab:times_mnistSudoku}.

\begin{table}[ht]
\centering

\begin{tabular}{@{}cccccc@{}}
\toprule
Size & Training set size & Test set size  & Approach & Learning time & Inference time per test sample \\ \midrule
4x4  & 200 & 100 & $\Pi$-NeSy-1 & 154.39 $\pm$ 2.96 & 0.03 $\pm$ 0.00  \\
4x4  & 200 & 100 & $\Pi$-NeSy-2 & 155.30 $\pm$ 2.57 & 0.03 $\pm$ 0.00 \\
9x9  & 200 & 100 & $\Pi$-NeSy-1 & 8814.38 $\pm$ 51.23 & 0.67 $\pm$ 0.01 \\
9x9  & 200 & 100 & $\Pi$-NeSy-2 & 8784.21 $\pm$ 66.18 & 0.60 $\pm$ 0.01  \\ \bottomrule
\end{tabular}

\caption{Each row corresponds to a MNIST Sudoku problem processed by $\Pi$-NeSy-1 or $\Pi$-NeSy-2. Each row lists the corresponding learning time in seconds using training data (averaged over 10 runs) and the average inference time  for a test sample in seconds (averaged over $10\cdot 100$ observations) from these experiments. Both methods exhibit similar processing times.}
\label{tab:times_mnistSudoku}
\end{table}

In Section \ref{subsec:appendix:posslearn:mnistsudoku}, we present a more detailed analysis of the empirical results for  possibilistic learning obtained using $\Pi$-NeSy-1 and $\Pi$-NeSy-2 on the MNIST Sudoku 4x4 and MNIST Sudoku 9x9 problems. Likewise for MNIST Addition-$k$ problems, we show that the thresholds used by the two approaches are always close to zero for these problems, that possibilistic learning always yields values close to zero for the rule parameters, and that a small amount of training data is considered reliable. In the analysis, we also show that the MNIST Sudoku 4x4 problem does not require much memory (RAM) to be processed (around 750 MB), whereas the MNIST Sudoku 9x9 problem requires significantly more memory (around 6.7 GB).

\paragraph{Ablation study}

We conducted an ablation study with $\Pi$-NeSy on MNIST Sudoku problems, following the same procedure as the ablation study for MNIST Addition-$k$ problems. We used the same experimental setup as in the standard  settings, except that we do not apply possibilistic learning: no thresholds were used to select training data samples, and all rule parameters in the possibilistic rule-based system were manually set to zero, assuming the rules to be certain. As we operated with the same fixed random seeds as those used in the experiments reported in Table \ref{tab:visual_Sudoku_classification} and with deterministic pytorch operations, we obtained the same results as those reported in Table \ref{tab:visual_Sudoku_classification}.  

\paragraph{Low data settings}

We also tackled MNIST Sudoku problems in low data settings, where for each experiment, the training dataset $\mathcal{A}_{\text{train}}$ is composed of 10 puzzles (5 correct, 5 incorrect), while the validation dataset and the test dataset are generated as in the standard settings. The results are presented in Table \ref{tab:mnist_sudokuresults_low_data_settings} where, for each configuration (4x4 and 9x9), the reported result is the average of the accuracies reached for 10 independent runs. Our approach $\Pi$-NeSy appears as quite robust when few data are available, even if for the 9x9 configuration, performance is almost equivalent to the one that would be achieved by a random choice. 

Note that we kept  $h=2$ for generating the  set $\mathcal{T}$ of candidate thresholds (see (\ref{eq:setTvalidationstep})) in the experiments conducted, as for those carried out in low data settings to address MNIST Addition-$k$ problems.  

\begin{table}[ht]
\centering
\renewcommand{\arraystretch}{1.5}
\begin{tabular}{lcc}
\hline
{Approach} & {4x4 Sudoku} & {9x9 Sudoku}\\
\hline
$\Pi$-NeSy-1 & 73.8 $\pm$ 6.79 & 51.5 $\pm$ 0.67\\
$\Pi$-NeSy-2 & 72.9 $\pm$ 11.78 & 51.2 $\pm$ 1.17\\
\hline
\end{tabular}
\caption{Accuracy and standard deviation on the test dataset for the MNIST-Sudoku problem over 10 runs in low data settings.}
\label{tab:mnist_sudokuresults_low_data_settings}
\end{table}
In Subsection \ref{subsec:appendix:lds:posslearn:mnistsudoku}, an in-depth analysis of the characteristics and performance of our possibilistic learning method for these experiments is presented. 
We did not address MNIST Sudoku problems using DeepSoftLog CNN (its implementation for this task is not publicly available).

\subsection{Discussion on \texorpdfstring{$\Pi$-NeSy}{Pi-NeSy}}
\label{subsec:disc}
\paragraph*{Design of a possibilistic rule-based system}
One challenge when using $\Pi$-NeSy is to design a possibilistic rule-based system suited to the problem at hand. In a possibilistic rule-based system,  possibilistic rules encode negative information \cite{dubois2003new}  and, when performing inference, the rules are combined conjunctively with the  function $\min$. Let us remind the notion of negative information in the case of a crisp rule ``if $X$ is $A$ then $Z$ is $O$'' in classical logic,  
where $A \subseteq U$ and $O \subseteq V$
are subsets of the domains of the variables $X$ and $Z$,  respectively, and the rule links the two
universes of discourse $U$ and $V$ by their local restrictions $A$ and $O$. The rule encodes negative information when the rule is interpreted as a \textit{constraint} ``if $X$ is $A$ then $Z$ must be $O$”:  it asserts in an implicitly negative way that the values outside $O$ are excluded when $X$ takes its values in $A$. Otherwise, if the rule represents positive information,  the rule is viewed as a condition of the form ``if $X$ is $A$ then $Z$ \textit{can be} $O$'' and
 asserts that when $X$ takes its value in $A$, then {\it any} value in $O$ is eligible for $Z$.

It should be noted that for the possibilistic rule-based systems used in the experiments, each rule is designed in order to  exhibit a form of equivalence between ``if $p$ then $q$'' and ``if $q$ then $p$'' (viewed as ``if $\neg p$ then $\neg q$''). When performing possibilistic learning, rule parameters are thus expected to have values close to or equal to zero, \textit{which means that both the rules and the converse rules are certain}. This is what we observed experimentally, see Section \ref{sec:appendix:expadd}.


Each possibilistic rule-based system used in the experiments also exhibit the following property: assuming that  the rule parameters values are set to zero, if each of the input possibility distributions used for the inference of the possibilistic rule-based system is such that one value of the distribution equals one while all other values are zero (this type of  distribution is  called a one-point distribution \cite{MarekProbability1963}), then, after inference,  the resulting output possibility distribution is also a one-point distribution. Stated differently, the conclusion of the rule is certain since both the rule and its premise are certain.

\paragraph*{Possibilistic learning}
Before performing possibilistic learning in practice, a selection step that aims to exhibit most reliable training data samples has been carried out (see Subsection \ref{subsec:thresholdsdetermination}). This is a new feature for neuro-symbolic approaches. Such a selection process was used in the experiments  to ensure that the equation systems built by stacking $(\mathbf{\Sigma})$, see (\ref{eq:stackedES}), are of a reasonable size. With regard to our experimental results, it appears that our selection method is  efficient enough, and it  filtered well reliable data from unreliable ones according to a threshold. 
The time needed to select a threshold can be reduced by choosing the smallest one that work for applying possibilistic learning,  i.e., we set each threshold  to the minimum value that ensures that at least one training data sample is selected and is therefore considered reliable. 

Not handling the inconsistency of the equation systems constructed with training data
thanks to \cite{BAAJ2024} would have led to unexpected results. In all our experiments, when applying possibilistic learning, the values of the rule parameters were always equal to zero or very close to zero (see Section \ref{sec:appendix:expadd}). Thus, for each set of rules, the learning method allows us to state that the rules are certain with respect to the training data. 


We have also tested the performance of the learning method without achieving first the selection process using data of varying reliability (this can be tested by setting the threshold(s) to a value strictly greater than $1$). This test has been done for the possibilistic rule-based systems used for MNIST Additions-$k$ problems and MNIST Sudoku problems. Since a consistent equation system has been obtained for each training data sample using Definition \ref{def:reduced-es}, the learning phase has been achieved successfully.

Finally, the possibilistic learning method introduced in this article appears to be well performing in practice. Of course, additional experiments will be required to confirm its efficiency in other settings and considering other neuro-symbolic problems.

\paragraph*{Limitations}

When a possibilistic rule-based system contains a very large number of rules and/or the domain of the output attribute has a very large number of output values, the size of the matrix $\dot{\Gamma}$ of the equation system $(\Upsigma)$, see (\ref{eq:upsigma}), can be very large. In such a case, depending on the number of training data samples, the size of the matrix of the equation system constructed by stacking $(\mathbf{\Sigma})$, see (\ref{eq:stackedES}) can be de facto very large as well. This drawback is slightly mitigated by the use of thresholds, which can be leveraged to reduce the size of the matrix of the equation system $(\mathbf{\Sigma})$ as they allow us to select the most reliable training data samples. 

In the experiments carried out, our possibilistic learning method has been applied to rules ``if $p$ then $q$'' exhibiting a form of equivalence with ``if $q$ then $p$'' (viewed as ``if $\neg p$ then $\neg q$''). Possibilistic learning was not tested for other types of rule. Furthermore, possibilistic learning was only tested for multi-class classification problems where the targeted output possibility distribution associated with a training data sample
is such that one output value has a possibility degree equal to one, while all others have a possibility degree equal to zero. 

\paragraph*{Related work}

Neuro-symbolic methods can mainly be categorized into two approaches \cite{vermeulen2023experimental}: regularization-based approaches, which use logical constraints to regularize the neural network during training (such as Semantic Based Regularization (SBR) \cite{diligenti2017semantic} or Logic Tensor Networks (LTNs) \cite{badreddine2022logic}); and logic programming-based approaches, which interface a neural network with a probabilistic or fuzzy extension of a logic programming framework (e.g., DeepProbLog \cite{manhaeve2018deepproblog}, NeurASP \cite{yang2023neurasp}, DeepStochLog \cite{winters2022deepstochlog}). 

$\Pi$-NeSy follows a logic programming-based approach, in the sense that it interfaces a neural network with a possibilistic rule-based system. Possibilistic rule-based systems are based on Possibility Theory, an uncertainty theory developed since the late 70s which lies at the crossroads between fuzzy sets, probability, and nonmonotonic reasoning  \cite{dubois2015possibility}. Unlike probabilistic neuro-symbolic approaches which perform approximate inference based on weighted model counting (WMC) \cite{maene2024hardness}, $\Pi$-NeSy inference is not computationally expensive. As such, it scales well, and allows us to make exact inferences. 

The learning method for possibilistic rule-based systems is a new approach, which differs from classic machine learning paradigms (at work in state-of-the-art neuro-symbolic approaches) by its use of an optimization method based on the solving of $\min-\max$ equation systems and the ability of handling inconsistent $\min-\max$ equation systems. This approach brings two main new features to the development of neuro-symbolic approaches: the measure of  reliability of a training data sample, and the capacity to make minimal modifications to the classification of a data sample in order to make it compatible with a set of rules.

\subsection{Further developments of neuro-symbolic approaches based on possibilistic rule-based systems}

In what follows, two interesting features of possibilistic rule-based systems, which may be of interest for the development of neuro-symbolic approaches based on these systems, are highlighted. The first feature is a backpropagation mechanism that can be defined for improving the neural model used. The second feature concerns the XAI-compliance of the approach: it is possible to derive explanations for the inference results of possibilistic rule-based systems.

\subsubsection{Backpropagation mechanism}
\label{subsub:backpropagation}

In this subsection, we show how a process similar to  backpropagation \cite{rumelhart1986learning} can be performed when dealing with a possibilistic rule-based system.  Let a possibilistic rule-based system be composed of a set of $n$ possibilistic rules, where the values of the rule parameters are set. The rule premises are constructed from $k$ input attributes $a_1,a_2,\dots,a_k$ and the output attribute is denoted by $b$. Let us suppose that the totally ordered set $(\Lambda^{(n)}, \preceq^{(n)})$ associated with the  possibilistic rule-based system has been constructed (see Subsection \ref{subsec:partitionred}). Using (\ref{eq:reducedonmnconstructed}), the matrix $\mathcal{M}_n$ composed of the rule parameters can be built with $(\Lambda^{(n)}, \preceq^{(n)})$.  

Let us now assume a targeted normalized output possibility distribution $\widetilde \pi_{b(x)}$ associated with the output attribute $b$. Using $(\Lambda^{(n)}, \preceq^{(n)})$, the mapping $\sigma^{(n)}$, see (\ref{eq:sigman}) and  the possibility measure $\widetilde \Pi$ associated with $\widetilde \pi_{b(x)}$, the following vector can be formed:
\begin{equation}\label{eq:otilde}
\widetilde O  = [\widetilde o_\mu]_{\mu\in\Lambda^{(n)}} = \begin{bmatrix}
    \widetilde\Pi(\sigma^{(n)}(\mu_1))\\
    \widetilde\Pi(\sigma^{(n)}(\mu_2))\\
    \vdots\\
    \widetilde \Pi(\sigma^{(n)}(\mu_{\omega^{(n)}}))\\
\end{bmatrix} 
\end{equation}

The construction of this vector is similar to the one of the output vector in the matrix relation, see (\ref{eq:reducedonmnconstructed}). As we suppose that the output possibility distribution $\widetilde \pi_{b(x)}$ is normalized and since the non-empty subsets $(\sigma^{(n)}(\mu_i)_{1 \leq i \leq \omega^{(n)}}$ form a partition of the set $D_b$, we have:
\begin{equation}\label{eq:Nor}
 \exists i\in \{1 , 2 , \dots , \omega^{(n)}\} \text{ such that }  
 \widetilde \Pi(\sigma^{(n)}(\mu_i)) = 1.
\end{equation}

For defining a backpropagation process, the following $\min-\max$ system can be used:
\begin{equation}\label{eq:sysOmega}
   (\Omega_n): \widetilde O  = \mathcal{M}_n \Box_{\max}^{\min} X,
\end{equation}
\noindent where $X = \begin{bmatrix}
    \lambda_1\\
    \rho_1 \\
    \lambda_2\\
    \rho_2\\
    \vdots \\
    \lambda_n\\
    \rho_n
\end{bmatrix}$ is an unknown vector, whose components are the possibility degrees of the rule premises ($\lambda_i$ is the possibility degree of the premise $p_i$ and $\rho_i$ is the possibility degree of $\neg p_i$, see Subsection \ref{subsec:possibilistic-rbs}).

A backpropagation mechanism can be implemented as follows:  \textit{
If the equation system $(\Omega_n)$ is consistent (i.e., if the equation system $(\Omega_n)$ has a solution),   normalized input  possibility distributions can be generated. Indeed, inference from the possibilistic rule-based system with these generated input possibility distributions using the  matrix relation associated with the set of rules, see (\ref{eq:redFP}), produce the desired output vector $ \widetilde{O} $ and thus the targeted output possibility distribution $ \widetilde{\pi}_{b(x)} $}. Therefore, the generated input possibility distributions are  ``targeted'' to produce the exact output vector $\widetilde{O}$  by inference with the matrix relation. 

Using the antipignistic probability-possibility transformation, see Subsection \ref{subsec:antipignistic}, the obtained targeted input possibility distributions can finally be transformed into  probability distributions. 

\paragraph{Solving the equation system \texorpdfstring{$(\Omega_n)$}{Omega n}}

In the following, we study the solving of the equation system $(\Omega_n)$. The study is similar to the one for the reduced equation system  $(\Upsigma_n)$ as shown in  Subsection \ref{subsec:learnruleparam}. We show that, if the system $(\Omega_n)$  is consistent,  its solutions can be used for backpropagation. If the equation system $(\Omega_n)$ is inconsistent, \cite{BAAJ2024} can be leveraged  to handle it.  

We remind that the coefficients of the matrix $\mathcal{M}_n$  of the equation system  $(\Omega_n): \widetilde O  = \mathcal{M}_n \Box_{\max}^{\min} X$ are given by the   formulas  (\ref{eq:coefredM})  derived from the  $n$ partitions 
$\Lambda^{(n)} = \Lambda_j^{n, \top} \cup \Lambda_j^{n, \bot}$ with $j \in \{1,2,\cdots,n\}$, see Definition \ref{def:Lj} and (\ref{eq:Lnpart}). We introduce the following column-vectors:  
\begin{equation}F^{n,\downarrow} = {\mathcal{M}_n^t} \Box_{\epsilon}^{\max}  \widetilde O = \begin{bmatrix}
    \lambda_1^{n,\downarrow}\\
    \rho_1^{n,\downarrow} \\
    \lambda_2^{n,\downarrow}\\
    \rho_2^{n,\downarrow}\\
    \vdots \\
    \lambda_n^{n,\downarrow}\\
    \rho_n^{n,\downarrow}
\end{bmatrix}  \quad \text{and} \quad 
    F^{n,\uparrow} = \begin{bmatrix}
    \lambda_1^{n,\uparrow}\\
    \rho_1^{n,\uparrow} \\
    \lambda_2^{n,\uparrow}\\
    \rho_2^{n,\uparrow}\\
    \vdots \\
    \lambda_n^{n,\uparrow}\\
    \rho_n^{n,\uparrow}
\end{bmatrix}.
\end{equation}
where the vector $F^{n,\downarrow}$ is computed by a max-epsilon matrix product. The coefficients of the two vectors $F^{n,\downarrow}$ and $F^{n,\uparrow}$  are obtained using the  rule parameters and  the coefficients of  the   vector   $\widetilde O =[\widetilde o_\mu]_{\mu\in\Lambda^{(n)}}$ using the sets $\Lambda_j^{n, \top}$ and $\Lambda_j^{n, \bot}$ with $j \in \{1,2,\dots,n\}$:

\begin{definition}\label{def:coFU}
The coefficients of  the vector   $F^{n,\uparrow}$  are  defined by:
\begin{equation}\label{eq:formulasejtopOMEGA}
  \text{for } j=1,2,\dots,n,  
  \quad \lambda_j^{n,\uparrow}: = \max_{\mu\in \Lambda_j^{n, \top}} \widetilde o_\mu \quad 
\text{ and }\quad 
    \rho_j^{n,\uparrow} := \max_{\mu\in \Lambda_j^{n, \bot}} \widetilde o_\mu.
\end{equation}
 If $\Lambda_j^{n, \top} = \emptyset$ (resp. $\Lambda_j^{n, \bot} = \emptyset$), we have $\max_{\mu\in \Lambda_j^{n, \top}} \widetilde o_\mu= 0$ (resp. $\max_{\mu\in \Lambda_j^{n, \bot}} \widetilde o_\mu= 0$).
    
\end{definition}
\begin{restatable}{lemma}{lemcoEDcmdOMEGA}\label{lem:coEDOMEGA}
The coefficients of  the vector $F^{n,\downarrow}:= \mathcal{M}^t_n  \Box_\epsilon^{\max } \widetilde O$ are  given by: \\
For $j=1,2,\dots,n$: 
\begin{equation}\label{eq:eninitialformulaOMEGA}
    \lambda_j^{n,\downarrow} = s_j\, \epsilon \max_{\mu\in \Lambda_j^{n, \top}} \widetilde o_\mu  = s_j \,\epsilon \, \lambda^{n,\uparrow}_{j} \text{ and } \rho_j^{n,\downarrow} = r_j \, \epsilon \max_{\mu\in \Lambda_j^{n, \bot}} \widetilde o_\mu = r_j\, \epsilon \,\rho^{n,\uparrow}_{j}.
\end{equation}
For $l= 1, 2, \dots, n$, we have $\lambda_j^{n,\downarrow} \leq \lambda_j^{n,\uparrow}$ and $\rho_j^{n,\downarrow} \leq \rho_j^{n,\uparrow}$, i.e., the vector inequality 
$F^{n,\downarrow} \leq F^{n,\uparrow}$ holds.
\end{restatable}
\noindent See Subsection \ref{subsec:proof:coEDOMEGA} for the proof. We also have the following result:

\begin{restatable}{proposition}{propcoEDUcmdOMEGA}\label{prop:coEDUOMEGA}
Let $\mathcal{M}_n \Box_{\max}^{\min} F^{n,\downarrow} = [u_\mu]_{\mu\in \Lambda^{(n)}}$. Then, for each $\mu\in \Lambda^{(n)}$, we have:
\begin{equation}\label{eq:autreenOMEGA}
u_\mu = \min_{1 \leq j \leq n} z_{\mu, j}\quad \text{; } \quad 
z_{\mu, j} := \begin{cases}
\max(s_j, \lambda_{j}^{n, \uparrow})     & \text{ if }  \mu\in \Lambda_j^{n, \top}   \\
  \max(r_j, \rho_{j}^{n, \uparrow}) & \text{ if }  \mu\in \Lambda_j^{n, \bot} 
\end{cases}.
    \end{equation}
Moreover, the following equality holds: 
\begin{equation}\label{eq:eeOMEGA}
\mathcal{M}_n \Box_{\max}^{\min} F^{n,\downarrow} = \mathcal{M}_n \Box_{\max}^{\min} F^{n,\uparrow} 
.
\end{equation}

\end{restatable}
\noindent See Subsection \ref{subsec:proof:coEDUOMEGA} for the proof. 
By Sanchez's result \cite{sanchez1976resolution}, the  equation system $(\Omega_n)$ is consistent iff $\widetilde O = \mathcal{M}_n \Box_{\max}^{\min} F^{n,\downarrow}$. If this equality is satisfied, the lowest solution of the system $(\Omega_n)$ is the vector $F^{n,\downarrow}$ and by     Proposition \ref{prop:coEDUOMEGA},  the vector $F^{n,\uparrow}$ is another solution of the system $(\Omega_n)$. In fact,    any vector $X$ such that $F^{n,\downarrow}\leq X \leq  F^{n,\uparrow}$  (where $\leq$ denotes the usual component-wise order between vectors) is a solution of the  equation system $(\Omega_n)$:
\begin{corollary}\label{cor:consistOMEGA}
Suppose that the equation system  $(\Omega_n): \widetilde O_n = \mathcal{M}_n \Box_{\max}^{\min} X$ is consistent. Then  any vector $X$ such that $F^{n,\downarrow}\leq X \leq  F^{n,\uparrow}$ is a solution of the equation system $(\Omega_n)$.   
\end{corollary}
\begin{proof}
As by Lemma \ref{lem:coEDOMEGA}, we have  $F^{n,\downarrow} \leq F^{n,\uparrow}$, we can deduce that for any vector $X$ such that $F^{n,\downarrow}\leq X \leq  F^{n,\uparrow}$:
\[ \widetilde O= \mathcal{M}_n  \Box_{\max}^{\min} F^{n,\downarrow} \leq \mathcal{M}_n \Box_{\max}^{\min} X \leq 
  \mathcal{M}_n  \Box_{\max}^{\min} F^{n,\uparrow} = \widetilde O. \]
\end{proof}

Accordingly, whenever the equation system $(\Omega_n)$ is consistent, it has an uncountable number of solutions. Some solutions exhibit properties which allow us to perform backpropagation:
\begin{restatable}{proposition}{propSOLOMEGA}\label{prop:SOLOMEGA}\mbox{}
\begin{enumerate}
    \item 
    The components of the vector $F^{n,\uparrow}$ exhibit the following property:
    \[ \forall j \in \{1,2,\cdots,n\}, \text{ we have } \max(\lambda_j^{n,\uparrow}, \rho_j^{n,\uparrow}) = 1.\]
    \item Suppose that all the rule parameters of the possibilistic rule-based system have a value that is strictly lower than $1$, i.e., $\forall j \in \{1,2,\cdots,n\}$, we have $s_j < 1$ and $r_j < 1$. Then, the components of the vector $F^{n,\downarrow}$ exhibit the following property:
    \[ \forall j \in \{1,2,\cdots,n\}, \text{ we have } \max(\lambda_j^{n,\downarrow}, \rho_j^{n,\downarrow}) = 1.\]
    \item Suppose that all the rule parameters of the possibilistic rule-based system have a value that is strictly lower than $1$, i.e.,  $\forall j \in \{1,2,\cdots,n\}$, we have $s_j < 1$ and $r_j < 1$. Then  the components of any vector $X^\ast = \begin{bmatrix}\lambda_1^\ast & \rho_1^\ast & \cdots & \lambda_n^\ast & \rho_n^\ast\end{bmatrix}^t$ such that $F^{n,\downarrow}\leq X^\ast \leq  F^{n,\uparrow}$ exhibit the following property:
    \[ \forall j \in \{1,2,\cdots,n\}, \text{ we have } \max(\lambda_j^\ast, \rho_j^\ast) = 1.\]
\end{enumerate}
\end{restatable}
\noindent See Subsection \ref{subsec:proof:SOLOMEGA} for the proof.

\paragraph{Generating targeted input possibility distributions}

\noindent Let us now assume that the equation system $(\Omega_n)$ is consistent and the vector $X^\ast = \begin{bmatrix}\lambda_1^\ast & \rho_1^\ast & \cdots & \lambda_n^\ast & \rho_n^\ast\end{bmatrix}^t$ is one of its solution. 
If the components of $X^\ast$ are such that for all $i \in \{1,2,\dots,n\}$, $\max(\lambda_i^\ast, \rho_i^\ast)=1$ (Proposition \ref{prop:SOLOMEGA}), then generating input possibility distributions associated with the input attributes of the possibilistic rule-based system from $X^\ast$  based on the definition of the rule premises is feasible.  Indeed, one can state from the vector $X^\ast$ conditions on the possibility degrees $\lambda_1, \rho_1, \lambda_2, \rho_2,\cdots, \lambda_n, \rho_n$ of the rule premises: for all $i \in \{1,2,\dots,n\}$, we must have $\lambda_i = \lambda_i^\ast$ and $\rho_i = \rho_i^\ast$. Then, we can formulate conditions on the propositions composing the rule premises and deduce conditions on the targeted input possibility distributions of the possibilistic rule-based system: 

 \begin{itemize}
            \item  The condition $\lambda_i = \lambda_i^\ast$ relates to the possibility degree  of the  premise $p_i = p_1 \wedge p_2 \wedge \dots \wedge p_{l_i}$. We must have: 
\[ \lambda_i = \min_{1 \leq j \leq l_i} \pi(p_j) = \lambda_i^\ast.\] 
\item The condition $\rho_i = \rho_i^\ast$ relates to the possibility degree of the negation of the premise $p_i$. We must have:  
\begin{equation*}
 \rho_i = \max_{1 \leq j \leq l_i} \pi(\neg p_j) = \rho_i^\ast.
\end{equation*}
        \end{itemize}

\noindent The above conditions on the propositions composing the rule premises can be transformed to conditions on the input possibility distributions of the possibilistic rule-based system.  For instance, for a proposition $p_j$ of the form ``$a_j(x) \in P_j$'',  the condition $\pi(p_j) = \tau$ is equivalent to $\max_{v \in P_j} \pi_{a_j(x)}(v) = \tau$. Thus, targeted input possibility distributions can be generated.

Inference from the possibilistic rule-based system using the targeted input possibility distributions $\pi^t_{a_1(x)},\pi^t_{a_2(x)}, \cdots, \pi^t_{a_k(x)}$ that we obtained for the input attributes $a_1, a_2, \cdots, a_k$ yields the targeted output possibility distribution $\widetilde \pi_{b(x)}$. In fact, by computing the possibility degrees of the rule premises from $\pi^t_{a_1(x)},\pi^t_{a_2(x)}, \cdots, \pi^t_{a_k(x)}$ and then constructing the input vector $I_n$ of the matrix relation using (\ref{eq:inputvector}), the following basic equality is obtained:  $I_n = X^\ast$. Therefore, inference from the possibilistic rule-based system yields the output vector $\widetilde O$ and thus $\widetilde \pi_{b(x)}$: 
\begin{equation}
    \mathcal{M}_n \Box_{\max}^{\min} I_n = \mathcal{M}_n  \Box_{\max}^{\min} X^\ast = \widetilde O.
\end{equation}

\begin{remark}
    The backpropagation mechanism assumes that the rules of the possibilistic rule-based system are well-specified, and therefore that conflicting conditions on the possibility degrees of the rule premises cannot be reached using $X^\ast$, which would preclude the generation of targeted input possibility distributions. An example of conflicting conditions arises when two distinct possibility distributions, $ \pi_{a_j(x)}$ and $\pi'_{a_j(x)}$, are derived from $X^\ast$ for the same attribute $ a_j $ from different rules. If these distributions conflict for a given value $v \in D_{a_j}$, e.g., $ \pi_{a_j(x)}(v) = 1 $ and $ \pi'_{a_j(x)}(v) = 0 $, a contradiction occurs.

    A case when it is easy to obtain input possibility distributions is when every input attribute of the possibilistic rule-based system has a domain that is binary $\{0,1\}$, and every rule premise is a single proposition of the form ``$a_i(x) \in \{0\}$'' or ``$a_i(x) \in \{1\}$'', involving a single input attribute $a_i$ and a single input attribute value.
\end{remark}

A backpropagation mechanism can also be defined while taking account for a training data sample. Let us consider a training data sample associated with a set of possibilistic rules, based on a set of normalized input  possibility distributions $\widetilde \pi_{a_1(x)}, \widetilde  \pi_{a_2(x)}, \cdots, \widetilde  \pi_{a_k(x)}$ associated with the input attributes $a_1,a_2,\dots,a_k$ and a targeted normalized output possibility distribution $\widetilde \pi_{b(x)}$ associated with the output attribute $b$. The input vector $\widetilde I$, see (\ref{eq:inputvector}), can be generated from the possibility degrees of the rule premises computed with  $\widetilde \pi_{a_1(x)}, \widetilde  \pi_{a_2(x)}, \cdots, \widetilde  \pi_{a_k(x)}$  and the output vector $\widetilde O$ associated with $\widetilde \pi_{b(x)}$ can be built using (\ref{eq:otilde}).  By measuring the distance (in terms of $L_\infty$ norm for multi-class classification problems)  between the vector $\mathcal{M}_n \Box_{\max}^{\min} \widetilde I$ and the vector $\widetilde O$, the extent to which the desired result $\widetilde O$ is obtained from $\widetilde I$ can be determined. 

Given that the equation system $(\Omega_n)$ is consistent, one can check whether the  targeted input possibility distributions $\pi^t_{a_1(x)},\pi^t_{a_2(x)}, \cdots, \pi^t_{a_k(x)}$ generated from a solution $X^*$ such that $\max(\lambda_i^* , \rho_i^*) = 1$ (Proposition \ref{prop:SOLOMEGA}) match the input possibility distributions $\widetilde \pi_{a_1(x)}, \widetilde  \pi_{a_2(x)}, \cdots, \widetilde  \pi_{a_k(x)}$ of the considered training data sample.

\begin{example}\label{ex:backprop}
{(Example \ref{ex:firstrules-set}, 
cont'ed)} \\
Let us reuse the matrix relation (Example \ref{ex:matrixrelred}) of the  set of possibilistic rules introduced in Example \ref{ex:firstrules-set}. We assume that the rule parameters are all set to zero and that the matrix $\mathcal{M}_n$ has been constructed. 

Suppose that the following training data sample is considered: the input possibility distributions are [$\widetilde\pi_{a_1(x)}(0), \widetilde\pi_{a_1(x)}(1)$] = $[1, 0]$
and [$\widetilde\pi_{a_2(x)}(0), \widetilde\pi_{a_2(x)}(1)$] = $[0, 1]$, while the  targeted output possibility distribution  associated with the output attribute $b$ assigns a possibility degree of $1$ to $(1,0)$, i.e., $\widetilde \pi_{b(x)}((1,0)) = 1$, and a possibility degree of zero for all other output values. 

From the targeted output possibility distribution, the following output vector of the equation system $(\Omega)$ can be generated:
$$\widetilde O =  \begin{bmatrix}
    \widetilde \pi_{b(x)}((1,1))\\
     \widetilde \pi_{b(x)}((0,1))\\
      \widetilde \pi_{b(x)}((1,0))\\
       \widetilde \pi_{b(x)}((0,0))
\end{bmatrix}=\begin{bmatrix}
   0\\
    0 \\
     1 \\
     0 
\end{bmatrix}.$$ Using the vector $\widetilde O$ and the matrix $\mathcal{M}_n$, the equation system $(\Omega)$ can be constructed:
$$
\begin{bmatrix}
    0\\
    0\\
    1\\
    0
\end{bmatrix}
    =
    \begin{bmatrix}
        1 & 0 & 0 & 1 & 1 & 0 & 0 & 1\\
        0 & 1 & 1 & 0 & 1 & 0 & 0 & 1\\
        1 & 0 & 0 & 1 & 0 & 1 & 1 & 0\\
        0 & 1 & 1 & 0 & 0 & 1 & 1 & 0
    \end{bmatrix} \Box_{\max}^{\min} \begin{bmatrix}
    \lambda_1\\
    \rho_1\\
    \lambda_2\\
    \rho_2\\
    \vdots\\
    \lambda_4\\
    \rho_4
\end{bmatrix}.$$

\noindent This equation system is consistent. Its lowest solution is the vector $F^{n,\downarrow} = \begin{bmatrix}
    \lambda_1^\ast\\
    \rho_1^\ast\\
    \lambda_2^\ast\\
    \rho_2^\ast\\
    \lambda_3^\ast\\
    \rho_3^\ast\\
    \lambda_4^\ast\\
    \rho_4^\ast
\end{bmatrix}= \begin{bmatrix}
    0\\ 1\\ 1\\ 0\\ 1\\ 0\\ 0\\ 1
\end{bmatrix}$.
From the first two rules, which use the same attribute $a_1$, in order to have $ \lambda_1=\lambda_1^\ast, \rho_1=\rho_1^\ast, \lambda_2=\lambda_2^\ast, \rho_2=\rho_2^\ast$, the input possibility distribution $\pi_{a_1(x)}$ must be as follows: $[\pi_{a_1(x)}(0), \pi_{a_1(x)}(1)]$ = $[0,1]$. It is different from the input possibility distribution $\widetilde\pi_{a_1(x)}$ of the training data sample. Similarly, for having $\lambda_3=\lambda_3^\ast, \rho_3=\rho_3^\ast, \lambda_4=\lambda_4^\ast, \rho_4=\rho_4^\ast$, the input possibility distribution $\pi_{a_2(x)}$ must be as follows: $[\pi_{a_2(x)}(0), \pi_{a_2(x)}(1)]$ = $[1,0]$. Again, this distribution is different from the input possibility distribution $\widetilde\pi_{a_2(x)}$ of the training data sample.

Using the anti-pignistic transform (see Subsection \ref{subsec:antipignistic}), the targeted input possibility distributions derived from $F^{\downarrow}$ can finally be transformed into probability distributions. We obtain the following probability distributions:  $[P_{a_1(x)}(0) = 0, P_{a_1(x)}(1) = 1]$ and $[P_{a_2(x)}(0) = 1, P_{a_2(x)}(1) = 0]$.
\end{example}

\subsubsection{Deriving explanations from possibilistic rule-based systems}
\label{subsub:explainability}

A valuable feature offered by possibilistic rule-based systems is their explanation facility. 
Thus, finding explanations for the inferences achieved from possibilistic rule-based systems has been studied by Farreny and Prade in \cite{farrency2013positive} and subsequently developed in \cite{baaj2021representation} (see also \cite{baajtel-03647652}). \cite{baaj2021representation,farrency2013positive} address two explanatory purposes for an output value $u \in D_b$ (where $D_b$ is the domain of the output attribute $b$), which can be formulated as two questions:
\begin{itemize}
    \item (i) What are the \textit{conditions} on the input possibility distributions   for having the possibility degree $\pi^\ast_{b(x)}(u)$ strictly greater (or lower) than a given $\tau \in [0,1]$?
    \item (ii) What are the \textit{inferred} rule premises justifying $\pi^\ast_{b(x)}(u) = \tau$? 
\end{itemize}

Regarding (i) in the context of our neuro-symbolic approach $\Pi$-NeSy, in the case of incorrect classification, i.e., when the possibility degree $\pi^\ast_{b(x)}(u)$ of an output value $u$ is not equal to the expected value, necessary and sufficient conditions on the input possibility distributions in order to obtain the expected  value can be identified using the explanatory methods of \cite{baaj2021representation}.  When the input possibility distributions result from the  antipignistic possibility-probability transformation  (see Subsection \ref{subsec:antipignistic}), these conditions 
can be turned into conditions on the probability distributions produced by the neural network (let us remind that the antipignistic method defines a one-to-one correspondence between probability measures and possibility measures).

Regarding (ii), to the best of our knowledge, current neuro-symbolic approaches are not able to provide detailed and useful explanations of their inference results. In contrast,  for our neuro-symbolic approach $\Pi$-NeSy,  two types of explanations of an inference result can be derived: 
\begin{itemize}
    \item  explanations having the form of justification statements, composed of possibilistic expressions that are sufficient to justify the possibility degree of the output value,
    \item explanations given as unexpectedness statements,  alias semifactual explanations \cite{DBLP:conf/nips/KennyH23,DBLP:conf/aaai/Aryal24}, of the form 
``Even if $X$ holds, the conclusion $b(x)$ is $u$ with a possibility degree of $\pi^\ast_{b(x)}(u)$ remains valid.''
This kind of explanations is in the same vein as the  explanations named "even-if-because" of Darwiche and Hirth \cite{DarwicheH20}. The unexpectedness $X$ is a set of possible or certain possibilistic expressions, which are not involved in the determination of the considered inference result but would not question it if they were present.
\end{itemize}

\begin{example}
{(Example \ref{ex:firstrules-set}, 
cont'ed)} \\
We illustrate the explanatory methods of \cite{baaj2021representation} using  the possibilistic rule-based system considered in Example \ref{ex:firstrules-set}. All rule parameters are set to zero. Inference is performed from the possibilistic rule-based system using the following input possibility distributions: $\pi_{a_1(x)}(0) = 0, \pi_{a_1(x)}(1) = 1$ and $\pi_{a_2(x)}(0) = 1, \pi_{a_2(x)}(1) = 0$. The resulting output possibility distribution $\pi^\ast_{b(x)}$ assigns a possibility degree equal to $1$ to the output value $(1,0)$, while all other output values have a possibility degree equal to zero.  Using \cite{baaj2021representation}, the inference result $\pi^\ast_{b(x)}((1,0)) = 1$ can be justified by triplets of the form \((p, \text{sem}, d)\), where \(p\) is a premise, \(\text{sem}\) denotes the semantics  (\(\textsf{P}\) for Possible, \(\textsf{C}\) for Certain) attached to the degree $d \in \{\pi(p),n(p)\}$  which is either the possibility degree of $p$, i.e., $\pi(p)$, or the necessity degree of $p$, i.e., $n(p) = 1 - \pi(\neg p)$. For this example, the four triplets composing the justification of $\pi^\ast_{b(x)}((1,0)) = 1$ are: $(p_1, \textsf{C}, n(p_1))$, $(p_2, \textsf{P}, \pi(p_2))$, $(p_3, \textsf{P}, \pi(p_3))$ and $(p_4, \textsf{C}, n(p_4))$. The premise $p$ in a triplet \((p, \text{sem}, d)\)  can be reduced  using the premise reduction functions presented in \cite{baaj2021representation}.

From this extraction, a natural language explanation could be generated: ``\textit{It is possible that  $b$ is $(1,0)$. In fact, $a_1$ is $1$ and $a_2$ is $0$. In addition, it is assessed as  impossible  that $a_1$ is 0 or $a_2$ is $1$}''.

\end{example}

\section{Conclusion and perspectives}

In this article, we introduced $\Pi$-NeSy, a neuro-symbolic approach that combines a low-level perception task performed by a neural network with a high-level reasoning task performed by a possibilistic rule-based
system. We focused on the development of methods for building two key data structures associated with the possibilistic rule-based system  at hand: the matrix relation for performing inferences and the equation system for performing learning. For this purpose, we have introduced a practical method for constructing the partition of the domain  of the  output attribute of a set of possibilistic rules according to the conclusions of the rules, which is the essential tool for constructing the matrix relation and the  equation system in a way that scales up. We then have presented a method for performing possibilistic learning in practice, using recent results on the handling of inconsistent $\min-\max$ equation systems. This work led us to propose the neuro-symbolic approach $\Pi$-NeSy where the connection between the neural network and the possibilistic rule-based system is established by transforming the output probability distributions of a neural network into possibility distributions. We have carried out experiments with $\Pi$-NeSy on well-known neuro-symbolic datasets: MNIST-Additions-$k$ and  MNIST Sudoku puzzles problems. These experiments led to quite good results with $\Pi$-NeSy, globally better than those offered by current state-of-the-art methods. 

A perspective for further research is to perform additional experiments with other datasets, and to apply our approach to neuro-symbolic problems that involve knowledge graphs \cite{hitzler2020neural}.
For the future, the development of several features of our neuro-symbolic approach also appears as promising. 
Especially, it would be useful to determine how to automatically acquire uncertain domain knowledge (under the form of possibilistic rules) from data, and how to integrate such an acquisition step into our neuro-symbolic approach. Another perspective consists in exploring the backpropagation mechanism suited to possibilistic rule-based systems, in order to develop a neuro-symbolic learning method based on this mechanism.

\section*{Acknowledgements}

This article was initiated during Ismaïl Baaj's postdoctoral contract at the Centre de Recherche en Informatique de Lens (CRIL), Université d'Artois / CNRS. This contract was partially supported by TAILOR, a project  funded by the EU Horizon 2020 research and innovation programme under grant agreement No. 952215. Ismaïl Baaj expresses his gratitude to the CRIL lab and to TAILOR for their support and for providing a research environment that greatly facilitated this work.

\smallskip 
\noindent Pierre Marquis has benefited from the support of the AI Chair EXPE\textcolor{orange}{KC}TA\-TION (ANR-19-CHIA-0005-01) and of the France 2030 MAIA Project (ANR-22-EXES-0009) of the French National Research Agency (ANR).

\bibliography{sn-bibliography}

\clearpage

\appendix 

\section{Proofs}
\label{sec:proofs}

In this section, we present the proofs of the results presented in the paper (Section \ref{sec:practical-procedure-building-reduced-ES} and Subsection \ref{subsec:learnruleparam}). 

\subsection{Proof of Lemma \ref{lemma:subsettouniquetuple}}
\label{subsec:proof:lemmasubsettouniquetuple}
\lemmasubsettouniquetuplecmd*
\begin{proof}
The first two statements are easily deduced by induction on $n$ using the definition of the partition, see (\ref{eq:partition}). 

\noindent
Suppose we have two   tuples $(T_1, T_2, \dots, T_n) $ and $(T'_1, T'_2, \dots, T'_n) $ such that: 
\[ T_1 \cap T_2 \cap \dots \cap T_n = T'_1 \cap T'_2 \cap \dots \cap T'_n \not =\emptyset \quad \text{with} \quad T_i\in\{Q_i, \overline{Q_i}\} \text{ and }    T'_i\in\{Q_i, \overline{Q_i}\}\quad \text{for all  $i$}.    \]
\noindent
Let us take an element $e$ belonging to  $T_1 \cap T_2 \cap \dots \cap T_n = T'_1 \cap T'_2 \cap \dots \cap T'_n$. For any $i\in\{1, 2, \dots, n\}$, we have $e\in T_i \cap T'_i$, so necessarily $T_i = T'_i = Q_i$ or $T_i = T'_i = \overline{Q_i}$.   We have proved the third statement.
\end{proof}

\subsection{Proof of Lemma \ref{lemma:LambdaN}}
\label{subsec:proof:LambdaN}
\LambdaNcmd*
\begin{proof}
The first statement follows from the definition of the partition $E^{(1)}$: we have $E^{(1)}_1 = Q_1$  and 
$E^{(1)}_2 = \overline{Q_1}$.

\noindent
To prove the second statement, we observe that for $1 < i \leq n$
and  $(t_1, t_2, \dots, t_i)$ such that $t_k = \pm k$, we have:
\begin{equation*}
    \sigma^{(i)}(t_1, t_2, \dots, t_i) =\begin{cases}
     \sigma^{(i -1)}(t_1, t_2, \dots, t_{i-1}) \cap Q_i     & \text{ if } \quad t_i = i  \\
 \sigma^{(i -1)}(t_1, t_2, \dots, t_{i-1}) \cap \overline{Q_i}     & \text{ if }  \quad t_i = -i.
\end{cases}
\end{equation*}
\end{proof}
\subsection{Proof of Lemma \ref{lemma:mutokmu}}
\label{subsec:proof:lemmamutokmu}
\lemmamutokmucmd*
\begin{proof}
The proof is organized into three steps:
\begin{enumerate}
\item Let $\mu = (t_1, t_2, \dots, t_n)$ be in $\Lambda^{(n)}$. Then by using the definition of the mapping $\sigma^{(n)}$, see (\ref{eq:sigman}),  and Lemma  \ref{lemma:subsettouniquetuple}, we obtain that   $\sigma^{(n)}(\mu)$ is a non-empty set of the partition   $(E_k^{(n)})_{1 \leq k \leq 2^n}$.
As for all $k, k'\in\{1, 2, \dots, 2^n\}$, we have:
\[ E_k^{(n)} \cap E_{k'}^{(n)} \not=\emptyset \Longrightarrow k = k'\]
there exists a unique $k_\mu \in \{1, 2, \dots, 2^n\}$ such that $\sigma^{(n)}(\mu) = E_{k_\mu}^{(n)}   $. By definition of the set $J^{(n)}$, see Notation \ref{nottrois}, $k_\mu\in J^{(n)}$. 

\item Let  $\mu = (t_1, t_2, \dots, t_n)$  and $\mu' = (t'_1, t'_2, \dots, t'_n)$ be in $\Lambda^{(n)}$.  Let:
 \begin{equation*}
    T_i  =\begin{cases}
      Q_i     & \text{ if } \quad t_i = i  \\
 \overline{Q_i}     & \text{ if } \quad t_i = - i 
\end{cases}, \quad T'_i  =\begin{cases}
      Q_i     & \text{ if } \quad t'_i = i  \\
 \overline{Q_i}     & \text{ if } \quad t'_i = - i 
\end{cases} \quad \text{for all $i = 1, 2 \dots, n$}.
\end{equation*}
\noindent
From (\ref{eq:sigman}), we know that  $\sigma^{(n)}(\mu) = T_1 \cap T_2 \cap \dots \cap T_n = 
E_{k_\mu}^{(n)}$ and $\sigma^{(n)}(\mu') = T'_1 \cap T'_2 \cap \dots \cap T'_n = E_{k_{\mu'}}^{(n)}$. If we have $k_\mu = k_{\mu'}$, as $E_{k_\mu}^{(n)} \not=\emptyset  $, we deduce from Lemma  \ref{lemma:subsettouniquetuple} that $T_i = T'_i$  for all  $i$ and therefore $t_i = t'_i$  for all  $i$: we have proved the injectivity of  the mapping $\Psi$. 

\item Let $k\in J^{(n)}$, i.e.,  the set $E_k^{(n)}$  is non-empty.  Let $(T_1, T_2, \dots, T_n)$ the unique tuple such that $T_i\in \{Q_i, \overline{Q_i}\}$ for all $i$ and  $E_k^{(n)} = T_1 \cap T_2 \cap \dots \cap T_n$, see Lemma \ref{lemma:subsettouniquetuple}. Let  $t_i  =\begin{cases}
      i    & \text{ if } \quad T_i = Q_i  \\
 -i     & \text{ if } \quad T_i = \overline{Q_i} 
\end{cases} \,$  and $\,\mu:= (t_1, t_2, \dots, t_n)$ As we have 
$E_k^{(n)} = \sigma^{(n)}(\mu)$, see (\ref{eq:sigman}),  we deduce that $\mu\in \Lambda^{(n)}$ and  
by   the first part of the proof, 
$k = k_\mu$. We have proved the surjectivity  of the mapping $\Psi$.
\end{enumerate}

\noindent
As  $(E_k^{(n)})_{k\in J^{(n)}}$ is a  partition of the set $D_b$ by non-empty sets and  $J^{(n)} \subseteq \{1, 2, \dots, 2^n\}$, we get the inequality  $\text{card}\, J^{(n)} \leq \min(\text{card}\, D_b, 2^n)$.

\end{proof}
\subsection{Proof of Lemma \ref{lemma:muto}}
\label{subsec:proof:lemmamuto}
\lemmamutocmd*
\begin{proof}
    As   $\mu = (t_1, t_2, \dots, t_n)\in \Lambda^{(n)}$, we know that $\sigma^{(n)}(\mu) \not=\emptyset$ and 
    $t_n = \pm n$, see Definition \ref{def:LambdaN}. Moreover, by definition of the mappings  $\sigma^{(n)}$ and $\sigma^{(n-1)}$, see  (\ref{eq:sigman}), we have:
\begin{equation}\label{eq:kmusig}
\sigma^{(n)}(\mu) = \begin{cases}
 \sigma^{(n-1)}(\widetilde \mu) \cap Q_n    & \text{ if }\,  t_n = n  \\
  \sigma^{(n-1)}(\widetilde \mu) \cap \overline{Q_n}    & \text{ if }\,  t_n =  -n \end{cases}
  .
 \end{equation}
 We deduce that $\widetilde\mu\in\Lambda^{(n-1)}$.
 From the construction (\ref{eq:partition}) and the equalities 
 $\sigma^{(n)}(\mu) = E_{k_\mu}^{(n)}$ and $\sigma^{(n-1)}(\widetilde\mu) = E_{k_{\widetilde \mu}}^{(n-1)}$, we get:
\begin{equation}\label{eq:kmu11}
E^{(n)}_{k_\mu} = \begin{cases}
     E_{k_{\widetilde \mu}}^{(n-1)} \cap Q_n    & \text{ if }\,  t_n = n \\
  E_{k_{\widetilde \mu}}^{(n-1)} \cap  \overline{Q_n}& \text{ if }\,  t_n = - n
\end{cases} = \begin{cases}
     E_{k_{ \mu}}^{(n-1)} \cap Q_n    & \text{ if }\,  k_\mu \leq 2^{n-1} \\
  E_{k_{ \mu} - 2^{n-1}}^{(n-1)} \cap  \overline{Q_n}& \text{ if }\,  2^{n-1} < k_\mu  
\end{cases}.
\end{equation}
\noindent
As we have $E_{k_\mu}^{(n)} \not=\emptyset$, we deduce the equivalences 
$k_\mu \leq 2^{n-1}\Longleftrightarrow t_n = n$ and 
$k_\mu > 2^{n-1}\Longleftrightarrow t_n = - n$.

\begin{itemize}
\item If $k_\mu \leq  2^{n-1}$, then  $ E_{k_{\widetilde \mu}}^{(n-1)} \cap   Q_n =E_{k_{ \mu}}^{(n-1)} \cap Q_n  \not=\emptyset$, so 
 $E_{k_{\widetilde \mu}}^{(n-1)} \cap E_{k_{ \mu}}^{(n-1)}  \not=\emptyset$. As $(E_k^{(n-1)})_{1 \leq k \leq 2^{n-1}}$   is a partition of the set $D_b$, we get      $k_\mu = k_{\widetilde\mu}$.

\item If $k_\mu >   2^{n-1}$,  by replacing $Q_n$ by $\overline{Q_n}$ in the previous reasoning, we obtain in this case $k_\mu = k_{\widetilde\mu} + 2^{n-1}$.
\end{itemize}
 
\end{proof}

\subsection{Proof of Proposition \ref{prop:isom}}
\label{subsec:proof:propisom}
 \propisomcmd*
\begin{proof}
\noindent
To prove   the equivalence 
\[ \mu \preceq^{(n)} \mu' \Longleftrightarrow k_\mu \leq k_{\mu'} \quad \text{for all} \quad \mu, \mu' \in \Lambda^{(n)}, \]
which means that $\Psi$ is an isomorphism for order relations,
we proceed by induction on $n$ and rely on Lemma \ref{lemma:muto}.

\noindent
$\bullet\,$ Case  $n = 1$. From  Lemma \ref{lemma:LambdaN} and the definition  (\ref{eq:partition}) of the partition $(E_k^{(1)})_{1 \leq k \leq 2}$, we  deduce: 
  \[ \Lambda^{(1)} = \begin{cases} \{(1)\}& \text{ if } \quad \overline{Q_1} = \emptyset\\
 \{(-1)\} & \text{ if } \quad Q_1 = \emptyset\\
 \{(1),(-1)\} & \text{otherwise.}\end{cases} \quad \text{and} \quad 
 J^{(1)} = \begin{cases} \{1\}& \text{ if }\quad  \overline{Q_1} = \emptyset\\
 \{2\} & \text{ if }\quad Q_1 = \emptyset\\
 \{1,2\} & \text{otherwise.}
\end{cases}\]
\noindent
If $\overline{Q_1} = \emptyset$ or  $Q_1 = \emptyset$, the mapping  $\Psi : \Lambda^{(1)} \rightarrow J^{(1)}$ respects order relations. In the remaining case, we have   $(1) \prec^{(1)} (-1)$ and 
$\Psi((1)) = 1 < 2 = \Psi((-1))$.

\noindent
$\bullet\,$ Suppose that $n > 1$ and  the mapping 
\begin{center}
\begin{tabular}{ll}
$\widetilde \Psi :$ & $\Lambda^{n-1} \rightarrow J^{(n-1)}$\\
&  $\widetilde \mu \mapsto k_{{\widetilde \mu}}$ 
\end{tabular}
\end{center}
respects order relations. Let us prove that the mapping 
\begin{center}
\begin{tabular}{ll}
$\Psi :$ & $\Lambda^{n} \rightarrow J^{(n)}$\\
&  $ \mu \mapsto k_{\mu}$ 
\end{tabular}
\end{center}
respects order relations.

\noindent Let  $\mu = (t_1, t_2, \dots, t_n), \mu' = (t'_1, t'_2, \dots, t'_n) $ be in $\Lambda^{(n)}$. We must prove the equivalence  
$\mu \preceq^{(n)} \mu' \Longleftrightarrow k_\mu \leq k_{\mu'}  $.

\noindent Set      $\widetilde\mu = (t_1, t_2, \dots, t_{n-1})$ and $\widetilde{\mu'} = (t'_1, t'_2, \dots, t'_{n-1})$, so  $\mu = (\widetilde\mu, t_n)$ and $\mu' = (\widetilde{\mu'}, t'_n)$. 
From Definition \ref{def:LambdaN} and Lemma \ref{lemma:LambdaN}, we know  that:
\[ \widetilde \mu\in \Lambda^{(n-1)} \text{ and } t_n = \pm n \quad ;\quad \widetilde{\mu'}\in \Lambda^{(n-1)}  \text{ and } t'_n = \pm n\]
We remind   that the order relations $\preceq^{(n)}$ and $\preceq^{(n-1)}$ satisfy the following equivalence, see Definition \ref{def:orderrel}:
\[ \mu \preceq^{(n)} \mu' \Longleftrightarrow 
[t_n = n \text{ and } t'_n = -n ]\quad \text{or} \quad 
[\widetilde \mu \preceq^{(n -1)}\widetilde{\mu'} \text{ and } t_n = t'_n].\]
\noindent
To prove the equivalence  $\mu \preceq^{(n)} \mu' \Longleftrightarrow k_\mu \leq k_{\mu'}  $, we rely on (\ref{eq:kmu10}) and the induction hypothesis:
\begin{itemize}
    \item If  $\mu \preceq^{(n)} \mu'$  and  $t_n = n $  and $t'_n = -n$, then we deduce from (\ref{eq:kmu10})  $k_\mu \leq 2^{n-1} < k_{\mu'}$. 
    \item If  $\mu \preceq^{(n)} \mu'$  and  $t_n = t'_n $,  then  we have, see Definition \ref{def:orderrel}, 
  $\widetilde \mu \preceq^{(n -1)}\widetilde{\mu'} $. By the induction hypothesis,  we know that $k_{\widetilde \mu} \leq k_{\widetilde {\mu'}}$ and from  (\ref{eq:kmu10}) and  $t_n = t'_n$, it is easy to deduce the inequality $k_\mu \leq k_{\mu'}$.
\end{itemize}
 \noindent We have proven the implication $\mu \preceq^{(n)} \mu' \Longrightarrow k_\mu \leq k_{\mu'}  $.

\noindent
Suppose now that we have $ k_\mu \leq k_{\mu'}$,  then the indices $k_\mu$ and $k_{\mu'}$ necessarily satisfy:
\[ k_\mu \leq 2^{n-1} < k_{\mu'}\quad \text{or} \quad 
k_\mu \leq k_{\mu'} \leq 2^{n-1}\quad \text{or} \quad 
 2^{n-1} < k_\mu \leq k_{\mu'}.\]
 \noindent
By relying on (\ref{eq:kmu10}) and Definition \ref{def:orderrel}, we deduce:
 \begin{itemize}
     \item If $k_\mu \leq 2^{n-1} < k_{\mu'}$, then $t_n = n$ and  $t'_n = -n$,  which implies $\mu \preceq^{(n)} \mu'    $.
     \item If $k_\mu \leq k_{\mu'} \leq 2^{n-1}$, then $t_n = t'_n =n$ and $k_{\widetilde \mu} = k_\mu \leq k_{\mu'} = k_{\widetilde{\mu'}}$. By the induction hypothesis, we know that $\widetilde \mu \preceq^{(n -1)}\widetilde{\mu'} $ and by definition of the order relation $\preceq^{(n)}$, we have  $\mu \preceq^{(n)} \mu'$.
     \item If $2^{n-1} < k_\mu \leq k_{\mu'}  $, then $t_n = t'_n = - n$ and $k_{\widetilde \mu} = k_\mu  - 2^{n-1}\leq k_{\mu'} - 2^{n-1}= k_{\widetilde{\mu'}}$. We get $\mu \preceq^{(n)} \mu'    $ as in the previous case.
 \end{itemize}

\noindent
 We have proved the implication $\mu \preceq^{(n)} \mu' \Longleftarrow k_\mu \leq k_{\mu'}  $.

\end{proof}
\subsection{Proof of Proposition \ref{prop:complexity:nbOp}}
\label{subsec:proof:propnumberop}
In order to prove Proposition \ref{prop:complexity:nbOp}, we rely on  Lemma \ref{lemma:LambdaN} and the following two additional lemmas:
\begin{lemma}\label{comp1}
Let $n \geq 2$. Then, we have the following partition of the set $\Lambda^{(n-1)}$:
$$\Lambda^{(n -1)} = X_{n -1} \cup Y_{n -1} \cup Z_{n -1} $$
where:\\
\noindent $\bullet$ $X_{n-1} = \{\mu =(t_1, t_2, \dots, t_{n-1})\in \Lambda^{(n -1)}\, \mid 
 \sigma^{n-1}(\mu) \cap Q_n  \not=\emptyset \,\text{ and }\, 
 \sigma^{n-1}(\mu) \cap \overline{Q_n}  =\emptyset \}$,
\noindent $\bullet$ $Y_{n-1} = \{\mu =(t_1, t_2, \dots, t_{n-1})\in \Lambda^{(n -1)}\, \mid 
 \sigma^{n-1}(\mu) \cap Q_n =\emptyset \,\text{ and }\, 
 \sigma^{n-1}(\mu) \cap \overline{Q_n}  \not=\emptyset \}$,
\noindent $\bullet$ $Z_{n-1} = \{\mu =(t_1, t_2, \dots, t_{n-1})\in \Lambda^{(n -1)}\,  \mid 
 \sigma^{n-1}(\mu) \cap Q_n \not=\emptyset \,\text{ and }\, 
 \sigma^{n-1}(\mu) \cap \overline{Q_n}  \neq\emptyset \}$.


    
\end{lemma}
\begin{proof}
For any  $\mu\in \Lambda^{(n-1)}$, the subset $\sigma^{(n-1)}(\mu)$ of the set $D_b$ is non-empty, see Definition \ref{def:LambdaN}. As we have the obvious partition $D_b = Q_n \cup \overline{Q_n}$ and the equality:
    \[ \sigma^{(n-1)}(\mu) = (\sigma^{(n-1)}(\mu)\cap Q_n)\cup (\sigma^{(n-1)}(\mu)\cap \overline{Q_n})\]
it follows that  $\Lambda^{(n -1)} = X_{n -1} \cup Y_{n -1} \cup Z_{n -1} $  is a partition of the set $\Lambda^{(n-1)}$.    
\end{proof}

\noindent We rely on Lemma \ref{lemma:LambdaN} and reuse the notations of   Lemma \ref{comp1}  to  prove:
\begin{lemma}\label{comp2}~
We have:
\begin{itemize}
   \item $\text{card}\,(X_{n-1} \cup Z_{n-1}) = \text{card}\,(\Lambda^{(n)}_+)$,
   \item $\text{card}\,(Y_{n-1} \cup Z_{n-1}) = \text{card}\,(\Lambda^{(n)}_-)$,
    \item $\, 2\cdot \,\text{card}\,(Z_{n-1}) \leq \text{card}\,(\Lambda^{(n)}) \leq \text{card}\,(D_b)$.
\end{itemize}
\end{lemma}
\begin{proof}
Using Lemma \ref{lemma:LambdaN} and the definition of the set  $X_{n-1}, Y_{n-1}, Z_{n-1}$, we get that the two  mappings:
\begin{center}
\begin{tabular}{l}
$X_{n-1} \cup Z_{n-1} \rightarrow \Lambda^{(n)}_+$\\
$(t_1, t_2, \dots, t_{n-1}) \mapsto   (t_1, t_2, \dots, t_{n-1}, n)$ 
\end{tabular}
\end{center}
\noindent and
\begin{center}
\begin{tabular}{l}
$Y_{n-1} \cup Z_{n-1} \rightarrow \Lambda^{(n)}_-$\\
$(t_1, t_2, \dots, t_{n-1}) \mapsto   (t_1, t_2, \dots, t_{n-1} ,-n)$
\end{tabular}
\end{center}

\noindent are bijective.

\noindent
Then, we deduce:
\begin{align}
\, 2\,  \text{card}\,(Z_{n-1}) & \leq 
\text{card}\,(\Lambda^{(n)}_+) + 
\text{card}\,(\Lambda^{(n)}_-) \nonumber\\
&= \text{card}\,(\Lambda^{(n)}) \nonumber\\
& \leq 
\text{card}\,(D_b). \nonumber   
\end{align}
\end{proof}

\noindent We remind Proposition \ref{prop:complexity:nbOp} and prove it:
\propnumberopcmd*
\begin{proof}
\noindent
Let $\beta_n$  be the number of operations required to construct the set $\Lambda^{(n)}$.  We have $\beta_1 = 1$. Let us  prove:
\[ \forall n \geq 2,\, \beta_n = \beta_{n-1} + \omega^{(n-1)} + \text{card}\,(Z_{n-1}) \quad  \text{where} \quad \omega^{(n-1)} = \text{card}\, (\Lambda^{(n-1)}).\]
To construct $\Lambda^{(n)}$  from $\Lambda^{(n-1)}$, we have to construct   the two sets $\Lambda^{(n)}_+$  and  $\Lambda^{(n)}_-$ from $\Lambda^{(n-1)}$, which by    Lemma \ref{comp2} requires 
$\text{card}\,(X_{n-1} \cup Z_{n-1}) + \text{card}\,(Y_{n-1} \cup Z_{n-1})$ operations. Using Lemma \ref{comp1}, we deduce:
\begin{align}
 \beta_n   & = \beta_{n-1} + 
 \text{card}\,(X_{n-1} \cup Z_{n-1}) + \text{card}\,(Y_{n-1}\cup Z_{n-1})\nonumber\\
 &= \beta_{n-1} + 
 \text{card}\,(X_{n-1}) + \text{card}\,(Y_{n-1}) + 2 \,\text{card}\, (Z_{n-1})\nonumber\\
 & = \beta_{n-1} + \omega^{(n-1)} + \text{card}\, (Z_{n-1}).\nonumber\\
\end{align}
 \noindent
As we  have $\omega^{(n-1)} \leq \text{card}\, (D_b)$, see Lemma \ref{lemma:mutokmu}, 
and by  the  third statement of Lemma \ref{comp2}, we deduce   $\text{card}\,(Z_{n-1}) \leq \dfrac{\text{card}\,(D_b)}{2}$, we get the inequality: 
\[ \beta_n \leq \beta_{n-1} + \dfrac{3\cdot \text{card}(D_b)  }{2}.\]
\noindent
Using that $\beta_1 = 1$, we obtain  that
$\forall n \geq 1,\, \beta_n \leq 1 + \dfrac{3\cdot \text{card}(D_b) \cdot (n-1)}{2} \leq \dfrac{3\cdot \text{card}(D_b) \cdot n}{2}$.
\end{proof}

\subsection{Proof of Lemma \ref{lemma:hmutokmu}}
\label{subsec:proof:Hmu}
\Hmucmd*
\begin{proof}
We proceed by induction on $n$ and rely on   (\ref{eq:matrixmatrixrel}) and   (\ref{eq:kmu10}).

\noindent
$\bullet\,$ Case  $n = 1$. From  Lemma \ref{lemma:LambdaN} and  the definition of the partition $(E_k^{(1)})_{1 \leq k \leq 2}$, see (\ref{eq:partition}), we  deduce: 
  \[ \Lambda^{(1)} = \begin{cases} \{(1)\}& \text{ if } \quad \overline{Q_1} = \emptyset\\
 \{(-1)\} & \text{ if } \quad Q_1 = \emptyset\\
 \{(1),(-1)\} & \text{otherwise.}\end{cases} \quad \text{and} \quad 
 J^{(1)} = \begin{cases} \{1\}& \text{ if }\quad  \overline{Q_1} = \emptyset\\
 \{2\} & \text{ if }\quad Q_1 = \emptyset\\
 \{1,2\} & \text{otherwise.}
\end{cases}\]
\noindent
We remind that $M_1 = \begin{bmatrix}
s_1 & 1 \\
1 & r_1
\end{bmatrix}   $. Set $\mu = (1)$ and $\mu' = (-1)$. We distinguish the following cases: 

\begin{itemize}
\item  If  $\Lambda^{(1)} = \{(1)\}$ (respectively   $\Lambda^{(1)} = \{(-1)\}$),  then  we have $k_\mu = 1$
(respectively  $k_{\mu'} = 2$) and $H_\mu$ is the first row  (respectively $H_{\mu'} $ is the second row) of the matrix $M_1$.

\item  If  $\Lambda^{(1)} = \{(1),(-1)\}$, then  $k_\mu = 1$, 
   $k_{\mu'} = 2$  and $H_\mu = \begin{bmatrix}
       s_1 & 1
   \end{bmatrix}$,    $H_{\mu'} = \begin{bmatrix}
       1 & r_1
   \end{bmatrix}$, thus $H_\mu$   (respectively $H_{\mu'}$)  is the first row  (respectively   is the second row) of the matrix $M_1$.
\end{itemize}

\medskip
   \noindent
$\bullet\,$ Case  $n > 1$.  Suppose that for each $\widetilde\mu\in \Lambda^{(n-1)}$, the row matrix $H_{\widetilde\mu}$ is equal to    $k_{\widetilde\mu}$-th row     of the matrix $M_{n-1}$. 
Let $\mu = (t_1, \dots, t_{n-1}, t_n)\in \Lambda^{(n)}$. Set $\widetilde\mu = (t_1, \dots, t_{n-1})$. Then, by Lemma  \ref{lemma:LambdaN},    we have  $\widetilde\mu\in \Lambda^{(n-1)}$ and $t_n = \pm n$.

\noindent
Using  (\ref{eq:taumapping}), we get $H_\mu = \begin{bmatrix}
    H_{\widetilde\mu} &  L_{t_n}
\end{bmatrix}$. We distinguish the following cases: 

\begin{itemize}
\item If $t_n = n$, then by (\ref{eq:kmu10}), we have $k_\mu = k_{\widetilde\mu} \leq 2^{n-1}$.  Using  (\ref{eq:matrixmatrixrel})  and the induction hypothesis, we get:
 $$H_\mu = \begin{bmatrix}
    H_{\widetilde\mu}  & L_{n}
\end{bmatrix} = \begin{bmatrix}
    H_{\widetilde\mu}  & s_n  & 1 
\end{bmatrix}$$
is the $k_{\mu}$-th row  of the matrix $M_n$.

\item If $t_n = - n$, then by (\ref{eq:kmu10}), we have $k_\mu = k_{\widetilde\mu} +  2^{n-1} > 2^{n-1}$.  Using  (\ref{eq:matrixmatrixrel})  and the induction hypothesis, we get:
 $$H_\mu = \begin{bmatrix}
    H_{\widetilde\mu}  & L_{-n}
\end{bmatrix} = \begin{bmatrix}
    H_{\widetilde\mu}  & 1  & r_n 
\end{bmatrix}$$
is the $k_{\mu}$-th row  of the matrix $M_n$.
\end{itemize}

\end{proof}

\subsection{Proof of Lemma \ref{lemma:coM}}
\label{subsec:proof:coM}

\noindent
The idea behind the proof of the  formulas (\ref{eq:coefredM}) in Lemma \ref{lemma:coM} is based on:
\begin{enumerate}
    \item The use of formulas for computing  the coefficients of the unreduced matrix $M_n$ given  in  \cite{baaj2022learning}, which are based on the $n$   partitions $ (I_j^{n,\top}, I_j^{n, \bot})$ of  $I_n=\{1, 2, \dots, 2^n\}$.
    \item The following lemma,    which asserts that    for all   $j\in\{1, 2, \dots,n\}$,  the  bijective mapping  
\begin{center}
\begin{tabular}{ll}
$\Psi : $ & $\Lambda^{(n)} \rightarrow J^{(n)}$\\
& $\mu  \rightarrow  k_\mu$ 
\end{tabular}
\end{center}
    (see Lemma \ref{lemma:mutokmu}) allows us to map    the partition   $ (\Lambda_j^{n,\top}, \Lambda_j^{n, \bot})$ of   the set  $\Lambda^{(n)}$, see (\ref{eq:Lnpart}), onto the partition  $(I_j^{n,\top} \cap J^{{n)}}, I_j^{n, \bot} \cap J^{(n)})$ of the set $J^{(n)}$.
\end{enumerate}

\begin{lemma}\label{lemma:Ltb}
Let $\mu =(t_1, t_2, \dots, t_n)\in \Lambda^{(n)}$
and $j\in\{1, 2, \dots, n\}$. Then, we have: 
\begin{enumerate}
    \item $k_\mu \in I_j^{n, \top} \Longleftrightarrow \mu\in \Lambda_j^{n, \top}$.
    \item $k_\mu \in I_j^{n, \bot} \Longleftrightarrow \mu\in \Lambda_j^{n, \bot}$.
\end{enumerate}
\end{lemma}

\begin{proof}
We remind that by Definition  \ref{def:Lj},  $\mu\in  \Lambda_j^{n, \top}$ (respectively $\mu\in \Lambda_j^{n, \bot}$) means that   $t_j = j$  (respectively $t_j = - j$).  
 We prove the two equivalences by induction on $n$.

\smallskip
\noindent For $n = 1$, we have:
 \[ \Lambda^{(1)} = \begin{cases} \{(1)\}& \text{ if } \quad \overline{Q_1} = \emptyset\\
 \{(-1)\} & \text{ if } \quad Q_1 = \emptyset\\
 \{(1),(-1)\} & \text{otherwise.}\end{cases}, \quad 
 J^{(1)} = \begin{cases} \{1\}& \text{ if }\quad  \overline{Q_1} = \emptyset\\
 \{2\} & \text{ if }\quad Q_1 = \emptyset\\
 \{1,2\} & \text{otherwise.}
\end{cases}\]
\noindent \text{and} \[ \quad I_1^{1, \top} = \{1\}, \quad I_1^{1,\bot} = \{2\}. \]
\noindent
By using the bijective mapping  
\begin{center}
\begin{tabular}{ll}
$\Psi : $ & $\Lambda^{(1)} \rightarrow J^{(1)}$\\
& $\mu  \mapsto  k_\mu$ 
\end{tabular}
\end{center}
(see Lemma \ref{lemma:mutokmu}), we check that the two equivalences hold.

 \smallskip
\noindent
Assume now that $n > 1$ and that both equivalences hold  for $n - 1$. 
Let $\widetilde\mu = (t_1, t_2, \dots, t_{n-1})\in \Lambda^{(n-1)}$, so we have  $\mu = (\widetilde\mu, t_n)$  and $t_n = \pm n$. 
We rely on Lemma \ref{lemma:muto}
  and the following equalities:
\begin{equation}\label{eq:equal}
 \forall j\in \{1, 2, \dots, n-1\},\,
I_j^{n, \top} = I_j^{n -1, \top}\cup \tau_{2^{n-1}}I_j^{n - 1, \top} \text{ and }
I_j^{n, \bot} = I_j^{n -1, \bot}\cup \tau_{2^{n-1}}I_j^{n - 1, \bot}     
\end{equation}

\noindent where $\tau_{2^{n-1}}$ is the mapping $k \mapsto k + 2^{n-1}$.\\ 
$\bullet\,$
If $t_n = n\,$, i.e., $\,\mu\in\Lambda_n^{n,\top}$,  then we have: 
\[ k_\mu = k_{\widetilde\mu} \leq 2^{n-1}  \quad \text{(Lemma \ref{lemma:muto})}.\]
Thus, for $1 \leq j \leq n -1$, the two equivalences follow from the induction hypothesis and (\ref{eq:equal}).
For $j = n$, we have $I_n^{n, \top} = \{1, 2, \dots, 2^{n-1}\}$  and   $k_\mu \in I_n^{n, \top}$.

\noindent $\bullet\,$
If $t_n = -n\,$, i.e.,  $\,\mu\in\Lambda_n^{n,\bot}$, then we have: 
\[ k_\mu = k_{\widetilde\mu} + 2^{n-1}  \quad \text{(Lemma \ref{lemma:muto})}.  \]
Thus, for $1 \leq j \leq n -1$, the two equivalences follow from the induction hypothesis and (\ref{eq:equal}).
For $j = n$, we have $I_n^{n, \bot} = \{2^{n-1 }+ 1,   \dots, 2^{n}\}$ and   $k_\mu \in I_n^{n, \bot}$.

\end{proof}
\noindent We are now in position to prove Lemma \ref{lemma:coM}:
\lemcoMcmd*
\begin{proof}
\cite{baaj2022learning} proved the following formulas for the coefficients of the matrix $M_n = [m_{ij}]_{1 \leq i \leq 2^n, 1 \leq j \leq 2n}$:    
\begin{equation}\label{eq:coefMrap1}
   m_{i, 2j-1} =
   \begin{cases}
 s_j     & \text{ if }  i\in I_j^{n, \top} \\
  1 & \text{ if }  i\in I_j^{n, \bot}  
  \end{cases}
, \quad 
     m_{i, 2j} =
   \begin{cases}
 1     & \text{ if }  i\in I_j^{n, \top} \\
  r_j & \text{ if }  i\in I_j^{n, \bot} 
  \end{cases}.
  \end{equation}
As by Lemma \ref{lemma:hmutokmu}, the row of the matrix ${\cal M}_n$ with index $\mu\in \Lambda^{(n)}$ is equal to the row of the matrix $M_n$ with index $k_\mu$, we immediately deduce  from Lemma \ref{lemma:Ltb} and (\ref{eq:coefMrap1}) the formulas: 
\begin{equation} 
 \dot{m}_{\mu, 2j-1}  = m_{k_\mu, 2j -1} = 
   \begin{cases}
 s_j     & \text{ if }  \mu\in \Lambda_j^{n, \top} \\
  1 & \text{ if }  \mu\in \Lambda_j^{n, \bot}  
  \end{cases}
, \quad 
   \dot{m}_{\mu, 2j} = m_{k_\mu, 2j} =
   \begin{cases}
 1     & \text{ if }  \mu\in \Lambda_j^{n, \top} \\
  r_j & \text{ if }  \mu\in \Lambda_j^{n, \bot}  
   \end{cases}.
 \nonumber\end{equation}
\end{proof}

\subsection{Proof of Theorem \ref{theorem:YO}}
\label{subsec:proof:YO}
\theoremYOcmd*
\begin{proof}
First, let us recall that that we have $\dot{Y}_n = [\widetilde{\Pi}(\sigma^{(n)}(\mu)]_{\mu\in \Lambda^{(n)}}$ where 
$\sigma^{(n)}(\mu) = E_{k_\mu}^{(n)}$, see (\ref{eq:buildupsigman}) and Lemma \ref{lemma:mutokmu}. From  (\ref{eq:PiO}), we also have 
${\cal O}_n = [\Pi(E_{k_\mu}^{(n)}]_{\mu\in \Lambda^{(n)}}$.

\noindent
$\bullet\,$ As $X = [x_l]_{1\leq l \leq 2n}$ is a solution of the system  $(\Upsigma_n)$, we have for each $\mu\in \Lambda^{(n)}$:
\begin{align}\label{eq:A}
\widetilde{\Pi}(E_{k_\mu}^{(n)}) & = \min_{1 \leq l \leq 2n} \max(\dot{\gamma}_{\mu, l},  x_l)\nonumber\\
& = \min_{1 \leq j \leq n}\min(\max(\dot{\gamma}_{\mu, 2j -1}, x_{2j -1}), \max(\dot{\gamma}_{\mu, 2j}, x_{2j}))\nonumber\\
& = \min_{1 \leq j \leq n}\min(\max(\dot{\gamma}_{\mu, 2j -1}, s_j), \max(\dot{\gamma}_{\mu, 2j}, r_j))\\   
\nonumber\end{align}
\noindent where $\dot{\gamma}_{\mu, l}$ for $\mu\in \Lambda^{(n)}, 1\leq l \leq 2n$ are the coefficients of the matrix $\dot{\Gamma}_n$ of the system $(\Upsigma_n)$, see Lemma \ref{lem:coG}.\\
$\bullet\,$ As ${\cal O}_n := {\cal M}_n \Box_{\max}^{\min} I_n$, from (\ref{eq:PiO}), we get for each $\mu\in \Lambda^{(n)}$:
\begin{equation}\label{eq:B}
\Pi(E_{k_\mu}^{(n)})  
  = \min_{1 \leq j \leq n}\min(\max(\dot{m}_{\mu, 2j -1}, \lambda_j), \max(\dot{m}_{\mu, 2j}, \rho_{j})).    
\end{equation}
But, by using the commutativity of $\max$ and the formulas (\ref{eq:coefredM}) and (\ref{eq:coefredG}),  we get   for all $j\in\{1, 2, \dots n\}$ and $\mu\in \Lambda^{(n)}$: 
\[ \max(\dot{\gamma}_{\mu, 2j -1}, s_j) = \begin{cases}
 \max(\lambda_j, s_j)     & \text{ if }  \mu\in \Lambda_j^{n, \top} \\
  1 & \text{ if }  \mu\in \Lambda_j^{n, \bot}  
  \end{cases}=\max(\dot{m}_{\mu, 2j -1}, \lambda_j);  \]
\[ \max(\dot{\gamma}_{\mu, 2j}, r_j) = \begin{cases}
 1     & \text{ if }  \mu\in \Lambda_j^{n, \top} \\
  \max(\rho_j, r_j) & \text{ if }  \mu\in \Lambda_j^{n, \bot}  
   \end{cases} = \max(\dot{m}_{\mu, 2j}, \rho_j).  \]
It then follows from (\ref{eq:A}) and (\ref{eq:B}) that we have $\widetilde{\Pi}(E_{k_\mu}^{(n)}) = \Pi(E_{k_\mu}^{(n)})$ for all $\mu\in\Lambda^{(n)}$.
\end{proof}

\subsection{Proof of Lemma \ref{lem:coED}}
\label{subsec:proof:coED}
\lemcoEDcmd*

\begin{proof}
 Let us remind that for all $x, y, z$  in $[0, 1]$, we have:
 \[   x \epsilon \max(y, z) = \max(x \, \epsilon \, y, x \,\epsilon \, z), \quad \max(x,  x \,\epsilon \, y) = \max(x, y). \]
 Using (\ref{eq:coefredG}) and      for any $j\in \{1, 2, \dots n\}$, the partition $\Lambda^{(n)} = \Lambda^{n, \top}_j \cup \Lambda^{n, \bot}_j$ (see Definition  \ref{def:Lj}), we get:
\[ e_{2j -1}^{n, \downarrow}   = \max_{\mu\in \Lambda^{(n)}}\, \dot{\gamma}_{\mu, 2j - 1} \,\epsilon \dot{y}_{\mu}= \max_{\mu\in \Lambda_j^{n, \top}}\, \lambda_j  \,\epsilon \dot{y}_\mu=   \lambda_j  \,\epsilon \,\max_{\mu\in \Lambda_j^{n, \top}}\dot{y}_\mu 
  =   \lambda_j  \,\epsilon \,e_{2j -1}^{n, \uparrow}. \]
  \[ e_{2j}^{n, \downarrow}   = \max_{\mu\in \Lambda^{(n)}}\, \dot{\gamma}_{\mu, 2j} \,\epsilon \dot{y}_{\mu}= \max_{\mu\in \Lambda_j^{n, \bot}}\, \rho_j  \,\epsilon \dot{y}_\mu=   \rho_j  \,\epsilon \,\max_{\mu\in \Lambda_j^{n, \bot}}\dot{y}_\mu 
  =   \rho_j  \,\epsilon \,e_{2j}^{n, \uparrow}. \]
\noindent
From the obvious inequality $x \, \epsilon \, y \leq y$, we deduce that for all $l\in\{1, 2, \dots, 2n\}$, we have  $e_{l}^{n, \downarrow} \leq e_{l  }^{n, \uparrow}$.
\end{proof}

\subsection{Proof of Proposition \ref{prop:coEDU}}
\label{subsec:proof:coEDU}
\propcoEDUcmd*

\begin{proof}
Let
$\dot{\Gamma_n} \Box_{\max}^{\min}{E}^{n,\uparrow} = [v_\mu]_{\mu\in \Lambda^{(n)}}$, we prove that for any $\mu\in \Lambda^{(n)}$, we have:
\[v_\mu = \min_{1 \leq j \leq n} z_{\mu, j} = u_\mu \quad \text{where} \quad 
z_{\mu, j} := 
   \begin{cases}
 \max(\lambda_j, e_{2j-1}^{n, \uparrow})     & \text{ if }  \mu\in \Lambda_j^{n, \top}   \\
  \max(\rho_j, e_{2j}^{n, \uparrow}) & \text{ if }  \mu\in \Lambda_j^{n, \bot}  
  \end{cases}.\]
In fact, we have:\[ u_\mu  = \min_{1 \leq l \leq 2n} \, 
  \max(\dot{\gamma}_{\mu, l}, e_l^{n, \downarrow}) 
   = \min_{1 \leq j \leq n} \, 
  \min(\max(\dot{\gamma}_{\mu, 2j -1}, e_{2j -1}^{n, \downarrow}), \max(\dot{\gamma}_{\mu, 2j }, e_{2j}^{n, \downarrow})),\]\noindent
   \[v_\mu = \min_{1 \leq l \leq 2n} \, 
  \max(\dot{\gamma}_{\mu, l}, e_l^{n, \uparrow}) 
   = \min_{1 \leq j \leq n} \, 
  \min(\max(\dot{\gamma}_{\mu, 2j -1}, e_{2j -1}^{n, \uparrow}), \max(\dot{\gamma}_{\mu, 2j }, e_{2j}^{n, \uparrow})).\]
 
 Using (\ref{eq:coefredG}), we can easily check that for all $j\in \{1, 2, \dots, n\}$, we have: 
 \[ \min(\max(\dot{\gamma}_{\mu, 2j -1}, e_{2j -1}^{n, \downarrow}), \max(\dot{\gamma}_{\mu, 2j }, e_{2j}^{n, \downarrow})) = 
 \begin{cases} 
 \max(\lambda_j, e_{2j-1}^{n, \downarrow})     & \text{ if }  \mu\in \Lambda_j^{n, \top}   \\
  \max(\rho_j, e_{2j}^{n, \downarrow}) & \text{ if }  \mu\in \Lambda_j^{n, \bot}  
  \end{cases}.\]
\[ \min(\max(\dot{\gamma}_{\mu, 2j -1}, e_{2j -1}^{n, \uparrow}), \max(\dot{\gamma}_{\mu, 2j }, e_{2j}^{n, \uparrow})) = 
 \begin{cases} 
 \max(\lambda_j, e_{2j-1}^{n, \uparrow})     & \text{ if }  \mu\in \Lambda_j^{n, \top}   \\
  \max(\rho_j, e_{2j}^{n, \uparrow}) & \text{ if }  \mu\in \Lambda_j^{n, \bot}  
  \end{cases}  =  z_{\mu, j}   .\]

As by Lemma \ref{lem:coED},   we have 
$e_{2j -1}^{n, \downarrow} = \lambda_j  \,\epsilon \,e_{2j -1}^{n, \uparrow}$ and 
$e_{2j}^{n, \downarrow} = \rho_j  \,\epsilon \,e_{2j}^{n, \uparrow}$, 
  the well-known relation $\max(x, x \,\epsilon\, y) = \max(x, y)$ leads to the equality
$$z_{\mu, j} = \begin{cases} 
 \max(\lambda_j, e_{2j-1}^{n, \downarrow})     & \text{ if }  \mu\in \Lambda_j^{n, \top}   \\
  \max(\rho_j, e_{2j}^{n, \downarrow}) & \text{ if }  \mu\in \Lambda_j^{n, \bot}  
  \end{cases}$$ for all $j\in\{1, 2, \dots, n\}$ and we obtain 
  $v_\mu = \min_{1 \leq j \leq n} z_{\mu, j} = u_\mu$.\footnote{The equality  $u_\mu = \min_{1 \leq j \leq n} z_{\mu, j}   $ was proved  in \cite{baaj2022learning}, with other notations,  for the non-reduced equation system $(\Sigma_n)$. Unfortunately, the proof of Lemma 2 of \cite{baaj2022learning} contains misprints (a min being written as a max).}
  \end{proof}

\subsection{Proof of Lemma \ref{lem:coEDOMEGA}}
\label{subsec:proof:coEDOMEGA}
\lemcoEDcmdOMEGA*

\begin{proof}We adapt the proof of Lemma \ref{lem:coED}, see Subsection \ref{subsec:proof:coED}, to prove Lemma \ref{lem:coEDOMEGA}.
 Using (\ref{eq:coefredM}) and   considering   for any $j\in \{1, 2, \dots n\}$, the partition $\Lambda^{(n)} = \Lambda^{n, \top}_j \cup \Lambda^{n, \bot}_j$ (see Definition  \ref{def:Lj}), we get:
\[ \lambda_j^{n, \downarrow}   = \max_{\mu\in \Lambda^{(n)}}\, \dot{m}_{\mu, 2j - 1} \,\epsilon \widetilde o_{\mu}= \max_{\mu\in \Lambda_j^{n, \top}}\, s_j  \,\epsilon \widetilde o_{\mu}=   s_j  \,\epsilon \,\max_{\mu\in \Lambda_j^{n, \top}} \widetilde o_\mu 
  =  s_j  \,\epsilon \,\lambda_j^{n, \uparrow}. \]
  \[ \rho_j^{n, \downarrow}   = \max_{\mu\in \Lambda^{(n)}}\, \dot{m}_{\mu, 2j} \,\epsilon \widetilde o_{\mu}= \max_{\mu\in \Lambda_j^{n, \bot}}\, r_j  \,\epsilon \widetilde o_{\mu}=   r_j  \,\epsilon \,\max_{\mu\in \Lambda_j^{n, \bot}}\widetilde o_\mu
  =   r_j  \,\epsilon \, \rho_j^{n, \uparrow}. \]
\noindent
From the obvious inequality $x \, \epsilon \, y \leq y$, we deduce that for each $l\in\{1, 2, \dots, n\}$, we have  $\lambda_{l}^{n, \downarrow} \leq \lambda_{l  }^{n, \uparrow}$ and $\rho_{l}^{n, \downarrow} \leq \rho_{l  }^{n, \uparrow}$.
\end{proof}

\subsection{Proof of Proposition \ref{prop:coEDUOMEGA}}
\label{subsec:proof:coEDUOMEGA}
\propcoEDUcmdOMEGA*

\begin{proof}
We adapt the proof of Proposition \ref{prop:coEDU}, see Subsection \ref{subsec:proof:coEDU}, to prove Proposition \ref{prop:coEDUOMEGA}. Let
$\mathcal{M}_n \Box_{\max}^{\min}{F}^{n,\uparrow} = [v_\mu]_{\mu\in \Lambda^{(n)}}$, we prove that for any $\mu\in \Lambda^{(n)}$, we have:
\[v_\mu = \min_{1 \leq j \leq n} z_{\mu, j} = u_\mu \quad \text{where} \quad 
z_{\mu, j} := 
   \begin{cases}
 \max(s_j, \lambda_{j}^{n, \uparrow})     & \text{ if }  \mu\in \Lambda_j^{n, \top}   \\
  \max(r_j, \rho_{j}^{n, \uparrow}) & \text{ if }  \mu\in \Lambda_j^{n, \bot}  
  \end{cases}.\]
In fact, we have:\[ u_\mu    
   = \min_{1 \leq j \leq n} \, 
  \min(\max(\dot{m}_{\mu, 2j -1}, \lambda_{j}^{n, \downarrow}), \max(\dot{m}_{\mu, 2j }, \rho_{j}^{n, \downarrow})),\]\noindent
   \[v_\mu = \min_{1 \leq j \leq n} \, 
  \min(\max(\dot{m}_{\mu, 2j -1}, \lambda_{j}^{n, \uparrow}), \max(\dot{m}_{\mu, 2j }, \rho_{j}^{n, \uparrow})).\]
 
 Using (\ref{eq:coefredM}), we can easily check that for each $j\in \{1, 2, \dots, n\}$, we have: 
 \[ \min(\max(\dot{m}_{\mu, 2j -1}, \lambda_{j}^{n, \downarrow}), \max(\dot{m}_{\mu, 2j }, \rho_{j}^{n, \downarrow})) = 
 \begin{cases} 
 \max(s_j, \lambda_{j}^{n, \downarrow})     & \text{ if }  \mu\in \Lambda_j^{n, \top}   \\
  \max(r_j, \rho_{j}^{n, \downarrow}) & \text{ if }  \mu\in \Lambda_j^{n, \bot}  
  \end{cases}.\]
\[ \min(\max(\dot{m}_{\mu, 2j -1}, \lambda_{j}^{n, \uparrow}), \max(\dot{m}_{\mu, 2j }, \rho_{j}^{n, \uparrow})) = 
 \begin{cases} 
 \max(s_j, \lambda_{j}^{n, \uparrow})     & \text{ if }  \mu\in \Lambda_j^{n, \top}   \\
  \max(r_j, \rho_{j}^{n, \uparrow}) & \text{ if }  \mu\in \Lambda_j^{n, \bot}  
  \end{cases}  =  z_{\mu, j}   .\]

\noindent
As by Lemma \ref{lem:coEDOMEGA},   we have 
$\lambda_j^{n, \downarrow} = s_j  \,\epsilon \,\lambda_{j}^{n, \uparrow}$ and 
$\rho_{j}^{n, \downarrow} = r_j  \,\epsilon \,\rho_{j}^{n, \uparrow}$, 
  the well-known relation $\max(x, x \,\epsilon\, y) = \max(x, y)$ leads to the equality
$$z_{\mu, j} = \begin{cases} 
 \max(s_j, \lambda_{j}^{n, \downarrow})     & \text{ if }  \mu\in \Lambda_j^{n, \top}   \\
  \max(r_j, \rho_{j}^{n, \downarrow}) & \text{ if }  \mu\in \Lambda_j^{n, \bot}  
  \end{cases}$$ for each $j\in\{1, 2, \dots, n\}$ and we obtain 
  $v_\mu = \min_{1 \leq j \leq n} z_{\mu, j} = u_\mu$.
  \end{proof}

\subsection{Proof of Proposition \ref{prop:SOLOMEGA}}
\label{subsec:proof:SOLOMEGA}
\propSOLOMEGA*

\begin{proof}

\noindent
We remind that for each $j\in\{1 , 2 , \dots , n\}$ , the two sets
$\Lambda^{n, \top}_j , \Lambda^{n, \bot}_j$ form a partition of the set $\Lambda^{(n)}$, and that  we have an element $\mu\in \Lambda^{(n)}$ such that 
$\widetilde o_\mu = 1$, see  (\ref{eq:Nor}). From  the formulas  (\ref{eq:formulasejtopOMEGA}),   we conclude  that statement 1. holds: we have 
$\max(\lambda^{n , \uparrow}_j , \rho^{n , \uparrow}_j) = 1$.

Furthermore, for any $0 \leq x < 1$, we have $x \,\epsilon \,1 = 1$.

From the hypothesis of statement 2., the formulas (\ref{eq:eninitialformulaOMEGA}) and statement 1. (already proved), we deduce that for each $j\in\{1 , 2 , \dots , n\}$:
\[\max(\lambda^{n , \downarrow}_j , \rho^{n , \downarrow}_j) = 
\max(s_j \,\epsilon\, \lambda^{n , \uparrow}_j , r_j \,\epsilon \,\rho^{n , \uparrow}_j) = 1.\]

Since the mapping $(x , y) \mapsto \max(x , y)$ is increasing, given the hypothesis of statement 3. and the already proved statements 1. and 2., we deduce that:
\[ 1 = \max(\lambda^{n , \downarrow}_j , \rho^{n , \downarrow}_j) \leq 
\max(\lambda^{\ast}_j , \rho^{\ast}_j) \leq \max(\lambda^{n , \uparrow}_j , \rho^{n , \uparrow}_j) = 1.\]

This concludes the proof for statement 3.
\end{proof}

 \clearpage

 \section{Additional experimental results in normal settings}
\label{sec:appendix:expadd}

In this section, we provide an in-depth analysis of our possibilistic learning method (Method \ref{meth:learningincascade}), as applied in the two approaches described in Section \ref{sec:exp}: $\Pi$-NeSy-1, based on the antipignistic probability-possibility transformation, and $\Pi$-NeSy-2, which relies on the probability-possibility transformation based on the minimum specificity principle.

\subsection{\texorpdfstring{MNIST Addition-$k$ problems}{MNIST Addition k problems}} 
\label{subsec:appendix:posslearn:mnistadd} 

We analyzed our possibilistic learning method (Method \ref{meth:learningincascade}) applied to the MNIST Addition-$k$ problems  with $k \in \{1,2,4,15,100\}$, leading to the results presented in Table \ref{tab:mnist_add_results}.

For MNIST Addition-$k$ problems (see Subsection \ref{subsec:exp:mnistadd}), our possibilistic learning method uses three thresholds  to select reliable training data. These thresholds, referred to as Threshold \#1, Threshold \#2, and Threshold \#3, correspond to the sets of rules whose output attributes are $c_{\ast}$, $w_{\ast}$, and $y_{\ast}$, respectively. Threshold \#1 is used for learning the rule parameters of rules with output attributes $c_1, c_2, \dots, c_k$ (grouped as $c_{\ast}$). Similarly, Threshold \#2 and Threshold \#3 are associated with the output attributes $w_{\ast}$ and $y_{\ast}$, respectively, and are used for learning the rule parameters of rules with output attributes $w_1, w_2, \dots, w_k$ and $y_0, y_1, y_2, \dots, y_k$.

For MNIST Addition-$k$ problems with $k \in \{1, 2, 4, 15, 100\}$, both $\Pi$-NeSy-1 and $\Pi$-NeSy-2 consistently utilize the same threshold value of $4.12\times10^{-8}$ for Thresholds \#1, \#2, and \#3. The value $4.12\times10^{-8}$ is the smallest of the set $\mathcal{T}$ of candidate thresholds, see (\ref{eq:setTvalidationstep}), given the parameters defined for the experiments (see Section \ref{sec:exp}).

This finding highlights that the threshold values remain very close to zero across all values of $k$, indicating a highly restrictive selection of reliable training data. The absence of threshold variation as $k$ increases further suggests that this fixed value is suitable for all MNIST Addition-$k$ problem instances.

In Tables~\ref{tab:mnist-add-1-pi-nesy-1}--\ref{tab:mnist-add-100-pi-nesy-2}, we present  the average values of the rule parameter values $s$ and $r$ for the output attributes $c_{\ast}$, $w_{\ast}$, and $y_{\ast}$ obtained using our possibilistic learning method. For $c_{\ast}$, the average value is the mean of the rule parameter values $s$ (resp. $r$) across all rules whose output attributes are $c_1, c_2, \dots, c_k$, and similar calculations are performed for $w_{\ast}$ and $y_{\ast}$. Additionally, the tables provide the absolute number and percentage of reliable training samples, as well as the average and maximum memory usage for each experiment.
We can observe that:

\begin{itemize}
    \item The parameters $s$ and $r$ associated with the rules maintain average values close to zero with very low variability, which indicates that the learning process for the rule parameters is stable.
    
    \item  Larger $k$ values increase task complexity, while $N_{\text{train}}$, the total number of training samples, decreases for larger $k$. This results in slightly lower variability for the percentage of training data samples considered reliable. When $k$ increases, the absolute number and percentage of reliable training samples tend to decrease, except for the $w_\ast$ attributes. 

    \item Memory usage remains reasonable and decreases as k increases: from around 1350MB when k = 1 to around 920MB when k = 100, which is partly due to the fact that the percentage of training data samples considered reliable decreases when k increases.
\end{itemize}

\begin{table}[ht]
\centering
\setlength{\tabcolsep}{1.5pt}
\renewcommand{\arraystretch}{1.5}
\begin{tabular}{lcccc}
\toprule
\makecell{Set of rules\\(output attribute)} & \makecell{Average value of the\\rule parameters $s_{\ast}$} & \makecell{Average value of the\\rule parameters $r_{\ast}$} & \makecell{Average number of\\training data samples\\considered reliable} & \makecell{Percent of\\training data samples\\considered reliable} \\
\midrule
$c_{\ast}$ & 1.95e-08 $\pm$ 4.48e-09 & 1.85e-08 $\pm$ 4.61e-09 & 9018.20 $\pm$ 1606.28 & 36.073 $\pm$ 6.425\%\\
$w_{\ast}$ & 0.00e+00 $\pm$ 0.00e+00 & 0.00e+00 $\pm$ 0.00e+00 & 10851.10 $\pm$ 1521.76 & 43.404 $\pm$ 6.087\%\\
$y_{\ast}$ & 2.49e-08 $\pm$ 1.18e-08 & 0.00e+00 $\pm$ 0.00e+00 & 9934.65 $\pm$ 1813.23 & 39.739 $\pm$ 7.253\%\\
\midrule
\multicolumn{5}{l}{\text{Average memory (RAM) of an experiment:} 1353.76 $\pm$ 51.17 MB}\\
\multicolumn{5}{l}{\text{Maximum memory (RAM) of an experiment:} 2251.45 $\pm$ 64.23 MB}\\
\bottomrule
\end{tabular}
\caption{Results for MNIST Addition with $k=1$ ($N_{\text{train}}$ = 25000) using $\Pi$-Nesy-1 over 10 experiments. Averages for the rule parameter values ($s_{\ast}$ and $r_{\ast}$) for the $c_{\ast}$-, $w_{\ast}$-, and $y_{\ast}$-rules are computed from data comprising $10\cdot 20$, $10\cdot 1$, and $10\cdot 11$ observations, respectively. Reliable sample counts for the $c_{\ast}$-, $w_{\ast}$-, and $y_{\ast}$-rules are averaged over $10\cdot 1$, $10\cdot 1$, and $10\cdot 2$ observations, respectively, and memory metrics are averaged over 10 experiments.}
\label{tab:mnist-add-1-pi-nesy-1}
\end{table}

\begin{table}[ht]
\centering
\setlength{\tabcolsep}{1.5pt}
\renewcommand{\arraystretch}{1.5}
\begin{tabular}{lcccc}
\toprule
\makecell{Set of rules\\(output attribute)} & \makecell{Average value of the\\rule parameters $s_{\ast}$} & \makecell{Average value of the\\rule parameters $r_{\ast}$} & \makecell{Average number of\\training data samples\\considered reliable} & \makecell{Percent of\\training data samples\\considered reliable} \\
\midrule
$c_{\ast}$ & 1.95e-08 $\pm$ 4.48e-09 & 1.89e-08 $\pm$ 4.42e-09 & 9373.60 $\pm$ 1254.40 & 37.494 $\pm$ 5.018\%\\
$w_{\ast}$ & 0.00e+00 $\pm$ 0.00e+00 & 0.00e+00 $\pm$ 0.00e+00 & 11039.80 $\pm$ 1077.52 & 44.159 $\pm$ 4.310\%\\
$y_{\ast}$ & 2.50e-08 $\pm$ 1.18e-08 & 0.00e+00 $\pm$ 0.00e+00 & 10206.70 $\pm$ 1435.74 & 40.827 $\pm$ 5.743\%\\
\midrule
\multicolumn{5}{l}{\text{Average memory (RAM) of an experiment:} 1353.37 $\pm$ 43.42 MB}\\
\multicolumn{5}{l}{\text{Maximum memory (RAM) of an experiment:} 2240.44 $\pm$ 54.86 MB}\\
\bottomrule
\end{tabular}
\caption{Results for MNIST Addition with $k=1$ ($N_{\text{train}}$ = 25000) using $\Pi$-Nesy-2 over 10 experiments. Averages for the rule parameter values ($s_{\ast}$ and $r_{\ast}$) for the $c_{\ast}$-, $w_{\ast}$-, and $y_{\ast}$-rules are computed from data comprising $10\cdot 20$, $10\cdot 1$, and $10\cdot 11$ observations, respectively. Reliable sample counts for the $c_{\ast}$-, $w_{\ast}$-, and $y_{\ast}$-rules are averaged over $10\cdot 1$, $10\cdot 1$, and $10\cdot 2$ observations, respectively, and memory metrics are averaged over 10 experiments.}
\label{tab:mnist-add-1-pi-nesy-2}
\end{table}

\begin{table}[ht]
\centering
\setlength{\tabcolsep}{1.5pt}
\renewcommand{\arraystretch}{1.5}
\begin{tabular}{lcccc}
\toprule
\makecell{Set of rules\\(output attribute)} & \makecell{Average value of the\\rule parameters $s_{\ast}$} & \makecell{Average value of the\\rule parameters $r_{\ast}$} & \makecell{Average number of\\training data samples\\considered reliable} & \makecell{Percent of\\training data samples\\considered reliable} \\
\midrule
$c_{\ast}$ & 1.89e-08 $\pm$ 5.33e-09 & 1.68e-08 $\pm$ 5.28e-09 & 3048.50 $\pm$ 1337.70 & 24.388 $\pm$ 10.702\%\\
$w_{\ast}$ & 0.00e+00 $\pm$ 0.00e+00 & 0.00e+00 $\pm$ 0.00e+00 & 4742.00 $\pm$ 747.76 & 37.936 $\pm$ 5.982\%\\
$y_{\ast}$ & 2.57e-08 $\pm$ 1.07e-08 & 0.00e+00 $\pm$ 0.00e+00 & 3442.30 $\pm$ 1268.65 & 27.538 $\pm$ 10.149\%\\
\midrule
\multicolumn{5}{l}{\text{Average memory (RAM) of an experiment:} 1209.47 $\pm$ 54.21 MB}\\
\multicolumn{5}{l}{\text{Maximum memory (RAM) of an experiment:} 2030.62 $\pm$ 106.39 MB}\\
\bottomrule
\end{tabular}
\caption{Results for MNIST Addition with $k=2$ ($N_{\text{train}}$ = 12500) using $\Pi$-Nesy-1 over 10 experiments. Averages for the rule parameter values ($s_{\ast}$ and $r_{\ast}$) for the $c_{\ast}$-, $w_{\ast}$-, and $y_{\ast}$-rules are computed from data comprising $10\cdot 41$, $10\cdot 2$, and $10\cdot 21$ observations, respectively. Reliable sample counts for the $c_{\ast}$-, $w_{\ast}$-, and $y_{\ast}$-rules are averaged over $10\cdot 2$, $10\cdot 2$, and $10\cdot 3$ observations, respectively, and memory metrics are averaged over 10 experiments.}
\label{tab:mnist-add-2-pi-nesy-1}
\end{table}

\begin{table}[ht]
\centering
\setlength{\tabcolsep}{1.5pt}
\renewcommand{\arraystretch}{1.5}
\begin{tabular}{lcccc}
\toprule
\makecell{Set of rules\\(output attribute)} & \makecell{Average value of the\\rule parameters $s_{\ast}$} & \makecell{Average value of the\\rule parameters $r_{\ast}$} & \makecell{Average number of\\training data samples\\considered reliable} & \makecell{Percent of\\training data samples\\considered reliable} \\
\midrule
$c_{\ast}$ & 1.88e-08 $\pm$ 5.33e-09 & 1.70e-08 $\pm$ 5.22e-09 & 3654.85 $\pm$ 1410.19 & 29.239 $\pm$ 11.282\%\\
$w_{\ast}$ & 0.00e+00 $\pm$ 0.00e+00 & 0.00e+00 $\pm$ 0.00e+00 & 5375.40 $\pm$ 740.02 & 43.003 $\pm$ 5.920\%\\
$y_{\ast}$ & 2.57e-08 $\pm$ 1.05e-08 & 0.00e+00 $\pm$ 0.00e+00 & 4077.77 $\pm$ 1340.02 & 32.622 $\pm$ 10.720\%\\
\midrule
\multicolumn{5}{l}{\text{Average memory (RAM) of an experiment:} 1280.83 $\pm$ 57.95 MB}\\
\multicolumn{5}{l}{\text{Maximum memory (RAM) of an experiment:} 2159.32 $\pm$ 101.99 MB}\\
\bottomrule
\end{tabular}
\caption{Results for MNIST Addition with $k=2$ ($N_{\text{train}}$ = 12500) using $\Pi$-Nesy-2 over 10 experiments. Averages for the rule parameter values ($s_{\ast}$ and $r_{\ast}$) for the $c_{\ast}$-, $w_{\ast}$-, and $y_{\ast}$-rules are computed from data comprising $10\cdot 41$, $10\cdot 2$, and $10\cdot 21$ observations, respectively. Reliable sample counts for the $c_{\ast}$-, $w_{\ast}$-, and $y_{\ast}$-rules are averaged over $10\cdot 2$, $10\cdot 2$, and $10\cdot 3$ observations, respectively, and memory metrics are averaged over 10 experiments.}
\label{tab:mnist-add-2-pi-nesy-2}
\end{table}

\begin{table}[ht]
\centering
\setlength{\tabcolsep}{1.5pt}
\renewcommand{\arraystretch}{1.5}
\begin{tabular}{lcccc}
\toprule
\makecell{Set of rules\\(output attribute)} & \makecell{Average value of the\\rule parameters $s_{\ast}$} & \makecell{Average value of the\\rule parameters $r_{\ast}$} & \makecell{Average number of\\training data samples\\considered reliable} & \makecell{Percent of\\training data samples\\considered reliable} \\
\midrule
$c_{\ast}$ & 1.84e-08 $\pm$ 5.69e-09 & 1.41e-08 $\pm$ 6.72e-09 & 1265.83 $\pm$ 674.65 & 20.253 $\pm$ 10.794\%\\
$w_{\ast}$ & 0.00e+00 $\pm$ 0.00e+00 & 0.00e+00 $\pm$ 0.00e+00 & 2385.62 $\pm$ 385.31 & 38.170 $\pm$ 6.165\%\\
$y_{\ast}$ & 2.52e-08 $\pm$ 1.01e-08 & 0.00e+00 $\pm$ 0.00e+00 & 1463.20 $\pm$ 734.34 & 23.411 $\pm$ 11.749\%\\
\midrule
\multicolumn{5}{l}{\text{Average memory (RAM) of an experiment:} 1047.38 $\pm$ 40.41 MB}\\
\multicolumn{5}{l}{\text{Maximum memory (RAM) of an experiment:} 1625.33 $\pm$ 74.88 MB}\\
\bottomrule
\end{tabular}
\caption{Results for MNIST Addition with $k=4$ ($N_{\text{train}}$ = 6250) using $\Pi$-Nesy-1 over 10 experiments. Averages for the rule parameter values ($s_{\ast}$ and $r_{\ast}$) for the $c_{\ast}$-, $w_{\ast}$-, and $y_{\ast}$-rules are computed from data comprising $10\cdot 83$, $10\cdot 4$, and $10\cdot 41$ observations, respectively. Reliable sample counts for the $c_{\ast}$-, $w_{\ast}$-, and $y_{\ast}$-rules are averaged over $10\cdot 4$, $10\cdot 4$, and $10\cdot 5$ observations, respectively, and memory metrics are averaged over 10 experiments.}
\label{tab:mnist-add-4-pi-nesy-1}
\end{table}

\begin{table}[ht]
\centering
\setlength{\tabcolsep}{1.5pt}
\renewcommand{\arraystretch}{1.5}
\begin{tabular}{lcccc}
\toprule
\makecell{Set of rules\\(output attribute)} & \makecell{Average value of the\\rule parameters $s_{\ast}$} & \makecell{Average value of the\\rule parameters $r_{\ast}$} & \makecell{Average number of\\training data samples\\considered reliable} & \makecell{Percent of\\training data samples\\considered reliable} \\
\midrule
$c_{\ast}$ & 1.82e-08 $\pm$ 5.67e-09 & 1.42e-08 $\pm$ 6.67e-09 & 1460.75 $\pm$ 667.51 & 23.372 $\pm$ 10.680\%\\
$w_{\ast}$ & 0.00e+00 $\pm$ 0.00e+00 & 0.00e+00 $\pm$ 0.00e+00 & 2627.97 $\pm$ 290.17 & 42.048 $\pm$ 4.643\%\\
$y_{\ast}$ & 2.53e-08 $\pm$ 9.57e-09 & 0.00e+00 $\pm$ 0.00e+00 & 1666.82 $\pm$ 731.00 & 26.669 $\pm$ 11.696\%\\
\midrule
\multicolumn{5}{l}{\text{Average memory (RAM) of an experiment:} 1059.29 $\pm$ 37.25 MB}\\
\multicolumn{5}{l}{\text{Maximum memory (RAM) of an experiment:} 1649.35 $\pm$ 69.11 MB}\\
\bottomrule
\end{tabular}
\caption{Results for MNIST Addition with $k=4$ ($N_{\text{train}}$ = 6250) using $\Pi$-Nesy-2 over 10 experiments. Averages for the rule parameter values ($s_{\ast}$ and $r_{\ast}$) for the $c_{\ast}$-, $w_{\ast}$-, and $y_{\ast}$-rules are computed from data comprising $10\cdot 83$, $10\cdot 4$, and $10\cdot 41$ observations, respectively. Reliable sample counts for the $c_{\ast}$-, $w_{\ast}$-, and $y_{\ast}$-rules are averaged over $10\cdot 4$, $10\cdot 4$, and $10\cdot 5$ observations, respectively, and memory metrics are averaged over 10 experiments.}
\label{tab:mnist-add-4-pi-nesy-2}
\end{table}

\begin{table}[ht]
\centering
\setlength{\tabcolsep}{1.5pt}
\renewcommand{\arraystretch}{1.5}
\begin{tabular}{lcccc}
\toprule
\makecell{Set of rules\\(output attribute)} & \makecell{Average value of the\\rule parameters $s_{\ast}$} & \makecell{Average value of the\\rule parameters $r_{\ast}$} & \makecell{Average number of\\training data samples\\considered reliable} & \makecell{Percent of\\training data samples\\considered reliable} \\
\midrule
$c_{\ast}$ & 1.67e-08 $\pm$ 6.17e-09 & 7.27e-09 $\pm$ 7.63e-09 & 240.44 $\pm$ 109.22 & 14.432 $\pm$ 6.556\%\\
$w_{\ast}$ & 0.00e+00 $\pm$ 0.00e+00 & 0.00e+00 $\pm$ 0.00e+00 & 587.69 $\pm$ 78.87 & 35.275 $\pm$ 4.734\%\\
$y_{\ast}$ & 2.14e-08 $\pm$ 1.03e-08 & 0.00e+00 $\pm$ 0.00e+00 & 261.28 $\pm$ 134.45 & 15.683 $\pm$ 8.070\%\\
\midrule
\multicolumn{5}{l}{\text{Average memory (RAM) of an experiment:} 911.12 $\pm$ 11.25 MB}\\
\multicolumn{5}{l}{\text{Maximum memory (RAM) of an experiment:} 1328.48 $\pm$ 19.72 MB}\\
\bottomrule
\end{tabular}
\caption{Results for MNIST Addition with $k=15$ ($N_{\text{train}}$ = 1666) using $\Pi$-Nesy-1 over 10 experiments. Averages for the rule parameter values ($s_{\ast}$ and $r_{\ast}$) for the $c_{\ast}$-, $w_{\ast}$-, and $y_{\ast}$-rules are computed from data comprising $10\cdot 314$, $10\cdot 15$, and $10\cdot 151$ observations, respectively. Reliable sample counts for the $c_{\ast}$-, $w_{\ast}$-, and $y_{\ast}$-rules are averaged over $10\cdot 15$, $10\cdot 15$, and $10\cdot 16$ observations, respectively, and memory metrics are averaged over 10 experiments.}
\label{tab:mnist-add-15-pi-nesy-1}
\end{table}

\begin{table}[ht]
\centering
\setlength{\tabcolsep}{1.5pt}
\renewcommand{\arraystretch}{1.5}
\begin{tabular}{lcccc}
\toprule
\makecell{Set of rules\\(output attribute)} & \makecell{Average value of the\\rule parameters $s_{\ast}$} & \makecell{Average value of the\\rule parameters $r_{\ast}$} & \makecell{Average number of\\training data samples\\considered reliable} & \makecell{Percent of\\training data samples\\considered reliable} \\
\midrule
$c_{\ast}$ & 1.71e-08 $\pm$ 5.99e-09 & 8.00e-09 $\pm$ 7.79e-09 & 356.03 $\pm$ 119.59 & 21.370 $\pm$ 7.178\%\\
$w_{\ast}$ & 0.00e+00 $\pm$ 0.00e+00 & 0.00e+00 $\pm$ 0.00e+00 & 734.45 $\pm$ 78.21 & 44.084 $\pm$ 4.695\%\\
$y_{\ast}$ & 2.29e-08 $\pm$ 1.00e-08 & 0.00e+00 $\pm$ 0.00e+00 & 378.73 $\pm$ 146.45 & 22.733 $\pm$ 8.790\%\\
\midrule
\multicolumn{5}{l}{\text{Average memory (RAM) of an experiment:} 924.22 $\pm$ 10.14 MB}\\
\multicolumn{5}{l}{\text{Maximum memory (RAM) of an experiment:} 1324.28 $\pm$ 8.94 MB}\\
\bottomrule
\end{tabular}
\caption{Results for MNIST Addition with $k=15$ ($N_{\text{train}}$ = 1666) using $\Pi$-Nesy-2 over 10 experiments. Averages for the rule parameter values ($s_{\ast}$ and $r_{\ast}$) for the $c_{\ast}$-, $w_{\ast}$-, and $y_{\ast}$-rules are computed from data comprising $10\cdot 314$, $10\cdot 15$, and $10\cdot 151$ observations, respectively. Reliable sample counts for the $c_{\ast}$-, $w_{\ast}$-, and $y_{\ast}$-rules are averaged over $10\cdot 15$, $10\cdot 15$, and $10\cdot 16$ observations, respectively, and memory metrics are averaged over 10 experiments.}
\label{tab:mnist-add-15-pi-nesy-2}
\end{table}

\begin{table}[ht]
\centering
\setlength{\tabcolsep}{1.5pt}
\renewcommand{\arraystretch}{1.5}
\begin{tabular}{lcccc}
\toprule
\makecell{Set of rules\\(output attribute)} & \makecell{Average value of the\\rule parameters $s_{\ast}$} & \makecell{Average value of the\\rule parameters $r_{\ast}$} & \makecell{Average number of\\training data samples\\considered reliable} & \makecell{Percent of\\training data samples\\considered reliable} \\
\midrule
$c_{\ast}$ & 1.11e-08 $\pm$ 7.60e-09 & 1.81e-09 $\pm$ 4.88e-09 & 25.10 $\pm$ 8.58 & 10.041 $\pm$ 3.434\%\\
$w_{\ast}$ & 0.00e+00 $\pm$ 0.00e+00 & 0.00e+00 $\pm$ 0.00e+00 & 77.11 $\pm$ 11.18 & 30.845 $\pm$ 4.471\%\\
$y_{\ast}$ & 1.25e-08 $\pm$ 1.01e-08 & 0.00e+00 $\pm$ 0.00e+00 & 25.60 $\pm$ 9.94 & 10.239 $\pm$ 3.978\%\\
\midrule
\multicolumn{5}{l}{\text{Average memory (RAM) of an experiment:} 910.22 $\pm$ 7.01 MB}\\
\multicolumn{5}{l}{\text{Maximum memory (RAM) of an experiment:} 1307.98 $\pm$ 6.10 MB}\\
\bottomrule
\end{tabular}
\caption{Results for MNIST Addition with $k=100$ ($N_{\text{train}}$ = 250) using $\Pi$-Nesy-1 over 10 experiments. Averages for the rule parameter values ($s_{\ast}$ and $r_{\ast}$) for the $c_{\ast}$-, $w_{\ast}$-, and $y_{\ast}$-rules are computed from data comprising $10\cdot 2099$, $10\cdot 100$, and $10\cdot 1001$ observations, respectively. Reliable sample counts for the $c_{\ast}$-, $w_{\ast}$-, and $y_{\ast}$-rules are averaged over $10\cdot 100$, $10\cdot 100$, and $10\cdot 101$ observations, respectively, and memory metrics are averaged over 10 experiments.}
\label{tab:mnist-add-100-pi-nesy-1}
\end{table}

\begin{table}[ht]
\centering
\setlength{\tabcolsep}{1.5pt}
\renewcommand{\arraystretch}{1.5}
\begin{tabular}{lcccc}
\toprule
\makecell{Set of rules\\(output attribute)} & \makecell{Average value of the\\rule parameters $s_{\ast}$} & \makecell{Average value of the\\rule parameters $r_{\ast}$} & \makecell{Average number of\\training data samples\\considered reliable} & \makecell{Percent of\\training data samples\\considered reliable} \\
\midrule
$c_{\ast}$ & 1.20e-08 $\pm$ 7.66e-09 & 2.32e-09 $\pm$ 5.44e-09 & 45.21 $\pm$ 12.00 & 18.085 $\pm$ 4.798\%\\
$w_{\ast}$ & 0.00e+00 $\pm$ 0.00e+00 & 0.00e+00 $\pm$ 0.00e+00 & 104.54 $\pm$ 13.87 & 41.817 $\pm$ 5.549\%\\
$y_{\ast}$ & 1.38e-08 $\pm$ 9.91e-09 & 0.00e+00 $\pm$ 0.00e+00 & 45.81 $\pm$ 13.39 & 18.323 $\pm$ 5.356\%\\
\midrule
\multicolumn{5}{l}{\text{Average memory (RAM) of an experiment:} 926.24 $\pm$ 4.83 MB}\\
\multicolumn{5}{l}{\text{Maximum memory (RAM) of an experiment:} 1320.76 $\pm$ 6.32 MB}\\
\bottomrule
\end{tabular}
\caption{Results for MNIST Addition with $k=100$ ($N_{\text{train}}$ = 250) using $\Pi$-Nesy-2 over 10 experiments. Averages for the rule parameter values ($s_{\ast}$ and $r_{\ast}$) for the $c_{\ast}$-, $w_{\ast}$-, and $y_{\ast}$-rules are computed from data comprising $10\cdot 2099$, $10\cdot 100$, and $10\cdot 1001$ observations, respectively. Reliable sample counts for the $c_{\ast}$-, $w_{\ast}$-, and $y_{\ast}$-rules are averaged over $10\cdot 100$, $10\cdot 100$, and $10\cdot 101$ observations, respectively, and memory metrics are averaged over 10 experiments.}
\label{tab:mnist-add-100-pi-nesy-2}
\end{table}

\FloatBarrier

\subsection{MNIST Sudoku problems}
\label{subsec:appendix:posslearn:mnistsudoku}

 We also analyzed our possibilistic learning method applied to the MNIST Sudoku problems, which led to the results reported in Table \ref{tab:visual_Sudoku_classification}.

For MNIST Sudoku problems (Subsection \ref{subsec:exp:sudoku}), our possibilistic  learning method uses two thresholds to select reliable training data: Threshold \#1 corresponds to the sets of rules whose output attributes are $b_{(i,j,i',j')}$ where $(i,j,i',j') \in F$ and Threshold \#2 corresponds to the set of rules whose output attribute is $c$. 

Table~\ref{tab:mnist-sudoku-thresholds} presents the average values of these thresholds for both $\Pi$-NeSy-1 and $\Pi$-NeSy-2 approaches, computed over ten experiments. These thresholds are very small. 

In Tables~\ref{tab:mnist-sudoku-4-pi-nesy-1}--\ref{tab:mnist-sudoku-9-pi-nesy-2}, we report the average values  of the rule parameter values $s$ and $r$ for the output attributes $b_{(i,j,i',j')}$ and $c$ obtained using our possibilistic learning method. For $b_{(i,j,i',j')}$, the average values of $s$ and $r$ represent the mean of the rule parameter values across all rules whose output attributes are $b_{(i,j,i',j')}$ where $(i,j,i',j') \in F$, while for $c$, the corresponding values are shown for the rules whose output attribute is $c$. In addition, the tables provide the absolute number and percentage of reliable training data samples, as well as the average and maximum memory usage for each problem.

From these results, we observe:

\begin{itemize}
    \item The parameters $s$ and $r$ associated with the rules have values very close to or equal to zero. The variability of the parameters is low.

    \item The percentage of reliable training samples relative to $N_{\text{train}} = 200$ varies with the two problems. For MNIST Sudoku 4x4, $\Pi$-NeSy-2 considers $37.69 \pm 7.20\%$ of reliable training data for $b_{(i,j,i',j')}$ compared to $30.90 \pm 8.85\%$ for $\Pi$-NeSy-1. For MNIST Sudoku 9x9, these percentages drop to $ 20.02 \pm 11.30\%$ for $\Pi$-NeSy-2 and $13.72 \pm 8.41\%$ for $\Pi$-NeSy-1.

    \item For MNIST Sudoku 4x4, memory usage averages 750 MB and remains reasonable for both approaches. For MNIST Sudoku 9x9, memory usage increases considerably, exceeding 6.7 GB for both approaches.
\end{itemize}

\begin{table}[h!]
\centering
\begin{tabular}{@{}llll@{}}
\toprule
Dimension & Approach & Average threshold for $b_{(i,j,i',j')}$ & Average threshold for $c$  \\ \midrule
4x4 & $\Pi$-NeSy-1 & 1.318e-06 $\pm$ 0.000e+00 & 1.689e-07 $\pm$ 3.831e-07 \\
4x4 & $\Pi$-NeSy-2 & 1.318e-06 $\pm$ 0.000e+00 & 4.119e-08 $\pm$ 0.000e+00 \\
9x9 & $\Pi$-NeSy-1 & 1.689e-07 $\pm$ 3.831e-07 & 4.119e-08 $\pm$ 0.000e+00 \\
9x9 & $\Pi$-NeSy-2 & 2.966e-07 $\pm$ 5.108e-07 & 4.119e-08 $\pm$ 0.000e+00 \\
\bottomrule
\end{tabular}
\caption{Average thresholds for MNIST-Sudoku problems used for selecting reliable training data during possibilistic learning. In the 4x4 configuration (resp. 9x9 configuration), averages for $b_{(i,j,i',j')}$ are computed over 10 observations and for $c$ over 10 observations for both approaches.}
\label{tab:mnist-sudoku-thresholds}\end{table}

\begin{table}[ht]
\centering
\setlength{\tabcolsep}{1.5pt}
\renewcommand{\arraystretch}{1.5}
\begin{tabular}{lcccc}
\toprule
\makecell{Set of rules\\(output attribute)} & \makecell{Average value of the\\rule parameters $s$} & \makecell{Average value of the\\rule parameters $r$} & \makecell{Average number of\\training data samples\\considered reliable} & \makecell{Percent of\\training data samples\\considered reliable} \\
\midrule
$b_{(i,j,i',j')}$ & 4.30e-07 $\pm$ 2.40e-07 & 2.08e-07 $\pm$ 2.50e-07 & 61.8 $\pm$ 17.7 & 30.90 $\pm$ 8.85\% \\
c & 0.00e+00 $\pm$ 0.00e+00 & 0.00e+00 $\pm$ 0.00e+00 & 43.4 $\pm$ 35.8 & 21.70 $\pm$ 17.90\% \\
\midrule
\multicolumn{5}{l}{\text{Average memory (RAM) of an experiment:} 754.8 $\pm$ 4.9 MB}\\
\multicolumn{5}{l}{\text{Maximum memory (RAM) of an experiment:} 803.8 $\pm$ 5.3 MB}\\
\bottomrule
\end{tabular}
\caption{Possibilistic learning results on the MNIST Sudoku 4x4 problem using $\Pi$-NeSy-1. Results are averaged over ten experiments. The respective averages for the rule parameters corresponding to $b_{(i,j,i',j')}$ are computed from $10\cdot 448$ observations, while those for the attribute $c$ are computed from $10\cdot 1$ observations. Reliable training sample counts for $b_{(i,j,i',j')}$ and $c$ are computed from $10\cdot 56$ and $10\cdot 1$ observations, respectively.}
\label{tab:mnist-sudoku-4-pi-nesy-1}
\end{table}
\begin{table}[ht]
\centering
\setlength{\tabcolsep}{1.5pt}
\renewcommand{\arraystretch}{1.5}
\begin{tabular}{lcccc}
\toprule
\makecell{Set of rules\\(output attribute)} & \makecell{Average value of the\\rule parameters $s$} & \makecell{Average value of the\\rule parameters $r$} & \makecell{Average number of\\training data samples\\considered reliable} & \makecell{Percent of\\training data samples\\considered reliable} \\
\midrule
$b_{(i,j,i',j')}$ & 4.60e-07 $\pm$ 2.30e-07 & 2.29e-07 $\pm$ 2.55e-07 & 75.4 $\pm$ 14.4 & 37.69 $\pm$ 7.20\% \\
c & 0.00e+00 $\pm$ 0.00e+00 & 0.00e+00 $\pm$ 0.00e+00 & 48.0 $\pm$ 40.2 & 24.00 $\pm$ 20.09\% \\
\midrule
\multicolumn{5}{l}{\text{Average memory (RAM) of an experiment:} 755.4 $\pm$ 4.9 MB}\\
\multicolumn{5}{l}{\text{Maximum memory (RAM) of an experiment:} 804.6 $\pm$ 5.2 MB}\\
\bottomrule
\end{tabular}
\caption{Possibilistic learning results on the MNIST Sudoku 4x4 problem using $\Pi$-NeSy-2. Results are averaged over ten experiments. The respective averages for the rule parameters corresponding to $b_{(i,j,i',j')}$ are computed from $10\cdot 448$ observations, while those for the attribute $c$ are computed from $10\cdot 1$ observations. Reliable training sample counts for $b_{(i,j,i',j')}$ and $c$ are computed from $10\cdot 56$ and $10\cdot 1$ observations, respectively.}
\label{tab:mnist-sudoku-4-pi-nesy-2}
\end{table}
\begin{table}[ht]
\centering
\setlength{\tabcolsep}{1.5pt}
\renewcommand{\arraystretch}{1.5}
\begin{tabular}{lcccc}
\toprule
\makecell{Set of rules\\(output attribute)} & \makecell{Average value of the\\rule parameters $s$} & \makecell{Average value of the\\rule parameters $r$} & \makecell{Average number of\\training data samples\\considered reliable} & \makecell{Percent of\\training data samples\\considered reliable} \\
\midrule
$b_{(i,j,i',j')}$ & 5.79e-08 $\pm$ 1.54e-07 & 1.09e-08 $\pm$ 6.71e-08 & 27.4 $\pm$ 16.8 & 13.72 $\pm$ 8.41\% \\
c & 0.00e+00 $\pm$ 0.00e+00 & 0.00e+00 $\pm$ 0.00e+00 & 46.0 $\pm$ 14.8 & 23.00 $\pm$ 7.42\% \\
\midrule
\multicolumn{5}{l}{\text{Average memory (RAM) of an experiment:} 6713.6 $\pm$ 16.9 MB}\\
\multicolumn{5}{l}{\text{Maximum memory (RAM) of an experiment:} 7597.5 $\pm$ 16.8 MB}\\
\bottomrule
\end{tabular}
\caption{Possibilistic learning results on the MNIST Sudoku 9x9 problem using $\Pi$-NeSy-1. Results are averaged over ten experiments. The respective averages for the rule parameters corresponding to $b_{(i,j,i',j')}$ are computed from $10\cdot 14580$ observations, while those for the attribute $c$ are computed from $10\cdot 1$ observations. Reliable training sample counts for $b_{(i,j,i',j')}$ and $c$ are computed from $10\cdot 810$ and $10\cdot 1$ observations, respectively.}
\label{tab:mnist-sudoku-9-pi-nesy-1}
\end{table}
\begin{table}[ht]
\centering
\setlength{\tabcolsep}{1.5pt}
\renewcommand{\arraystretch}{1.5}
\begin{tabular}{lcccc}
\toprule
\makecell{Set of rules\\(output attribute)} & \makecell{Average value of the\\rule parameters $s$} & \makecell{Average value of the\\rule parameters $r$} & \makecell{Average number of\\training data samples\\considered reliable} & \makecell{Percent of\\training data samples\\considered reliable} \\
\midrule
$b_{(i,j,i',j')}$ & 1.07e-07 $\pm$ 2.11e-07 & 2.20e-08 $\pm$ 9.68e-08 & 40.0 $\pm$ 22.6 & 20.02 $\pm$ 11.30\% \\
c & 0.00e+00 $\pm$ 0.00e+00 & 0.00e+00 $\pm$ 0.00e+00 & 52.0 $\pm$ 20.7 & 26.00 $\pm$ 10.34\% \\
\midrule
\multicolumn{5}{l}{\text{Average memory (RAM) of an experiment:} 6715.6 $\pm$ 13.3 MB}\\
\multicolumn{5}{l}{\text{Maximum memory (RAM) of an experiment:} 7616.1 $\pm$ 10.5 MB}\\
\bottomrule
\end{tabular}
\caption{Possibilistic learning results on the MNIST Sudoku 9x9 problem using $\Pi$-NeSy-2. Results are averaged over ten experiments. The respective averages for the rule parameters corresponding to $b_{(i,j,i',j')}$ are computed from $10\cdot 14580$ observations, while those for the attribute $c$ are computed from $10\cdot 1$ observations. Reliable training sample counts for $b_{(i,j,i',j')}$ and $c$ are computed from $10\cdot 810$ and $10\cdot 1$ observations, respectively.}
\label{tab:mnist-sudoku-9-pi-nesy-2}
\end{table}

\FloatBarrier
\section{Additional experimental results in low data settings}

We have also analyzed the characteristics and performance of the possibilistic learning method (Method \ref{meth:learningincascade}) applied to the MNIST Addition-$k$ problems with $k \in \{1,2,4,15,100\}$ and to MNIST Sudoku problems in low data settings. \ref{tab:mnist_sudokuresults_low_data_settings}. 
In the following, we reuse the notations given in Section \ref{sec:appendix:expadd}.

When dealing with low data settings, let us first remind that we encountered a few difficulties in comparing floats, linked to the thresholds used in Method \ref{meth:learningincascade}.
To overcome this problem, we modified the parameters for generating the set $\mathcal{T}$ of candidate thresholds, see (\ref{eq:setTvalidationstep}), by setting $h = 2$, which generates candidate thresholds with higher values.

\subsection{\texorpdfstring{MNIST Addition-$k$ problems in low data settings}{MNIST Addition k problems}} 
\label{subsec:appendix:lds:posslearn:mnistadd}

In the following, we present results obtained by our possibilistic learning method  applied to the MNIST Addition-$k$ problems with $k \in \{1,2,4,15,100\}$ in low data setings, see Table \ref{tab:mnist_add_results_low_data_settings}.
The thresholds used for MNIST Addition-$k$ problems with $k \in \{1, 2, 4, 15, 100\}$ in low data settings are presented in Table \ref{tab:lds:thresholds:mnistadd}. We remind  that in  normal settings, both $\Pi$-NeSy-1 and $\Pi$-NeSy-2 consistently use the fixed threshold value of $4.12\times10^{-8}$ for Thresholds \#1, \#2, and \#3 (see Subsection \ref{subsec:appendix:posslearn:mnistadd}). Contrastingly, in low data settings, for $k \in \{1,2,4,15\}$, both approaches yield an average threshold value of $1.11\times10^{-3} \pm 2.17\times10^{-19}$, while for $k = 100$, the average threshold increases to $2.11\times10^{-3} \pm 1.53\times10^{-3}$ for Thresholds \#1 and \#2. These values, which are higher than those in normal settings, indicate a slightly less restrictive selection of training data in low data settings.

\begin{table}[ht!]
\centering
\begin{tabular}{|c|c|c|c|c|}
\hline
k & Approach & Threshold \#1 & Threshold \#2 & Threshold \#3 \\
\hline
1 & $\Pi$-Nesy-1 & 1.11e-03 $\pm$ 2.17e-19 & 1.11e-03 $\pm$ 2.17e-19 & 1.11e-03 $\pm$ 2.17e-19 \\
1 & $\Pi$-Nesy-2 & 1.11e-03 $\pm$ 2.17e-19 & 1.11e-03 $\pm$ 2.17e-19 & 1.11e-03 $\pm$ 2.17e-19 \\
2 & $\Pi$-Nesy-1 & 1.11e-03 $\pm$ 2.17e-19 & 1.11e-03 $\pm$ 2.17e-19 & 1.11e-03 $\pm$ 2.17e-19 \\
2 & $\Pi$-Nesy-2 & 1.11e-03 $\pm$ 2.17e-19 & 1.11e-03 $\pm$ 2.17e-19 & 1.11e-03 $\pm$ 2.17e-19 \\
4 & $\Pi$-Nesy-1 & 1.11e-03 $\pm$ 2.17e-19 & 1.11e-03 $\pm$ 2.17e-19 & 1.11e-03 $\pm$ 2.17e-19 \\
4 & $\Pi$-Nesy-2 & 1.11e-03 $\pm$ 2.17e-19 & 1.11e-03 $\pm$ 2.17e-19 & 1.11e-03 $\pm$ 2.17e-19 \\
15 & $\Pi$-Nesy-1 & 1.11e-03 $\pm$ 2.17e-19 & 1.11e-03 $\pm$ 2.17e-19 & 1.11e-03 $\pm$ 2.17e-19 \\
15 & $\Pi$-Nesy-2 & 1.11e-03 $\pm$ 2.17e-19 & 1.11e-03 $\pm$ 2.17e-19 & 1.11e-03 $\pm$ 2.17e-19 \\
100 & $\Pi$-Nesy-1 & 2.11e-03 $\pm$ 1.53e-03 & 1.11e-03 $\pm$ 2.17e-19 & 2.11e-03 $\pm$ 1.53e-03 \\
100 & $\Pi$-Nesy-2 & 2.11e-03 $\pm$ 1.53e-03 & 1.11e-03 $\pm$ 2.17e-19 & 2.11e-03 $\pm$ 1.53e-03 \\
\hline
\end{tabular}
\caption{Average thresholds for MNIST Addition-$k$ problems with  $\Pi$-Nesy-1 and $\Pi$-Nesy-2 in low data settings.}
\label{tab:lds:thresholds:mnistadd}
\end{table}

\begin{table}[ht]
\centering
\begin{tabular}{@{}cccccc@{}}
\toprule
k & Training set size & Test set size & Approach & Learning time & Inference time per test sample \\
\midrule
1 & 1250 & 5000 & $\Pi$-NeSy-1 & 221.17 $\pm$ 29.95 & 0.001 $\pm$ 0.00 \\
1 & 1250 & 5000 & $\Pi$-NeSy-2 & 249.63 $\pm$ 38.46 & 0.001 $\pm$ 0.00 \\
2 & 625 & 2500 & $\Pi$-NeSy-1 & 200.72 $\pm$ 30.29 & 0.002 $\pm$ 0.00 \\
2 & 625 & 2500 & $\Pi$-NeSy-2 & 265.12 $\pm$ 30.67 & 0.002 $\pm$ 0.00 \\
4 & 312 & 1250 & $\Pi$-NeSy-1 & 210.02 $\pm$ 24.95 & 0.004 $\pm$ 0.00 \\
4 & 312 & 1250 & $\Pi$-NeSy-2 & 217.97 $\pm$ 27.30 & 0.004 $\pm$ 0.00 \\
15 & 83 & 333 & $\Pi$-NeSy-1 & 169.70 $\pm$ 10.50 & 0.015 $\pm$ 0.00 \\
15 & 83 & 333 & $\Pi$-NeSy-2 & 173.36 $\pm$ 10.30 & 0.015 $\pm$ 0.00 \\
100 & 12 & 50 & $\Pi$-NeSy-1 & 161.61 $\pm$ 3.27 & 0.101 $\pm$ 0.00 \\
100 & 12 & 50 & $\Pi$-NeSy-2 & 159.66 $\pm$ 4.16 & 0.100 $\pm$ 0.00 \\
\bottomrule
\end{tabular}

\caption{Each row in the table represents a MNIST Addition-$k$ problem addressed using either $\Pi$-NeSy-1 or $\Pi$-NeSy-2 in low data settings.
Each row includes the average learning time in seconds (averaged over 10 runs) and the average inference time per test sample in seconds (averaged over $10 \cdot N_{\text{test}}$ test samples) from these experiments. }
\label{tab:lds:times_mnistadd}
\end{table}

Empirically, using the specified configuration (detailed in Subsection \ref{subsec:mnistaddotherresults}), the inference and learning times of  $\Pi$-NeSy in low data settings are small, see Table \ref{tab:lds:times_mnistadd}. We remark that  inference times are  less important than in normal settings (see Table \ref{tab:times_mnistadd} for inference and learning times obtained under normal conditions) which may be due to the fact that the program uses less memory (see below).

Tables~\ref{tab:lds:mnist-add-1-pi-nesy-1}--\ref{tab:lds:mnist-add-100-pi-nesy-2} report the average values of the rule parameters obtained by possibilistic learning, the absolute number and percentage of reliable training data samples used, and the memory usage. As with normal settings, the parameters $s$ and $r$ maintain average values close to zero with very low variability. In low data settings, the percentage of training data samples considered reliable is higher than in normal settings, and remains stable for all $k \in \{1, 2, 4, 15, 100\}$. To run the experiments in low data settings, $\Pi$-NeSy needs around 615MB of memory (which is significantly less than the memory required (750MB) under normal conditions).

\begin{table}[ht]
\centering
\setlength{\tabcolsep}{1.5pt}
\renewcommand{\arraystretch}{1.5}
\begin{tabular}{lcccc}
\toprule
\makecell{Set of rules\\(output attribute)} & \makecell{Average value of the\\rule parameters $s_{\ast}$} & \makecell{Average value of the\\rule parameters $r_{\ast}$} & \makecell{Average number of\\training data samples\\considered reliable} & \makecell{Percent of\\training data samples\\considered reliable} \\
\midrule
$c_{\ast}$ & 0.00e+00 $\pm$ 0.00e+00 & 0.00e+00 $\pm$ 0.00e+00 & 711.90 $\pm$ 137.88 & 56.952 $\pm$ 11.030\%\\
$w_{\ast}$ & 0.00e+00 $\pm$ 0.00e+00 & 0.00e+00 $\pm$ 0.00e+00 & 876.50 $\pm$ 113.76 & 70.120 $\pm$ 9.101\%\\
$y_{\ast}$ & 0.00e+00 $\pm$ 0.00e+00 & 0.00e+00 $\pm$ 0.00e+00 & 794.20 $\pm$ 150.83 & 63.536 $\pm$ 12.066\%\\
\midrule
\multicolumn{5}{l}{\text{Average memory (RAM) of an experiment:} 604.64 $\pm$ 5.06 MB}\\
\multicolumn{5}{l}{\text{Maximum memory (RAM) of an experiment:} 705.89 $\pm$ 7.53 MB}\\
\bottomrule
\end{tabular}
\caption{Results for MNIST Addition with $k=1$ in low data settings ($N_{\text{train}}$ = 1250) using $\Pi$-Nesy-1 over 10 experiments. Averages for the rule parameter values ($s_{\ast}$ and $r_{\ast}$) for the $c_{\ast}$-, $w_{\ast}$-, and $y_{\ast}$-rules are computed from data comprising $10\cdot 20$, $10\cdot 1$, and $10\cdot 11$ observations, respectively. Reliable sample counts for the $c_{\ast}$-, $w_{\ast}$-, and $y_{\ast}$-rules are averaged over $10\cdot 1$, $10\cdot 1$, and $10\cdot 2$ observations, respectively, and memory metrics are averaged over 10 experiments.}
\label{tab:lds:mnist-add-1-pi-nesy-1}
\end{table}

\begin{table}[ht]
\centering
\setlength{\tabcolsep}{1.5pt}
\renewcommand{\arraystretch}{1.5}
\begin{tabular}{lcccc}
\toprule
\makecell{Set of rules\\(output attribute)} & \makecell{Average value of the\\rule parameters $s_{\ast}$} & \makecell{Average value of the\\rule parameters $r_{\ast}$} & \makecell{Average number of\\training data samples\\considered reliable} & \makecell{Percent of\\training data samples\\considered reliable} \\
\midrule
$c_{\ast}$ & 7.10e-08 $\pm$ 5.69e-07 & 0.00e+00 $\pm$ 0.00e+00 & 845.00 $\pm$ 155.58 & 67.600 $\pm$ 12.446\%\\
$w_{\ast}$ & 0.00e+00 $\pm$ 0.00e+00 & 0.00e+00 $\pm$ 0.00e+00 & 1002.00 $\pm$ 125.23 & 80.160 $\pm$ 10.018\%\\
$y_{\ast}$ & 0.00e+00 $\pm$ 0.00e+00 & 0.00e+00 $\pm$ 0.00e+00 & 923.50 $\pm$ 161.57 & 73.880 $\pm$ 12.926\%\\
\midrule
\multicolumn{5}{l}{\text{Average memory (RAM) of an experiment:} 611.66 $\pm$ 4.94 MB}\\
\multicolumn{5}{l}{\text{Maximum memory (RAM) of an experiment:} 717.52 $\pm$ 9.90 MB}\\
\bottomrule
\end{tabular}
\caption{Results for MNIST Addition with $k=1$  in low data settings ($N_{\text{train}}$ = 1250) using $\Pi$-Nesy-2 over 10 experiments. Averages for the rule parameter values ($s_{\ast}$ and $r_{\ast}$) for the $c_{\ast}$-, $w_{\ast}$-, and $y_{\ast}$-rules are computed from data comprising $10\cdot 20$, $10\cdot 1$, and $10\cdot 11$ observations, respectively. Reliable sample counts for the $c_{\ast}$-, $w_{\ast}$-, and $y_{\ast}$-rules are averaged over $10\cdot 1$, $10\cdot 1$, and $10\cdot 2$ observations, respectively, and memory metrics are averaged over 10 experiments.}
\label{tab:lds:mnist-add-1-pi-nesy-2}
\end{table}

\begin{table}[ht]
\centering
\setlength{\tabcolsep}{1.5pt}
\renewcommand{\arraystretch}{1.5}
\begin{tabular}{lcccc}
\toprule
\makecell{Set of rules\\(output attribute)} & \makecell{Average value of the\\rule parameters $s_{\ast}$} & \makecell{Average value of the\\rule parameters $r_{\ast}$} & \makecell{Average number of\\training data samples\\considered reliable} & \makecell{Percent of\\training data samples\\considered reliable} \\
\midrule
$c_{\ast}$ & 1.69e-06 $\pm$ 8.92e-06 & 6.40e-08 $\pm$ 8.30e-07 & 276.40 $\pm$ 99.13 & 44.224 $\pm$ 15.860\%\\
$w_{\ast}$ & 0.00e+00 $\pm$ 0.00e+00 & 0.00e+00 $\pm$ 0.00e+00 & 392.60 $\pm$ 79.82 & 62.816 $\pm$ 12.772\%\\
$y_{\ast}$ & 1.07e-05 $\pm$ 3.64e-05 & 0.00e+00 $\pm$ 0.00e+00 & 311.63 $\pm$ 105.96 & 49.861 $\pm$ 16.953\%\\
\midrule
\multicolumn{5}{l}{\text{Average memory (RAM) of an experiment:} 611.69 $\pm$ 7.21 MB}\\
\multicolumn{5}{l}{\text{Maximum memory (RAM) of an experiment:} 714.01 $\pm$ 16.97 MB}\\
\bottomrule
\end{tabular}
\caption{Results for MNIST Addition with $k=2$  in low data settings ($N_{\text{train}}$ = 625) using $\Pi$-Nesy-1 over 10 experiments. Averages for the rule parameter values ($s_{\ast}$ and $r_{\ast}$) for the $c_{\ast}$-, $w_{\ast}$-, and $y_{\ast}$-rules are computed from data comprising $10\cdot 41$, $10\cdot 2$, and $10\cdot 21$ observations, respectively. Reliable sample counts for the $c_{\ast}$-, $w_{\ast}$-, and $y_{\ast}$-rules are averaged over $10\cdot 2$, $10\cdot 2$, and $10\cdot 3$ observations, respectively, and memory metrics are averaged over 10 experiments.}
\label{tab:lds:mnist-add-2-pi-nesy-1}
\end{table}

\begin{table}[ht]
\centering
\setlength{\tabcolsep}{1.5pt}
\renewcommand{\arraystretch}{1.5}
\begin{tabular}{lcccc}
\toprule
\makecell{Set of rules\\(output attribute)} & \makecell{Average value of the\\rule parameters $s_{\ast}$} & \makecell{Average value of the\\rule parameters $r_{\ast}$} & \makecell{Average number of\\training data samples\\considered reliable} & \makecell{Percent of\\training data samples\\considered reliable} \\
\midrule
$c_{\ast}$ & 2.57e-07 $\pm$ 1.63e-06 & 0.00e+00 $\pm$ 0.00e+00 & 422.95 $\pm$ 64.76 & 67.672 $\pm$ 10.361\%\\
$w_{\ast}$ & 0.00e+00 $\pm$ 0.00e+00 & 0.00e+00 $\pm$ 0.00e+00 & 524.75 $\pm$ 31.43 & 83.960 $\pm$ 5.029\%\\
$y_{\ast}$ & 4.65e-07 $\pm$ 3.36e-06 & 0.00e+00 $\pm$ 0.00e+00 & 454.73 $\pm$ 71.85 & 72.757 $\pm$ 11.496\%\\
\midrule
\multicolumn{5}{l}{\text{Average memory (RAM) of an experiment:} 618.52 $\pm$ 4.75 MB}\\
\multicolumn{5}{l}{\text{Maximum memory (RAM) of an experiment:} 735.42 $\pm$ 8.18 MB}\\
\bottomrule
\end{tabular}
\caption{Results for MNIST Addition with $k=2$  in low data settings ($N_{\text{train}}$ = 625) using $\Pi$-Nesy-2 over 10 experiments. Averages for the rule parameter values ($s_{\ast}$ and $r_{\ast}$) for the $c_{\ast}$-, $w_{\ast}$-, and $y_{\ast}$-rules are computed from data comprising $10\cdot 41$, $10\cdot 2$, and $10\cdot 21$ observations, respectively. Reliable sample counts for the $c_{\ast}$-, $w_{\ast}$-, and $y_{\ast}$-rules are averaged over $10\cdot 2$, $10\cdot 2$, and $10\cdot 3$ observations, respectively, and memory metrics are averaged over 10 experiments.}
\label{tab:lds:mnist-add-2-pi-nesy-2}
\end{table}

\begin{table}[ht]
\centering
\setlength{\tabcolsep}{1.5pt}
\renewcommand{\arraystretch}{1.5}
\begin{tabular}{lcccc}
\toprule
\makecell{Set of rules\\(output attribute)} & \makecell{Average value of the\\rule parameters $s_{\ast}$} & \makecell{Average value of the\\rule parameters $r_{\ast}$} & \makecell{Average number of\\training data samples\\considered reliable} & \makecell{Percent of\\training data samples\\considered reliable} \\
\midrule
$c_{\ast}$ & 3.71e-05 $\pm$ 1.18e-04 & 1.11e-05 $\pm$ 6.29e-05 & 180.32 $\pm$ 45.32 & 57.796 $\pm$ 14.525\%\\
$w_{\ast}$ & 0.00e+00 $\pm$ 0.00e+00 & 0.00e+00 $\pm$ 0.00e+00 & 239.85 $\pm$ 28.04 & 76.875 $\pm$ 8.987\%\\
$y_{\ast}$ & 1.08e-04 $\pm$ 1.67e-04 & 0.00e+00 $\pm$ 0.00e+00 & 192.02 $\pm$ 48.16 & 61.545 $\pm$ 15.435\%\\
\midrule
\multicolumn{5}{l}{\text{Average memory (RAM) of an experiment:} 612.64 $\pm$ 4.18 MB}\\
\multicolumn{5}{l}{\text{Maximum memory (RAM) of an experiment:} 700.34 $\pm$ 9.42 MB}\\
\bottomrule
\end{tabular}
\caption{Results for MNIST Addition with $k=4$  in low data settings ($N_{\text{train}}$ = 312) using $\Pi$-Nesy-1 over 10 experiments. Averages for the rule parameter values ($s_{\ast}$ and $r_{\ast}$) for the $c_{\ast}$-, $w_{\ast}$-, and $y_{\ast}$-rules are computed from data comprising $10\cdot 83$, $10\cdot 4$, and $10\cdot 41$ observations, respectively. Reliable sample counts for the $c_{\ast}$-, $w_{\ast}$-, and $y_{\ast}$-rules are averaged over $10\cdot 4$, $10\cdot 4$, and $10\cdot 5$ observations, respectively, and memory metrics are averaged over 10 experiments.}
\label{tab:lds:mnist-add-4-pi-nesy-1}
\end{table}

\begin{table}[ht]
\centering
\setlength{\tabcolsep}{1.5pt}
\renewcommand{\arraystretch}{1.5}
\begin{tabular}{lcccc}
\toprule
\makecell{Set of rules\\(output attribute)} & \makecell{Average value of the\\rule parameters $s_{\ast}$} & \makecell{Average value of the\\rule parameters $r_{\ast}$} & \makecell{Average number of\\training data samples\\considered reliable} & \makecell{Percent of\\training data samples\\considered reliable} \\
\midrule
$c_{\ast}$ & 2.46e-05 $\pm$ 9.44e-05 & 8.12e-06 $\pm$ 5.34e-05 & 193.45 $\pm$ 45.18 & 62.003 $\pm$ 14.482\%\\
$w_{\ast}$ & 0.00e+00 $\pm$ 0.00e+00 & 0.00e+00 $\pm$ 0.00e+00 & 254.20 $\pm$ 23.90 & 81.474 $\pm$ 7.659\%\\
$y_{\ast}$ & 4.12e-05 $\pm$ 1.16e-04 & 0.00e+00 $\pm$ 0.00e+00 & 205.38 $\pm$ 48.08 & 65.827 $\pm$ 15.411\%\\
\midrule
\multicolumn{5}{l}{\text{Average memory (RAM) of an experiment:} 614.09 $\pm$ 5.47 MB}\\
\multicolumn{5}{l}{\text{Maximum memory (RAM) of an experiment:} 702.24 $\pm$ 7.63 MB}\\
\bottomrule
\end{tabular}
\caption{Results for MNIST Addition with $k=4$  in low data settings ($N_{\text{train}}$ = 312) using $\Pi$-Nesy-2 over 10 experiments. Averages for the rule parameter values ($s_{\ast}$ and $r_{\ast}$) for the $c_{\ast}$-, $w_{\ast}$-, and $y_{\ast}$-rules are computed from data comprising $10\cdot 83$, $10\cdot 4$, and $10\cdot 41$ observations, respectively. Reliable sample counts for the $c_{\ast}$-, $w_{\ast}$-, and $y_{\ast}$-rules are averaged over $10\cdot 4$, $10\cdot 4$, and $10\cdot 5$ observations, respectively, and memory metrics are averaged over 10 experiments.}
\label{tab:lds:mnist-add-4-pi-nesy-2}
\end{table}

\begin{table}[ht]
\centering
\setlength{\tabcolsep}{1.5pt}
\renewcommand{\arraystretch}{1.5}
\begin{tabular}{lcccc}
\toprule
\makecell{Set of rules\\(output attribute)} & \makecell{Average value of the\\rule parameters $s_{\ast}$} & \makecell{Average value of the\\rule parameters $r_{\ast}$} & \makecell{Average number of\\training data samples\\considered reliable} & \makecell{Percent of\\training data samples\\considered reliable} \\
\midrule
$c_{\ast}$ & 1.40e-04 $\pm$ 1.87e-04 & 2.59e-05 $\pm$ 9.17e-05 & 39.13 $\pm$ 16.06 & 47.149 $\pm$ 19.348\%\\
$w_{\ast}$ & 0.00e+00 $\pm$ 0.00e+00 & 0.00e+00 $\pm$ 0.00e+00 & 59.22 $\pm$ 11.93 & 71.349 $\pm$ 14.368\%\\
$y_{\ast}$ & 1.69e-04 $\pm$ 2.16e-04 & 0.00e+00 $\pm$ 0.00e+00 & 40.31 $\pm$ 16.54 & 48.569 $\pm$ 19.932\%\\
\midrule
\multicolumn{5}{l}{\text{Average memory (RAM) of an experiment:} 609.51 $\pm$ 7.60 MB}\\
\multicolumn{5}{l}{\text{Maximum memory (RAM) of an experiment:} 667.62 $\pm$ 13.40 MB}\\
\bottomrule
\end{tabular}
\caption{Results for MNIST Addition with $k=15$  in low data settings ($N_{\text{train}}$ = 83) using $\Pi$-Nesy-1 over 10 experiments. Averages for the rule parameter values ($s_{\ast}$ and $r_{\ast}$) for the $c_{\ast}$-, $w_{\ast}$-, and $y_{\ast}$-rules are computed from data comprising $10\cdot 314$, $10\cdot 15$, and $10\cdot 151$ observations, respectively. Reliable sample counts for the $c_{\ast}$-, $w_{\ast}$-, and $y_{\ast}$-rules are averaged over $10\cdot 15$, $10\cdot 15$, and $10\cdot 16$ observations, respectively, and memory metrics are averaged over 10 experiments.}
\label{tab:lds:mnist-add-15-pi-nesy-1}
\end{table}

\begin{table}[ht]
\centering
\setlength{\tabcolsep}{1.5pt}
\renewcommand{\arraystretch}{1.5}
\begin{tabular}{lcccc}
\toprule
\makecell{Set of rules\\(output attribute)} & \makecell{Average value of the\\rule parameters $s_{\ast}$} & \makecell{Average value of the\\rule parameters $r_{\ast}$} & \makecell{Average number of\\training data samples\\considered reliable} & \makecell{Percent of\\training data samples\\considered reliable} \\
\midrule
$c_{\ast}$ & 1.46e-04 $\pm$ 1.92e-04 & 2.71e-05 $\pm$ 9.41e-05 & 45.18 $\pm$ 14.58 & 54.434 $\pm$ 17.569\%\\
$w_{\ast}$ & 0.00e+00 $\pm$ 0.00e+00 & 0.00e+00 $\pm$ 0.00e+00 & 64.22 $\pm$ 9.83 & 77.373 $\pm$ 11.847\%\\
$y_{\ast}$ & 1.69e-04 $\pm$ 2.27e-04 & 0.00e+00 $\pm$ 0.00e+00 & 46.39 $\pm$ 15.01 & 55.889 $\pm$ 18.079\%\\
\midrule
\multicolumn{5}{l}{\text{Average memory (RAM) of an experiment:} 609.63 $\pm$ 7.05 MB}\\
\multicolumn{5}{l}{\text{Maximum memory (RAM) of an experiment:} 668.93 $\pm$ 11.88 MB}\\
\bottomrule
\end{tabular}
\caption{Results for MNIST Addition with $k=15$  in low data settings ($N_{\text{train}}$ = 83) using $\Pi$-Nesy-2 over 10 experiments. Averages for the rule parameter values ($s_{\ast}$ and $r_{\ast}$) for the $c_{\ast}$-, $w_{\ast}$-, and $y_{\ast}$-rules are computed from data comprising $10\cdot 314$, $10\cdot 15$, and $10\cdot 151$ observations, respectively. Reliable sample counts for the $c_{\ast}$-, $w_{\ast}$-, and $y_{\ast}$-rules are averaged over $10\cdot 15$, $10\cdot 15$, and $10\cdot 16$ observations, respectively, and memory metrics are averaged over 10 experiments.}
\label{tab:lds:mnist-add-15-pi-nesy-2}
\end{table}

\begin{table}[ht]
\centering
\setlength{\tabcolsep}{1.5pt}
\renewcommand{\arraystretch}{1.5}
\begin{tabular}{lcccc}
\toprule
\makecell{Set of rules\\(output attribute)} & \makecell{Average value of the\\rule parameters $s_{\ast}$} & \makecell{Average value of the\\rule parameters $r_{\ast}$} & \makecell{Average number of\\training data samples\\considered reliable} & \makecell{Percent of\\training data samples\\considered reliable} \\
\midrule
$c_{\ast}$ & 9.85e-05 $\pm$ 1.49e-04 & 1.37e-05 $\pm$ 6.56e-05 & 6.50 $\pm$ 2.27 & 54.125 $\pm$ 18.904\%\\
$w_{\ast}$ & 0.00e+00 $\pm$ 0.00e+00 & 0.00e+00 $\pm$ 0.00e+00 & 8.74 $\pm$ 2.04 & 72.867 $\pm$ 16.965\%\\
$y_{\ast}$ & 1.49e-04 $\pm$ 2.12e-04 & 0.00e+00 $\pm$ 0.00e+00 & 6.52 $\pm$ 2.28 & 54.356 $\pm$ 19.014\%\\
\midrule
\multicolumn{5}{l}{\text{Average memory (RAM) of an experiment:} 630.61 $\pm$ 3.79 MB}\\
\multicolumn{5}{l}{\text{Maximum memory (RAM) of an experiment:} 696.87 $\pm$ 3.76 MB}\\
\bottomrule
\end{tabular}
\caption{Results for MNIST Addition with $k=100$  in low data settings ($N_{\text{train}}$ = 12) using $\Pi$-Nesy-1 over 10 experiments. Averages for the rule parameter values ($s_{\ast}$ and $r_{\ast}$) for the $c_{\ast}$-, $w_{\ast}$-, and $y_{\ast}$-rules are computed from data comprising $10\cdot 2099$, $10\cdot 100$, and $10\cdot 1001$ observations, respectively. Reliable sample counts for the $c_{\ast}$-, $w_{\ast}$-, and $y_{\ast}$-rules are averaged over $10\cdot 100$, $10\cdot 100$, and $10\cdot 101$ observations, respectively, and memory metrics are averaged over 10 experiments.}
\label{tab:lds:mnist-add-100-pi-nesy-1}
\end{table}

\begin{table}[ht]
\centering
\setlength{\tabcolsep}{1.5pt}
\renewcommand{\arraystretch}{1.5}
\begin{tabular}{lcccc}
\toprule
\makecell{Set of rules\\(output attribute)} & \makecell{Average value of the\\rule parameters $s_{\ast}$} & \makecell{Average value of the\\rule parameters $r_{\ast}$} & \makecell{Average number of\\training data samples\\considered reliable} & \makecell{Percent of\\training data samples\\considered reliable} \\
\midrule
$c_{\ast}$ & 9.99e-05 $\pm$ 1.50e-04 & 1.23e-05 $\pm$ 6.23e-05 & 6.01 $\pm$ 2.53 & 50.108 $\pm$ 21.073\%\\
$w_{\ast}$ & 0.00e+00 $\pm$ 0.00e+00 & 0.00e+00 $\pm$ 0.00e+00 & 8.55 $\pm$ 1.97 & 71.233 $\pm$ 16.401\%\\
$y_{\ast}$ & 1.54e-04 $\pm$ 2.16e-04 & 0.00e+00 $\pm$ 0.00e+00 & 6.04 $\pm$ 2.54 & 50.330 $\pm$ 21.131\%\\
\midrule
\multicolumn{5}{l}{\text{Average memory (RAM) of an experiment:} 630.06 $\pm$ 3.98 MB}\\
\multicolumn{5}{l}{\text{Maximum memory (RAM) of an experiment:} 696.94 $\pm$ 4.35 MB}\\
\bottomrule
\end{tabular}
\caption{Results for MNIST Addition with $k=100$  in low data settings ($N_{\text{train}}$ = 12) using $\Pi$-Nesy-2 over 10 experiments. Averages for the rule parameter values ($s_{\ast}$ and $r_{\ast}$) for the $c_{\ast}$-, $w_{\ast}$-, and $y_{\ast}$-rules are computed from data comprising $10\cdot 2099$, $10\cdot 100$, and $10\cdot 1001$ observations, respectively. Reliable sample counts for the $c_{\ast}$-, $w_{\ast}$-, and $y_{\ast}$-rules are averaged over $10\cdot 100$, $10\cdot 100$, and $10\cdot 101$ observations, respectively, and memory metrics are averaged over 10 experiments.}
\label{tab:lds:mnist-add-100-pi-nesy-2}
\end{table}

\FloatBarrier
\subsection{MNIST Sudoku problems in low data settings}
\label{subsec:appendix:lds:posslearn:mnistsudoku}

We also analyzed our possibilistic learning method applied to MNIST Sudoku problems in low data settings, yielding the results presented in Table \ref{tab:mnist_sudokuresults_low_data_settings}.

Table~\ref{tab:lds:mnist-sudoku-thresholds} presents the average values of thresholds for both $\Pi$-NeSy-1 and $\Pi$-NeSy-2 approaches, computed over ten experiments in low data settings. These average values are very small. The threshold values of $\Pi$-NeSy-2 are sometimes slightly higher than those of $\Pi$-NeSy-1.

\begin{table}[ht!]
\centering
\begin{tabular}{@{}llll@{}}
\toprule
Dimension & Approach & Average threshold for $b_{(i,j,i',j')}$ & Average threshold for $c$  \\ \midrule
4x4 & $\Pi$-NeSy-1 & 5.561e-03 $\pm$ 2.224e-03 & 1.780e-03 $\pm$ 1.335e-03 \\
4x4 & $\Pi$-NeSy-2 & 1.034e-01 $\pm$ 2.992e-01 & 1.112e-03 $\pm$ 2.168e-19 \\
9x9 & $\Pi$-NeSy-1 & 1.446e-03 $\pm$ 1.001e-03 & 1.112e-03 $\pm$ 2.168e-19 \\
9x9 & $\Pi$-NeSy-2 & 3.337e-03 $\pm$ 2.724e-03 & 1.112e-03 $\pm$ 2.168e-19 \\
\bottomrule
\end{tabular}
\caption{Average thresholds for MNIST-Sudoku problems  in low data settings used for selecting reliable training data during possibilistic learning. In the 4x4 configuration (resp. 9x9 configuration), averages for $b_{(i,j,i',j')}$ are computed over 10 observations and for $c$ over 10 observations for both approaches.}
\label{tab:lds:mnist-sudoku-thresholds}\end{table}

As expected, using the specified configuration (detailed in subsection \ref{subsub:mnistsudokuoresults}), $\Pi$-NeSy's inference and learning times in low data settings are low, see Table \ref{tab:lds:times_mnistSudoku}. Learning times are significantly shorter than in normal settings (see Table \ref{tab:times_mnistSudoku} for inference and learning times obtained under normal conditions). Inference times are lower than in normal settings, which can be explained by the fact that the program uses less memory.

\begin{table}[ht]
\centering

\begin{tabular}{@{}cccccc@{}}
\toprule
Size & Training set size & Test set size & Approach & Learning time & Inference time per test sample \\
\midrule
4x4 & 10 & 100 & $\Pi$-NeSy-1 & 21.73 $\pm$ 0.11 & 0.01 $\pm$ 0.00 \\
4x4 & 10 & 100 & $\Pi$-NeSy-2 & 21.51 $\pm$ 0.32 & 0.01 $\pm$ 0.00 \\
9x9 & 10 & 100 & $\Pi$-NeSy-1 & 657.73 $\pm$ 3.50 & 0.48 $\pm$ 0.00 \\
9x9 & 10 & 100 & $\Pi$-NeSy-2 & 633.75 $\pm$ 3.92 & 0.40 $\pm$ 0.00 \\
\bottomrule
\end{tabular}

\caption{Each row corresponds to a MNIST Sudoku problem processed by $\Pi$-NeSy-1 or $\Pi$-NeSy-2 in low data settings. Each row lists the corresponding learning time in seconds using training data (averaged over 10 runs) and the average inference time  for a test sample in seconds (averaged over $10\cdot 100$ observations) from these experiments. }
\label{tab:lds:times_mnistSudoku}
\end{table}

In Tables~\ref{tab:lds:mnist-sudoku-4-pi-nesy-1}--\ref{tab:lds:mnist-sudoku-9-pi-nesy-2}, we present the average values  of the rule parameter values $s$ and $r$ for the output attributes $b_{(i,j,i',j')}$ and $c$ obtained with our possibilistic learning method, the absolute number and percentage of reliable training samples used, as well as the average and maximum memory usage for each problem. As in normal settings, the parameters $s$ and $r$ associated with the rules have values very close to or equal to zero, with low variability. The percentage of  training data samples considered reliable with respect to $N_{\text{train}} = 10$ is very high for for the output attributes $b_{(i,j,i',j')}$ and medium for $c$. 

For MNIST Sudoku 4x4, memory usage averages 650 MB  for both approaches. For MNIST Sudoku 9x9, memory usage averages 1250 MB, which is significantly lower than in the normal settings, where we needed around 6.7 GB.

\begin{table}[ht]
\centering
\setlength{\tabcolsep}{1.5pt}
\renewcommand{\arraystretch}{1.5}
\begin{tabular}{lcccc}
\toprule
\makecell{Set of rules\\(output attribute)} & \makecell{Average value of the\\rule parameters $s$} & \makecell{Average value of the\\rule parameters $r$} & \makecell{Average number of\\training data samples\\considered reliable} & \makecell{Percent of\\training data samples\\considered reliable} \\
\midrule
$b_{(i,j,i',j')}$ & 1.05e-05 $\pm$ 6.28e-05 & 4.08e-06 $\pm$ 4.06e-05 & 8.0 $\pm$ 2.1 & 80.25 $\pm$ 21.18\% \\
c & 0.00e+00 $\pm$ 0.00e+00 & 0.00e+00 $\pm$ 0.00e+00 & 2.8 $\pm$ 1.5 & 28.00 $\pm$ 14.76\% \\
\midrule
\multicolumn{5}{l}{\text{Average memory (RAM) of an experiment:} 657.5 $\pm$ 1.9 MB}\\
\multicolumn{5}{l}{\text{Maximum memory (RAM) of an experiment:} 678.3 $\pm$ 2.4 MB}\\
\bottomrule
\end{tabular}
\caption{Possibilistic learning results on the MNIST Sudoku 4x4 problem using $\Pi$-NeSy-1  in low data settings. Results are averaged over ten experiments. The respective averages for the rule parameters corresponding to $b_{(i,j,i',j')}$ are computed from $10\cdot 448$ observations, while those for the attribute $c$ are computed from $10\cdot 1$ observations. Reliable training sample counts for $b_{(i,j,i',j')}$ and $c$ are computed from $10\cdot 56$ and $10\cdot 1$ observations, respectively.}
\label{tab:lds:mnist-sudoku-4-pi-nesy-1}
\end{table}
\begin{table}[ht]
\centering
\setlength{\tabcolsep}{1.5pt}
\renewcommand{\arraystretch}{1.5}
\begin{tabular}{lcccc}
\toprule
\makecell{Set of rules\\(output attribute)} & \makecell{Average value of the\\rule parameters $s$} & \makecell{Average value of the\\rule parameters $r$} & \makecell{Average number of\\training data samples\\considered reliable} & \makecell{Percent of\\training data samples\\considered reliable} \\
\midrule
$b_{(i,j,i',j')}$ & 3.02e-05 $\pm$ 1.05e-04 & 7.46e-06 $\pm$ 5.45e-05 & 8.1 $\pm$ 2.5 & 81.41 $\pm$ 24.63\% \\
c & 0.00e+00 $\pm$ 0.00e+00 & 0.00e+00 $\pm$ 0.00e+00 & 4.6 $\pm$ 0.5 & 46.00 $\pm$ 5.16\% \\
\midrule
\multicolumn{5}{l}{\text{Average memory (RAM) of an experiment:} 656.0 $\pm$ 2.3 MB}\\
\multicolumn{5}{l}{\text{Maximum memory (RAM) of an experiment:} 676.4 $\pm$ 2.4 MB}\\
\bottomrule
\end{tabular}
\caption{Possibilistic learning results on the MNIST Sudoku 4x4 problem using $\Pi$-NeSy-2  in low data settings. Results are averaged over ten experiments. The respective averages for the rule parameters corresponding to $b_{(i,j,i',j')}$ are computed from $10\cdot 448$ observations, while those for the attribute $c$ are computed from $10\cdot 1$ observations. Reliable training sample counts for $b_{(i,j,i',j')}$ and $c$ are computed from $10\cdot 56$ and $10\cdot 1$ observations, respectively.}
\label{tab:lds:mnist-sudoku-4-pi-nesy-2}
\end{table}
\begin{table}[ht]
\centering
\setlength{\tabcolsep}{1.5pt}
\renewcommand{\arraystretch}{1.5}
\begin{tabular}{lcccc}
\toprule
\makecell{Set of rules\\(output attribute)} & \makecell{Average value of the\\rule parameters $s$} & \makecell{Average value of the\\rule parameters $r$} & \makecell{Average number of\\training data samples\\considered reliable} & \makecell{Percent of\\training data samples\\considered reliable} \\
\midrule
$b_{(i,j,i',j')}$ & 1.32e-04 $\pm$ 1.59e-04 & 1.49e-05 $\pm$ 6.75e-05 & 8.4 $\pm$ 1.8 & 84.39 $\pm$ 18.04\% \\
c & 0.00e+00 $\pm$ 0.00e+00 & 0.00e+00 $\pm$ 0.00e+00 & 5.1 $\pm$ 0.3 & 51.00 $\pm$ 3.16\% \\
\midrule
\multicolumn{5}{l}{\text{Average memory (RAM) of an experiment:} 1254.5 $\pm$ 16.2 MB}\\
\multicolumn{5}{l}{\text{Maximum memory (RAM) of an experiment:} 1336.5 $\pm$ 16.3 MB}\\
\bottomrule
\end{tabular}
\caption{Possibilistic learning results on the MNIST Sudoku 9x9 problem using $\Pi$-NeSy-1  in low data settings. Results are averaged over ten experiments. The respective averages for the rule parameters corresponding to $b_{(i,j,i',j')}$ are computed from $10\cdot 14580$ observations, while those for the attribute $c$ are computed from $10\cdot 1$ observations. Reliable training sample counts for $b_{(i,j,i',j')}$ and $c$ are computed from $10\cdot 810$ and $10\cdot 1$ observations, respectively.}
\label{tab:lds:mnist-sudoku-9-pi-nesy-1}
\end{table}
\begin{table}[ht]
\centering
\setlength{\tabcolsep}{1.5pt}
\renewcommand{\arraystretch}{1.5}
\begin{tabular}{lcccc}
\toprule
\makecell{Set of rules\\(output attribute)} & \makecell{Average value of the\\rule parameters $s$} & \makecell{Average value of the\\rule parameters $r$} & \makecell{Average number of\\training data samples\\considered reliable} & \makecell{Percent of\\training data samples\\considered reliable} \\
\midrule
$b_{(i,j,i',j')}$ & 1.14e-04 $\pm$ 1.57e-04 & 1.10e-05 $\pm$ 5.95e-05 & 8.6 $\pm$ 1.7 & 86.46 $\pm$ 17.34\% \\
c & 0.00e+00 $\pm$ 0.00e+00 & 0.00e+00 $\pm$ 0.00e+00 & 5.1 $\pm$ 1.2 & 51.00 $\pm$ 11.97\% \\
\midrule
\multicolumn{5}{l}{\text{Average memory (RAM) of an experiment:} 1247.6 $\pm$ 12.1 MB}\\
\multicolumn{5}{l}{\text{Maximum memory (RAM) of an experiment:} 1332.9 $\pm$ 12.1 MB}\\
\bottomrule
\end{tabular}
\caption{Possibilistic learning results on the MNIST Sudoku 9x9 problem using $\Pi$-NeSy-2  in low data settings. Results are averaged over ten experiments. The respective averages for the rule parameters corresponding to $b_{(i,j,i',j')}$ are computed from $10\cdot 14580$ observations, while those for the attribute $c$ are computed from $10\cdot 1$ observations. Reliable training sample counts for $b_{(i,j,i',j')}$ and $c$ are computed from $10\cdot 810$ and $10\cdot 1$ observations, respectively.}
\label{tab:lds:mnist-sudoku-9-pi-nesy-2}
\end{table}

\FloatBarrier
\section{Additional experimental results for MNIST Addition-\texorpdfstring{$k$}{k} problems using \texorpdfstring{$\Pi$}{Pi}-NeSy with  DeepSoftLog's CNN}
\label{sec:appendix:posslearn:withDSL}

To complete our empirical study, we also analyzed the characteristics and performance of the possibilistic learning method (Method \ref{meth:learningincascade}) applied to the MNIST Addition-$k$ problems with $k \in \{1,2,4,15,100\}$ when $\Pi$-NeSy takes advantage of  DeepSoftLog's CNN, yielding to the results presented in Table \ref{tab:deepsoftlog_cnn_mnist_add_results}. We reuse the notations given in Section \ref{sec:appendix:expadd}.

For MNIST Addition-$k$ problems with $k \in \{1, 2, 4, 15, 100\}$, both $\Pi$-NeSy-1 and $\Pi$-NeSy-2 consistently utilize the same threshold value of $4.12\times10^{-8}$ for Thresholds \#1, \#2, and \#3, as in the normal settings (see Subsection \ref{subsec:appendix:posslearn:mnistadd}). This threshold value is the smallest in the set $\mathcal{T}$ of candidate thresholds, see (\ref{eq:setTvalidationstep}).

\begin{table}[ht]
\centering
\begin{tabular}{@{}cccccc@{}}
\toprule
k & Training set size & Test set size & Approach & Learning time & Inference time per test sample \\
\midrule
1 & 30000 & 5000 & $\Pi$-NeSy-1 & 5121.64 $\pm$ 23.21 & 0.002 $\pm$ 0.00 \\
1 & 30000 & 5000 & $\Pi$-NeSy-2 & 4914.52 $\pm$ 313.03 & 0.002 $\pm$ 0.00 \\
2 & 15000 & 2500 & $\Pi$-NeSy-1 & 5266.28 $\pm$ 35.31 & 0.004 $\pm$ 0.00 \\
2 & 15000 & 2500 & $\Pi$-NeSy-2 & 5272.54 $\pm$ 86.30 & 0.004 $\pm$ 0.00 \\
4 & 7500 & 1250 & $\Pi$-NeSy-1 & 3779.06 $\pm$ 13.97 & 0.007 $\pm$ 0.00 \\
4 & 7500 & 1250 & $\Pi$-NeSy-2 & 3783.56 $\pm$ 41.99 & 0.007 $\pm$ 0.00 \\
15 & 2000 & 333 & $\Pi$-NeSy-1 & 2051.94 $\pm$ 19.01 & 0.027 $\pm$ 0.00 \\
15 & 2000 & 333 & $\Pi$-NeSy-2 & 2058.92 $\pm$ 9.31 & 0.027 $\pm$ 0.00 \\
100 & 300 & 50 & $\Pi$-NeSy-1 & 1418.69 $\pm$ 28.05 & 0.181 $\pm$ 0.01 \\
100 & 300 & 50 & $\Pi$-NeSy-2 & 1409.68 $\pm$ 11.33 & 0.177 $\pm$ 0.00 \\
\bottomrule
\end{tabular}

\caption{Each row in the table represents a MNIST Addition-$k$ problem addressed using either $\Pi$-NeSy-1 or $\Pi$-NeSy-2 with DeepSoftLog's CNN.
Each row includes the average learning time in seconds (averaged over 10 runs) and the average inference time per test sample in seconds (averaged over $10 \cdot N_{\text{test}}$ test samples) from these experiments. }
\label{tab:dsl:times_mnistadd}
\end{table}

Empirically, using the specified configuration (given  in Subsection \ref{subsec:mnistaddotherresults}), the inference and learning times of  $\Pi$-NeSy with DeepSoftLog's CNN are reasonable, see Table \ref{tab:dsl:times_mnistadd}, but learning times are  much more important than in normal settings (see Table \ref{tab:times_mnistadd} for inference and learning times obtained under normal conditions). This is explained by the results presented in Tables~\ref{tab:dsl:mnist-add-1-pi-nesy-1}--\ref{tab:dsl:mnist-add-100-pi-nesy-2}, which report the average values of the rule parameters obtained by possibilistic learning, the absolute number and percentage of reliable training data samples used, and the memory usage  obtained using $\Pi$-NeSy with DeepSoftLog's CNN when addressing MNIST Addition-$k$ problems.  The parameters $s$ and $r$ associated with the rules maintain average values close to zero with very low variability. We observe that a very large number of training data samples are considered reliable (approx. 99.8\%), resulting in much higher memory requirements (approx. 2.8 GB for k = 1 to 1150 MB for k = 100) than under normal conditions (approx. 1350 MB for k = 1 to 920 MB when k = 100). This observation is partly due to the fact that, as the CNN has been trained with a very large number of epochs (130) and a very good  accuracy has been reached, the quality of the probability distributions produced by the CNN on MNIST training images is very good, in the sense that they are very close to one-point distributions, where, for each  distribution, the probability of one is associated with the correct number to be predicted.

    \begin{table}[ht]
    \centering
    \setlength{\tabcolsep}{1.5pt}
    \renewcommand{\arraystretch}{1.5}
    \begin{tabular}{lcccc}
    \toprule
    \makecell{Set of rules\\(output attribute)} & \makecell{Average value of the\\rule parameters $s_{\ast}$} & \makecell{Average value of the\\rule parameters $r_{\ast}$} & \makecell{Average number of\\training data samples\\considered reliable} & \makecell{Percent of\\training data samples\\considered reliable} \\
    \midrule
    $c_{\ast}$ & 1.91e-08 $\pm$ 4.51e-09 & 1.43e-08 $\pm$ 6.31e-09 & 29895.00 $\pm$ 0.00 & 99.650 $\pm$ 0.000\%\\
    $w_{\ast}$ & 0.00e+00 $\pm$ 0.00e+00 & 0.00e+00 $\pm$ 0.00e+00 & 29924.70 $\pm$ 4.52 & 99.749 $\pm$ 0.015\%\\
    $y_{\ast}$ & 2.12e-08 $\pm$ 1.29e-08 & 0.00e+00 $\pm$ 0.00e+00 & 29909.85 $\pm$ 15.19 & 99.700 $\pm$ 0.051\%\\
    \midrule
    \multicolumn{5}{l}{\text{Average memory (RAM) of an experiment:} 2782.00 $\pm$ 10.06 MB}\\
    \multicolumn{5}{l}{\text{Maximum memory (RAM) of an experiment:} 3245.57 $\pm$ 22.42 MB}\\
    \bottomrule
    \end{tabular}
    \caption{Results for MNIST Addition with $k=1$ ($N_{\text{train}}$ = 30000) using $\Pi$-Nesy-1 with DeepSoftLog's CNN over 10 experiments. Averages for the rule parameter values ($s_{\ast}$ and $r_{\ast}$) for the $c_{\ast}$-, $w_{\ast}$-, and $y_{\ast}$-rules are computed from data comprising $10\cdot 20$, $10\cdot 1$, and $10\cdot 11$ observations, respectively. Reliable sample counts for the $c_{\ast}$-, $w_{\ast}$-, and $y_{\ast}$-rules are averaged over $10\cdot 1$, $10\cdot 1$, and $10\cdot 2$ observations, respectively, and memory metrics are averaged over 10 experiments.}
    \label{tab:dsl:mnist-add-1-pi-nesy-1}
    \end{table}

    \begin{table}[ht]
    \centering
    \setlength{\tabcolsep}{1.5pt}
    \renewcommand{\arraystretch}{1.5}
    \begin{tabular}{lcccc}
    \toprule
    \makecell{Set of rules\\(output attribute)} & \makecell{Average value of the\\rule parameters $s_{\ast}$} & \makecell{Average value of the\\rule parameters $r_{\ast}$} & \makecell{Average number of\\training data samples\\considered reliable} & \makecell{Percent of\\training data samples\\considered reliable} \\
    \midrule
    $c_{\ast}$ & 1.92e-08 $\pm$ 4.62e-09 & 1.07e-08 $\pm$ 6.98e-09 & 29972.00 $\pm$ 0.00 & 99.907 $\pm$ 0.000\%\\
    $w_{\ast}$ & 0.00e+00 $\pm$ 0.00e+00 & 0.00e+00 $\pm$ 0.00e+00 & 29981.20 $\pm$ 2.23 & 99.937 $\pm$ 0.007\%\\
    $y_{\ast}$ & 1.82e-08 $\pm$ 1.28e-08 & 0.00e+00 $\pm$ 0.00e+00 & 29976.60 $\pm$ 4.86 & 99.922 $\pm$ 0.016\%\\
    \midrule
    \multicolumn{5}{l}{\text{Average memory (RAM) of an experiment:} 2793.14 $\pm$ 7.91 MB}\\
    \multicolumn{5}{l}{\text{Maximum memory (RAM) of an experiment:} 3265.04 $\pm$ 11.51 MB}\\
    \bottomrule
    \end{tabular}
    \caption{Results for MNIST Addition with $k=1$ ($N_{\text{train}}$ = 30000) using $\Pi$-Nesy-2 with DeepSoftLog's CNN over 10 experiments. Averages for the rule parameter values ($s_{\ast}$ and $r_{\ast}$) for the $c_{\ast}$-, $w_{\ast}$-, and $y_{\ast}$-rules are computed from data comprising $10\cdot 20$, $10\cdot 1$, and $10\cdot 11$ observations, respectively. Reliable sample counts for the $c_{\ast}$-, $w_{\ast}$-, and $y_{\ast}$-rules are averaged over $10\cdot 1$, $10\cdot 1$, and $10\cdot 2$ observations, respectively, and memory metrics are averaged over 10 experiments.}
    \label{tab:dsl:mnist-add-1-pi-nesy-2}
    \end{table}

    \begin{table}[ht]
    \centering
    \setlength{\tabcolsep}{1.5pt}
    \renewcommand{\arraystretch}{1.5}
    \begin{tabular}{lcccc}
    \toprule
    \makecell{Set of rules\\(output attribute)} & \makecell{Average value of the\\rule parameters $s_{\ast}$} & \makecell{Average value of the\\rule parameters $r_{\ast}$} & \makecell{Average number of\\training data samples\\considered reliable} & \makecell{Percent of\\training data samples\\considered reliable} \\
    \midrule
    $c_{\ast}$ & 1.80e-08 $\pm$ 5.54e-09 & 1.12e-08 $\pm$ 7.03e-09 & 14929.90 $\pm$ 17.68 & 99.533 $\pm$ 0.118\%\\
    $w_{\ast}$ & 0.00e+00 $\pm$ 0.00e+00 & 0.00e+00 $\pm$ 0.00e+00 & 14959.35 $\pm$ 6.43 & 99.729 $\pm$ 0.043\%\\
    $y_{\ast}$ & 2.02e-08 $\pm$ 1.22e-08 & 0.00e+00 $\pm$ 0.00e+00 & 14938.00 $\pm$ 18.52 & 99.587 $\pm$ 0.123\%\\
    \midrule
    \multicolumn{5}{l}{\text{Average memory (RAM) of an experiment:} 2876.65 $\pm$ 14.56 MB}\\
    \multicolumn{5}{l}{\text{Maximum memory (RAM) of an experiment:} 3411.16 $\pm$ 22.02 MB}\\
    \bottomrule
    \end{tabular}
    \caption{Results for MNIST Addition with $k=2$ ($N_{\text{train}}$ = 15000) using $\Pi$-Nesy-1 with DeepSoftLog's CNN over 10 experiments. Averages for the rule parameter values ($s_{\ast}$ and $r_{\ast}$) for the $c_{\ast}$-, $w_{\ast}$-, and $y_{\ast}$-rules are computed from data comprising $10\cdot 41$, $10\cdot 2$, and $10\cdot 21$ observations, respectively. Reliable sample counts for the $c_{\ast}$-, $w_{\ast}$-, and $y_{\ast}$-rules are averaged over $10\cdot 2$, $10\cdot 2$, and $10\cdot 3$ observations, respectively, and memory metrics are averaged over 10 experiments.}
    \label{tab:dsl:mnist-add-2-pi-nesy-1}
    \end{table}

    \begin{table}[ht]
    \centering
    \setlength{\tabcolsep}{1.5pt}
    \renewcommand{\arraystretch}{1.5}
    \begin{tabular}{lcccc}
    \toprule
    \makecell{Set of rules\\(output attribute)} & \makecell{Average value of the\\rule parameters $s_{\ast}$} & \makecell{Average value of the\\rule parameters $r_{\ast}$} & \makecell{Average number of\\training data samples\\considered reliable} & \makecell{Percent of\\training data samples\\considered reliable} \\
    \midrule
    $c_{\ast}$ & 1.78e-08 $\pm$ 5.82e-09 & 7.38e-09 $\pm$ 7.19e-09 & 14981.50 $\pm$ 4.59 & 99.877 $\pm$ 0.031\%\\
    $w_{\ast}$ & 0.00e+00 $\pm$ 0.00e+00 & 0.00e+00 $\pm$ 0.00e+00 & 14990.00 $\pm$ 2.86 & 99.933 $\pm$ 0.019\%\\
    $y_{\ast}$ & 1.87e-08 $\pm$ 1.15e-08 & 0.00e+00 $\pm$ 0.00e+00 & 14984.00 $\pm$ 5.34 & 99.893 $\pm$ 0.036\%\\
    \midrule
    \multicolumn{5}{l}{\text{Average memory (RAM) of an experiment:} 2887.93 $\pm$ 12.65 MB}\\
    \multicolumn{5}{l}{\text{Maximum memory (RAM) of an experiment:} 3428.11 $\pm$ 18.24 MB}\\
    \bottomrule
    \end{tabular}
    \caption{Results for MNIST Addition with $k=2$ ($N_{\text{train}}$ = 15000) using $\Pi$-Nesy-2 with DeepSoftLog's CNN over 10 experiments. Averages for the rule parameter values ($s_{\ast}$ and $r_{\ast}$) for the $c_{\ast}$-, $w_{\ast}$-, and $y_{\ast}$-rules are computed from data comprising $10\cdot 41$, $10\cdot 2$, and $10\cdot 21$ observations, respectively. Reliable sample counts for the $c_{\ast}$-, $w_{\ast}$-, and $y_{\ast}$-rules are averaged over $10\cdot 2$, $10\cdot 2$, and $10\cdot 3$ observations, respectively, and memory metrics are averaged over 10 experiments.}
    \label{tab:dsl:mnist-add-2-pi-nesy-2}
    \end{table}

    \begin{table}[ht]
    \centering
    \setlength{\tabcolsep}{1.5pt}
    \renewcommand{\arraystretch}{1.5}
    \begin{tabular}{lcccc}
    \toprule
    \makecell{Set of rules\\(output attribute)} & \makecell{Average value of the\\rule parameters $s_{\ast}$} & \makecell{Average value of the\\rule parameters $r_{\ast}$} & \makecell{Average number of\\training data samples\\considered reliable} & \makecell{Percent of\\training data samples\\considered reliable} \\
    \midrule
    $c_{\ast}$ & 1.69e-08 $\pm$ 5.97e-09 & 8.13e-09 $\pm$ 7.37e-09 & 7457.68 $\pm$ 10.33 & 99.436 $\pm$ 0.138\%\\
    $w_{\ast}$ & 0.00e+00 $\pm$ 0.00e+00 & 0.00e+00 $\pm$ 0.00e+00 & 7478.05 $\pm$ 4.22 & 99.707 $\pm$ 0.056\%\\
    $y_{\ast}$ & 2.02e-08 $\pm$ 1.12e-08 & 0.00e+00 $\pm$ 0.00e+00 & 7461.50 $\pm$ 12.12 & 99.487 $\pm$ 0.162\%\\
    \midrule
    \multicolumn{5}{l}{\text{Average memory (RAM) of an experiment:} 2216.00 $\pm$ 9.02 MB}\\
    \multicolumn{5}{l}{\text{Maximum memory (RAM) of an experiment:} 2636.40 $\pm$ 11.77 MB}\\
    \bottomrule
    \end{tabular}
    \caption{Results for MNIST Addition with $k=4$ ($N_{\text{train}}$ = 7500) using $\Pi$-Nesy-1 with DeepSoftLog's CNN over 10 experiments. Averages for the rule parameter values ($s_{\ast}$ and $r_{\ast}$) for the $c_{\ast}$-, $w_{\ast}$-, and $y_{\ast}$-rules are computed from data comprising $10\cdot 83$, $10\cdot 4$, and $10\cdot 41$ observations, respectively. Reliable sample counts for the $c_{\ast}$-, $w_{\ast}$-, and $y_{\ast}$-rules are averaged over $10\cdot 4$, $10\cdot 4$, and $10\cdot 5$ observations, respectively, and memory metrics are averaged over 10 experiments.}
    \label{tab:dsl:mnist-add-4-pi-nesy-1}
    \end{table}

    \begin{table}[ht]
    \centering
    \setlength{\tabcolsep}{1.5pt}
    \renewcommand{\arraystretch}{1.5}
    \begin{tabular}{lcccc}
    \toprule
    \makecell{Set of rules\\(output attribute)} & \makecell{Average value of the\\rule parameters $s_{\ast}$} & \makecell{Average value of the\\rule parameters $r_{\ast}$} & \makecell{Average number of\\training data samples\\considered reliable} & \makecell{Percent of\\training data samples\\considered reliable} \\
    \midrule
    $c_{\ast}$ & 1.59e-08 $\pm$ 6.66e-09 & 5.11e-09 $\pm$ 6.73e-09 & 7489.15 $\pm$ 3.45 & 99.855 $\pm$ 0.046\%\\
    $w_{\ast}$ & 0.00e+00 $\pm$ 0.00e+00 & 0.00e+00 $\pm$ 0.00e+00 & 7494.70 $\pm$ 1.99 & 99.929 $\pm$ 0.027\%\\
    $y_{\ast}$ & 1.61e-08 $\pm$ 1.05e-08 & 0.00e+00 $\pm$ 0.00e+00 & 7490.16 $\pm$ 3.75 & 99.869 $\pm$ 0.050\%\\
    \midrule
    \multicolumn{5}{l}{\text{Average memory (RAM) of an experiment:} 2225.07 $\pm$ 11.33 MB}\\
    \multicolumn{5}{l}{\text{Maximum memory (RAM) of an experiment:} 2647.62 $\pm$ 18.61 MB}\\
    \bottomrule
    \end{tabular}
    \caption{Results for MNIST Addition with $k=4$ ($N_{\text{train}}$ = 7500) using $\Pi$-Nesy-2 with DeepSoftLog's CNN over 10 experiments. Averages for the rule parameter values ($s_{\ast}$ and $r_{\ast}$) for the $c_{\ast}$-, $w_{\ast}$-, and $y_{\ast}$-rules are computed from data comprising $10\cdot 83$, $10\cdot 4$, and $10\cdot 41$ observations, respectively. Reliable sample counts for the $c_{\ast}$-, $w_{\ast}$-, and $y_{\ast}$-rules are averaged over $10\cdot 4$, $10\cdot 4$, and $10\cdot 5$ observations, respectively, and memory metrics are averaged over 10 experiments.}
    \label{tab:dsl:mnist-add-4-pi-nesy-2}
    \end{table}

    \begin{table}[ht]
    \centering
    \setlength{\tabcolsep}{1.5pt}
    \renewcommand{\arraystretch}{1.5}
    \begin{tabular}{lcccc}
    \toprule
    \makecell{Set of rules\\(output attribute)} & \makecell{Average value of the\\rule parameters $s_{\ast}$} & \makecell{Average value of the\\rule parameters $r_{\ast}$} & \makecell{Average number of\\training data samples\\considered reliable} & \makecell{Percent of\\training data samples\\considered reliable} \\
    \midrule
    $c_{\ast}$ & 1.36e-08 $\pm$ 6.92e-09 & 3.60e-09 $\pm$ 6.17e-09 & 1987.43 $\pm$ 3.42 & 99.371 $\pm$ 0.171\%\\
    $w_{\ast}$ & 0.00e+00 $\pm$ 0.00e+00 & 0.00e+00 $\pm$ 0.00e+00 & 1993.90 $\pm$ 2.24 & 99.695 $\pm$ 0.112\%\\
    $y_{\ast}$ & 1.64e-08 $\pm$ 9.79e-09 & 0.00e+00 $\pm$ 0.00e+00 & 1987.78 $\pm$ 3.60 & 99.389 $\pm$ 0.180\%\\
    \midrule
    \multicolumn{5}{l}{\text{Average memory (RAM) of an experiment:} 1430.21 $\pm$ 53.85 MB}\\
    \multicolumn{5}{l}{\text{Maximum memory (RAM) of an experiment:} 1724.20 $\pm$ 63.85 MB}\\
    \bottomrule
    \end{tabular}
    \caption{Results for MNIST Addition with $k=15$ ($N_{\text{train}}$ = 2000) using $\Pi$-Nesy-1 with DeepSoftLog's CNN over 10 experiments. Averages for the rule parameter values ($s_{\ast}$ and $r_{\ast}$) for the $c_{\ast}$-, $w_{\ast}$-, and $y_{\ast}$-rules are computed from data comprising $10\cdot 314$, $10\cdot 15$, and $10\cdot 151$ observations, respectively. Reliable sample counts for the $c_{\ast}$-, $w_{\ast}$-, and $y_{\ast}$-rules are averaged over $10\cdot 15$, $10\cdot 15$, and $10\cdot 16$ observations, respectively, and memory metrics are averaged over 10 experiments.}
    \label{tab:dsl:mnist-add-15-pi-nesy-1}
    \end{table}

    \begin{table}[ht]
    \centering
    \setlength{\tabcolsep}{1.5pt}
    \renewcommand{\arraystretch}{1.5}
    \begin{tabular}{lcccc}
    \toprule
    \makecell{Set of rules\\(output attribute)} & \makecell{Average value of the\\rule parameters $s_{\ast}$} & \makecell{Average value of the\\rule parameters $r_{\ast}$} & \makecell{Average number of\\training data samples\\considered reliable} & \makecell{Percent of\\training data samples\\considered reliable} \\
    \midrule
    $c_{\ast}$ & 1.21e-08 $\pm$ 7.38e-09 & 2.32e-09 $\pm$ 5.15e-09 & 1996.77 $\pm$ 1.83 & 99.838 $\pm$ 0.091\%\\
    $w_{\ast}$ & 0.00e+00 $\pm$ 0.00e+00 & 0.00e+00 $\pm$ 0.00e+00 & 1998.53 $\pm$ 1.20 & 99.927 $\pm$ 0.060\%\\
    $y_{\ast}$ & 1.35e-08 $\pm$ 8.84e-09 & 0.00e+00 $\pm$ 0.00e+00 & 1996.88 $\pm$ 1.84 & 99.844 $\pm$ 0.092\%\\
    \midrule
    \multicolumn{5}{l}{\text{Average memory (RAM) of an experiment:} 1444.54 $\pm$ 73.58 MB}\\
    \multicolumn{5}{l}{\text{Maximum memory (RAM) of an experiment:} 1736.21 $\pm$ 91.18 MB}\\
    \bottomrule
    \end{tabular}
    \caption{Results for MNIST Addition with $k=15$ ($N_{\text{train}}$ = 2000) using $\Pi$-Nesy-2 with DeepSoftLog's CNN over 10 experiments. Averages for the rule parameter values ($s_{\ast}$ and $r_{\ast}$) for the $c_{\ast}$-, $w_{\ast}$-, and $y_{\ast}$-rules are computed from data comprising $10\cdot 314$, $10\cdot 15$, and $10\cdot 151$ observations, respectively. Reliable sample counts for the $c_{\ast}$-, $w_{\ast}$-, and $y_{\ast}$-rules are averaged over $10\cdot 15$, $10\cdot 15$, and $10\cdot 16$ observations, respectively, and memory metrics are averaged over 10 experiments.}
    \label{tab:dsl:mnist-add-15-pi-nesy-2}
    \end{table}

    \begin{table}[ht]
    \centering
    \setlength{\tabcolsep}{1.5pt}
    \renewcommand{\arraystretch}{1.5}
    \begin{tabular}{lcccc}
    \toprule
    \makecell{Set of rules\\(output attribute)} & \makecell{Average value of the\\rule parameters $s_{\ast}$} & \makecell{Average value of the\\rule parameters $r_{\ast}$} & \makecell{Average number of\\training data samples\\considered reliable} & \makecell{Percent of\\training data samples\\considered reliable} \\
    \midrule
    $c_{\ast}$ & 7.06e-09 $\pm$ 7.29e-09 & 1.15e-09 $\pm$ 3.85e-09 & 298.05 $\pm$ 1.32 & 99.350 $\pm$ 0.441\%\\
    $w_{\ast}$ & 0.00e+00 $\pm$ 0.00e+00 & 0.00e+00 $\pm$ 0.00e+00 & 299.08 $\pm$ 0.92 & 99.695 $\pm$ 0.306\%\\
    $y_{\ast}$ & 9.53e-09 $\pm$ 7.14e-09 & 0.00e+00 $\pm$ 0.00e+00 & 298.06 $\pm$ 1.32 & 99.353 $\pm$ 0.441\%\\
    \midrule
    \multicolumn{5}{l}{\text{Average memory (RAM) of an experiment:} 1160.71 $\pm$ 10.02 MB}\\
    \multicolumn{5}{l}{\text{Maximum memory (RAM) of an experiment:} 1475.45 $\pm$ 278.41 MB}\\
    \bottomrule
    \end{tabular}
    \caption{Results for MNIST Addition with $k=100$ ($N_{\text{train}}$ = 300) using $\Pi$-Nesy-1 with DeepSoftLog's CNN over 10 experiments. Averages for the rule parameter values ($s_{\ast}$ and $r_{\ast}$) for the $c_{\ast}$-, $w_{\ast}$-, and $y_{\ast}$-rules are computed from data comprising $10\cdot 2099$, $10\cdot 100$, and $10\cdot 1001$ observations, respectively. Reliable sample counts for the $c_{\ast}$-, $w_{\ast}$-, and $y_{\ast}$-rules are averaged over $10\cdot 100$, $10\cdot 100$, and $10\cdot 101$ observations, respectively, and memory metrics are averaged over 10 experiments.}
    \label{tab:dsl:mnist-add-100-pi-nesy-1}
    \end{table}

    \begin{table}[ht]
    \centering
    \setlength{\tabcolsep}{1.5pt}
    \renewcommand{\arraystretch}{1.5}
    \begin{tabular}{lcccc}
    \toprule
    \makecell{Set of rules\\(output attribute)} & \makecell{Average value of the\\rule parameters $s_{\ast}$} & \makecell{Average value of the\\rule parameters $r_{\ast}$} & \makecell{Average number of\\training data samples\\considered reliable} & \makecell{Percent of\\training data samples\\considered reliable} \\
    \midrule
    $c_{\ast}$ & 5.58e-09 $\pm$ 6.35e-09 & 8.56e-10 $\pm$ 3.08e-09 & 299.48 $\pm$ 0.70 & 99.827 $\pm$ 0.232\%\\
    $w_{\ast}$ & 0.00e+00 $\pm$ 0.00e+00 & 0.00e+00 $\pm$ 0.00e+00 & 299.76 $\pm$ 0.48 & 99.920 $\pm$ 0.160\%\\
    $y_{\ast}$ & 7.94e-09 $\pm$ 5.90e-09 & 0.00e+00 $\pm$ 0.00e+00 & 299.49 $\pm$ 0.69 & 99.828 $\pm$ 0.231\%\\
    \midrule
    \multicolumn{5}{l}{\text{Average memory (RAM) of an experiment:} 1145.01 $\pm$ 8.14 MB}\\
    \multicolumn{5}{l}{\text{Maximum memory (RAM) of an experiment:} 1366.51 $\pm$ 9.96 MB}\\
    \bottomrule
    \end{tabular}
    \caption{Results for MNIST Addition with $k=100$ ($N_{\text{train}}$ = 300) using $\Pi$-Nesy-2 with DeepSoftLog's CNN over 10 experiments. Averages for the rule parameter values ($s_{\ast}$ and $r_{\ast}$) for the $c_{\ast}$-, $w_{\ast}$-, and $y_{\ast}$-rules are computed from data comprising $10\cdot 2099$, $10\cdot 100$, and $10\cdot 1001$ observations, respectively. Reliable sample counts for the $c_{\ast}$-, $w_{\ast}$-, and $y_{\ast}$-rules are averaged over $10\cdot 100$, $10\cdot 100$, and $10\cdot 101$ observations, respectively, and memory metrics are averaged over 10 experiments.}
    \label{tab:dsl:mnist-add-100-pi-nesy-2}
    \end{table}

\end{document}